\documentclass[aos, preprint]{imsart}
\usepackage[usenames]{color}
\usepackage{ifthen}
\usepackage{enumerate}
\usepackage[shortlabels]{enumitem}
\usepackage[latin1]{inputenc}
 \usepackage[authoryear]{natbib}
\usepackage{graphicx}
\usepackage{latexsym}
\usepackage{amssymb}
\usepackage{amsmath}
\usepackage{amsfonts,amsthm,enumerate,CJK,array}
\usepackage{bm,dsfont,multicol}
\usepackage{caption}
\usepackage[hypertexnames=false]{hyperref}
\usepackage{chngcntr} 
\usepackage[]{algorithm2e}
\usepackage{mathrsfs} 
\usepackage{adjustbox}
\usepackage{multirow}

\startlocaldefs
\makeatletter 
\newcommand{\labitem}[2]{%
\def\@itemlabel{\textbf{#1.}}
\item
\def\@currentlabel{#1}\label{#2}}
\makeatother

\makeatletter 
\newcommand{\labitemc}[2]{%
\def\@itemlabel{\textbf{#1}}
\item
\def\@currentlabel{#1}\label{#2}}
\makeatother

\newboolean{tab}
\setboolean{tab}{false}

\makeatletter\@addtoreset{equation}{section}\makeatother
\renewcommand{\theequation}{\thesection.\arabic{equation}}

\newcommand{\Var}{\mbox{Var}}

\def \R{\mathbb{R}}
\def \E{\mathbb{E}}
\def \N{\mathbb{N}}

\def \Cov{\mbox{Cov}}

\newcommand{\var}{ \mbox{\sl Var} \ }

\newcommand{\pto}{\stackrel{P}{\rightarrow}}

\newcommand{\asto}{\stackrel{a.s.}{\rightarrow}}

\newcommand{\weakto}{\rightsquigarrow}

\newcommand{\Ac}{\mathcal{A}}

\newcommand{\Wc}{\mathcal{W}}

\newcommand{\Gc}{\mathcal{G}}
\newcommand{\Dc}{\mathcal{D}}

\newcommand{\Cc}{\mathcal{C}}
\newcommand{\Kc}{\mathcal{K}}

\newcommand{\Fc}{\mathcal{F}}
\newcommand{\Zc}{\mathcal{Z}}
\newcommand{\Lc}{\mathcal{L}}
\newcommand{\Nc}{\mathcal{N}}

\newcommand{\Tc}{\mathcal{T}}

\newcommand{\Qc}{\mathcal{Q}}
\newcommand{\Rc}{\mathcal{R}}
\newcommand{\Sc}{\mathcal{S}}
\newcommand{\Pc}{\mathcal{P}}
\newcommand{\Uc}{\mathcal{U}}

\newcommand{\Vc}{\mathcal{V}}

\newcommand{\bb}{\mathbf{b}}

\newcommand{\bx}{\mathbf{x}}
\newcommand{\by}{\mathbf{y}}

\newcommand{\bv}{\mathbf{v}}
\newcommand{\bu}{\mathbf{u}}
\newcommand{\bw}{\mathbf{w}}

\newcommand{\br}{\mathbf{r}}

\newcommand{\BB}{\mathbf{B}}
\newcommand{\DD}{\mathbf{D}}

\newcommand{\ZZ}{\mathbf{Z}}

\newcommand{\VV}{\mathbf{V}}
\newcommand{\WW}{\mathbf{W}}

\newcommand{\UU}{\mathbf{U}}
\newcommand{\MM}{\mathbf{M}}
\newcommand{\PP}{\mathbf{P}}
\newcommand{\Id}{\mathbf{I}}

\newcommand{\zsig}{\boldsymbol\sigma}

\newcommand{\tP}{\tilde\Psi}

\newcommand{\tPh}{\tilde\Phi}
\newcommand{\tPs}{\tilde\Psi^{*}}

\newcommand{\bV}{\bm{V}}

\newcommand{\bX}{\bm{X}}

\newcommand{\bean}{\begin{eqnarray*}}
\newcommand{\eean}{\end{eqnarray*}}
\newcommand{\bea}{\begin{eqnarray}}
\newcommand{\eea}{\end{eqnarray}}
\newcommand{\be}{\begin{eqnarray}}
\newcommand{\ee}{\end{eqnarray}}
\newcommand{\beq}{\begin{equation}}
\newcommand{\eeq}{\end{equation}}

\renewcommand{\hat}{\widehat}
\renewcommand{\tilde}{\widetilde}

\newcommand{\IF}{\boldsymbol{1}} 

\newcommand{\dom}{\mbox{dom}\,}

\newcommand{\itg}{\lfloor t/\gamma \rfloor}

\DeclareMathOperator{\vect}{vec}

\newtheorem{theo}{Theorem}[section]
\newtheorem{lemma}[theo]{Lemma}
\newtheorem{cor}[theo]{Corollary}
\theoremstyle{definition}
\newtheorem{rem}[theo]{Remark}

\newtheorem{defin}[theo]{Definition}
\newtheorem{example}[theo]{Example}
\counterwithin{figure}{section}
\renewenvironment{proof}[1][\proofname]{{\noindent\bfseries #1.}}{\qed} 

\DeclareMathOperator{\sgn}{sgn}
\DeclareMathOperator{\diag}{diag}

\endlocaldefs

\counterwithin{table}{section}

\begin{document}

\begin{frontmatter}

\title{A Generalization of Regularized Dual Averaging and its Dynamics\protect\thanksref{T1}}
\runtitle{gRDA and its dynamics}
\thankstext{T1}{Shih--Kang Chao wants to thank the following for invaluable discussions and comments: Iosif Pinelis at Michigan Technological University, Samy Tindel at Purdue University, Jia--Yuan Dai at National Taiwan University and Jos\'e E. Figueroa--L\'opez at Washington University in St. Louis. Guang Cheng is a visiting member of Institute for Advanced Study, Princeton in the Fall of 2019; he would like to thank the IAS for its hospitality}

\begin{aug} 
  \author{\fnms{Shih--Kang}  \snm{Chao}\corref{}\thanksref{t2}\ead[label=e1]{chaosh@missouri.edu}}
    \and
  \author{\fnms{Guang} \snm{Cheng}\thanksref{t3}\ead[label=e2]{chengg@purdue.edu}}

  \thankstext{t2}{Partially supported by the Research Council grant of University of Missouri--Columbia.}
  \thankstext{t3}{Partially supported by NSF DMS-1712907, DMS-1811812, DMS-1821183, and Office of Naval Research (ONR N00014-18-2759).}

  \runauthor{S.--K. Chao and G. Cheng}

  \affiliation{University of Missouri--Columbia and Purdue University}

  \address{Department of Statistics\\
  University of Missouri--Columbia\\
  Middlebush Hall 146\\
  Columbia, MO 65201\\
  U.S.A.\\ 
          \printead{e1}}

  \address{Department of Statistics\\
  Purdue University\\
  250 N. University Street\\
  West Lafayette, IN 47907\\
  U.S.A.\\
          \printead{e2}}

\end{aug}

\begin{abstract}
Excessive computational cost for learning large data and streaming data can be alleviated by using stochastic algorithms, such as stochastic gradient descent and its variants. Recent advances improve stochastic algorithms on convergence speed, adaptivity and structural awareness. However, distributional aspects of these new algorithms are poorly understood, especially for structured parameters. To develop statistical inference in this case, we propose a class of {\em generalized} regularized dual averaging (gRDA) algorithms with constant step size, which improves RDA \citep{X10,FB17}. Weak convergence of gRDA trajectories are studied, and as a consequence, for the first time in the literature, the asymptotic distributions for online $\ell_1$ penalized problems become available. These general
results apply to both convex and non-convex differentiable loss functions, and in particular, recover the existing regret bound for convex losses \citep{NJLS09}. As important applications, statistical inferential theory on online sparse linear regression and online sparse principal component analysis are developed, and are supported by extensive numerical analysis. Interestingly, when gRDA is properly tuned, support recovery and central limiting distribution (with mean zero) hold simultaneously in the online setting, which is in contrast with the biased central limiting distribution of batch Lasso \citep{KF00}. Technical devices, including weak convergence of stochastic mirror descent, are developed as by-products with independent interest. Preliminary empirical analysis of modern image data shows that learning very sparse deep neural networks by gRDA does not necessarily sacrifice testing accuracy.
\end{abstract}



\end{frontmatter}

\section{Introduction}\label{sec:intro}

Excessively large data and streaming data sets create unique challenges for modern statistics. These data arise from fast-growing applications including image recognition, social media, e--commerce, environmental surveillance and numerous others. Analyzing such data requires new computational methods in order to overcome storage and processing constraints.

To meet the computational challenges, stochastic optimization methods, such as stochastic gradient descent (SGD, \cite{RM51}), become widely used. In particular, SGD concerns statistical estimation of $\bw^*$, which is a solution of 
\begin{align}
	G(\bw)=0.\label{eq:sa}
\end{align}
A common choice of $G(\bw)$ is $\E_Z[\nabla f(\bw;Z)]$, where $f:\R^d\times\Zc\to\R$ is some loss function and $\nabla f(\bw;Z)$ is its gradient or subgradient w.r.t. $\bw$. For non-convex $f$, the $\bw^*$ satisfying \eqref{eq:sa} may not be unique, though. In reality, the form of $G$ is often unknown such that one may resort to its stochastic version $G(\bw;Z)$ satisfying $\E_Z[G(\bw;Z)]=G(\bw)$. For example, given i.i.d. observations $\{Z_1,...,Z_N\}\ni Z_n$,
$G(\bw;Z_n)$ can be chosen as $\nabla f(\bw;Z_n)$. SGD approximates $\bw^*$ as follows:
\begin{align}\tag{\texttt{SGD}}
	\begin{split}\label{eq:sgd}
	\bw_{n+1}^\gamma &= \bw_n^\gamma - \gamma \nabla f(\bw_n^\gamma;Z_{n+1}),
	\end{split}
\end{align}
where $\gamma>0$ is a small constant step size, and is chosen and fixed from the initialization. Setting a different value of $\gamma$ would change the entire process, not only a single $\bw_n^\gamma$. Also, it is known that SGD produces statistically unbiased estimate for $\bw^*$ \citep{BMP90,KY03}. From now on, if no confusion occurs, $\bw_n^\gamma$ may be shortened as $\bw_n$. 

 
A well-known drawback of SGD is, however, that it fails to adapt to the intrinsic structure of $\bw^*$ such as sparsity. It is also known that including the penalty such as the $\ell_1$ norm as part of the loss function $f$ in SGD fails to penalize $\bw_n$ effectively \citep{DS09}. This motivates the regularized dual averaging (RDA) algorithm \citep{X10,FB17}, which has found wide applications in online learning and reinforcement learning \citep{ML12,LML12,MLT14,Y18}. Specifically, it includes a convex penalty function $\Pc(\bw)$ in the following way: for a constant $c_0>0$,
\begin{align}\tag{\texttt{RDA}}
		\bw_{n+1}=\arg\,\min_{\bw\in\R^d}\bigg\{\gamma\bw^\top \underbrace{\sum_{i=0}^n\nabla f(\bw_i;Z_{i+1})}_{(*)} + c_0 n\gamma \Pc(\bw) + F(\bw)\bigg\}.\label{eq:rda}
\end{align}
In the above, $F$ is a deterministic and strongly convex regularizer that stabilizes the iterates in the spirit of the ``follow-the-regularized-leader'' in Section 2.3 of \cite{S12}, which functions differently from the penalty $\Pc(\bw)$. The \eqref{eq:rda} essentially performs two iterative steps: accumulating the gradients as in ($\ast$), and then performing regularization with $F(\bw)$ and penalization with $c_0 n\gamma \Pc(\bw)$. If we set $\Pc(\bw)=0$ and $F(\bw)=\frac{1}{2}\|\bw\|_2^2$, \eqref{eq:rda} reduces to \eqref{eq:sgd}. Therefore, \eqref{eq:sgd} may be viewed as an un-penalized version of \eqref{eq:rda}. If we further change $\Pc(\bw)$ to $\|\bw\|_1$, sparse solutions can be produced \citep{X10, LW12}. 

Our first contribution is to prove that the penalization of \eqref{eq:rda} may be so aggressive that the estimates are biased in some important cases, in contrast with \eqref{eq:sgd}. A look at \eqref{eq:rda} reveals that the diverging rate of the factor $n\gamma$ (noting that $\gamma$ is fixed 
) is the reason for overly aggressive penalization. This observation motivates us to design a class of new algorithms, named as generalized RDA (gRDA), that can adjust the level of penalization with time through a tuning function $g(n,\gamma)$ as follows
\begin{align}\tag{\texttt{gRDA}}
	\bw_{n+1}=\arg\,\min_{\bw\in\R^d}\bigg\{\bw^\top \Big(-\bw_0+\gamma\sum_{i=0}^n\nabla f(\bw_i;Z_{i+1})\Big) + g(n,\gamma) \Pc(\bw) + F(\bw)\bigg\},\label{eq:grda_gen}
\end{align}
where $g(n,\gamma)\geq 0$ is a deterministic non-negative function of $n,\gamma$. It is clear that $g(n,\gamma)=n\gamma$ and $\bw_0=0$ reduce \eqref{eq:grda_gen} to \eqref{eq:rda}. The initializer $\bw_0$ in gRDA is usually selected randomly with normal or uniform distribution centered at 0. The general framework of (\ref{eq:grda_gen}) also covers \texttt{AdaGrad} of \cite{DHS11} that sets $F$ as Mahalanobis norm adapted to past gradients, and \texttt{FTPRL} of \cite{MS10} that sets $F$ as a squared $\ell_2$ norm centered at the last iterate.

Figure \ref{fig:cifar10} compares the performance of multiple popular algorithms (see Chap. 8 of \cite{GBC16}) with \eqref{eq:grda_gen} on training a simple 6-layer convolutional neural network (CNN) with around 1.25 million coefficients, using the modern CIFAR--10 image classification data set\footnote{Data and task description: \url{https://www.cs.toronto.edu/~kriz/cifar.html}; a description of the CNN used here is in the documentation of \texttt{Keras}: \url{https://keras.io/examples/cifar10_cnn/}.}. For the illustration purpose, we set $F(\bw)=\frac{1}{2}\|\bw\|_2^2$, $\Pc(\bw)=\|\bw\|_1$ and \begin{align}
	g(n,\gamma)=c \gamma^{1/2}(n\gamma-t_0)_+^\mu, \label{eq:tunel1}
\end{align}
where $c=\gamma=0.005$, and $t_0\geq 0$ is the time mean dynamics (see \eqref{eq:convode} below) reaches a neighborhood of $\bw^*$. In this empirical demonstration, we set $t_0=0$, which works well. With $\mu=1$, the gRDA behaves similarly to \eqref{eq:rda} as epoch (or $n$) gets larger, and it creates high sparsity in the CNN but sacrifices much testing accuracy than the other algorithms. Among all the algorithms, \eqref{eq:grda_gen} with $\mu=0.7$ provides the best trade-off between accuracy and sparsity. In parctice, the sparse neural networks adjust the over-parametrization to a proper level so that requirements on storage and memory can be reduced; see e.g. \cite{dnnspa18, HMD16, LWK18}. Note that our aim here is to illustrate the difference between algorithms, without paying effort in maximizing the accuracy of image recognition.
	\begin{figure}[!h]  
 \includegraphics[scale=0.3]{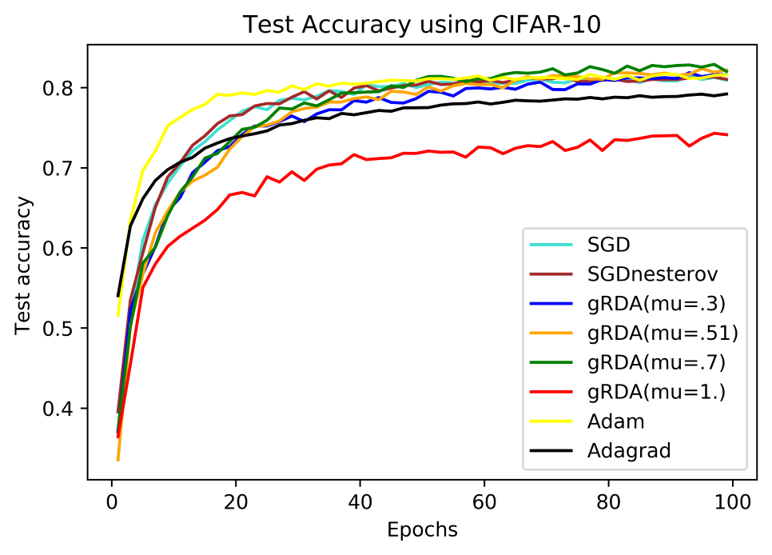}
 \includegraphics[scale=0.3]{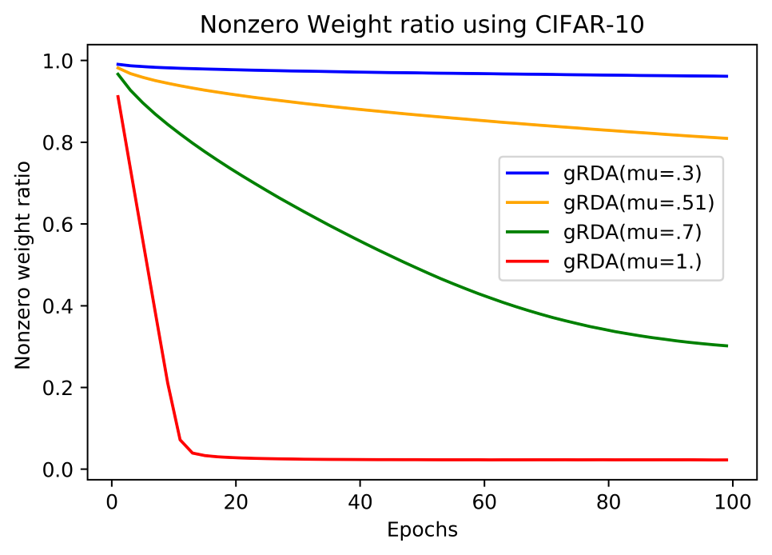}
 \caption{Left: testing accuracy of CIFAR--10 data with different algorithms. Right: the ratio of nonzero coefficients to the number of total coefficients. The initial step size for \texttt{Adagrad} and \texttt{Adam} are 0.005 and 0.0001 respectively. Minibatch size = 10. Except for gRDA, no other algorithms here can generate sparsity.}\label{fig:cifar10}
\vspace{-0.3cm}
\end{figure}

Understanding the above observation and finding the best choice of $g(n,\gamma)$ in \eqref{eq:grda_gen} require us to develop distributional analysis that provides a complete picture on the algorithm. Despite the well established distributional analysis for \eqref{eq:sgd}, e.g., \cite{KY03,BMP90}, the existing theoretical analysis for \eqref{eq:rda}, e.g. \cite{X10,LW12,OCC15,FB17}, only focus on regret and convergence (with decreasing step size) analysis. One technical reason could be that \eqref{eq:rda} performs $\arg\min$ in every step with a time-varying penalty such that the martingale weak convergence theory \citep{EK86} in the analysis of \eqref{eq:sgd} cannot be directly applied.


Our second contribution is a rigorous and thorough distributional analysis for the learning process of \eqref{eq:grda_gen}, which holds for both convex and non-convex differentiable $f$. Define $\bw_\gamma(t):=\bw_{\itg}^\gamma$ as a piecewise constant stochastic process indexed by $t>0$. The process convergence is developed for $\bw_\gamma(t)$ by passing $\gamma\to 0$. 
We prove that as $\gamma\to 0$,
{\small
\begin{align}
	\bw_\gamma(t) \weakto \bw(t) := \psi_{g}(t,\bv(t)), \mbox{ where } \frac{d\bv(t)}{dt} = G(\psi_{g}(t,\bv(t))), \  \bv(0)=\bw_0.\label{eq:convode}
\end{align}}
where $\weakto$ denotes the weak convergence, and $G$ is defined in \eqref{eq:sa}. Note that the ordinary differential equation (ODE) in \eqref{eq:convode} may have multiple solutions. If the solution is unique, then the weak convergence can be improved to convergence in probability, which is analogous to the classical weak law of large number. For some Gaussian process $\bX(t)$ with independent increments, we develop the central limit theorem as follows: consider the re-scaled process $\WW_\gamma(t)=\frac{\bw_\gamma(t)-\bw(t)}{\sqrt{\gamma}}$, we prove
{\small 
\begin{align}
	\WW_\gamma(t)\weakto \WW(t):=\tilde\psi_{g}(t,\VV(t)), \mbox{ where } d\VV(t)=\nabla G(\tilde\psi_{g}(t,\VV(t)))dt + d\bX(t),\label{eq:convsde}
\end{align}
}
with $\VV(0)=0$, where $\nabla G$ is a matrix of derivative of $G$; $\psi_{g}$ and $\tilde \psi_{g}$ are some functions depending on $g$, and will be made explicit in the later sections. 
Properties of $\bv(t)$ and $\VV(t)$ can be used to infer the behavior of $\bw_\gamma(t)$ and $\WW_\gamma(t)$.


Among many interesting choices of $F$ and $\Pc$ for \eqref{eq:grda_gen} (see Example \ref{ex:FP}), we focus on the leading example $F(\bw)=\frac{1}{2}\|\bw\|_2^2$ and $\Pc(\bw)=\|\bw\|_1$ that potentially introduces sparsity. In this setup, a consequence of \eqref{eq:convode} is that \eqref{eq:rda}, i.e., $g(n,\gamma)=n\gamma$, is biased for some important learning problems (similar as LASSO for batch learning) with bias $c_0$. On the other hand, to ensure a non-degenerate limiting distribution, i.e. \eqref{eq:convsde}, our theory suggests $g(n,\gamma)\asymp\sqrt{\gamma}$. 
In addition, if we replace $c_0$ in \eqref{eq:rda} by a time decreasing function, e.g. $t^{\mu-1}=(n\gamma)^{\mu-1}$ with $\mu<1$, the bias will vanish. This motivates the form of \eqref{eq:tunel1}. A caveat is that $\mu$ also needs to be large enough to ensure support recovery. We observe that the components of $\VV(t)$ in \eqref{eq:convsde} corresponding to inactive coefficients asymptotically evolve like a Gaussian process that grows in rate $t^{1/2}$ (up to an iterative logarithmic term), while the active coefficients diverge like $t^{\max\{1/2,\mu\}}$. Hence, \eqref{eq:tunel1} with $\mu>0.5$ creates a contrast in learning dynamics that distinguishes the active and inactive set. The intuition above is rigorously justified in the setup of least square regression with orthogonal Hessian matrix for an arbitrarily small $t_0>0$ in \eqref{eq:tunel1}. The orthogonality condition is suggested to be unnecessary by the simulations in Section \ref{sec:simlr}. In addition, the same conclusion seems to continue to hold for online sparse principal component analysis (see Section \ref{sec:simpca}). Encouraged by the preliminary analysis in Figure \ref{fig:cifar10}, we conjecture that \eqref{eq:tunel1} is a universal recipe for using \eqref{eq:grda_gen} with the $\ell_1$ penalization and can work for more difficult task such as deep learning.

Both results, i.e., \eqref{eq:convode} and \eqref{eq:convsde}, apply to the entire training process $\{\bw_n^\gamma\}_{n\in\N}$, and their analysis relies on the so-called ``stochastic mirror descent'' (SMD) representation. The SMD representation for $\bw_\gamma(t)$ in \eqref{eq:convode} is known since \cite{OCC15}; see \eqref{eq:md}. However, to analyze $\WW_\gamma(t)$ in \eqref{eq:convsde}, we need to construct a new SMD representation using the localized Bregman divergence at $\bw(t)$; see Lemma \ref{lem:exprW}. Stochastic mirror descent represents a rich family of algorithms, but no distributional result has been considered for either time-varying or time-invariant regularizers; see the Related Works below. In this respect, our general theory is also new to the constant step size SMD literature.


Our theory not only provides a fundamental understanding to \eqref{eq:grda_gen}, but also, for the first time in the literature, makes uncertainty quantification possible for online sparse algorithms penalized by convex functions. The asymptotic confidence bands for the complete training process can be constructed with our theory. In particular, the confidence bands are non-smooth at the point where the mean trajectory switches sign. To illustrate, we consider asymptotic confidence bands for online sparse linear regression and online sparse principal component analysis, and validate the coverage by simulations. 

\texttt{Python} code for \eqref{eq:grda_gen} is on Github: \href{https://github.com/donlan2710/gRDA-Optimizer}{\texttt{gRDA-Optimizer}}.

The rest of this paper is organized as follows. In Section \ref{sec:cvx}, the stochastic mirror descent representation of \eqref{eq:grda_gen} is introduced. Section \ref{sec:gentheo} contains the main theoretical results of \eqref{eq:grda_gen}. After the general theory is introduced, we tailor it for the important $\ell_1$ norm penalization, and discuss the sufficient conditions that warrant the asymptotic analysis for this case. In Section \ref{sec:lrl1}, we discuss the oracle properties that arise when applying \eqref{eq:grda_gen} to the online sparse linear regression. Section \ref{sec:pca} focuses on the online sparse principal component analysis and its dynamics. In Section \ref{sec:sim}, we validate the relevance of our asymptotic results for non-infinitesimal step sizes through extensive Monte Carlo experiments. Proofs of all theoretical results are deferred to the online supplement. 

\noindent{\bf Related works.} 
We discuss some existing theoretical works of \eqref{eq:sgd}, \eqref{eq:rda} and stochastic mirror descent (SMD). SGD can also be implemented with shrinking step size $\gamma_n$ in $n$, and statistical inference based on this has been studied in \cite{SZ18} and \cite{CLTZ19}. But in this case, it often takes longer to converge for some modern learning tasks, and is empirically harder to tune as the results are sensitive to the decreasing rate \citep{DDB17,CT18}. On the other hand, SGD with a constant step size is simpler to implement, and has been found effective in navigating through complex loss landscape of non-smooth convex loss \citep{BM13} or non-convex loss \citep[Chapter 5]{GBC16}.

In the literature, \eqref{eq:rda} is often considered with shrinking step size in $n$, e.g., $\gamma_n=c_1 n^{-\alpha}$, where the values of $c_1,\alpha>0$ need to be selected with great care. For example, \cite{X10} proposed $\alpha=1/2$ for some constant $c_1>0$. For $\Pc(\bw)=\|\bw\|_1$, \cite{LW12} prove that as $n\to\infty$, $\bw_n$ converges to the minimizer of $L(\bw):=\E_Z[f(\bw;Z)]+c_0\|\bw\|_1$ where $f(\bw;Z)$ is required to be convex and smooth in $\bw$ (neither is required by us). This implicitly suggests that $\lim_{n\to\infty}\bw_n$ is biased for $\arg\min\E_Z[f(\bw;Z)]$ \citep{fan2001variable}. Our results show that \eqref{eq:rda} with constant step size is biased, which is remotely related to \cite{LW12}. In practice, the performance of RDA with shrinking $\gamma_n$ is very sensitive to the constant $c_0$ that controls the penalization strength in \eqref{eq:rda}, for either convex loss \citep{X10} or non-convex loss in deep learning \citep{mrda18}.

The SMD evolves from the classical mirror descent proposed by \cite{NY83}, and is arguably more robust than \eqref{eq:sgd}. The literature for the SMD is large and still growing fast. 
Traditionally, SMD is implemented with time-invariant regularizers, and convergence in mean or almost surely was obtained for convex and non-convex losses; see, e.g. \cite{NJLS09,DAJJ12,ZMBBG17,LZ18,ZH18,JNNT19}. Recently, many authors implicitly or explicitly consider generalized SMD with time-varying regularizers, with proven regret bounds; see e.g. \cite{V01,AW01,BHR07,ST10,OCC15}.

\noindent{\bf Notations.} For $a\geq 0$, $\lfloor a\rfloor$ returns the greatest integer less than or equal to $a$. Define $D([0,\infty))^d$ and $C([0,\infty))^d$ as spaces consisting of {c\'adl\'ag} functions and continuous functions mapping from $[0,\infty)$ to $\R^d$, respectively. The Skorohod metric $\rho_d^\infty$ on $D([0,\infty))^d$ is defined by $\rho_d^\infty(\bx,\by) = \sum_{j=1}^d \rho^\infty(x_j,y_j)$, where $x_j,y_j\in D[0,\infty)$ and
\begin{align}\label{eq:metric}
	\rho^\infty(x,y) := \inf_{\nu\in\Vc}\Big[&\sup_{0\leq t<s}\Big|\log\frac{\nu(s)-\nu(t)}{s-t}\Big|\\
	&\vee \int_0^\infty e^{-u} \sup_{0\leq t\leq u} \big\{\|x(t)-y(\nu(t))\|_2 \wedge 1\big\}du\Big].\notag 
\end{align}
Here, $\Vc$ is a set of Lipschitz continuous functions mapping from $[0,\infty)$ onto $[0,\infty)$. A sequence of random elements $X_\gamma$ in $\big(D([0,\infty))^d,\rho_d^\infty\big)$ weakly converges to $X$, if $\E[g(X_\gamma)]\to\E[g(X)]$ as $\gamma\to 0$ for any real valued bounded continuous function $g$ in the above space \citep{VW96}. 

\section{Stochastic mirror descent representation}\label{sec:cvx}

In the theoretical analysis, the stochastic mirror descent representation of \eqref{eq:grda_gen} is used. Define
\begin{align}
	\Psi_{n,\gamma}(\bw) := g(n,\gamma) \Pc(\bw) + F(\bw), \label{eq:smdreg}
\end{align}  
where $g(n,\gamma)$ is the tuning function. For any $\bv\in\R^d$, define the Fenchel conjugate
\begin{align}
\Psi_{n,\gamma}^*(\bv):=\max_{\bw\in\R^d}\big\{\langle \bw,\bv\rangle - \Psi_{n,\gamma}(\bw)\big\}. \label{eq:fenchel}
\end{align}

We need the following condition on $\Pc(\cdot)$ and $F(\cdot)$ throughout the paper. 
\begin{itemize}
	\labitemc{(R)}{as:R} $\Pc(\bw)$ is convex, $F(\bw)$ is $\beta$-strongly convex, where $\beta>0$ is a constant, and both are lower semicontinuous (l.s.c.)\footnote{A function $h:\R^d\to\R$ is {lower semicontinuous} if $\liminf_{\bu\to\bu_0} h(\bu)\geq h(\bu_0)$ for every $\bu_0\in\R^d$.} and finite.
\end{itemize}
Under the Condition \ref{as:R}, Proposition 11.3 of \cite{RW09} implies that $\Psi^*_{n,\gamma}$ is differentiable, and its derivative satisfies
\begin{align}
\nabla\Psi_{n,\gamma}^*(\bv)= \arg\,\min_{\bw\in\R^d}\big\{\Psi_{n,\gamma}(\bw)-\bw^\top\bv\big\}. \label{eq:dconj}
\end{align}

With these notations, \eqref{eq:grda_gen} can be rewritten in the SMD representation \citep{OCC15}:
 \begin{align}\tag{\texttt{gRDA-SMD}}
	\begin{split}\label{eq:md}
	\bv_{n+1} &= \bv_n - \gamma \nabla f(\bw_n;Z_{n+1}),\\
	\bw_{n+1} &= \nabla\Psi_{n+1,\gamma}^*(\bv_{n+1}),%
	\end{split}
\end{align}
where $\bv_0=\bw_0$, $\bw_n=(w_{n,1},w_{n,2},...,w_{n,d})^\top$, $\bv_n=(v_{n,1},v_{n,2},...,v_{n,d})^\top$ and
$$
\nabla\Psi_{n,\gamma}^*(\bv)=\big(\nabla\Psi_{n,\gamma,1}^*(\bv),\nabla\Psi_{n,\gamma,2}^*(\bv),...,\nabla\Psi_{n,\gamma,d}^*(\bv)\big)^\top\in\R^d.
$$ 
The parameter $\bv_n$ is an accumulator of gradients in the dual space spanned by the gradients. 

We next provide some examples for $F(\bw)$ and $\Pc(\bw)$.
\begin{example}[The common choices of $F(\bw)$ and $\Pc(\bw)$]\label{ex:FP}
$F$ can be any strongly convex function, such as $F(\bw) = \frac{1}{2}\|\bw\|_2^2$ and $F(\bw) = \frac{1}{2}\bw^\top A \bw$ with some p.d. pre-conditioning matrix $A$. $\Pc(\bw)$ are usually norms of $\bw$, such as $\ell_p$ norm, where often $p=1,2$ and $\infty$. In particular, $\Pc(\bw)=\|\bw\|_1$ is arguably the most popular choice. Other common penalties include elastic net of \cite{zou2005regularization} and group LASSO of \cite{YL06}. They can be incorporated into \eqref{eq:grda_gen}, by setting
	\begin{align}
		\mbox{LASSO: }\Pc(\bw)&=\|\bw\|_1;\\
		\mbox{Elastic Net: }\Pc(\bw) &= \frac{\kappa}{2}\|\bw\|_2^2+\|\bw\|_1;\\
		\mbox{Group LASSO: }\Pc(\bw) &= \sum_{a\in\Gc} \|\bw_{a}\|_2,
		\end{align}
		where $\Gc$ is a partition of $\{1,2,...,d\}$, and $\bw_{a}$ is the group $a$ coefficients of $\bw=(w_{1},w_{2},...,w_{d})^\top$. Given that $F(\bw) = \frac{1}{2}\|\bw\|_2^2$, \eqref{eq:md} can take advantage of the closed-form proximal operator of $g(n,\lambda)\Pc(\bw)$ \citep{PB14} such that for $j=1,2,...,d$,
		{\small 
	\begin{align}
		\mbox{LASSO: }\nabla\Psi_{n,\gamma,j}^*(\bv) &= \sgn(v_j)\cdot\big(|v_j|-g(n,\gamma)\big)_+,\tag{\texttt{gRDA-$\ell_1$}},\label{eq:grdal1}\\
	\mbox{Elastic Net: }\nabla\Psi_{n,\gamma,j}^*(\bv) &= \frac{1}{1+\kappa g(n,\gamma)}\sgn(v_j)\cdot\big(|v_j|- g(n,\gamma)\big)_+, \label{eq:elcon1}\\
			\mbox{Group LASSO: }\nabla\Psi_{n,\gamma,a}^*(\bv) &= \Big(1-\frac{g(n,\gamma)}{\|\bv_{a}\|_2}\Big) \bv_{a}, \quad a\in\Gc,
	\end{align}}
	where $\bv_{a}$ is the group $a$ coefficients of $\bv=(v_{1},v_{2},...,v_{d})^\top$. 
\end{example}

We remark that for the above selection of $\Pc$, the computational cost per iteration in \eqref{eq:md} is as cheap as SGD due to the closed-form proximal operators. Therefore, the leading example in this paper is $F(\bw) = \frac{1}{2}\|\bw\|_2^2$.

\section{Asymptotic analysis of gRDA}\label{sec:gentheo}

In this section, we study asymptotic behaviors of the two sequences $\bv_n$ and $\bw_n$ defined in \eqref{eq:md} as the step size $\gamma\to 0$ along a countable sequence. The asymptotic trajectory is characterized by a system of {time-inhomogeneous} ordinary differential equations (ODE). As for the distributional dynamics, we find that it is in general not an Ornstein-Uhlenbeck type stochastic differential equations (SDE), and thus is different from the well studied SGD case \citep{BMP90}. As a corollary, the asymptotic distribution for the averaged estimator is also developed. In the end, we 
tailor our general theory to study a special case that $\Pc(\bw)=\|\bw\|_1$, i.e. \eqref{eq:grdal1}.

\subsection{Asymptotic trajectory}\label{sec:asymtr}

We need the following regularity conditions. 
\begin{itemize}
	\labitemc{(S)}{as:S} Data sequence $Z_1, Z_2,\ldots,Z_n$ are i.i.d.  
	\labitemc{(M)}{as:M} $G(\bw):\R^d\to\R^d$ is continuous on $\R^d$,
	and for any $K>0$,
	\begin{align}
		\E\Big[\sup_{\bw:\|\bw\|_2\leq K}\big\|\nabla f\big(\bw,Z)\big)\big\|_2\Big]<\infty.\label{eq:GL1}
	\end{align}
\end{itemize}
Our theory can be readily generalized to Markovian data related to the modern ``adversarial'' setting in online learning \citep{S12}, by setting $Z_{n+1}=Q(\bw_n,\xi_n)$ for some measurable function $Q:\R^d\times\Sc\to\Zc$, and i.i.d. random variables $\xi_n$. 

Recall from \eqref{eq:smdreg} that $\Psi_{n,\gamma}(\bw) = g(n,\gamma) \Pc(\bw) + F(\bw)$. For simplicity, we adopt the following reparametrization:
\begin{align}
\Psi_\gamma(t,\bw):=\Psi_{\itg,\gamma}(\bw)\;\;\mbox{for $t\geq 0$}. \label{eq:trcond}
\end{align}
Similarly, \eqref{eq:dconj} can be re-parametrized as
\begin{align}
	\nabla\Psi_\gamma^*(t,\bv)= \arg\,\max_{\bw\in\R^d}\big\{\langle \bw,\bv\rangle - \Psi_\gamma(t,\bw)\big\}.\label{eq:rpdconj}
\end{align}
The $\bv_n$ and $\bw_n$ in the SMD representation \eqref{eq:md} can also be re-expressed as
\begin{align}
	\bv_\gamma(t) &:= \bv_{\lfloor t/\gamma\rfloor}, \quad \bw_\gamma(t) := \bw_{\lfloor t/\gamma\rfloor} = \nabla\Psi_\gamma^*(t,\bv_\gamma(t)).\label{eq:defvw0}
\end{align}
In the expression of $\Psi_\gamma(t,\bw)$, we have $g(\lfloor t/\gamma \rfloor,\gamma)$ whose limit is guaranteed to exist by the following condition:
\begin{itemize}[itemindent=45pt,leftmargin=0pt]	
	\labitemc{($g$-lim)}{as:Plim} Assume $g(\lfloor\cdot/\gamma\rfloor,\gamma)\in D([0,\infty))$, and there exists a continuous non-negative function $g^\dagger:[0,\infty)\to[0,\infty)$ such that 
	\begin{align}
			\lim_{\gamma\to 0}\sup_{t\in[0,T]}\big|g(\itg,\gamma)-g^\dagger(t)\big|=0 \quad \mbox{for every $T>0$}. \label{eq:Plim}
	\end{align} 
\end{itemize}
This gives the limit of $\Psi_\gamma(t,\bw)$ (when $\gamma\to 0$) as
\begin{align}
	\Psi(t,\bw):=F(\bw) + g^\dagger(t)\Pc(\bw) \label{def:Psi}
\end{align}
Note that $g^\dagger(t)$ is allowed to be zero for all $t$. The Fenchel conjugate of $\Psi(t,\bw)$ and its derivative at $\bv\in\R^d$ can be defined similarly as \eqref{eq:fenchel} and \eqref{eq:dconj}:
\begin{align}
	\Psi^*(t,\bv)&:=\sup_{\bw\in\R^d}\big\{\bw^\top\bv-\Psi(t,\bw)\big\}, \notag\\
	\nabla\Psi^*(t,\bv)&:=\arg\,\min_{\bw\in\R^d}\big\{\Psi(t,\bw)-\bw^\top\bv\big\}.\label{eq:conjctucvx}
\end{align}

We are now ready to present our first main theorem characterizing the mean dynamics of \eqref{eq:md}, which can be regarded as the law of large number for the learning process as $\gamma\to 0$. By virtue of a powerful weak convergence theorem adapted from \cite{K09}, we do not need to directly impose conditions on the operator $\nabla\Psi_{n,\gamma}^*$, which can be difficult to verify. Therefore, in this respect Theorem \ref{th:at} gives a generally applicable tool on studying the asymptotic behavior of \eqref{eq:md}. As far as we know, even its application to \eqref{eq:rda} is new.

For any $K>0$, define a stopping time $\tau_\gamma^K:=\inf\{t:\|\bv_\gamma(t)\|_2\geq K\}$ and the corresponding ``stopped'' processes $\bv_\gamma^K(\cdot):=\bv_\gamma(\cdot\wedge \tau_\gamma^K)$ and $\bw_\gamma^K(\cdot):=\bw_\gamma(\cdot\wedge \tau_\gamma^K)$. We say a sequence $\{\bv_\gamma\}_\gamma$ is {relatively compact} in $D([0,\infty))^d$ if for every subsequence $\bv_{\gamma'}(t)$, there exists a further subsequence that weakly converges in $D([0,\infty))^d$ [see p.57 of \cite{B99}].
\begin{theo}[Asymptotic trajectory]\label{th:at}
	Suppose that \ref{as:R}, \ref{as:S}, \ref{as:M} and \ref{as:Plim} hold. 
	Let $\bv_0$ be a fixed initial vector for $\bv_n$. 
	Then, we have
	\begin{itemize}
		\item[(a)] For each $K$, $\{\bv_\gamma^K,\gamma>0\}$ is relatively compact in $D([0,\infty))^d$, and as $\gamma\to 0$, every limit of $(\bv_\gamma^K(t),\bw_\gamma^K(t))$ satisfies
	\begin{align}
		\bv(t)&=\bv_0-\int_0^t G\big(\nabla\Psi^*(s,\bv(s))\big)ds,\label{eq:weak0v}\\
		\bw(t)&=\nabla\Psi^*(t,\bv(t))\label{eq:weak0w},
	\end{align}
	for all $t<\tau^K=\inf\{t:\|\bv(t)\|_2\geq K\}$, where $\nabla\Psi^*(t,\cdot)$ is defined in \eqref{eq:conjctucvx}.
	\item[(b)] If the solution to \eqref{eq:weak0v} is unique, then the sequence converges in probability, i.e. for every $T,\epsilon>0$, $\lim_{\gamma\to 0} P\big(\sup_{0\leq t < T\wedge \tau^K}\|\bv_\gamma^K(t)-\bv(t)\|_2>\epsilon\big)=0$, and the same holds for $\bw_\gamma^K(t)$. 
	\item[(c)] If all the solutions to the ODE in \eqref{eq:weak0v} are bounded in finite time, then $(\bv_\gamma(t),\bw_\gamma(t))$ are relatively compact with every limit satisfying \eqref{eq:weak0v} and \eqref{eq:weak0w} for all $t\in[0,\infty)$.
	\end{itemize}
\end{theo}
See Section \ref{sec:pfat} for a proof of Theorem \ref{th:at}. Note that $\nabla\Psi^*(t,\bv)$ is just the $\psi_{g}(t,\bv)$ in \eqref{eq:convode}. When the solution of \eqref{eq:weak0v} is not unique, (b) in Theorem \ref{th:at} does not hold, but (c) still applies.


\begin{rem}[Sufficient conditions for the uniqueness in (b) of Theorem \ref{th:at}]\label{rmk:at}
	If $G$ is locally Lipschitz continuous, i.e. 
	\begin{align}
		\sup_{\bw\neq \bw', \bw,\bw'\in\Wc}\frac{\big|G(\bw)-G(\bw')\big|}{|\bw-\bw'|}<\infty, \mbox{ for every compact $\Wc\subset\R^d$},\label{eq:locliput}
	\end{align}
	then $G\big(\nabla\Psi^*(t,\cdot))$ is locally Lipschitz uniformly with respect to $t$ because $\nabla\Psi^*(t,\cdot)$ is globally Lipchitz uniformly with respect to $t$ by Lemma \ref{lem:uc}(a). Under the above condition, the Picard-Lindel\"of theorem (see Theorem 2.2 in \cite{T12}) asserts that there exists a unique local solution $\bv(t)$ of \eqref{eq:weak0v} for $t\leq\overline{T}$ and some $0<\overline{T}<\infty$. This can be extended to $\overline{T}=\infty$ if $G$ is globally Lipschitz continuous, see, e.g. Corollary 2.6 in \cite{T12}. 
\end{rem}


\subsection{Distributional dynamics}\label{sec:disdy}



To study the distributional dynamics, we need stronger moment and differentiability conditions on $G$ than those in Assumption \ref{as:M} as follows. For a fixed $\bw\in\R^d$, define the covariance matrix
\begin{align}
	\Sigma(\bw):=\E\big[\big(\nabla f(\bw;Z)-G(\bw)\big)\big(\nabla f(\bw;Z)-G(\bw)\big)^\top\big].\label{eq:sig}
\end{align}
\begin{itemize}
		\labitemc{(L)}{as:L} $G:\R^d\to\R^d$ is continuously differentiable on $\R^d$ (i.e. $\nabla \E_Z[\nabla f(\cdot;Z)]: \R^d\to\R^{d\times d}$ is continuous on $\R^d$). $\Sigma:\R^d\to\R^{d\times d}$ is continuous on $\R^d$. For any $K>0$,
		\begin{align}
			\E\Big[\sup_{\bw:\|\bw\|_2\leq K}\big\|\nabla f\big(\bw,Z\big)\big\|_2^2\Big]<\infty.\label{eq:mom2bdd}
		\end{align}
\end{itemize}
A consequence of Assumption~\ref{as:L} is that by Remark \ref{rmk:at}, $G(\nabla\Psi^*(t,\cdot))$ is locally Lipschitz uniformly with respect to $t$. So, there exists a unique solution $\bv(t)$ of \eqref{eq:weak0v} for $t\in[0,\overline{T}]$. Therefore, the quantities of interest to be studied in this section: 
\begin{align}
	\VV_\gamma(t) &:= \frac{\bv_{\gamma}(t)-\bv(t)}{\sqrt{\gamma}}, \quad \WW_\gamma(t) := \frac{\bw_{\gamma}(t)-\bw(t)}{\sqrt{\gamma}}\label{eq:defvw}
\end{align}
are well defined up to some time $\overline{T}$ under the condition \ref{as:L}.

\bigskip		
The first step in establishing the distributional dynamics is to link  $\VV_\gamma(t)$ with $\WW_\gamma(t)$ analytically. Specifically, it is through the local Bregman divergence at $(\bv(t),\bw(t))$ induced by $\Psi_\gamma(t,\bw)$:
	\begin{align}
		\tP_\gamma(t,\bu):=\gamma^{-1}\big(\Psi_\gamma(t,\bw(t)+\sqrt{\gamma}\bu)-\Psi_\gamma(t,\bw(t)) -\langle \sqrt{\gamma}\bu,\bv(t)\rangle\big).\label{eq:locdgf}
	\end{align}
In some sense, (\ref{eq:locdgf}) is similar to the local loss function used for deriving weak convergence of $M$-estimate, see \cite{VW96}. 

Lemma \ref{lem:exprW} shows that the SMD representation for $\WW_\gamma(t)$ can be derived from $\tP_\gamma(t,\bu)$, which facilitates the subsequent theoretical analysis. 
\begin{lemma}[SMD representation of $\WW_\gamma(t)$]\label{lem:exprW}
	For each $t\geq 0$ and $\gamma$,
	\begin{align}
		\WW_\gamma(t)=\nabla\tPs_\gamma(t,\VV_\gamma(t)), \label{eq:exprW}
	\end{align}
	where $\nabla\tPs_\gamma(t,\VV)$ is defined similarly as \eqref{eq:rpdconj} by replacing $\Psi_\gamma$ with $\tP_\gamma$ in \eqref{eq:locdgf}.
\end{lemma}
See Section \ref{sec:pfdisdy} for a proof of Lemma \ref{lem:exprW}. We give a concrete example on the form of local Bregman divergence $\tP_\gamma(t,\bu)$.


\begin{example}[Constant tuning function]\label{ex:ldgf}
Suppose that $g(n,\gamma):=c$ is a constant, and that $F$ and $\Pc$ are twice continuously differentiable. Hence, $\Psi_{n,\gamma}(\bw)=\Qc(\bw):=c\Pc(\bw)+F(\bw)$ in \eqref{eq:smdreg}. Suppose $\inf_{t\geq 0}\lambda_{\min}(\nabla^2\Qc(\bw(t))>0$, where $\bw(t)$ is the mean dynamics. Replacing $\Psi_\gamma(t,\bw)=\Qc(\bw)$ and $\bv(t)=\nabla \Qc(\bw(t))$ in \eqref{eq:locdgf}, Taylor expansion for multivariate function gives
	\begin{align*}
		\tP_\gamma(t,\bu) &= \bu^\top \Big( \int_0^1 (1-s) \nabla^2\Qc\big(\bw(t)+s \sqrt{\gamma}\bu\big) ds\Big)\bu.
	\end{align*}
\end{example}

In the above example, if additionally $\sup_{t\leq T}\big\|\nabla^2\Qc\big(\bw(t)+s \sqrt{\gamma}\bu\big)\big\|_2$ is bounded uniformly in $\gamma$ for every $T$, we obtain the following limit as $\gamma\to 0$ (by the dominated convergence and the continuity of $\nabla^2 \Qc$):
	\begin{align*}
		\tP(t,\bu) = \frac{1}{2}\bu^\top \nabla^2\Qc(\bw(t)) \bu,
	\end{align*}
	whose conjugate is $\tP^*(t,\bu) = \frac{1}{2}\bu^\top (\nabla^2\Qc)^{-1}(\bw(t)) \bu$.
	
In general, we need to assume the existence of such a limit together with some regulatory properties. 
\begin{itemize}[itemindent=45pt,leftmargin=0pt]	
	\labitemc{($\tP$-lim)}{as:tPlim} 
	 There exists $\tP:[0,\infty)\times\R^d\to\R$ such that
	\begin{itemize}[itemindent=60pt,leftmargin=0pt]	
		\labitemc{($\tP$-lim0)}{as:tPlim0} For any $T>0$, $\sup_{0\leq t\leq T}|\tP_\gamma(t,\bu) -\tP(t,\bu)|\to 0$ pointwise for every $\bu\in\R^d$.
		\labitemc{($\tP$-lim1)}{as:tPlim1} $\tP(t,\cdot)$ is l.s.c. for any $t\in[0,\infty)$ and $\tP(\cdot,\bu)$ is continuous on $[0,\infty)$ for any $\bu\in\R^d$.
		\labitemc{($\tP$-lim2)}{as:tPlim2} There exists a $\bu_0\in\R^d$ and a constant $c_T>0$ such that for any subgradient $\nabla\tP(t,\bu_0)\in\partial_\bu \tP(t,\bu_0)$ for $t\leq T$, $\sup_{0\leq t\leq T}\|\nabla\tP(t,\bu_0)\|_2 \leq c_T$. 
	\end{itemize}
\end{itemize}
Condition \ref{as:tPlim0} implicitly regularizes $F$, the penalty term $\Pc$ and tuning function $g(n,\gamma)$. Together with Condition \ref{as:Plim}, this consist of the only two conditions on $g(n,\gamma)$: Condition \ref{as:Plim} is for the asymptotic trajectory, while Condition \ref{as:tPlim0} is for the distributional dynamics.  

Now we are ready to present the second main result of this paper, which can be regarded as the central limit theorem of the learning process. 

\begin{theo}[Distributional dynamics]\label{th:dd}
Suppose the solution $\bv(t)$ of \eqref{eq:weak0v} uniquely exists for all $t\geq 0$, and $\bv_0$ be a constant initial value of the sequence $\bv_n$ in (\ref{eq:md}). In addition, assume that \ref{as:R}, \ref{as:S}, \ref{as:M}, \ref{as:L}, \ref{as:Plim} and \ref{as:tPlim} hold. If $\tP_\gamma(\cdot,\bu)\in D[0,\infty)$ for each $\bu\in\R^d$ and $\gamma$, then we have $\VV_\gamma(t)\weakto\VV(t)$ in $D([0,\infty))^d$ with
	\begin{align}
		\VV(t)&=-\int_0^t \nabla G(\bw(s)) \cdot \nabla\tPs(s,\VV(s))ds+\MM(t),\label{eq:weakv}
	\end{align}
	where $\MM(t)\in\R^d$ is a centered continuous process with independent Gaussian increment and covariance kernel
	\begin{align}
		\E\big[\MM(t)\MM(t)^\top\big] = \int_0^t \Sigma\big(\bw(s)\big)ds.\label{eq:covfn}
	\end{align}
	Moreover, $\WW_\gamma(t)\weakto\WW(t)$ in $D([0,\infty))^d$ with
	\begin{align}
		\WW(t)&=\nabla\tPs(t,\VV(t)).\label{eq:weakw}
	\end{align}
	where $\nabla\tPs(t,\VV)$ is defined similarly as \eqref{eq:conjctucvx} by replacing $\Psi$ with $\tP$.
\end{theo}
See Section \ref{sec:pfweak} for a proof of Theorem \ref{th:dd}. Note that $\nabla\tP^*(t,\bv)$ is just the $\tilde\psi_{g}(t,\bv)$ in \eqref{eq:convsde}.

A key step in the proof of Theorem \ref{th:dd} is a continuous mapping theorem adapted from \cite{K09}, which connects the weak convergence of $\VV_\gamma(t)$ and that of $\WW_\gamma(t)$. Additionally, we need to show the relative compactness of $(\VV_\gamma(t),\WW_\gamma(t))$. Theorem \ref{th:dd} recovers the result of \eqref{eq:sgd} in \cite{BMP90} and \cite{BKS93}, but does not apply to \eqref{eq:rda} which does not satisfy \ref{as:tPlim0}. 

Studying $\WW(t)$ inevitably requires understanding $\VV(t)$ in the first place. Unfortunately, the closed form solution of \eqref{eq:weakv} may be unavailable in many cases. Numerical methods, such as Euler-Maruyama or Milstein method \citep{KP92}, can be applied to obtain an approximate solution of $\VV(t)$, which suffices in many applications; e.g. constructing confidence band for the entire learning trajectory. 


\begin{rem}[Local convergence]\label{rem:thddex}
	The convergence in Theorem \ref{th:dd} is global in the sense that the convergence holds for the complete time domain $[0,\infty)$. Local convergence, namely $(\VV_\gamma(t),\WW_\gamma(t))\weakto (\VV(t),\WW(t))$ on $D(\Tc)^d$ for a compact $\Tc$, can be established by similar proof as that of Theorem \ref{th:dd}, under weaker assumptions that the \eqref{eq:Plim} in \ref{as:Plim} and \ref{as:tPlim0} hold for any compact subinterval $\Tc_1\subset\Tc$, instead of for $[0,T]$ for every $T>0$. The details are omitted.
\end{rem}

\bigskip
	For a sequence of estimates given by the stochastic gradient descent, it is known that averaging over them often leads to improved performance \citep{PJ92,R88,NJLS09}. Inspired by this observation, we consider
	\begin{align}
		\bar\bw_N = \frac{1}{N} \sum_{n=1}^N \bw_n. \label{eq:defavg}
	\end{align}
For any $T>0$, setting $N=\lfloor T/\gamma\rfloor$ and using the piecewise constant interpolation $\bw_\gamma(t)$ in \eqref{eq:defvw0}, \eqref{eq:defavg} can be rewritten as
\begin{align}
	\bar\bw_{\lfloor T/\gamma\rfloor} = \frac{1}{\gamma\lfloor T/\gamma\rfloor} \int_0^T \bw_\gamma(s) ds.\label{eq:barwT}
\end{align}
By applying Theorems \ref{th:at} and \ref{th:dd}, the asymptotic distribution of $\bar\bw_{\lfloor T/\gamma\rfloor}$ can be written as a functional of the limiting process $\WW(t)$ in \eqref{eq:weakw}. 
\begin{cor}[Asymptotic distribution of iterate average]\label{th:avg}
 Under the same conditions as Theorem \ref{th:dd}, we have for any fixed $T>0$,
 \begin{align}
 	\frac{\bar\bw_{\lfloor T/\gamma\rfloor}-T^{-1}\int_0^T\bw(s)ds}{\sqrt{\gamma}} \weakto \frac{1}{T}\int_0^T \WW(s) ds, \mbox{ as }\gamma\to 0,\label{eq:avg}
 \end{align}
where $\WW(s)$ is defined in \eqref{eq:weakw}.
\end{cor}
See Section \ref{sec:pfavg} for a proof of Corollary~\ref{th:avg}. Applying Corollary \ref{th:avg} to \eqref{eq:sgd} yields a distributional result complementary to that in Chapter 11.1.2 of \cite{KY03}, which concerns the convergence of $\bar\bw_{\lfloor (T+T_\gamma)/\gamma\rfloor}$, where $T_\gamma(\uparrow\infty$ as $\gamma$) is the transient time that $\bw_n$ takes to get closer to a stationary point $\bw^*$.

\begin{rem}[Excess risk of $\bar\bw_N$ under convex loss]\label{rem:excess}
Consider the same setting as Example \ref{ex:ldgf} that $\Psi_{n,\gamma}(\bw)=\Qc(\bw):=F(\bw)+c\Pc(\bw)$ is time-invariant. Suppose that the expected loss function $f_0(\bw):=\E_Z[f(\bw;Z)]$ is globally Lipschitz and convex. Corollary \ref{th:avg} implies an upper bound for the excess risk
	\begin{align}
		f_0(\bar\bw_{\lfloor T/\gamma\rfloor}) - \inf_{\bw} f_0(\bw) = O_p\bigg(\frac{\mbox{Breg}_\Qc(\bw_0,\bw^*)}{T}+  \sqrt{\frac{\gamma}{T}}\bigg). \label{eq:excess}
	\end{align}
where $\mbox{Breg}_\Qc(\bx,\by)=\Qc(\bx)-\Qc(\by)-\langle\nabla\Qc(\by),\bx-\by\rangle$ is the Bregman divergence induced by $\Qc$. See Section \ref{sec:pfavg} for a proof of \eqref{eq:excess}. If $F(\bw)=\frac{1}{2}\|\bw\|_2^2$, $\Pc(\bw)=0$ and $\bw_0=0$, then $\Qc(\bw)=\frac{1}{2}\|\bw\|_2^2$ and $\mbox{Breg}_\Qc(\bw_0,\bw^*)=\frac{1}{2}\|\bw^*\|_2^2\leq \frac{d}{2}\|\bw^*\|_\infty^2$. In this case, the rate \eqref{eq:excess} corresponds to the bound in page 1580 of \cite{NJLS09} (with $T = N\gamma$ and the optimal step size in (2.20) therein). The difference is that \eqref{eq:excess} holds in probability, which is slightly weaker than the bound in mean in   \cite{NJLS09}.
\end{rem}

\subsection{Distributional analysis under $\ell_1$ penalization}\label{sec:rdal1}
In this section, we apply the developed asymptotic results to an important case that 
$$F(\bw)=\frac{1}{2}\|\bw\|_2^2\;\;\mbox{and}\;\;\Pc(\bw)=\|\bw\|_1.$$ Define $\Lc_{n,\gamma}(\bw):=\frac{1}{2}\|\bw\|_2^2+g(n,\gamma)\|\bw\|_1.$ We obtain \eqref{eq:grdal1} by setting $\Psi_{n,\gamma}(\bw)=\Lc_{n,\gamma}(\bw)$ in the SMD representation \eqref{eq:md}. 

Using the above representation, we have Corollary \ref{cor:l1at} implied by Theorem \ref{th:at}. For $\bv=(v_1,...,v_d)$, define a soft-thresholding operator $\nabla\Lc^*(t,\bv)=(\nabla\Lc_1^*(t,v_1),...,\nabla\Lc_d^*(t,v_d))$ by
	\begin{align}
		\nabla\Lc_j^*(t,v_j) 
		:=  \sgn(v_j)\cdot\big(|v_j|-g^\dagger(t)\big)_+, \mbox{ }j=1,...,d,\label{eq:sop}
	\end{align}
which maps $\bv$ to 0 whenever $\bv\in [-g^\dagger(t),g^\dagger(t)]$. Recall that $g^\dagger(t)$ is the limit of $g(\lfloor t/\gamma\rfloor,\gamma)$ defined in \ref{as:Plim}.

\begin{cor}[Mean dynamics]\label{cor:l1at}
	Suppose that \ref{as:Plim}, \ref{as:S} and \ref{as:M} hold. If every solution to the ODE:
	\begin{align}
		\bv(t)&=\bv_0-\int_0^t G\big(\nabla\Lc^*(t,\bv(s))\big)ds,\label{eq:weak0vl1}
	\end{align}
	is bounded in finite time, then $$(\bv_\gamma(t),\bw_\gamma(t))\weakto (\bv(t),\bw(t))\;\;\mbox{on}\;\;D([0,\infty))^{2d},$$ where $(\bv_\gamma(t),\bw_\gamma(t))$ is defined in \eqref{eq:defvw0}, $\bv(t)$ is a solution of \eqref{eq:weak0vl1} and
		\begin{align}
			\bw(t)&=\nabla\Lc^*(t,\bv(t))\label{eq:weak0wl1}.
		\end{align}
\end{cor}
The proof of Corollary \ref{cor:l1at} is in Section \ref{sec:pfl1gen}. We remark that $\bw(t)$ is unique even when $\bv(t)$ may not be, and also that the sparsity of $\bw(t)$ is induced by the soft-thresholding operator $\nabla\Lc^*$. 

We now apply Corollary \ref{cor:l1at} to \eqref{eq:rda}, which reveals its bias in linear regression.
\begin{theo}[\eqref{eq:rda} is biased for active coefficients]\label{th:biasrda}
\eqref{eq:rda} with an arbitrary initializer $\bw_0$ and $c_0>0$ is biased for linear regression with $\ell_2$ loss (see \eqref{eq:modrdal1} for the model). In particular, if $X$ in \eqref{eq:modrdal1} is centered with $H=\E[XX^\top]=\mbox{diag}(\sigma_1^2,\sigma_2^2,...,\sigma_d^2)$, and if the mean trajectory $\bw(t)$ of \eqref{eq:rda} converges, i.e. $\lim_{t\to\infty}\bw(t)=\bw^\infty$, where $\bw^\infty\neq 0$, then $|w_j^\infty - w_j^*|=c_0/\sigma_j^2>0$ for all $j\in\{k:w_k^\infty\neq 0\}$.
\end{theo}
See Section \ref{sec:pfl1gen} for a proof of Theorem \ref{th:biasrda}. This theorem shows that the performance of \eqref{eq:rda} in the long run is biased like LASSO \citep{tibshirani1996regression,fan2001variable} in the batch learning, and is remotely related to the results of \cite{LW12} who treat \eqref{eq:rda} with shrinking step size. We remark that verifying the convergence of the mean trajectory $\bw(t)$ of \eqref{eq:rda} is nontrivial; see Remark \ref{rem:wsta} for more details.

\bigskip
We next characterize the distributional dynamics using Theorem \ref{th:dd}. According to (\ref{eq:locdgf}), the local Bregman divergence $\tilde\Lc_\gamma(t,\bu)$ induced by $\Lc_\gamma(t,\bu):=\Lc_{\itg,\gamma}(\bu)$:
\begin{align}
\tilde\Lc_\gamma(t,\bu):=\gamma^{-1}\big(\Lc_\gamma(t,(\bw(t)+\sqrt{\gamma}\bu)-\Lc_\gamma(t,(\bw(t)) -\langle \sqrt{\gamma}\bu,\bv(t)\rangle\big),\label{eq:locbregl1}
\end{align}
where $\bw(t)$ is the mean trajectory from \eqref{eq:weak0wl1}. Whether $\tilde\Lc_\gamma(t,\bu)$ satisfies the Condition \ref{as:tPlim} depends on the sign stability of the mean dynamics $\bw(t)$ defined below. Recall that $\bw(t) = \nabla \Psi^*(t,\bv(t))$ and $\bv(t)$ follows a time-inhomogeneous ODE in \eqref{eq:weak0v}. 
\begin{defin}[Sign stability]\label{def:sist}
A $d$-dimensional trajectory $\bw(t)=(w_1(t),w_2(t),...,w_d(t))^\top$ is \emph{sign stable on an interval $\Tc\subset [0,\infty)$} if for each $j=1,...,d$, $\sgn(w_j(t_1))=\sgn(w_j(t_2))$ for all $t_1,t_2\in\Tc$. If $\Tc= [0,\infty)$, then $\bw(t)$ is \emph{globally sign stable}. By convention, $\sgn(0)=0$. 
\end{defin} 

	\begin{figure}[!h] 
 \includegraphics[scale=0.12]{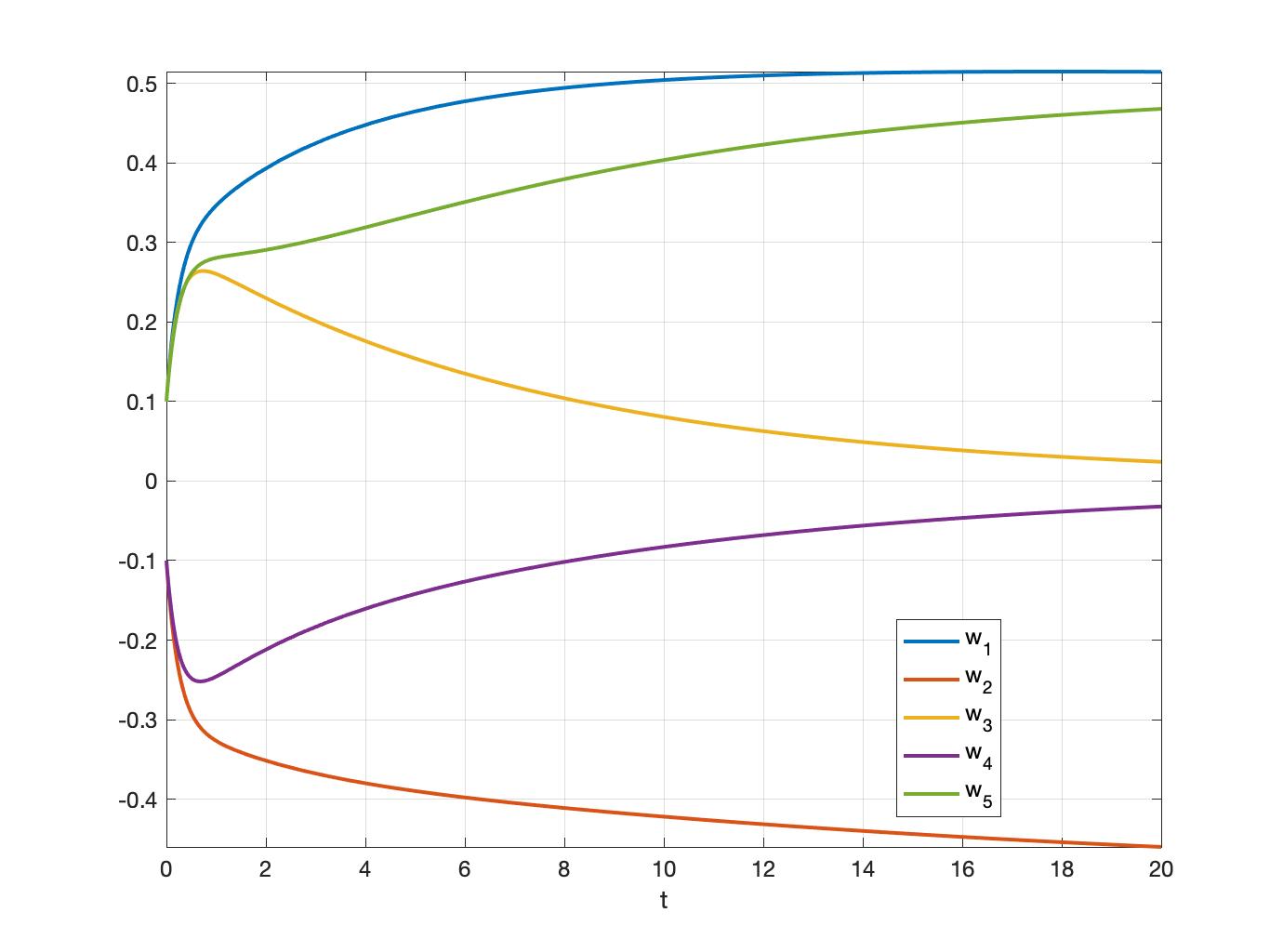}
 \includegraphics[scale=0.12]{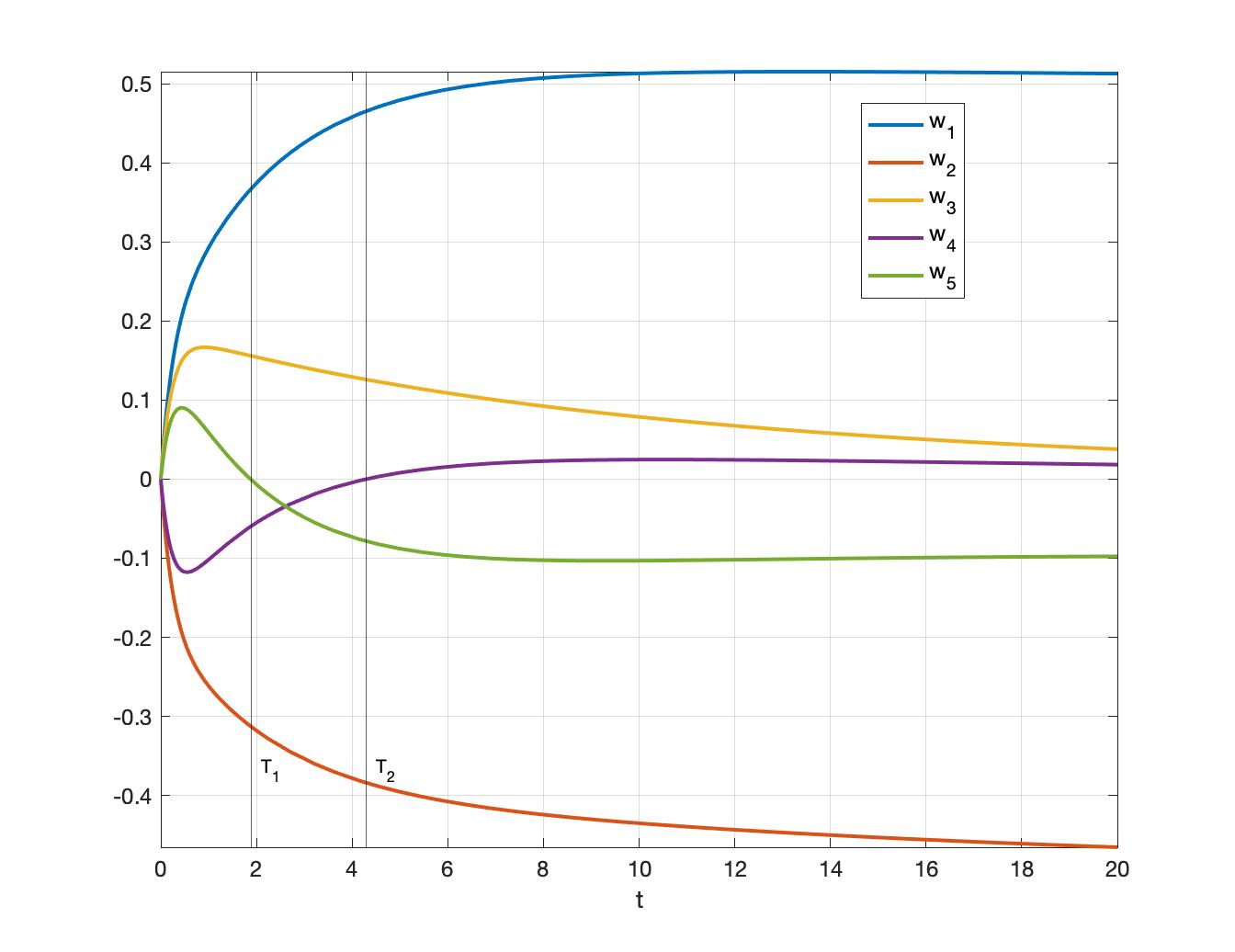}
 \caption{Left: globally sign stable path. Right: sign stable on $(0,T_1)$, $(T_1,T_2)$ and $(T_2,\infty)$.}\label{fig:sta} 
\vspace{-0.3cm}
\end{figure}

\bigskip
Whether sign stability holds depends on the specific form of $G$ and the regularizer $\Psi_{n,\gamma}(\bw)$. However, we show in the next lemma that the sign stability condition is equivalent to the condition that $\tilde\Lc_\gamma(t,\bu)$ has a limit as $\gamma\to 0$, i.e.,  \ref{as:tPlim0}.
\begin{lemma}[Equivalence between sign stability and \ref{as:tPlim0}]\label{lem:lbrel1time}
	Let $\Tc\subset[0,\infty)$ be a closed interval. Suppose there exists a continuous function $g^\ddagger(t)\geq 0$ with $g^\ddagger(t)>0$ for some $t\in\Tc$ satisfying
	\begin{align}
			\lim_{\gamma\to 0}\sup_{t\in\Tc}\big|g(\itg,\gamma)/\sqrt{\gamma}-g^\ddagger(t)\big|=0.\label{eq:lamcond}
	\end{align}
	Then, $\tilde\Lc_\gamma(t,\bu)$ satisfies 
	\begin{align}
	\lim_{\gamma\to 0}\sup_{t\in\Tc}|\tilde\Lc_\gamma(t,\bu) -\tilde\Lc(t,\bu)|= 0 \mbox{ pointwise for every $\bu\in\R^d$} \label{eq:tPlim1l1}
	\end{align}
	where for $t\in\Tc$,
    \begin{align}
			\tilde\Lc(t,\bu)= \frac{1}{2}\|\bu\|_2^2 + g^\ddagger(t) \sum_{j=1}^d\Big(u_j\sgn(w_j(t))\IF\{w_j(t)\neq 0\}+|u_j|\IF\{w_j(t)=0\}\Big), \label{eq:rdal1tP}
	\end{align}
	if and only if $\bw(t)$ is sign stable on $\Tc$.
\end{lemma}
The proof of Lemma \ref{lem:lbrel1time} is in Section \ref{sec:pfl1gen}. 

The following theorem describes the distributional dynamics of \eqref{eq:grdal1}. Under Condition \eqref{eq:lamcond}, the asymptotic mean trajectory of $\bw_n$ in \eqref{eq:grda_gen} is the same as that of \eqref{eq:sgd} (see, e.g. Theorem 1 of \cite{BKS93}). Rather, the effect of $\ell_1$ penalization enters in the distributional dynamics (second order asymptotics). 
\begin{theo}\label{th:l1dd}
 		Let $\bv_n$ and $\bw_n$ being iterates from \eqref{eq:grdal1}. Suppose \ref{as:S}, \ref{as:M} and \ref{as:L} hold, and that $g(\lfloor\cdot/\gamma\rfloor,\gamma)\in D([0,\infty))$ satisfies \eqref{eq:lamcond} with $\Tc=[0,T]$ for any $T>0$. 
Suppose the solution of
		\begin{align}
			\bv(t)&=\bv_0-\int_0^t G\big(\bv(s)\big)ds \label{eq:bvdd}
		\end{align}
	uniquely exists for all $t>0$. Then, 
	\begin{itemize}
		\item[(a)] In Corollary \ref{cor:l1at}, $g^\dagger(t)=0$ for all $t$ and $\bw(t)=\bv(t)$;
		\item[(b)] Suppose that $\bw(t)$ in (a) is {sign stable} on some closed interval $\tilde\Tc\subset[0,\infty)$. Let interval $\Tc\subset\tilde\Tc$ and $T_0:=\inf \Tc$. If $\VV_\gamma(T_0)\pto\VV(T_0)$ as random vector in $\R^d$, then $(\VV_\gamma,\WW_\gamma)\weakto(\VV,\WW)$ as $\gamma\to 0$ as random elements in $D(\Tc)^{2d}$ where
				\begin{align}
					\VV(t) &= \VV(T_0)-\int_0^t \nabla G(\bw(s)) \cdot \nabla \tilde\Lc^*(s,\VV(s)) ds + \MM(t),\label{eq:rdaV}\\
					 \WW(t)&= \nabla \tilde\Lc^*(t,\VV(t)),\label{eq:rdaW}
				\end{align}
				where $\tilde\Lc(t,\bw)$ is defined in \eqref{eq:rdal1tP}, and 
				\begin{align}\label{eq:rdasoftthreshold}
					\nabla \tilde\Lc^*_j(t,\bV) = \begin{cases}
						V_j-g^\ddagger(t), &\mbox{ for }j:w_j(t)>0;\\
						V_j+g^\ddagger(t), &\mbox{ for }j:w_j(t)<0;\\
						\sgn(V_j)\big[|V_j|- g^\ddagger(t)\big]_+, &\mbox{ for }j:w_j(t)=0.\\
					\end{cases} 
				\end{align}
				and
				$\MM(t)$ is a centered continuous process with independent Gaussian increment and covariance $\E[\MM(t)\MM(t)] = \int_{0}^t \Sigma(\bw(s)) ds$ where $\Sigma(\cdot)$ is defined in \eqref{eq:sig}.
		
		If $\bw(t)$ in \eqref{eq:weak0wl1} is {globally sign stable}, then the convergence of \eqref{eq:rdaV} and \eqref{eq:rdaW} holds in $D([0,\infty))^{2d}$ with $T_0=0$, and $\VV(0)=0$ almost surely.
	\end{itemize}	
\end{theo}
The proof of Theorem \ref{th:l1dd} is in Section \ref{sec:pfl1gen}. Note that (\ref{eq:rda}) is not covered by this theorem, as for any $\Tc\subset[0,\infty)$, $\lim_{\gamma\to 0}\sup_{t\in\Tc}c_0 \lfloor t/\gamma\rfloor \sqrt{\gamma}$ diverges to infinity. So,  \eqref{eq:lamcond} does not hold.

Sign stability ensures the weak convergence to a continuous limiting SDE. Simulation (Figure \ref{fig:LRpath_nsta} in Section \ref{sec:simlr}) shows that the empirical trajectories and limiting SDE of \eqref{eq:grdal1} are non-smooth (potentially a jump) at the boundary point of two adjacent sign stable intervals. In contrast, the limiting SDE of \eqref{eq:sgd} is globally continuous \citep{BMP90}.

\section{Online sparse linear regression}\label{sec:lrl1}

In this section, we consider $Z_n=(X_n,Y_n)\in\R^d\times\R$, and $Y_n$ is an i.i.d. sequence generated by 
\begin{align}
	Y_n = X_n^\top\bw^* + \varepsilon_n, \quad \E\varepsilon_n = 0,  \quad \Var(\varepsilon_n)=\sigma_\varepsilon^2.\label{eq:modrdal1}
\end{align}
Let $X_n$ be i.i.d. with $\E X_n = 0$, covariance matrix $H:=\E[X_nX_n^\top]$, and $\E[\|X_n\|_2^4]<\infty$. Assume there exists an absolute constant $C>0$ such that $C^{-1}<\sigma_{\min}(H)\leq\sigma_{\max}(H)<C$, where $\sigma_{\min}$ and $\sigma_{\max}$ denote the minimal and maximal eigenvalues.

We consider $f(\bw;Z_n)=(Y_n-X_n^\top\bw)^2/2$, which is a special case of Section~\ref{sec:rdal1}. As an application of Theorem \ref{th:l1dd}, the following Corollary shows that $\bw(t)\to\bw^*$ exponentially as $t\to\infty$, if the minimal eigenvalue of $H$ is bounded from below by a constant. Again, the mean dynamics given by \eqref{eq:weakatrda1} is the same as \eqref{eq:sgd} following the comments before Theorem \ref{th:l1dd}.   

\begin{cor}[Online sparse linear regression]\label{cor:lrrdaat}
	Under the linear regression model \eqref{eq:modrdal1}, assume that $g(\lfloor\cdot/\gamma\rfloor,\gamma)\in D([0,\infty))$ satisfies \eqref{eq:lamcond} with $\Tc=[0,T]$ for any $T>0$. Let $\bw_0$ be a fixed initial vector for $\bw_n$ in  \eqref{eq:grdal1}. Then, 
	\begin{itemize}
		\item[(a)] the asymptotic trajectory of $\bw_n$ satisfies the following linear ODE: 
		\begin{align}
			\bw(t)=\bw_0-\int_0^t H\big(\bw(s)-\bw^*\big)ds.\label{eq:weakatrda1}
		\end{align}
		In particular, the unique solution to \eqref{eq:weakatrda1} is 
		\begin{align}
			\bw(t)=e^{-H t}\bw_0 + (I_d-e^{-H t})\bw^* = P e^{-D t} P^{-1} \bw_0 + P(I_d-e^{-D t})P^{-1}\bw^*,
			\label{eq:solatrda1}
		\end{align}
		where $e^{-H t}:=\sum_{k=0}^\infty\frac{(-H)^k}{k!}$ and $e^{-D t} = \diag(e^{-D_{11} t},e^{-D_{22} t},...,e^{-D_{dd}t})$, and the diagonalization of $H = P D P^{-1}$ with the diagonal matrix $D$.
	\item[(b)] the conclusion of (b) in Theorem \ref{th:l1dd} holds with $\nabla G(\bw(s))=H=\E[X_nX_n^\top]$, and the covariance kernel of $\MM(t)$
	\begin{align}
		\Sigma(\bw)=\E\big[(XX^\top-H)(\bw-\bw^*)(\bw-\bw^*)^\top(XX^\top-H)\big] + \sigma_\varepsilon^2 H.\label{eq:covrdadd}
	\end{align}
	\end{itemize}
\end{cor}
\noindent The proof of Corollary \ref{cor:lrrdaat} is in Section \ref{sec:pfslr}. 

\bigskip
Support recovery is important for model interpretation and statistical inference. Existing methods in batch learning, e.g. LASSO \citep{tibshirani1996regression} or Bayesian variable selection \citep{GM93}, typically require to load the entire data set in processing unit. This is infeasible for large and streaming data. Although \eqref{eq:rda} is an alternative, it critically depends on the value of $c_0$. Increasing $c_0$ yields larger sparsity but also more bias (Theorem \ref{th:biasrda}). In the following, we show that by choosing the tuning function $g(n,\gamma)$ carefully, \eqref{eq:grda_gen} can not only recover the support consistently, but also provide unbiased stationary distribution, which facilitates statistical inference of active coordinates \citep{fan2001variable}.

Under the model \eqref{eq:modrdal1}, the dynamics $\VV(t)$ in \eqref{eq:rdaV} can be written specifically as
\begin{align}
\VV(t) &= -\int_0^t H \cdot \nabla \tilde\Lc^*(s,\VV(s)) ds + d\MM(t).
\end{align}
If $H$ is not diagonal, the drift in $\VV(t)$ fails to de-couple, and thus is complicated. For a more transparent exposition, we impose the following condition.
\begin{itemize}
	\labitemc{(D1)}{as:D1} $H=\E[X_nX_n^\top] = \mbox{diag}(\sigma_1^2,...,\sigma_d^2)$ with $\sigma_j\in(0,\infty)$ for all $j=1,...,d$. Moreover, $X_n$ are bounded in each component almost surely. 
\end{itemize} 
Under \ref{as:D1}, $\bw(t)$ in \eqref{eq:solatrda1} de-couples as 
\begin{align}
	\bw(t)=\begin{bmatrix}
	e^{-\sigma_1^2 t}w_{0,1}+(1-e^{-\sigma_1^2 t})w_1^*\\
	e^{-\sigma_2^2 t}w_{0,2}+(1-e^{-\sigma_2^2 t})w_2^*\\
	 \vdots \\
	e^{-\sigma_d^2 t}w_{0,d}+(1-e^{-\sigma_d^2 t})w_d^*\\
	\end{bmatrix},
	\label{eq:lintra}
\end{align}
whose path is a straight line connecting $\bw_0$ and $\bw^*$ in $\mathbb R^d$. The boundedness of $X_n$ in Assumption \ref{as:D1} is assumed for simplicity, and may be relaxed to sub-Gaussian tail, as supported by the simulations in Section \ref{sec:simlr}.

\begin{theo}[Online support recovery and stationary distribution]\label{th:lrsu} Consider the linear regression model \eqref{eq:modrdal1} and suppose \ref{as:D1} holds. Let $\bw_n$ be iterates generated by \eqref{eq:grdal1} with initialization $\bw_0=0$, 
	\begin{align}\label{eq:lrlamlong}
		g(n,\gamma)= c\gamma^{1/2} (n\gamma-t_0)_+^\mu
	\end{align}
	for an arbitrarily small fixed $t_0>0$ independent of $\gamma$ and $n$, some $\mu>0$ and $c >0$. 
Let $w_{\gamma,j}(t)$ be the $j$th component of $\bw_\gamma(t)$ and $w_j(t) = e^{-\sigma_j^2 t}w_{0,j}+(1-e^{-\sigma_j^2 t})w_j^*$ be the $j$th coordinate of \eqref{eq:lintra}. Then
	\begin{itemize}
		\labitemc{(a)}{lrsu1}(Inactive coordinates) For $j\in\{k:w_k^*=0\}$, if $\mu>1/2$, then $|w_{n,j}|=o_p(\sqrt{\gamma})$, as $\gamma\to 0$ and $n\to\infty$. 
		\labitemc{(b)}{lrsu2}(Active coordinates) For $j\in\{k:w_k^*\neq 0\}$, then 
		\begin{align}
		\frac{w_{\gamma,j}(t)-w_j(t)}{\sqrt{\gamma}}=W_{\gamma,j}(t)\weakto e^{-\sigma_j^2 (t-t_0)} W_j(t_0)+h_j(t)+U_j(t) \label{eq:limw}
		\end{align} 
	on $(D[t_0,\infty))^d$ as $\gamma\to 0$, where $W_j(t_0)\in\R$ is a random variable with $\E\big[W_j(t_0)^2\big]<\infty$, and
		\begin{align}
			h_j(t) &= -\sgn(w_j^*) e^{-\sigma_j^2 t} c \mu \int_{t_0}^t (s-t_0)^{\mu-1} e^{\sigma_j^2 s} ds,\label{eq:hj}\\ 
			U_j(t) &= e^{-\sigma_j^2 t}\int_{t_0}^t e^{\sigma_j^2 s}\  \Sigma_{j\cdot}^{1/2}(\bw(s))^\top d\BB(s),
			\label{eq:Uj}
		\end{align}
		where $\Sigma_{j\cdot}^{1/2}(\bw(s))$ is the $j$th row of the root matrix of $\Sigma(\bw(s))$, $\BB(t)=(B_1(t),B_2(t),...,B_d(t))^\top$ are independent standard Brownian motions. 
		
		\labitemc{(c)}{lrsu3}(Long-run bias and distribution) For any fixed $\mu>0$, $h_j(t)$ for $t\geq t_0$ in \eqref{eq:hj} satisfies
		\begin{align}
			\big|h_j(t) - (-1) \sgn(w_j^*) c \sigma_j^{-2} \mu (t-t_0)^{\mu-1}\big| \leq \tilde h_j(t),\label{eq:hest}
		\end{align}
		where $\tilde h_j(t)=o(1)$ as $t\to\infty$ for any $\mu>0$. 
		
		For $U_j(t)$ in \eqref{eq:Uj}, $U_j(t) \weakto \Nc(0,\sigma_\varepsilon^2/2)$ as $t\to\infty$ and 
		$$
		\lim_{t\to \infty}\Cov\big\{U_i(t),U_j(t)\big\}=0
		$$ 
		for all $i\neq j$ where $i,j\in\{k:w_k^*\neq 0\}$.
	\end{itemize}
\end{theo}
The proof of Theorem \ref{th:lrsu} is in Section \ref{sec:pfslr}. The $t_0>0$ in \eqref{eq:lrlamlong} is the time to allow $\bw(t)$ to achieve sign stability. In simulation we find that even with $t_0=0$, the conclusions of Theorem \ref{th:lrsu} continue to hold. 

In Theorem \ref{th:lrsu}(a), the $\mu>1/2$ in the tuning function \eqref{eq:lrlamlong} is set to dominate the magnitude of the maximum of Gaussian processes \citep{P96}. Consequently, the $V_j$, which is approximately a centered Gaussian process for $j\in\{k:w_k^*=0\}$, is eliminated as $t\to\infty$. However, for $j\in\{k:w_k^*\neq 0\}$, $V_j$ has an exploding trend in addition to the stochastic Gaussian term, and thus it cannot be eliminated as $t\to\infty$. Instead, Theorem \ref{th:lrsu}(b) shows $W_{\gamma,j}(t)$ for $j\in\{k:w_k^*\neq 0\}$ converges to a stationary distribution, which enables the asymptotic (infinitesimal $\gamma$) long-term analysis ($t\to\infty$). Since $W_j(t_0)$ has a finite second moment, $e^{-\sigma_j^2 t} W_j(t_0)\pto 0$ as $t\to\infty$. We summarize the above discussions in Corollary \ref{cor:longterm}.

\begin{cor}\label{cor:longterm}
	For $j\in\{k:w_k^*\neq 0\}$ and under the same conditions and notations in Theorem \ref{th:lrsu}(b), the distribution of $W_{\gamma,j}(t)$ as $t\to\infty$ is determined by only two components in \eqref{eq:limw}: the asymptotic {bias} $h_j(t)$, and the stochastic component $U_j(t)$ with $\E U_j(t)=0$ for all $t$. 
\end{cor}

From \eqref{eq:hest} and $\tilde h_j(t)$ (whose explicit form is given in the proof), we know that the bias $h_j(t)$ depends on $\mu$ in the definition of $g(n,\gamma)$. Their relation is described in Table \ref{tab:mu_bias} below. In particular, we want to point out the difference between the long-term ($t\to\infty$) central limiting distribution in Theorem \ref{th:lrsu} and that in batch Lasso ($n\to\infty$ there); see Theorem 2 of \cite{KF00}. For batch Lasso, the central limiting distribution is biased for active coefficients, while \eqref{eq:grdal1} can achieve unbiased central limit theorem (in particular, the limiting distribution is same as SGD; see Remark \ref{rem:comp_sgd} for details) when properly tuned, i.e. setting $1/2<\mu<1$, $c>0$ and arbitrarily small $t_0>0$ in \eqref{eq:lrlamlong}. This surprising phenomenon is an example that online learning can be a strong candidate for modern statistical learning, with rather different statistical properties from the corresponding batch optimization.

	\begin{table}[!h] 
		\centering
	\begin{tabular}{cccc}
		\hline\hline
				&$\mu<1$	&$\mu=1$ &$\mu>1$\\
		\hline
				$\lim_{t\to\infty}h_j(t)$ &$0$ &$-\sgn(w_j^*)\cdot c$ &$-\sgn(w_j^*)\cdot \infty$\\
		\hline\hline	  
	\end{tabular}
	\caption{The relation between $\mu$ in $g(n,\gamma)=c\gamma^{1/2} (n\gamma-t_0)_+^\mu$ and the long-run ($t\to\infty$) bias. For $j$ in the active set, if $\mu\geq 1$, the $w_{n,j}$ fails to converge to $w_j^*$ as $n\to\infty$.}\label{tab:mu_bias}
	\end{table}

\section{Online Sparse Principal Component Analysis (OSPCA)}\label{sec:pca}

Consider $Z_n=X_n\in\R^d$ are i.i.d. and centered with covariance $\Cc=\E[X_nX_n^\top]$. Classical principal component analysis (PCA) estimate the leading $k$ components $\UU_{\cdot 1}^*,...,\UU_{\cdot k}^*$ of $\Cc$ by solving the following optimization problem: 
\begin{align}
	\min_{\UU\in\Uc} \frac{1}{2}\mbox{tr}(\hat\Cc\UU\UU^\top), \quad\mbox{$\Uc$: $d\times k$ matrices with orthonormal columns}\label{eq:pca0}
\end{align}
based on its sample covariance matrix $\hat\Cc=N^{-1}\sum_{n=1}^N X_nX_n^\top$. Note that the sign is indeterminable for PCA. So, this optimization problem has two valid solutions, and the origin is a saddle point.

Suppose now that $X_n$'s are received in a streaming fashion. In this case, batch optimization of \eqref{eq:pca0} is infeasible, but the algorithms for online PCA are popular alternatives; see \cite{CD18} for a survey. 
Importantly, \cite{OK85} argue that finding the $j$th component of $\Cc$ can be cast as a stochastic approximation problem \eqref{eq:sa}, with a $d\times k$ matrix $G(\UU)=[G_1(\UU_{\cdot 1}),G_2(\UU_{\cdot 1},\UU_{\cdot 2}),...,G_k(\UU_{\cdot 1},\UU_{\cdot 2},...,\UU_{\cdot k})]$ [see Lemma 4 therein], where for $j=1,2,...,k$,
\begin{align}
	\begin{split}\label{eq:sa_pca}
		G_j(\UU_{\cdot 1},...,\UU_{\cdot j}) &= -\Ac_j\Cc\UU_{\cdot j}. 
	\end{split}
\end{align}
where 
\begin{align}
	\Ac_j=\Id_d-\UU_{\cdot j}\UU_{\cdot j}^\top-2\sum_{i=1}^{j-1}\UU_{\cdot i}\UU_{\cdot i}^\top. \label{eq:Acj}
\end{align}
By convention, $\sum_{i=1}^0 a_i=0$ for any sequence $a_i$. Using the orthogonality $\UU_{\cdot i_1}^* \perp \UU_{\cdot i_2}^*$, $\forall i_1\neq i_2$ where $i_1,i_2=1,...,k$, it can be easily verified that $\{\UU_{\cdot 1}^*,...,\UU_{\cdot j}^*\}$ is the root of $G_j$ for $j=1,...,k$, i.e. $G_j(\UU_{\cdot 1}^*,...,\UU_{\cdot j}^*)=0$. Using $X_{n+1}X_{n+1}^\top$ as the stochastic gradient to replace $\Cc=\E[X_{n+1}X_{n+1}^\top]$, \eqref{eq:sgd} can be applied as follows: for $j=1,2,...,k$,
\begin{align}\tag{\texttt{OPCA}}
	\begin{split}\label{eq:opca}
	\UU_{n+1, \cdot j} &= \UU_{n, \cdot j} + \gamma \Ac_{n,j} X_{n+1}X_{n+1}^\top\UU_{n, \cdot j},
	\end{split}
\end{align}
where $\Ac_{n,j}$ is defined by replacing $\UU_{\cdot j}$ in $\Ac_j$ by $\UU_{n,\cdot j}$. The distributional dynamics for \eqref{eq:opca} was derived in \cite{LWLZ17}.

Sparse principal component analysis (SPCA), which sparsifies the coefficients of principal components, has been a very active research area since \cite{JTU03}; see Section \ref{sec:respca} for a brief literature review and algorithms for online SPCA. 
Despite the flurry of recent studies on estimation, the asymptotic distribution of sparse PCA estimators is only addressed in the batch setting \citep{JG18}, whereas in the online setting it is still not well understood. We will fill this gap in this section. Motivated by \cite{SH08,SSM13}, we propose an online SPCA algorithm and estimate the $j$th principal component $\UU_{\cdot j}^*$ as follows
\begin{align}\tag{\texttt{OSPCA}}
	\begin{split}\label{eq:ospca}
	\tilde\UU_{n+1, \cdot j} &= \tilde\UU_{n, \cdot j} + \gamma\Ac_{n,j}X_{n+1}X_{n+1}^\top\UU_{n, \cdot j} \\
	\UU_{n+1, \cdot j} &= \sgn\big(\tilde\UU_{n+1, \cdot j}\big)\cdot\big(\big|\tilde\UU_{n+1, \cdot j}\big|-g(n,\gamma)\big)_+
	\end{split}
\end{align}
where the operators $\sgn(\cdot)$, $|\cdot|$ and $(\cdot)_+$ in the second line of \eqref{eq:ospca} are applied componentwisely. This algorithm is equivalent to performing \eqref{eq:grdal1} with $\bv_{n+1}=\tilde\UU_{n+1, \cdot j}$ and $\bw_{n+1}=\UU_{n+1, \cdot j}$ for each $j=1,2,...,k$.

The following Corollary is an application of Theorem \ref{th:l1dd}. For notational convenience, denote matrices $\tilde\UU_n := [\tilde\UU_{n,\cdot 1},\tilde\UU_{n,\cdot 2},...,\tilde\UU_{n,\cdot k}]$ and $\UU_n := [\UU_{n,\cdot 1},\UU_{n,\cdot 2},...,\UU_{n,\cdot k}]$. 

\begin{cor}[OSPCA]\label{cor:ospca}
	Assume the random vectors $X_n\in\R^d$ are i.i.d. with zero mean and are bounded in each component almost surely with covariance matrix $\Cc$. Suppose that the $k$ largest eigenvalues of $\Cc$ are positive and each of unit multiplicity. Let $\UU_{n, \cdot j}$ be iterates of \eqref{eq:ospca} with orthonormal initializer $\UU_{\cdot j,0}\neq 0$ for $j=1,...,k$. Given that $g(\lfloor\cdot/\gamma\rfloor,\gamma)\in D([0,\infty))$ satisfies \eqref{eq:lamcond} with $\Tc=[0,T]$ for any $T>0$, we have
	\begin{itemize}[itemindent=20pt,leftmargin=0pt]
		\item[(a)] $\UU_{\itg, \cdot j}\weakto\UU_{\cdot j}(t)$ as $\gamma\to 0$ on $D[0,t)^d$ for $j=1,...,k$, where $\UU_{\cdot j}(t)$ satisfies the ODE: 
		\begin{align}
			\begin{split}\label{eq:atospca}
				\dot\UU_{\cdot j}&=\Ac_j\Cc\UU_{\cdot j},\  \UU_{\cdot j}(0)=\UU_{\cdot j,0}, \mbox{ }j=1,2,...,k.
			\end{split}			
		\end{align}
		where $\Ac_j$ is defined in \eqref{eq:Acj}. Moreover, there exists a unique global solution of \eqref{eq:atospca}, and the trajectory $\UU_{\cdot j}(t)$ converges to $-\UU_{\cdot j}^*$ or $\UU_{\cdot j}^*$ as $t\to\infty$ if the initial value $\UU_{\cdot j,0}$ is within an attraction region centering around $-\UU_{\cdot j}^*$ or $\UU_{\cdot j}^*$ for $j=1,...,k$.
	\item[(b)] Setting 
	\small{\begin{align}
	\VV_\gamma(t)=\frac{\vect(\tilde \UU_{\itg})-\vect(\UU(t))}{\sqrt{\gamma}},\ \WW_\gamma(t)= \frac{\vect(\UU_{\itg})-\vect(\UU(t))}{\sqrt{\gamma}}. \label{eq:VWpca}
	\end{align}}
	Then the conclusions in Theorem \ref{th:l1dd}(b) hold for $(\VV_\gamma,\WW_\gamma)$, with $\nabla \tilde\Lc^*(t,\bV)$ defined in \eqref{eq:rdasoftthreshold}, 
	and the lower-triangular block matrix $\nabla G(\UU)\in\R^{kd\times kd}$:
	\begin{align}\label{eq:hesspca}
			\nabla G(\UU)=\left[
			\begin{array}{ccccc}
				\nabla G_{11}(\UU) &0 & 0&\cdots & 0\\
				\nabla G_{21}(\UU) &\nabla G_{22}(\UU) &0 &\cdots & 0\\
				\vdots & \vdots & \vdots & \ddots & \vdots \\
				\nabla G_{k1}(\UU) &\nabla G_{k2}(\UU) &\nabla G_{k3}(\UU) &\cdots & \nabla G_{kk}(\UU)\\
			\end{array}
			\right]
	\end{align}
where 
\begin{align*}
	\nabla G_{jj}(\UU) &= -\Cc + \Big(\Id_d \UU_{\cdot j}^\top \Cc \UU_{\cdot j} + 2 \sum_{i=1}^j \UU_{\cdot i} \UU_{\cdot i}^\top \Cc\Big),\quad j=1,...,k,\\
	\nabla G_{jl}(\UU) &= 2 \big(\Id_d \UU_{\cdot l}^\top \Cc \UU_{\cdot j} + \UU_{\cdot l} \UU_{\cdot j}^\top \Cc\big),\quad j=2,...,k, \quad l=1,...,j-1,
\end{align*}
	and the covariance kernel $\Sigma(\UU)\in\R^{kd\times kd}$ of $\MM(t)$:
	\begin{align}\label{eq:covpca}
		\Sigma(\UU)=\left[
			\begin{array}{ccccc}
				\Sigma_{11}(\UU) &\Sigma_{12}(\UU) & \Sigma_{13}(\UU)&\cdots & \Sigma_{1k}(\UU)\\
				\Sigma_{21}(\UU) &\Sigma_{22}(\UU) &\Sigma_{23}(\UU) &\cdots & \Sigma_{2k}(\UU)\\
				\vdots & \vdots & \vdots & \ddots & \vdots \\
				\Sigma_{k1}(\UU) &\Sigma_{k2}(\UU) &\Sigma_{k3}(\UU) &\cdots & \Sigma_{kk}(\UU)\\
			\end{array}
			\right]
	\end{align}
	where for $l,j=1,...,k$, $\Sigma_{lj}(\UU)=\Sigma_{jl}(\UU)^\top$ and
	\begin{align*}
		\Sigma_{jl}(\UU) 
		&=\Ac_j\E\big[\big(XX^\top-\Cc\big)\UU_{\cdot j}\UU_{\cdot l}^\top\big(XX^\top-\Cc\big)\big]\Ac_l,
	\end{align*}
	where $\Ac_j$ is defined in \eqref{eq:Acj}.
	\end{itemize}
\end{cor}

See Section \ref{sec:pfospca} for a proof of Corollary \ref{cor:ospca}. 
Corollary \ref{cor:ospca}(a) does not provide the size of the region of attraction. \cite{P95} considered the ODE of an orthogonal projection matrix $\PP=\UU(\UU^\top\UU)^{-1}\UU^\top$, where $\UU$ is the solution of \eqref{eq:atospca}, and showed that the region of attraction is {\em almost everywhere} on $\R^d$. Therefore, we conjecture the region of attraction in Corollary \ref{cor:ospca}(a) to be $\R^d$. In the simulation study (Section \ref{sec:simpca}), the initial $\UU_{0,\cdot j}$, $j=1,2$ are selected randomly, and no convergence issue is witnessed. 
See Figure \ref{fig:pca_ode} in Section \ref{sec:ode_ospca} of the supplementary material for an illustration of the mean dynamics in \eqref{eq:atospca}, and the sign stable regions.

\section{Simulation analysis}\label{sec:sim}
In this section, we validate the relevance of theory in Sections \ref{sec:lrl1} and Section \ref{sec:pca} for non-infinitesimal step size with synthetic data, by setting
\begin{align}
	g(n,\gamma)=\gamma^{1/2+\mu}n^{\mu}  \label{eq:lamsim}
\end{align}
where $\mu>0$, and $n=1,2,3,...$. This is motivated by \eqref{eq:lrlamlong} with $t_0=0$, as $t_0$ can be taken arbitrarily small in Theorem \ref{th:lrsu}. Section \ref{sec:simlr} concerns the linear regression model in Section \ref{sec:lrl1}, and Section \ref{sec:simpca} discusses the sparse online principal component analysis in Section \ref{sec:pca}.

\subsection{Online sparse linear regression}\label{sec:simlr}
We consider linear regression model \eqref{eq:modrdal1} in Section \ref{sec:lrl1} with dimension $d=100$. Specifically, the $(i, j)$-th element of $H$ is $H_{ij}=(-0.5)^{|i-j|}$ for any $1\leq i,j\leq d$. For the true coefficient $\bw^*$, the support of the active coordinates is selected randomly with size 30. The values of active coefficients are chosen randomly from an independent standard Gaussian distribution. Algorithm \eqref{eq:grdal1} is studied with initial value $\bw_0=0$.

The mean dynamics is computed from \eqref{eq:solatrda1}: 
\begin{align}
	\bw(t)=(w_1(t),w_2(t),...,w_d(t))^\top = e^{-H t}\bw_0 + (I_d-e^{-H t})\bw^*,
\end{align}
which is not globally sign stable (Definition \ref{def:sist}). The 95\% confidence bands are computed by first simulating 500 trajectories of $\VV(t)$ with the Euler-Maruyama method \citep{KP92} with $\Delta t=0.1$. The covariance kernel in \eqref{eq:covrdadd} is computed with a data set of size 5000 simulated independently from the data simulated for computing $\VV(t)$. Next, $\WW(t)$ is computed from $\VV(t)$ by \eqref{eq:rdaW}. Finally, quantiles $Q_{0.025,j}(t)$ and $Q_{0.975,j}(t)$ were taken from the 500 trajectories of $\WW(t)$, for each $t$. This procedure yields the 95\% confidence band:
\begin{align}
	\big[w_j(t)+\sqrt{\gamma}Q_{0.025,j}(t),w_j(t)+\sqrt{\gamma}Q_{0.975,j}(t)\big]\quad j=1,2,...,d.\tag{TACB}\label{eq:TACB}
\end{align}
This will be referred to as the theoretical asymptotic confidence band (TACB) subsequently. In all figures below, $t=n\gamma$.

First, we want to check what if the sign stability condition in Theorem \ref{th:l1dd} does not hold, by examining coefficients $j=1,54$ that are not sign stable in $(0,20)$. As shown in Figure \ref{fig:LRpath_nsta}, the TACB are not smooth, and a {jump} is observed in the neighborhood of time $t=t^\circ$, where $w_j(t^\circ)=0$, as can be seen in the second row panels of Figure \ref{fig:LRpath_nsta}. The jump is necessary for maintaining good coverage of the TACB for empirical trajectories after time $t^\circ$, as the empirical trajectories linger at 0 around time $t^\circ$. Such a phenomenon is more obvious when $\mu=0.7$ than when $\mu=0.4$. By contrast, the confidence band for SGD is everywhere continuous as shown in the first row panels of Figure \ref{fig:LRpath_nsta}. Note that for all algorithms at small $t$ (particularly between $t=0$ and $t=3$), the TACBs have larger dispersion, due to large covariance kernel \eqref{eq:covrdadd} resulted from large $\|\bw(t)-\bw^*\|_2$.

	\begin{figure}[!h] 
		\centering
	\includegraphics[scale=0.22]{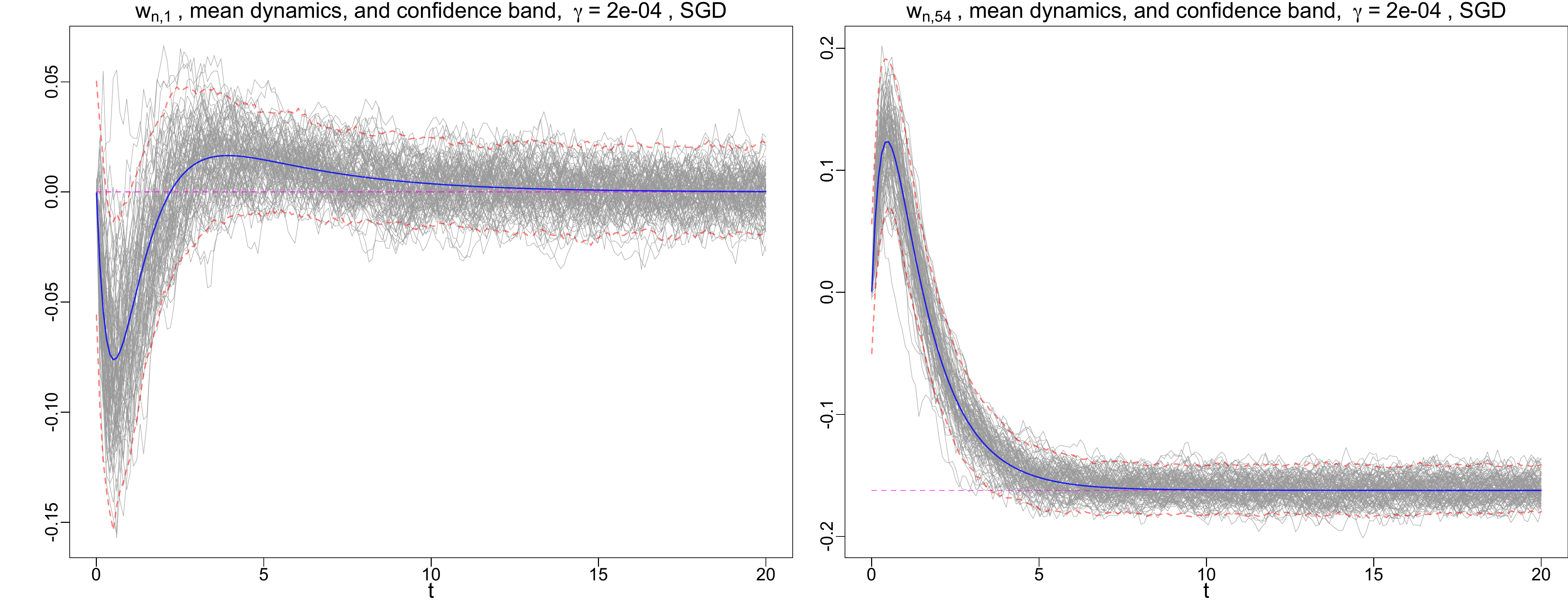}\\	
 \includegraphics[scale=0.22]{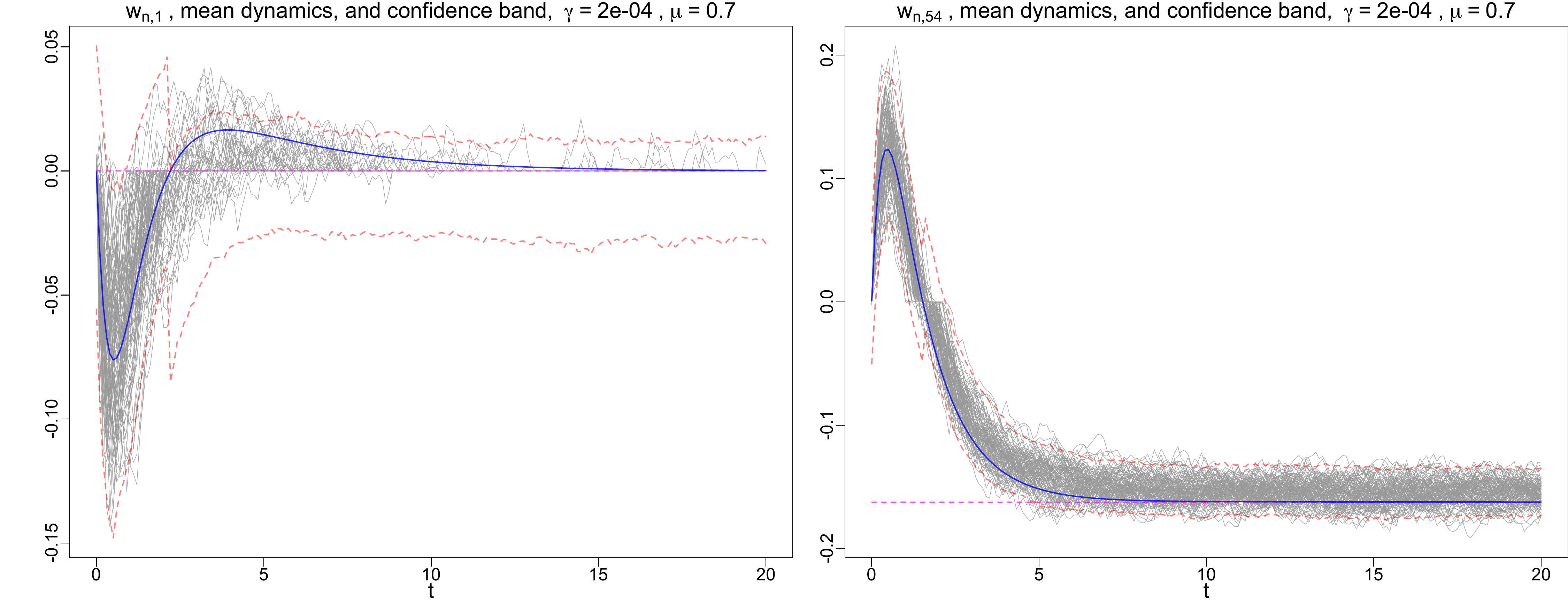}\\
 \caption{{Sign unstable coefficients in $(0,20)$:} 100 empirical trajectories $w_{n,j}$ (gray curves), $j=1,54$ mean dynamics (blue curve, \eqref{eq:solatrda1}) and confidence band (area bounded between two red dashed curves) for $w_1^*=0$ and $w_{54}^*=-0.1623$ (magenta dashed lines), under algorithms \eqref{eq:sgd} and \eqref{eq:grdal1} with $g(n,\gamma)=\gamma^{1/2+\mu}n^{\mu}$ with $\mu=0.7$, initiated at $\bw_0=0$. The number of steps is $N=20/\gamma$. The confidence band for \eqref{eq:solatrda1} has a jump in the neighborhood of where the mean trajectory takes zero.}\label{fig:LRpath_nsta}
\vspace{-0.3cm}
\end{figure}

To see the performance of our weak approximation result for different $\mu$ when step size $\gamma$ is non-infinitesimal, the left panel of Figure \ref{fig:cover_lr} shows the averaged coverage probabilities of \eqref{eq:TACB} over active coefficients at $\gamma=2\times 10^{-4}$. The coverage probabilities are close to the nominal level 95\% as $\mu<1$. For some $t$ (e.g. around $t=5$) the coverage probabilities of \eqref{eq:grdal1} slightly deviate from the nominal level, but they return to the nominal level as $t$ increases. For $\mu\geq 1$, the averaged coverage probabilities deviate from the nominal level after $t=9$, and never returns. The coverage probabilities improve when shrinking the step size; see additional simulation analysis in Figure \ref{fig:LRpath_mu15} and Figure \ref{fig:LRcover15_varyg} in Section \ref{sec:addsimlr}. 

	\begin{figure}[!h] 
		\centering
	\includegraphics[scale=0.3]{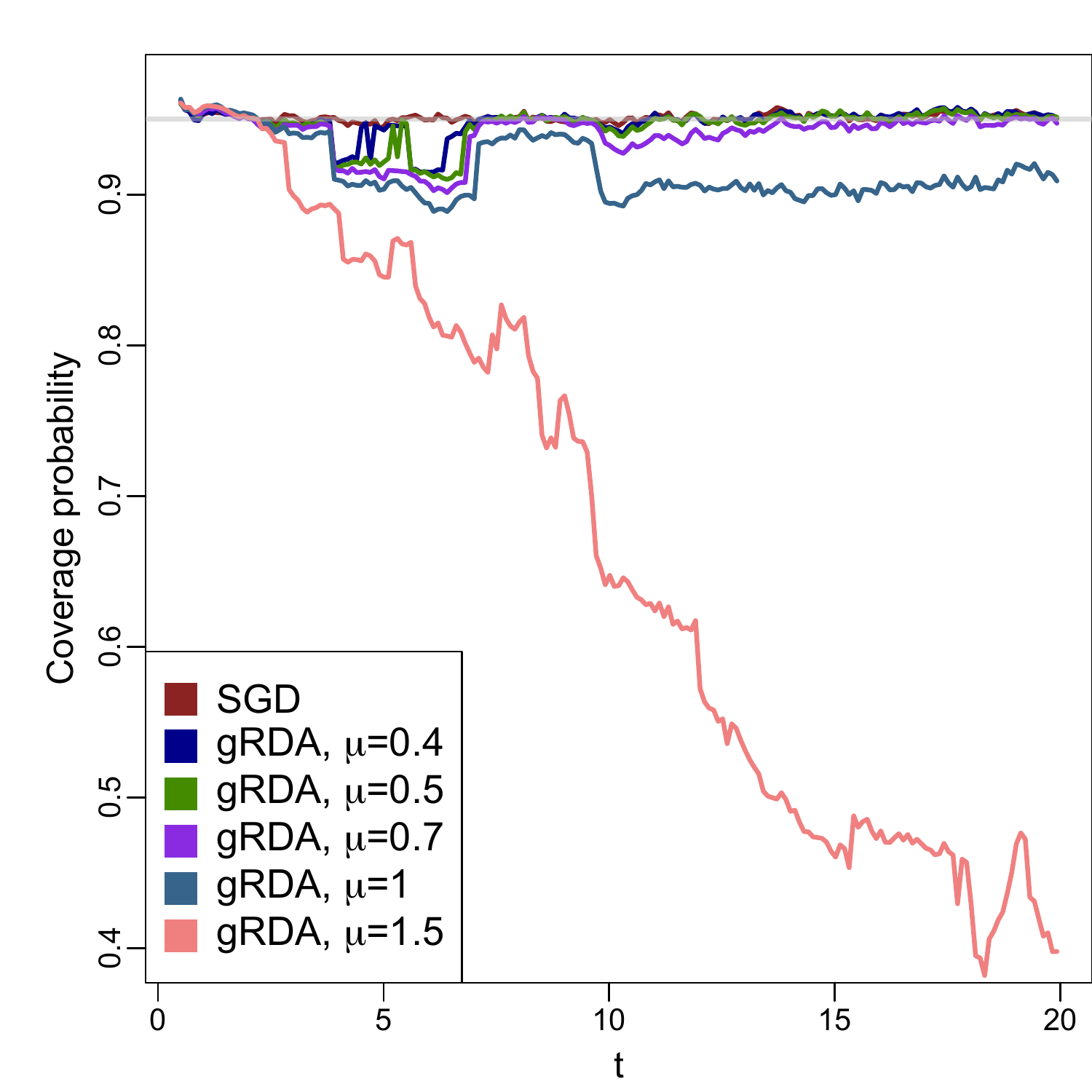}
	\includegraphics[scale=0.3]{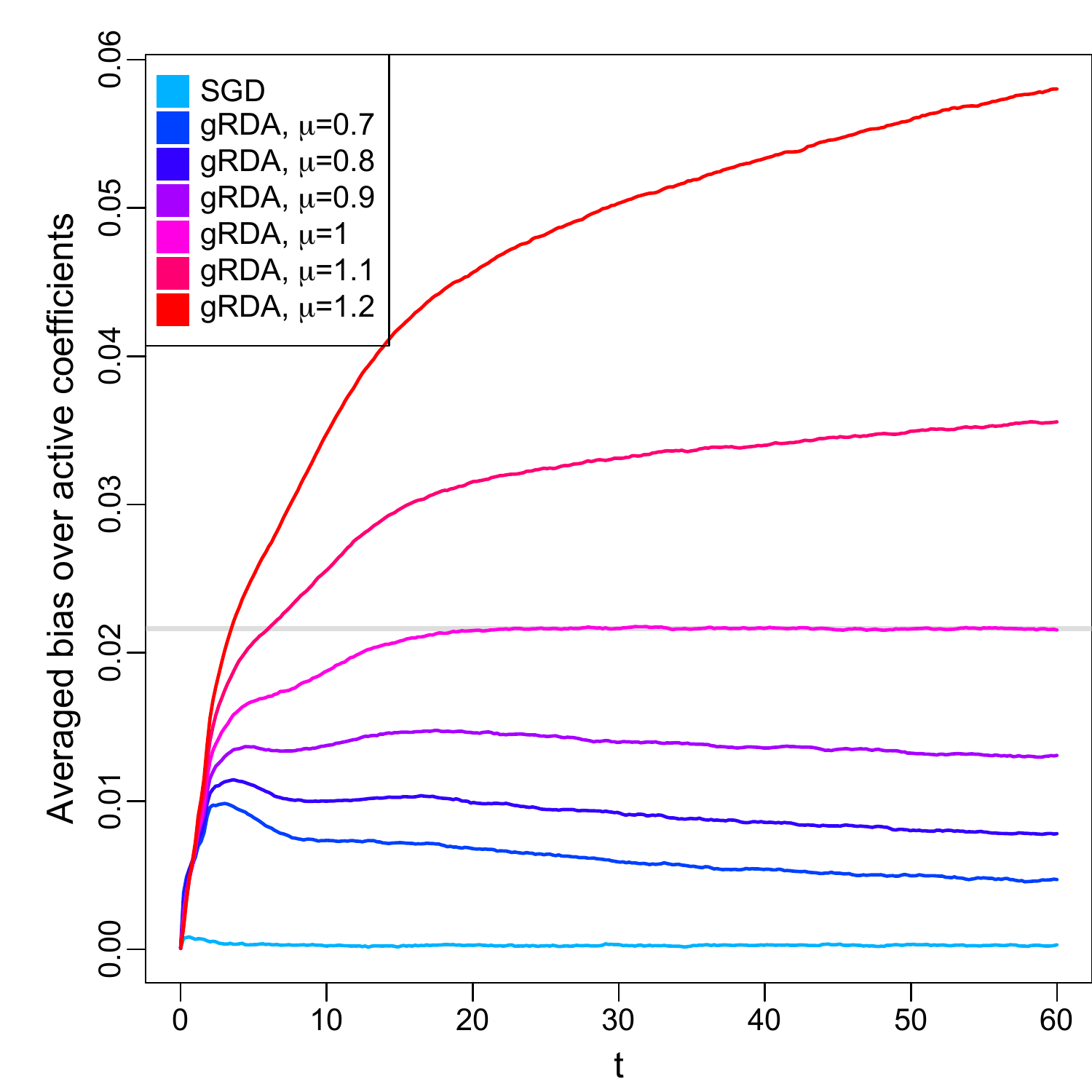}
	\caption{Left: The averaged coverage probability of \eqref{eq:TACB} over active coefficients. The gray line is the nominal level 95\%. Right: The averaged bias of \eqref{eq:TACB} over active coefficients, i.e. $30^{-1}\sum_{j\in\{k:w_k^*\neq 0\}}|\hat\E[w_{\itg,j}]-w_j(t)|$, where $\hat \E$ is computed based on 1000 simulations. The gray horizontal line is the long-term bias for \eqref{eq:grdal1} with $\mu=1$. Results are averages of 1000 simulation repetitions. Step size $\gamma=2\times 10^{-4}$.}\label{fig:cover_lr}
	\vspace{-0.3cm}
	\end{figure}

In Figure \ref{fig:LRpath_nsta}, the TACBs are asymmetric around $w_j(t)$ for $j\in\{k:w_k^*\neq 0\}$. This indicates the presence of bias in $\bw_n$. To better examine bias, the right panel of Figure \ref{fig:cover_lr} shows the averaged absolute bias of $w_{\itg,j}$ to $w_j(t)$ over the active set $j\in\{k:w_k^*\neq 0\}$. For $\mu<1$, the decreasing trend of the curves supports the bias order of $t^{\mu-1}$ in \eqref{eq:hj} for $t\geq 10$, even if $H$ is not diagonal here. The bias stays at a constant level for $\mu= 1$ and explodes for $\mu=1$ at large $t$, which also supports our finding in \eqref{eq:hj}. Therefore, we recommend to implement \eqref{eq:grdal1} with tuning function \eqref{eq:lamsim} and set $\mu<1$ therein. However, we cannot explain the increase of bias before $t\leq 5$ for $\mu\leq 1$, which is related to the non-orthogonality of $H$. We conjecture that the results in \eqref{eq:hj} hold for general $H$ only when $t$ is large enough.

	\begin{figure}[!h] 
		\centering
 \includegraphics[scale=0.3]{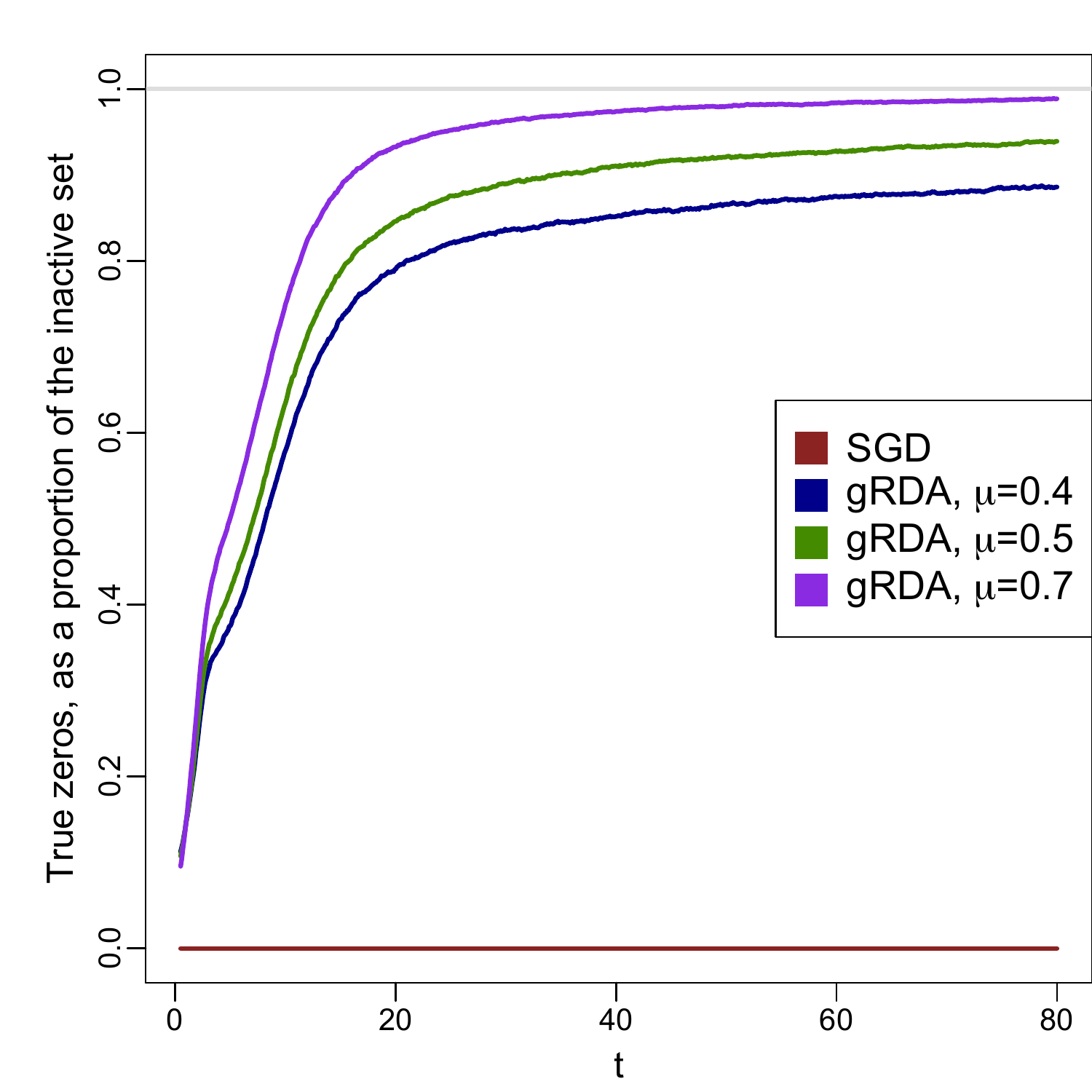}
 \includegraphics[scale=0.3]{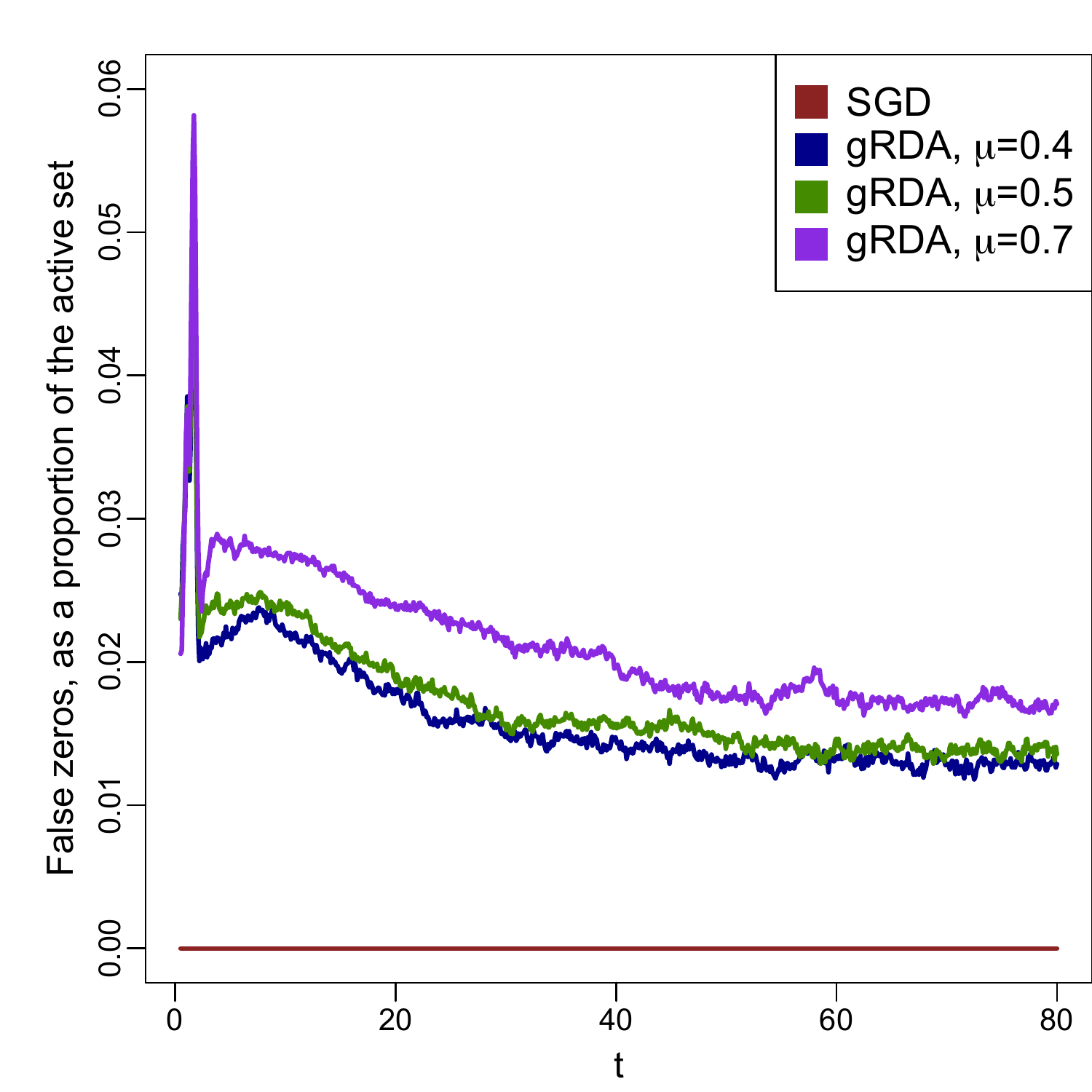}
 \caption{Left: the number of true zeros, or the cardinality $|\{j:w_{\itg,j}=0,w_j^*=0\}|$, as a proportion of the cardinality of inactive set $|\{j:w_j^*=0\}|$. Right: the number of false zeros, or the cardinality $|\{j:w_{\itg,j}=0,w_j^*\neq 0\}|$, as a proportion of the cardinality of active set $|\{j:w_j^*\neq 0\}|$. Step size $\gamma=2\times 10^{-4}$.}\label{fig:support_lr}
\vspace{-0.3cm}
\end{figure}

Support recovery performance is shown in Figure \ref{fig:support_lr}. The left panel of Figure \ref{fig:support_lr} presents the identified true zeros as a proportion of the cardinality of the inactive set $\{j:w_j^*= 0\}$, in which \eqref{eq:grdal1} with $\mu=0.7$ performs the best. \eqref{eq:sgd} cannot generate zero coefficients so its curve stays at 0. The right panel of Figure \ref{fig:support_lr} focuses on the identified false zeros as a proportion of the cardinality of the active set $\{j:w_j^*\neq 0\}$. The performance improve as $t$ increases. For the left and the right panel of Figure \ref{fig:support_lr}, results for $\mu=1$ and 1.5 are similar to those of $\mu=0.7$ and are omitted. Observations here on Figure \ref{fig:support_lr} suggest the Theorem \ref{th:lrsu} holds even under non-orthogonal $H$ and $t_0=0$ in \eqref{eq:lrlamlong}.

For additional simulation analysis on the sign stable coefficients and the effect of step size $\gamma$, see Section \ref{sec:addsimlr}.

	
		
\subsection{Online sparse PCA}\label{sec:simpca}

Now we turn to the sparse PCA problem in Section \ref{sec:pca}, by focusing on the top principal component, i.e. $k=1$. Results for the second principal component are in Section \ref{sec:addsimpca}. Consider i.i.d. random vectors $X_n\sim\Nc(0,\Cc)$ in $\R^d$ with $d=100$, where the covariance matrix 
$$
\Cc = 2*\UU_{\cdot 1}^*\UU_{\cdot 1}^{*\top}+\UU_{\cdot 2}^*\UU_{\cdot 2}^{*\top}+\Id_d.$$ 
This setting of covariance matrix is adopted from Section 4.3 of \cite{GWS18}. The solution of the mean ODE in \eqref{eq:atospca} has no closed form, and is computed by the numerical ODE solver \texttt{ode45} in \texttt{Matlab}. A known problem of this solver is that it introduces artificial but small oscillations to the solution; see \cite{SB08}. In our case, oscillations are observed for $\UU_{\cdot 1}(t)$ after $\UU_{\cdot 1}(t)$ has converged to $\UU_{\cdot 1}^*$. Therefore, as a remedy, we replace the \texttt{ode45} solution of $\UU_{\cdot 1}(t)$ with $\UU_{\cdot 1}^*$ after $t=12$, and treat this modified solution as $\UU_{\cdot 1}(t)$. 
Simulations on empirical trajectories of algorithms \eqref{eq:opca} and \eqref{eq:ospca} were repeated 1000 times with $g(n,\gamma)$ in \eqref{eq:lamsim}.

Procedures for computing quantiles in the 95\% \eqref{eq:TACB} is similar as those in Section \ref{sec:simlr}, except that here we use $\nabla G(\UU)$ in \eqref{eq:hesspca} and covariance kernel in \eqref{eq:covpca}, and the time grid is selected by \texttt{ode45} in \texttt{Matlab}.


\begin{figure}[!h]
	\centering
	 \includegraphics[scale=0.22]{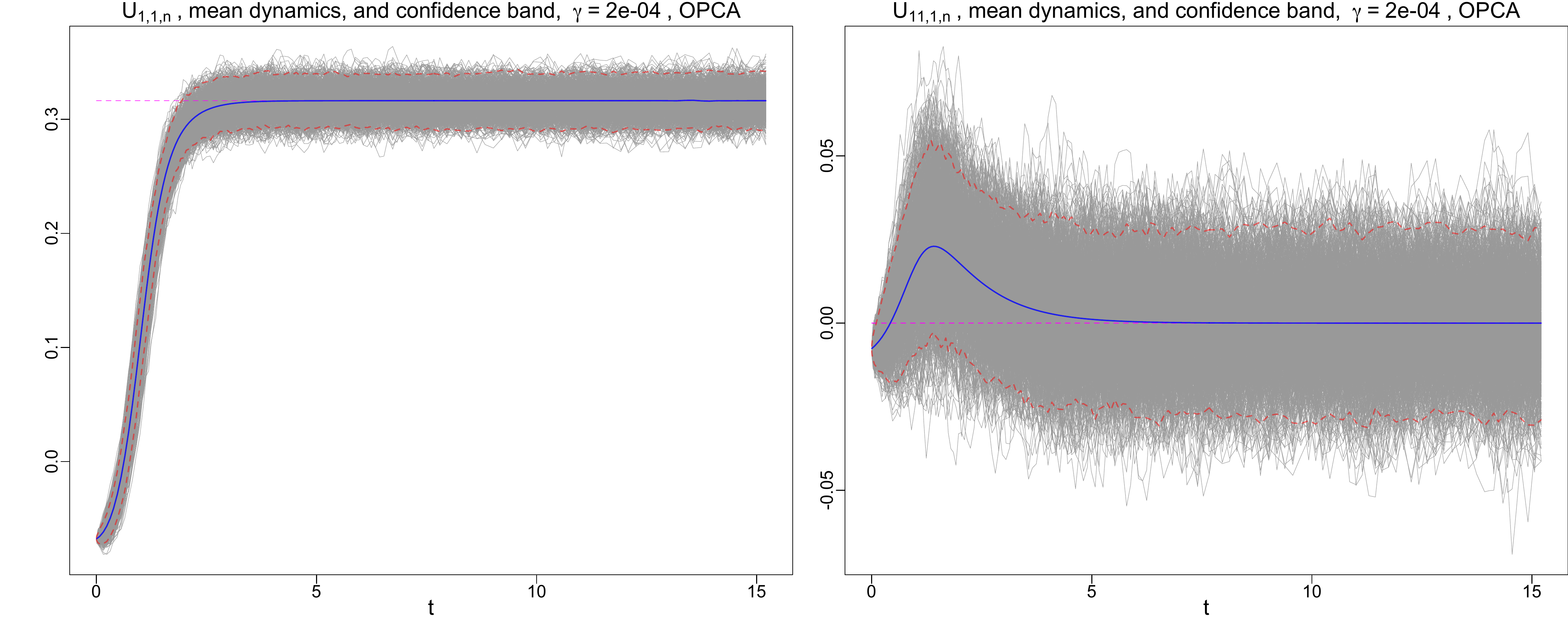}\\
	 \includegraphics[scale=0.22]{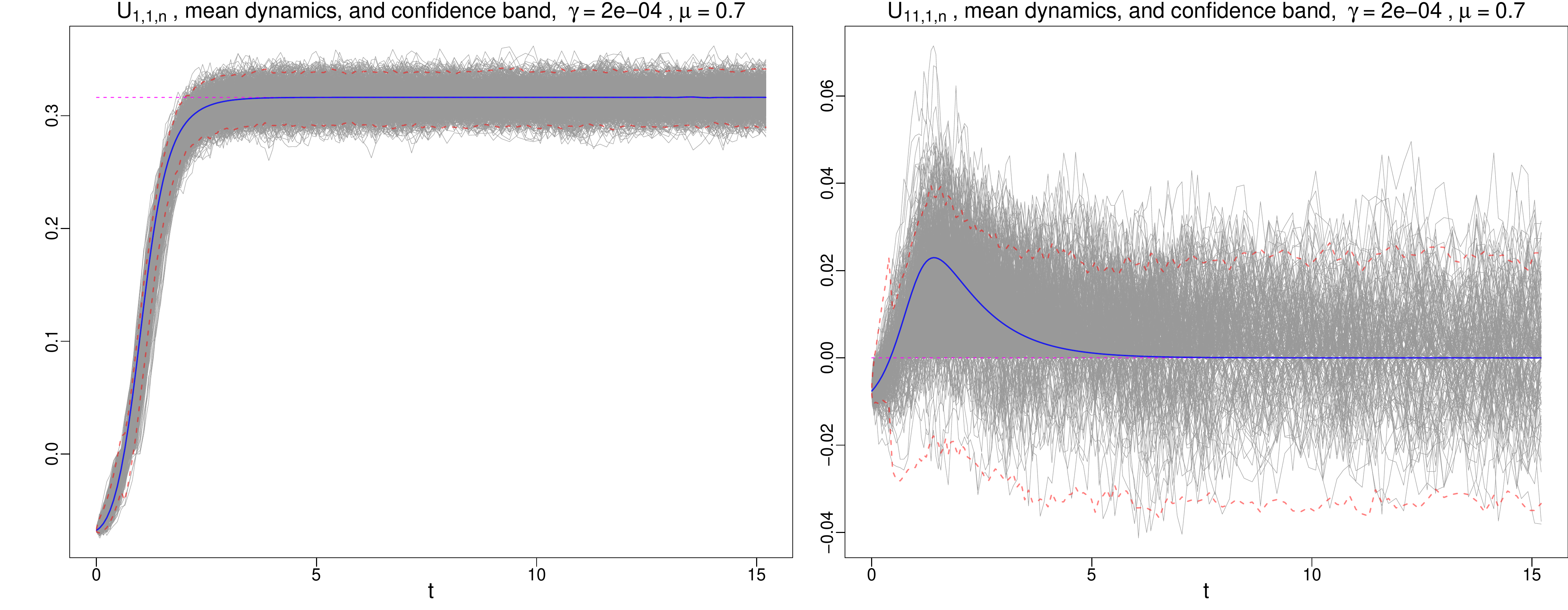}\\
	 \caption{{First} PC: 1000 empirical trajectories (gray curves), mean dynamics (blue curve) and confidence band (area bounded between two red dashed curves) for $\UU_{1,1}^*=10^{-1/2}$ and $\UU_{11,1}^*=0$ (magenta dashed lines). Empirical trajectories are computed from algorithms \eqref{eq:opca} and \eqref{eq:ospca} with random initiation on the unit sphere. The tuning function $g(n,\gamma)$ in \eqref{eq:lamsim} of \eqref{eq:ospca} depends on $\mu$. The number of steps is $N=15/\gamma$ with $\gamma=2\times 10^{-4}$.}\label{fig:ospca1}
\end{figure}

Figure \ref{fig:ospca1} focuses on an active coordinate $j=1$ with $\UU_{1,1}^*=10^{-1/2}$ and an inactive coordinate $j=11$ in the first principal component $\UU_{\cdot 1}^*$. Jumps in the TACB around the times $t$ where the mean trajectories $U_{1,1}(t)=0$ and $U_{11,1}(t)=0$ are observed. From the left column panels of Figure \ref{fig:ospca1} corresponding to $j=1$, the classical \eqref{eq:opca}, i.e., SGD, does not stuck at 0, while \eqref{eq:ospca} got stuck at 0 before continuing to increase. 
For the right column panels of Figure \ref{fig:ospca1} corresponding to $j=11$, trajectories of \eqref{eq:opca} are non-sparse, while \eqref{eq:ospca} has sparse trajectories. 

	\begin{figure}[!h]
		\centering
	\includegraphics[scale=0.3]{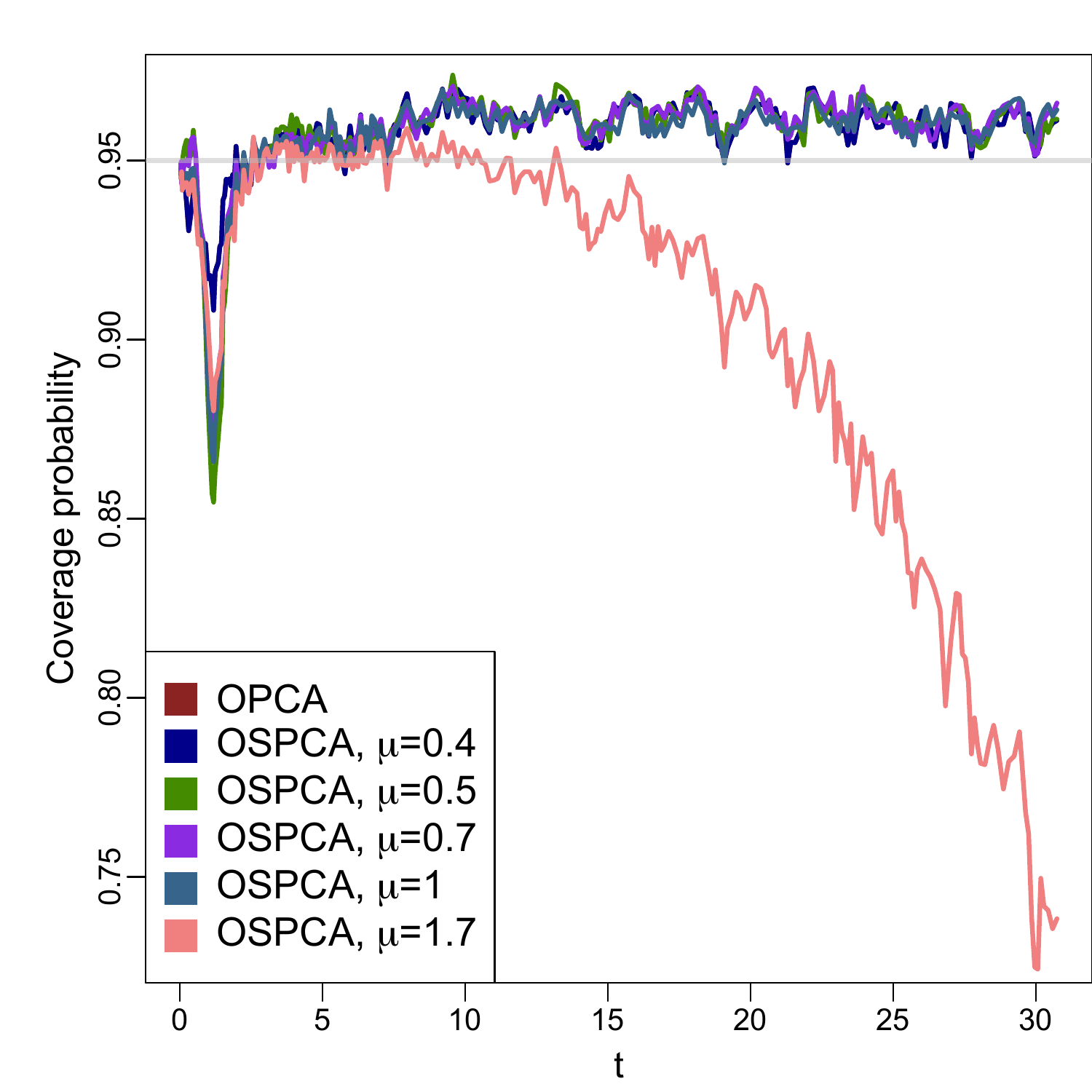}
	\includegraphics[scale=0.3]{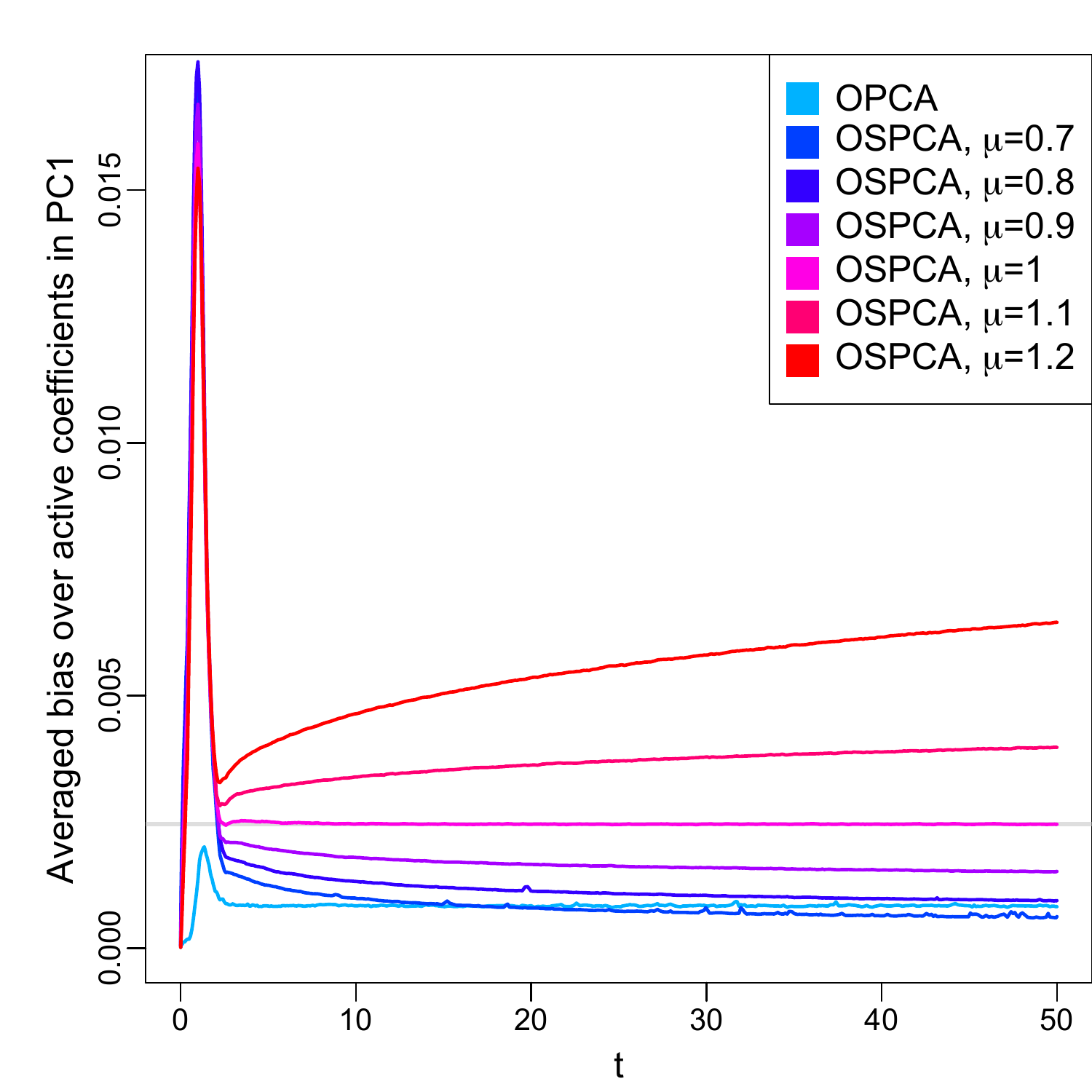}
	\caption{Left: The averaged coverage probability of \eqref{eq:TACB} over active coefficients in the first principal component. The tuning function $g(n,\gamma)$ in \eqref{eq:lamsim} of \eqref{eq:ospca} depends on $\mu$. \eqref{eq:opca} performs similarly as \eqref{eq:ospca} with $\mu=0.4$, so the curve for \eqref{eq:ospca} is covered by that of \eqref{eq:ospca} with $\mu=0.4$. The gray line is the nominal level 95\%. Right: The averaged bias of \eqref{eq:TACB} over active coefficients in the first principal component, i.e. $10^{-1}\sum_{j=1}^{10}|\hat\E[\UU_{\itg,j,1}]-\UU_{j,1}(t)|$, where $\hat \E$ is computed based on 1000 simulations. The gray horizontal line is the long-term bias for \eqref{eq:ospca} with $\mu=1$. Results are averages of 1000 simulation repetitions. Step size $\gamma=2\times 10^{-4}$.}\label{fig:cover_pca}
	\end{figure}

To see the performance of our weak approximation result under different $\mu$ when step size $\gamma$ is non-infinitesimal, Figure \ref{fig:cover_pca} presents the averaged coverage probability over active coefficients with $\gamma=2\times 10^{-4}$. \eqref{eq:TACB} does not cover the empirical trajectories well around $t=2$ for all algorithms, which is the time when empirical trajectories move quickly away from the initializer towards the minimizer. The coverage probability of \eqref{eq:ospca} with $\mu=1.7$ performs poorly as $t$ increases at step size $\gamma = 2\times 10^{-4}$, but it improves once the step size $\gamma$ decreases; see Figure \ref{fig:OSPCAcover17_varyg} in Section \ref{sec:addsimpca}. Another notable phenomenon is the cyclic pattern in coverage probability for \eqref{eq:opca} and \eqref{eq:ospca} with $\mu\leq 1$ for $t\geq 15$. Because the variation in \eqref{eq:TACB} is pretty small after, e.g. $t=5$ (see Figure \ref{fig:ospca1}), the cyclic pattern in the coverage probability must result from the cyclic behavior of iterates $\UU_{n,\cdot 1}$. Recently, \cite{CS18} suggest that the long-term behavior of \eqref{eq:sgd} could resemble closed loops with deterministic component. The long-term behavior of \eqref{eq:opca}, which is essentially \eqref{eq:sgd}, appears to support the finding of \cite{CS18}, and \eqref{eq:ospca} also has the cyclic pattern with the same deterministic component as \eqref{eq:opca}. This phenomenon may be interesting for future study. 

In the panels of Figure \ref{fig:ospca1} associated with \eqref{eq:ospca}, asymmetry of the TACB around $\bw(t)$ is observed. This is due to the bias in the SDE in Corollary \ref{cor:ospca}(b). Even though it is not proven explicitly as in the linear regression case in Theorem \ref{th:lrsu}, right panel of Figure \ref{fig:cover_pca} suggests that except for $t\leq 3$ where all algorithms have large bias, bias decreases as $\mu<1$, holds fixed as $\mu=1$ and explodes as $\mu>1$ as $t\to\infty$. This implies that $\bw_n$ is inaccurate for $\bw^*$ as $t\to\infty$ if $\mu$ is larger. This supports that the tuning function $g(n,\gamma)$ in \eqref{eq:lamsim} with $\mu<1$ should be a universal recipe, and applies to tasks beyond linear regression. Another notable observation is that the bias of \eqref{eq:opca} decreases slower than \eqref{eq:ospca} with $\mu=0.7$. This may be because the accuracy of \eqref{eq:opca} for active coefficients are negatively impacted by inactive coefficients which are not zeroed out with \eqref{eq:opca}.

	\begin{figure}[!h] 
		\centering
 \includegraphics[scale=0.3]{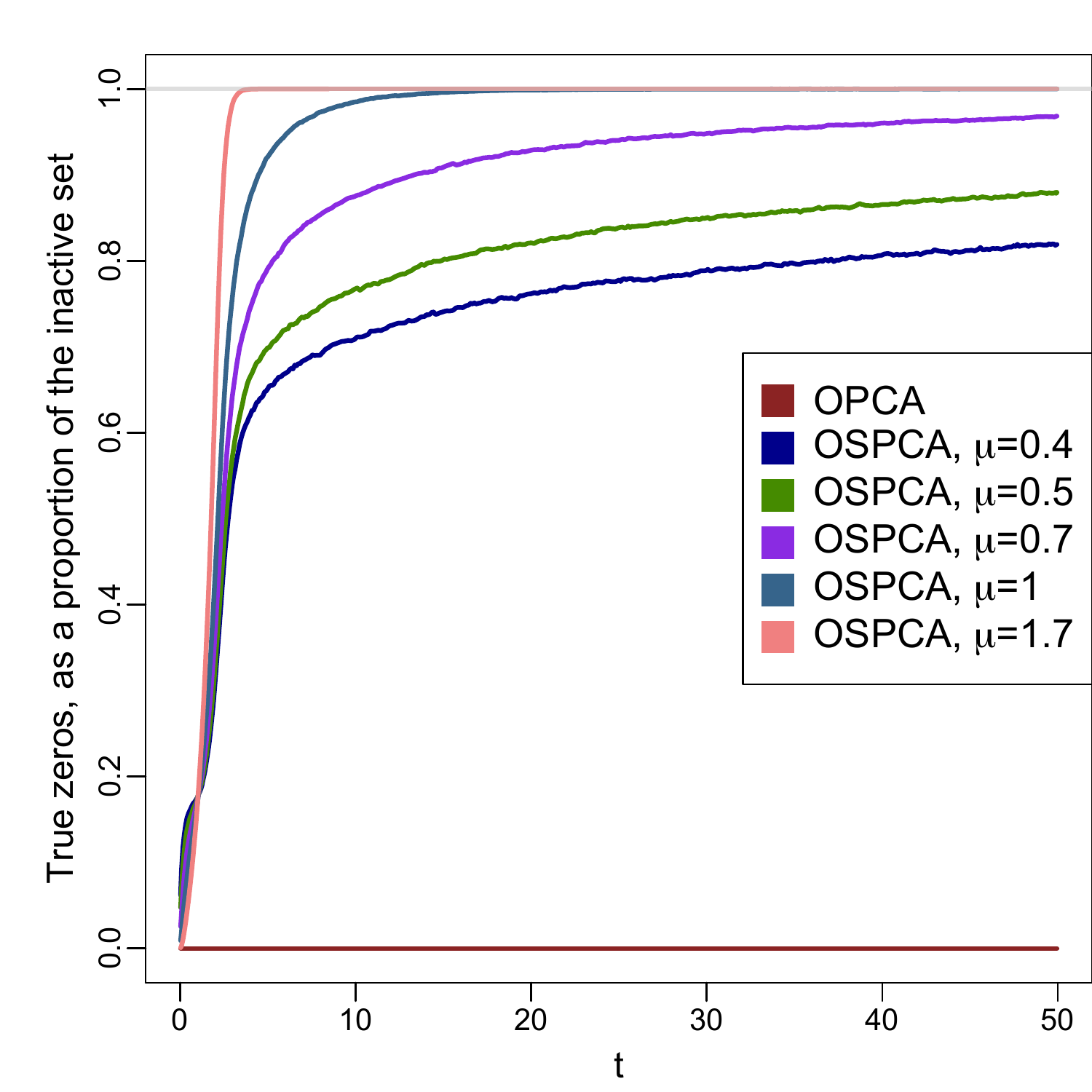}
 \includegraphics[scale=0.3]{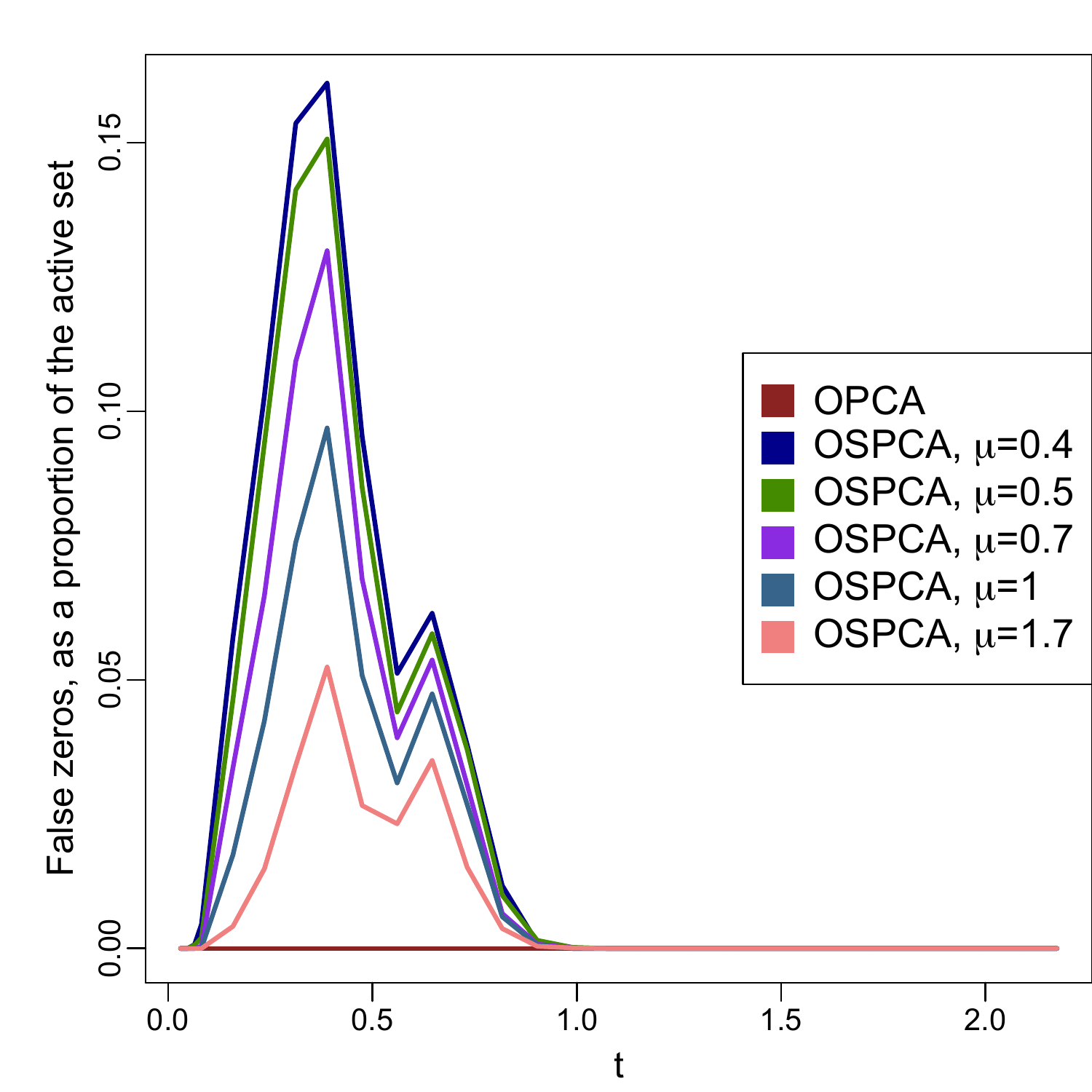}
 \caption{Support recovery performance of \eqref{eq:ospca}. Left: the number of true zeros, or the cardinality $|\{j:\UU_{\itg,j1}=0,\UU_{j1}^*=0\}|$, as a proportion of the cardinality of inactive set $|\{j:\UU_{j1}^*=0\}|$. Right: the number of false zeros, or the cardinality $|\{j:\UU_{\itg,j1}=0,\UU_{j1}^*\neq 0\}|$, as a proportion of the cardinality of active set $|\{j:\UU_{j1}^*\neq 0\}|$. The proportion of false zeros stays at 0 for time $t\geq 1$. Results are the average of 1000 simulation repetitions. Step size $\gamma=2\times 10^{-4}$.}\label{fig:support_pca}
\vspace{-0.3cm}
\end{figure}

To investigate the support recovery performance of \eqref{eq:ospca}, the left panel of Figure \ref{fig:support_pca} suggests that \eqref{eq:ospca} with $\mu=1.7$ or 1 identifies the true zeros in the early stage of training, while \eqref{eq:opca} cannot identify any zeros. The right panel of Figure \ref{fig:support_pca} provide evidence that except for the early stage of training, \eqref{eq:ospca} with the selected $\mu$ can correctly identify the active set for most of the time $t$. Notably, \eqref{eq:ospca} with $\mu=0.4$ has more false zeros than \eqref{eq:ospca} with $\mu=1.7$ for time $t=n\gamma<1$. This stems from the tuning function in \eqref{eq:lamsim}, in which $(n\gamma)^{\mu}$ is a strictly decreasing function in $\mu$ for $0<n\gamma<1$. Hence, the strength of penalization for $\mu=0.4$ is stronger than $\mu=1.7$ for $t<1$. This phenomenon also appears in the left panel of Figure \ref{fig:support_pca} and in both panels of Figure \ref{fig:cover_lr}. But it is most visible in the right panel of Figure \ref{fig:support_pca} due to the scale of time axis.

Section \ref{sec:addsimpca} contains additional simulation results for the second principal component, and for the effect of step size $\gamma$ on the first principal component.

	%

\section{Discussion and future works}
In this paper, we propose a generalization of regularized dual averaging algorithms, and develop its mean trajectories and distributional dynamics that apply to both convex and non-convex loss functions. There are three important consequences: (i) we prove that \eqref{eq:rda} is biased for some linear regression problems (Theorem \ref{th:biasrda}), due to strong penalization; (ii) our theory provides a theoretical guideline in selecting an appropriate penalization level; (iii) for the first time in the literature, we show that uncertainty quantification may be made for online penalized algorithms, and further discover an interesting observation that there exists discontinuity in the distributional dynamics for the $\ell_1$ penalized problems.

Preliminary analysis in Figure \ref{fig:cifar10} demonstrates promising outcomes for using gRDA to compress over-parametrized deep neural networks (DNNs) without sacrificing accuracy. A common practice of DNN is to start with a dense and over-parametrized model, and then train it with stochastic optimization algorithms \citep{GBC16}. However, many modern DNNs are too large to fit in devices with resource constraints, e.g. mobile phones. It is thus necessary to obtain compressed DNNs before deploying DNNs on such devices. 
Our gRDA sparsifies DNN during training, and is pretty robust to the hyperparameters. The learning rate can taken to be equal to that of SGD. Only two other hyperparameters are involved in the tuning function in \eqref{eq:tunel1} (with $t_0=0$): constant $c$ is an initial sparsity level which has limited influence on the outcome as the number of epochs gets large; $\mu$ is the key hyperparameter, and our theory provides insight on its selection. Hence, our method can be a step--stone toward a provable approach for sparsifying DNN, especially for large DNNs such as ResNet50. 

\begin{supplement}[id=suppA]
  \sname{ONLINE SUPPLEMENTARY MATERIAL}
  \stitle{A generalization of regularized dual averaging and its dynamics}
  \slink[doi]{COMPLETED BY THE TYPESETTER}
  \sdatatype{.pdf}
  \sdescription{The supplementary materials contain all the proofs, additional technical details, and additional numerical results and figures.}
\end{supplement}

\bibliography{spglm.bib}


\clearpage
\newpage
\vskip 2em \centerline{\Large \bf ONLINE SUPPLEMENTARY MATERIAL} \vskip -1em
\setcounter{page}{1}
\setcounter{section}{0}
\renewcommand{\thesection}{S.\arabic{section}}
\renewcommand{\thetheo}{S.\arabic{section}.\arabic{theo}}
\renewcommand{\theequation}{S.\arabic{section}.\arabic{equation}}
\renewcommand{\thesubsection}{S.\arabic{section}.\arabic{subsection}}

\section{Theory on equi-strongly convex functions}\label{sec:appcvx}

Denote $\partial\Psi(\bw)$ by the subdifferential of a function $\Psi$ at $\bw$, and the $\partial\Psi(\bw)$ is a singleton if and only if $\Psi(\bw)$ is differentiable (Corollary 2.4.10 in \cite{Z02}). Recall that a function $f:\R^d\to\R$ is {lower semicontinuous} (l.s.c.) if $\liminf_{\bu\to\bu_0} f(\bu)\geq f(\bu_0)$ for every $\bu_0\in\R^d$ (see, e.g. page 8 of \cite{RW09}).

A property shared by all $\Psi_\gamma(t,\bw), \tP_\gamma(t,\bw)$, $\Psi(t,\bw)$ and $\tP(t,\bw)$ in Section \ref{sec:gentheo} is stated in the next definition. This allows us to leverage powerful tools from convex and variational analysis.
\begin{defin}[$\beta$-equi strong convexity, or $\beta$-e.s.c.]\label{def:euc}
	A sequence of functions $\{\Phi_\gamma\}_\gamma$, where $\Phi_\gamma:[0,\infty) \times \R^d\to\R^d$, $(t,\bw)\mapsto \Phi_\gamma(t,\bw)$, is said to be {$\beta$-equi-strongly convex with respect to $t,\gamma$} if and only if there exists $\beta>0$ independent from $t,\gamma$ such that for any $\bw,\bw'\in\R^d$ and $\alpha\in[0,1]$,
		\begin{align}
			\Phi_\gamma\big(t,\alpha\bw+(1-\alpha)\bw'\big) \leq \alpha\Phi_\gamma(t,\bw)+(1-\alpha)\Phi_\gamma(t,\bw')-\beta\frac{\alpha(1-\alpha)}{2}\|\bw-\bw'\|_2^2;\label{eq:defuc}
		\end{align}
		similarly, a single function $\Phi:[0,\infty) \times \R^d\to\R^d$ is $\beta$-e.s.c. with respect to $t$ if and only if 
			\begin{align}
				\Phi(t,\alpha\bw+(1-\alpha)\bw') \leq \alpha\Phi(t,\bw)+(1-\alpha)\Phi(t,\bw')-\beta\frac{\alpha(1-\alpha)}{2}\|\bw-\bw'\|_2^2.\label{eq:defuccon}
			\end{align}
\end{defin}
A strongly convex function is naturally $\beta$-e.s.c. All results in this section can be extended to the class of uniformly convex functions \citep{Z02}, which is a larger class than the class of strongly convex functions.

The following lemma provides a convenient criterion for determining whether a function is e.s.c.
\begin{lemma}\label{lem:eucchar}
	 If an arbitrary function $\Phi_\gamma:[0,\infty)\times \R^d\to\R$ is in the form $\Phi_\gamma=\Phi_\gamma^{(1)}+\Phi_\gamma^{(2)}$, where $\Phi_\gamma^{(1)}(t,\bw)$ is $\beta$-e.s.c. with respect to $\gamma$ and $t$, and $\Phi_\gamma^{(2)}(t,\bw)$ is convex for any $\gamma,t$, then $\Phi_\gamma$ is $\beta$-e.s.c. The same holds if replacing the family $\{\Phi_\gamma\}_\gamma$ by a single function $\Phi:[0,\infty)\times \R^d\to\R$.
\end{lemma}
\begin{proof}[Proof of Lemma \ref{lem:eucchar}]
		Take arbitrary $\alpha\in [0,1]$, $\bw,\bw'\in\R^d$,
		\begin{align}
			\alpha\Phi_\gamma^{(1)}(t,\bw)+(1-\alpha)\Phi_\gamma^{(1)}(t,\bw')-\Phi_\gamma^{(1)}(t,\alpha\bw+(1-\alpha)\bw')
					\geq \beta\frac{\alpha(1-\alpha)}{2}\|\bw-\bw'\|_2^2 \label{eq:eucc1}
		\end{align}
		and by the convexity,
		\begin{align}
			\alpha\Phi_\gamma^{(2)}(t,\bw)+(1-\alpha)\Phi_\gamma^{(2)}(t,\bw')-\Phi_\gamma^{(2)}(t,\alpha\bw+(1-\alpha)\bw')
					\geq 0. \label{eq:eucc2}
		\end{align}
	Thus, combining \eqref{eq:eucc1} and \eqref{eq:eucc2}, we obtain that
	\begin{align*}
		&\alpha\Phi_\gamma(t,\bw)+(1-\alpha)\Phi_\gamma(t,\bw')-\Phi_\gamma(t,\alpha\bw+(1-\alpha)\bw')
		\geq 
		\beta\frac{\alpha(1-\alpha)}{2}\|\bw-\bw'\|_2^2.
	\end{align*}
	This proves that $\Phi_\gamma(t,\bw)$ is $\beta$-e.s.c. with respect to $\gamma,t$.
\end{proof}

\bigskip
The following lemma shows that $\Psi_\gamma(t,\bw)$ in \eqref{eq:trcond} and $\tP_\gamma(t,\bw)$ in \eqref{eq:locdgf} are e.s.c. with respect to $t,\gamma$ if \ref{as:R} holds, and $\Psi(t,\bw)$ in \eqref{def:Psi} and $\tP(t,\bw)$ in Condition \ref{as:tPlim} are e.s.c. with respect to $t$ if \ref{as:Plim} and \ref{as:tPlim} hold.

\begin{lemma}\label{lem:escreg}
	 Suppose Condition \ref{as:R} holds.
	 \begin{itemize}
	 	\item[(i)] $\Psi_\gamma(t,\bw), \tP_\gamma(t,\bw)$ are both $\beta$-e.s.c. with respect to $t,\gamma$ and $\tP_\gamma(t,\cdot)$ is l.s.c. for any $t$.
		\item[(ii)] If additionally \ref{as:Plim} holds, then $\Psi(t,\bw)$ is $\beta$-e.s.c. with respect to $t$.
		\item[(iii)] If additionally \ref{as:tPlim} holds, then $\tP(t,\bw)$ is $\beta$-e.s.c. with respect to $t$.
	 \end{itemize}
\end{lemma}
\begin{proof}[Proof of Lemma \ref{lem:escreg}]
	Since $F$ is strongly convex with constant $\beta$, and $g(\lfloor t/\gamma\rfloor,\gamma)\Pc$ is convex for any $t,\gamma$, $\Psi_\gamma(t,\bw)$ is $\beta$-e.s.c. by Lemma \ref{lem:eucchar}. The proof for $\Psi(t,\bw)$ is similar.
	
	$\tP_\gamma(t,\cdot)$ is l.s.c. by the fact that $\Psi_\gamma(t,\cdot)$ is l.s.c. for any $t$. For any $\alpha\in [0,1]$, and $\bu,\bu'\in\R^d$,
	\begin{align}
		&\alpha\tP_\gamma(t,\bu) + (1-\alpha)\tP_\gamma(t,\bu')-\tP_\gamma(t,\alpha\bu+(1-\alpha)\bu') \notag\\
		&=\gamma^{-1}\big[\alpha\Psi_\gamma(t,\bw(t)+\sqrt{\gamma}\bu)+(1-\alpha)\Psi_\gamma(t,\bw(t)+\sqrt{\gamma}\bu') \notag\\
		&\hspace{4cm}- \Psi_\gamma\{t,\alpha(\bw(t)+\sqrt{\gamma}\bu)+(1-\alpha)(\bw(t)+\sqrt{\gamma}\bu')\}\big]\notag\\
		&\geq \beta \frac{\alpha(1-\alpha)}{2}\|\bu-\bu'\|_2^2, \label{eq:tpsicvx}
	\end{align}
	where the last inequality follows from the $\beta$-e.s.c. of $\Psi_\gamma(t,\bw)$. This proves that $\tP_\gamma(t,\cdot)$ is $\beta$-e.s.c. with respect to $\gamma$ and $t$.
	
	Taking $\lim_{\gamma\to 0}$ on the both sides of \eqref{eq:tpsicvx}, the $\beta$-e.s.c. of $\tP(t,\bu)$ with respect to $t$ follows from \ref{as:tPlim0}.
\end{proof}

\bigskip
The next lemma states the properties shared by any family of functions that is e.s.c.

\begin{lemma}[Properties of e.s.c. family]\label{lem:uc}
	Let $\Phi_\gamma:[0,\infty)\times\R^d\to\R$ be a family of finite functions such that $\Phi_\gamma(t,\cdot)$ is l.s.c. for each $\gamma, t$ and $\beta$-e.s.c with respect to $\gamma,t$, and $\Phi:[0,\infty)\times\R^d\to\R$ is $\beta$-e.s.c with respect to $t$ and $\Phi(t,\cdot)$ is l.s.c. for each t. Then
	\begin{itemize}
		\item[(a)] $\Phi_\gamma^*(t,\cdot)$ is differentiable on $\R^d$, $\dom(\Phi_\gamma^*(t,\cdot))=\R^d$ for each $\gamma,t$, and for $\bv,\bv'\in\R^d$,
		 $$
		 \sup_{\gamma, t}\|\nabla \Phi_\gamma^*(t,\bv)-\nabla \Phi_\gamma^*(t,\bv')\|_2\leq \beta^{-1}\|\bv-\bv'\|_2.
		 $$ 
		 Same result holds for $\Phi^*(t,\cdot)$.
		\item[(b)] $\partial_\bv \Phi_\gamma^*(t,\bv)|_{\bv=\bv_0}$ is a singleton for any $\bv$. In particular, putting $\bv_0=0$ implies the minimizer of $\Phi_\gamma(t,\cdot)$ is unique for each $\gamma,t$. Same result holds for $\Phi^*(t,\cdot)$.
		\item[(c)] For $\bw_0\in\mbox{dom}(\Phi_\gamma(t,\cdot))$, $\bv_0\in\partial_{\bw} \Phi_\gamma(t,\bw)|_{\bw=\bw_0}$ if and only if $\bw_0=\nabla_\bv \Phi_\gamma^*(t,\bv)|_{\bv=\bv_0}$.
	\end{itemize}
\end{lemma}

\begin{proof}[Proof of Lemma \ref{lem:uc}]
We note that the definition of $\beta$-e.s.c. in \eqref{eq:defuc} of Definition \ref{def:euc} 
is effectively Corollary 3.5.11(i) of \cite{Z02} with $q=2, c_1=\beta$. Part (a) immediately follows by Corollary 3.5.11(x) on p.217-218 of \cite{Z02}. 

Part (b) follows by Part (a) and Corollary 2.4.10 on p.91 of \cite{Z02} (on finite dimensional normed vector spaces, G\^ateaux and Fr\'echet differentiability coincide). 

For Part (c), if $\bv_0\in\partial_{\bw} \Phi_\gamma(t,\bw)|_{\bw=\bw_0}$ (the domain of $\Phi_\gamma(t,\bw)$ is $\R^d$ because $\Phi_\gamma(t,\bw)$ is finite), then $\bw_0 \in\partial_\bw \Phi_\gamma^*(t, \bv)|_{\bv=\bv_0}$ by Theorem 2.4.4 (iv) on p.85 of \cite{Z02}. Since $\partial_\bv \Phi_\gamma^*(t,\bv)|_{\bv=\bv_0}=\{\nabla_{\bv} \Phi_\gamma^*(t,\bv)|_{\bv=\bv_0}\}$ by Part (b), the proof of part of (c) is complete. 
\end{proof}

\bigskip
The next lemma concerns $\nabla\Phi^*$ where $\Phi(t,\bw)$ is an arbitrary e.s.c. function.

\begin{lemma}[Joint continuity of $\nabla\Phi^*$]\label{lem:conticonj}
	Let $\Phi:[0,\infty)\times\R^d\to\R$ be such that $\Phi(t,\cdot)$ is $\beta$-e.s.c. with respect to $t$ and l.s.c. for each $t\in[0,\infty)$, and $\Phi(\cdot,\bu)$ is continuous on $[0,\infty)$ for any fixed $\bu$. If in addition, there exits a $\bu_0\in\R^d$ and a constant $c_T>0$ for any $T>0$ such that for any subgradient $\nabla\Phi(t,\bu_0)\in\partial_\bu \Phi(t,\bu_0)$ for $t\leq T$, $\sup_{0\leq t\leq T}\|\nabla\Phi(t,\bu_0)\|_2 \leq c_T$. Then $\nabla\Phi^*:[0,\infty)\times\R^d\to\R^d$ is jointly continuous on $[0,T]\times \R^d$ for any $T>0$.
\end{lemma} 

\begin{proof}[Proof of Lemma \ref{lem:conticonj}]	
	Because $\Phi(t,\cdot)$ is l.s.c. and $\beta$-e.s.c. for each $t\in [0,\infty)$ by the conditions in the Theorem, Applying Lemma \ref{lem:uc}(a) to $\Phi(t,\cdot)$, we obtain that $\nabla\Phi^*(t,\cdot)$ is $\beta^{-1}$-Lipschitz. By Lemma \ref{lem:jointconti}, it is enough to show that $\nabla\Phi^*(\cdot,\bv)$ is continuous on $[0,T]$ for any $\bv$ and $T>0$.
	
	Fix a $\bv\in\R^d$, to show that $\nabla\Phi^*(\cdot,\bv)$ is continuous on $[0,T]$, it is enough to show $\lim_{l\to\infty}\nabla\Phi^*(t_l,\bv) = \nabla\Phi^*(t_0,\bv)$, for any convergent sequence $\{t_l\}_{l\in\N}\subset[0,T]$ such that $t_l\to t_0$ as $l\to\infty$. To save notations, define 
	\begin{align*}
		\tPh(t,\bu)&:=\Phi(t,\bu)-\bu^\top\bv\\
		\tilde\bu(t)&:=\nabla\Phi^*(t,\bv)=\arg\min_{\bu\in\R^d}\tPh(t,\bu).
	\end{align*}
By Lemma \ref{lem:uc}(b), $\tPh(t,\bu)$ has a unique minimizer $\tilde\bu(t)$ for each $t$ since $\Phi(t,\bu)$ is e.u.c. with respect to $t$. Applying the conditions of $\Phi$ in this Theorem, $\Phi$ is continuous on $[0,T]\times\R^d$ by Lemma \ref{lem:rock107}, so $\tPh$ is also continuous on $[0,T]\times\R^d$.
	
	By the continuity of $\tPh(\cdot,\bu)$ on $[0,T]$ for each $\bu\in\R^d$, $\tPh(t_l,\tilde\bu(t_0))\to \tPh(t_0,\tilde\bu(t_0))$. Hence, there exists constant $K_0>0$ depending on $\tP$ and $t_0$ such that
	\begin{align}
		K_0 \geq \tPh(t_l,\tilde\bu(t_0)) \geq \tPh(t_l,\tilde\bu(t_l)), \label{eq:mint1}
	\end{align}
	where the second inequality is from the uniqueness of minimizer. 
	
	We claim that $\{\tilde\bu(t_l)\}_{l\in\N}$ is a bounded sequence. Suppose to the contrary that $\{\tilde\bu(t_l)\}_{l\in\N}$ is not bounded, then there is a subsequence $l'$ such that $\|\bu(t_{l'})\|_2\to\infty$. From \eqref{eq:mint1},
	\begin{align}
		K_0 \geq \lim_{l'\to\infty}\tPh(t_{l'},\tilde\bu(t_{l'})) &= \lim_{l'\to\infty} \inf_{s\in[0,T]}\tPh(s,\tilde\bu(t_{l'})) \notag\\
		&\geq \lim_{l'\to\infty} \inf_{s\in[0,T]} \bigg(\underbrace{\frac{\Phi(s,\bu(t_{l'}))-\|\bu(t_{l'})\|_2\|\bv\|_2}{\|\bu(t_{l'})\|_2^2}}_{\mbox{\scriptsize $\geq \beta/2$ as $\|\bu(t_{l'})\|_2\to\infty$}}\bigg)\|\bu(t_{l'})\|_2^2=\infty,
	\end{align}
	where the last equality follows from \eqref{eq:coer} in Lemma \ref{lem:coer}. 
	Therefore, we have a contradiction and conclude that $\{\tilde\bu(t_l)\}_{l\in\N}$ is a bounded sequence. 
	
	Since $\{\tilde\bu(t_l)\}_{l\in\N}$ is bounded, there exist at least one accumulation point for this sequence. We claim that $\tilde\bu(t_0)$ is the {only} accumulation point, and therefore $\tilde\bu(t_0)=\lim_{l\to\infty}\tilde\bu(t_l)$. To prove this claim, suppose $\bar\bu$ is an arbitrary accumulation point and the subsequence $\tilde\bu(t_{l''})\to\bar\bu$. Using the relation in \eqref{eq:mint1},
	\begin{align*}
		\tPh(t_0,\tilde\bu(t_0)) = \lim_{l''\to\infty}\tPh(t_{l''},\tilde\bu(t_0)) \geq \lim_{l''\to\infty} \tPh(t_{l''},\tilde\bu(t_{l''})) = \tPh(t_0,\bar\bu).
	\end{align*}
	where in the first equality we apply the continuity of $\tPh(\cdot,\tilde\bu(t_0))$, while in the third equality we apply the joint continuity of $\tPh$. Because $\tilde\bu(t_0)$ uniquely minimizes $\tPh(t_0,\cdot)$, it follows that $\tilde\bu(t_0)=\bar\bu=\lim_{l''\to\infty}\tilde\bu(t_{l''})$, for any convergent subsequence $\bu(t_{l''})$ of the bounded sequence $\bu(t_l)$. Therefore, the proof for the continuity of $\tilde\bu(\cdot)$ is complete. 
	
\end{proof}

\begin{lemma}\label{lem:jointconti}
	Let $\Tc\subset\R$ be a compact set, and $h:\Tc\times\R^d\to\R$ be such that $h(t,\cdot)$ is globally Lipschitz continuous on $\R^d$ with the Lipschitz constant $c$ independent of $t$, and $h(\cdot,\bu)$ is continuous for every $\bu\in\R^d$. Then $h$ is continuous on $\Tc\times\R^d$.
\end{lemma}
\begin{proof}[Proof of Lemma \ref{lem:jointconti}]
	Since $\Tc$ is compact, $h(\cdot,\bu)$ is uniformly continuous on $\Tc$ for every $\bu\in\R^d$. Take arbitrary convergent sequence $(t_n,\bu_n)\to(t_0,\bu_0)$, $n\in\N$. For any $\epsilon>0$, let $N_1\in\N$ depending on $\bu_0$ be such that $|h(t_n,\bu_0) - h(t_0,\bu_0)|\leq \epsilon/2$ for all $n\geq N_1$. Let $N_2\in\N$ be such that $\|\bu_n-\bu_0\|_2\leq \epsilon/(2c)$ for all $n\geq N_2$. Then, for all $n\geq \bar N:=\mbox\{N_1,N_2\}$,
	\begin{align*}
		|h(t_n,\bu_n) - h(t_0,\bu_0)| &\leq |h(t_n,\bu_n) - h(t_n, \bu_0)| + |h(t_n,\bu_0) - h(t_0,\bu_0)|\\
		&\leq c \|\bu_n-\bu_0\|_2 + \epsilon/2\\
		&\leq \epsilon.
	\end{align*}
\end{proof}

The following lemma is inspired by Proposition 3.5.8 of \cite{Z02}.

\begin{lemma}[Uniform coercity]\label{lem:coer}
Let $\Phi:[0,\infty)\times\R^d\to\R$ satisfy the conditions of Lemma \ref{lem:conticonj}. Then
  		\begin{align}
  			\liminf_{\|\bu\|_2\to\infty}\inf_{0\leq t\leq T}\frac{\Phi(t,\bu)}{\|\bu\|_2^2}\geq\frac{\beta}{2}, \mbox{ for any $T>0$}. \label{eq:coer}
  		\end{align}
\end{lemma}
\begin{proof}[Proof of Lemma \ref{lem:coer}]
	The definition of the subgradient implies
	\begin{align}
		\Phi(t,\bu) \geq \Phi(t,\bu_0) + \langle \bu-\bu_0,\nabla \Phi(t,\bu_0)\rangle, \mbox{ for any }\bu\in\R^d. \label{eq:coer1}
	\end{align} 
	Moreover, for any $\alpha\in (0,1)$, applying \eqref{eq:coer1} gives
	\begin{align}
		\Phi(t,\bu_0)+\alpha \langle \bu-\bu_0,\nabla \Phi(t,\bu_0)\rangle &\leq \Phi(t,(1-\alpha)\bu_0+\alpha \bu)\notag\\
		&\leq (1-\alpha)\Phi(t,\bu_0) + \alpha \Phi(t,\bu) - \alpha(1-\alpha)\rho_{\Phi(t,\cdot)}(\|\bu\|_2), \label{eq:coer2}
	\end{align}
	where the gage of uniform convexity (page 203 of \cite{Z02}, note that the domain of $\Phi(t,\cdot)$ is $\R^d$ for all $t$) for each $t$:
	\begin{align}
		\rho_{\Phi(t,\cdot)}(s):=\inf\Big\{\frac{(1-\alpha)\Phi(t,\bu)+\alpha \Phi(t,\bu')-\Phi(t,(1-\alpha)\bu+\alpha \bu')}{\alpha(1-\alpha)}\Big|\alpha\in(0,1), \|\bu-\bu'\|=s\Big\}. 
	\end{align}
	Applying the definition of $\beta$-e.s.c. (Definition \ref{def:euc}), we find that 
	\begin{align}
		\inf_{0\leq t<\infty} \rho_{\Phi(t,\cdot)}(1) \geq \beta/2>0. \label{eq:uniscvx1}
	\end{align}
	Dividing both sides of \eqref{eq:coer2} by $\alpha$ gives $\Phi(t,\bu_0)+ \langle \bu-\bu_0,\nabla \Phi(t,\bu_0)\rangle + (1-\alpha)\rho_{\Phi(t,\cdot)}(\|\bu\|_2) \leq \Phi(t,\bu)$. Letting $\alpha\to 0$ yields
	\begin{align}
		\Phi(t,\bu_0)+ \langle \bu-\bu_0,\nabla \Phi(t,\bu_0)\rangle + \rho_{\Phi(t,\cdot)}(\|\bu\|_2) \leq \Phi(t,\bu).\label{eq:coer3}
	\end{align}
	Dividing both sides of \eqref{eq:coer3} by $\|\bu\|_2^2$, and taking into account that $\langle \bu-\bu_0,\nabla \Phi(t,\bu_0)\rangle \geq -\|\bu\|_2 \|\nabla \Phi(t,\bu_0)\|_2$ for large $\bu$, we obtain from \eqref{eq:coer3} that
	\begin{align}
		&\liminf_{\|\bu\|_2\to\infty}\inf_{0\leq t\leq T}\frac{\Phi(t,\bu)}{\|\bu\|_2^2} \notag\\
		&\geq \liminf_{\|\bu\|_2\to\infty} \Big\{\frac{\inf_{0\leq t\leq T}\Phi(t,\bu_0)}{\|\bu\|_2^2}- \frac{\sup_{0\leq t\leq T}\|\nabla \Phi(t,\bu_0)\|_2}{\|\bu\|_2} + \frac{\inf_{0\leq t\leq T}\rho_{\Phi(t,\cdot)}(\|\bu\|_2)}{\|\bu\|_2^2}\Big\}\notag\\
		&\geq \liminf_{\|\bu\|_2\to\infty} \Big\{\frac{\inf_{0\leq t\leq T}\Phi(t,\bu_0)}{\|\bu\|_2^2}- \frac{\sup_{0\leq t\leq T}\|\nabla \Phi(t,\bu_0)\|_2}{\|\bu\|_2} + \inf_{0\leq t\leq T}\rho_{\Phi(t,\cdot)}(1)\Big\}\notag\\
		&\geq \frac{\beta}{2},\label{eq:coer4}
	\end{align}
		where the penultimate inequality applies the fact $\rho_{\Phi(t,\cdot)}(s) \geq s^2 \rho_{\Phi(t,\cdot)}(1)$ for any $s\geq 1$ by Proposition 3.5.1 on page 203 of \cite{Z02}; in the last inequality, we use the fact that $\sup_{0\leq t\leq T}|\Phi(t,\bu_0)|<\infty$ since $\Phi(\cdot,\bu_0)$ is continuous on $\R$, $\sup_{0\leq t\leq T}\|\nabla\Phi(t,\bu_0)\|_2 \leq c_T$ from the conditions of Lemma \ref{lem:conticonj} and \eqref{eq:uniscvx1}. 
\end{proof}

\section{Proofs for Section \ref{sec:gentheo}}\label{sec:pfgentheo}

In this section, we will view the restriction of an element in $D([0,\infty))^d$ on $[0,T]$ as an element in $D([0,T])^d$. We will use the results for $D([0,T])^d$ in \cite{B99} extensively, and  \cite{B99} metrizes $D([0,T])^d$ with metric $\rho_{d,\circ}^{T}$, defined by $\rho_{d,\circ}^{T}(\bx,\by):=\sum_{j=1}^d \rho_\circ^T(x_j,y_j)$ where 
\begin{align}		
\rho_\circ^T(x,y)=\inf_{\nu\in\Vc_T}\Big\{\sup_{0\leq t<s\leq T}\Big|\log\frac{\nu(s)-\nu(t)}{s-t}\Big|\vee \sup_{0\leq t\leq T}\big\|x(t)-y(\nu(t))\big\|_2\Big\}\label{eq:metricT}
\end{align}
where $\Vc_T$ is a class of nondecreasing functions on mapping $[0,T]$ onto itself. $\rho_\circ^T$ is one of the common metrics which topologize $D([0,T])$ which gives rise to the $J_1$ topology. 
This metric makes $D([0,T])$ a separable and complete space (see e.g. page 125-129 of \cite{B99}). 

Another metric on $D([0,\infty))^d$ can be defined by the product metric with base  
$$
d_\infty^\circ(x,y)=\sum_{m=1}^\infty 2^{-m}(1\wedge \rho_\circ^m(x^m,y^m))
$$ 
as page 168 of \cite{B99}, where $x^m$ and $y^m$ are the restrictions of $x,y\in D[0,\infty)$ on $[0,m]$. We note that the product metric with base $d_\infty^\circ$ generates the same topology as $\rho_d^\infty$ defined in \eqref{eq:metric}, because their base metric $d_\infty^\circ(x,y)\to 0$ if and only if $\rho^\infty(x,y)\to 0$, which follows by Theorem 16.1 on page 168 of \cite{B99} and Proposition 3.5.3 (a)$\Leftrightarrow$(b) on page 119 of \cite{EK86}. It follows from this property that the almost sure and in probability convergence are equivalent under the two metrics. Furthermore, by Lemma \ref{lem:fint}, the weak convergence in $D([0,\infty))^d$ is equivalent to the weak convergence in $(D[0,T))^d$ for every $T>0$.

\subsection{Proofs for Section \ref{sec:asymtr}}\label{sec:pfasymtr}

\begin{lemma}\label{lem:ui_iid}
	Suppose $S_n:=\sum_{i=1}^n X_n$, where $X_n$ is an i.i.d. sequence and $\E|X_1|<\infty$. Then $\{S_n/n\}_{n\in\N}$ is uniformly integrable. 
\end{lemma}
\begin{proof}[Proof of Lemma \ref{lem:ui_iid}]
It is sufficient to prove that for any $\epsilon>0$, there exists $M_2=M_2(\epsilon,\E|X_1|)$ such that $\sup_n\frac{1}{n}\sum_{i=1}^n \E\big[|X_n|\IF\{|S_n/n|>M_2\}\big]<\epsilon$. Then, the proof is completed by invoking the definition of the uniform integrability (e.g. Eq. (25.10) on page 338 of \cite{B95}).

	Take $M_1=M_1(\epsilon)$ such that $\E\big[|X_1|\IF\{|X_1|>M_1\}\big]<\epsilon/2$ (the existence of $M_1$ is guaranteed by the dominated convergence). Set $M_2:=2M_1\E|X_1|/\epsilon$, then
	\begin{align*}
		&\sup_n\E\big[|S_n/n|\IF\{|S_n/n|>M_2\}\big]\\ 
		&\leq \sup_n\frac{1}{n}\sum_{i=1}^n \E\big[|X_n|\IF\{|S_n/n|>M_2\}\big]\\
		&\leq \sup_n\frac{1}{n}\sum_{i=1}^n \big\{\E\big[|X_n|\IF\{|X_n|>M_1\}\big] + \E\big[|X_n|\IF\{|X_n|\leq M_1, |S_n/n|>M_2\}\big]\big\}\\
		&\leq \epsilon/2 + M_1 P(|S_n/n|>M_2)\\
		&\leq \epsilon/2 + M_1 (\E|X_1|/M_2) \\
		&= \epsilon,
	\end{align*}
	where the Markov inequality is applied in the penultimate inequality. 
\end{proof}

\subsubsection{Proof of Theorem \ref{th:at}}\label{sec:pfat}
First, we note that \ref{as:R} and \ref{as:Plim} have the following consequences for both $\Psi_\gamma(t,\bw)$ in \eqref{eq:trcond} and $\Psi(t,\bw)$ in \eqref{def:Psi}:
	\begin{itemize}[itemindent=55pt,leftmargin=0pt]	
		\labitemc{($\Psi$-lim0)}{as:Plim0}
		$\Psi(t,\cdot)$ is l.s.c. for any $t\in[0,\infty)$, and $\Psi(\cdot,\bu)$ is continuous on $[0,\infty)$ for any $\bu\in\R^d$.
		\labitemc{($\Psi$-lim1)}{as:Plim1} For every $T>0$, $\sup_{0\leq t\leq T}\big|\Psi_\gamma(t,\bu)-\Psi(t,\bu)\big|\to 0$ pointwise for every $\bu\in\R^d$.
		\labitemc{($\Psi$-lim2)}{as:Plim2} There exits a $\bu_0\in\R^d$ and a constant $c_T>0$ for any $T>0$ such that for any subgradient $\nabla_{\bu}\Phi(t,\bu_0)\in\partial_\bu \Phi(t,\bu_0)$ for $t\leq T$, $\sup_{0\leq t\leq T}\|\nabla\Phi(t,\bu_0)\|_2 \leq c_T$. 
	\end{itemize}
In particular, the first part of fact \ref{as:Plim0} follows by \ref{as:R}, and the second part follows by that $g^\dagger(t)$ is continuous in \ref{as:Plim}. \ref{as:Plim1} follows immediately from \ref{as:Plim}. \ref{as:Plim2} follows by that \ref{as:R} that both $F$ and $\Pc$ are finite.

\bigskip
Now we start with the main proof. Observe that from \eqref{eq:md} and Assumption \ref{as:M},
	\begin{align}
		\bv_{n+1} &= \bv_n - \gamma \nabla f_{n+1}(\bw_n)
				= \bv_n - \gamma\DD_{n+1}^\gamma-\gamma G(\bw_n).\label{eq:wtmp1_at}
	\end{align}
where $\DD_{n+1}^\gamma := \nabla f_{n+1}(\bw_n)-G(\bw_n)$ is a sequence of martingale differences. By iteration we have that the stopped process
\begin{align}
	\bv_\gamma(t\wedge \tau_\gamma^K)&= \bv_0 - \gamma \sum_{n=0}^{\lfloor (t\wedge \tau_\gamma^K)/\gamma\rfloor-1} \DD_{n+1}^\gamma - \gamma \sum_{n=0}^{\lfloor (t\wedge \tau_\gamma^K)/\gamma\rfloor-1} G(\bw_n)\notag\\
	&=: \bv_0 -L_\gamma(t\wedge \tau_\gamma^K) - r_\gamma(t \wedge \tau_\gamma^K).\label{eq:at1}
\end{align}

In \ref{p1}, we will show that $L_\gamma(\cdot \wedge \tau_\gamma^K) \weakto 0$ in $D([0,\infty))^d$. In \ref{p2}, it is shown that the stopped process $r_\gamma(\cdot \wedge \tau_\gamma^K)$ is relatively compact on $D([0,\infty))^d$. 

By \ref{p1-1} in \ref{p1} and \ref{p2}, all terms on the right of \eqref{eq:at1} converge to an element in $C([0,\infty))^d$ respectively. $C([0,\infty))^d$ is a Polish space, so both $L_\gamma(\cdot\wedge\tau_\gamma^K)$ and $r_\gamma(\cdot\wedge\tau_\gamma^K)$ are asymptotically tight by the discussion on page 23 of \cite{VW96}. It follows that $\bv_\gamma(\cdot\wedge\tau_\gamma^K)$ is relative compact by Lemma 1.4.3 of \cite{VW96} and Problem 3.22(c) on p.153 of \cite{EK86}.

\bigskip
Now we identify the limit. For any $K$ and any convergent subsequence $\bv_{\gamma'}(\cdot\wedge\tau_{\gamma'}^K)$ with limit $\bv_{K}'$ in $D([0,\infty))^d$, by the continuous mapping theorem and the fact that $\inf\{t:\cdot>K\}$ is a continuous,
\begin{align}
	(\bv_{\gamma'}(\cdot\wedge\tau_{\gamma'}^K),\tau_{\gamma'}^K)\weakto(\bv_{K}'(\cdot\wedge\tau^{K'}),\tau^{K'}) \quad\mbox{in }D([0,\infty))^d\times \R,\label{eq:jointstop_at}
\end{align}
for all but countably many $K$ at which $\tau^{K'}:=\inf\{t:\|\bv_{K}'(t)\|_2>K\}$ is discontinuous, and $\bv_{K}'(\cdot\wedge\tau^{K'})$ is continuous a.s. on $[0,\infty)$ by the almost sure continuity on $[0,\infty)$ of the limit of $L_\gamma(\cdot\wedge\tau_{\gamma'}^K)$ (from \ref{p1}) and $r_\gamma(\cdot\wedge\tau_{\gamma'}^K)$ (from \ref{p2}). 

	 A consequence followed by \eqref{eq:jointstop_at} and \ref{p3} is that 
\begin{align}
	\bw_{\gamma'}(\cdot\wedge\tau_{\gamma'}^K)\weakto\bw_{K}'(\cdot\wedge\tau^{K'}):=\nabla\tPs(\cdot,\bv_{K}'(\cdot\wedge\tau^{K'})) \quad\mbox{in $D([0,\infty))^d$}.\label{eq:weakwt}
\end{align}	  
where $\bw_{K}'(\cdot\wedge\tau^{K'})\in C([0,\infty))^d$ almost surely for all but countably many $K$, since $\nabla\tPs$ is continuous by Lemma \ref{lem:conticonj} [conditions there are verified by facts \ref{as:Plim1}, \ref{as:Plim2} and Lemma \ref{lem:escreg}(ii)] and $\bv_{K}'(\cdot\wedge\tau^{K'})\in C([0,\infty))^d$ a.s.

We will show that all possible limit $\bv_{K}'(\cdot\wedge\tau^{K'})$ satisfy
	\begin{align}
		\bv_{K}'(t\wedge \tau^{K'})&=\bv_0-\int_0^{t\wedge \tau^{K'}} G(\bw_K'(s\wedge \tau^{K'}))ds, \quad\mbox{for all $t\geq 0$},\label{eq:bvprime}
	\end{align}
and this suggests that $\bv_{K}'$ is one of the possibly multiple solutions of \eqref{eq:weak0v}, so we can write
\begin{align}
		(\bv_{\gamma'}(\cdot\wedge\tau_{\gamma'}^K),\tau_{\gamma'}^K)\weakto(\bv(\cdot\wedge\tau^K),\tau^K) \quad\mbox{in }D([0,\infty))^d\times \R,\label{eq:jointstop_at1}
\end{align}
for all but countably many $K$ at which $\tau^K$ is discontinuous, where $\bv(\cdot)$ is one of the solution of \eqref{eq:weak0v}. 

To see \eqref{eq:bvprime}, since $L_\gamma(\cdot\wedge \tau_\gamma^K)\weakto 0$ in $D([0,\infty))^d$ by \ref{p1}, it is sufficient to show 
\begin{align*}
	r_{\gamma'}(\cdot\wedge\tau_{\gamma'}^K)\rightsquigarrow \int_0^{\cdot} G(\bw_K'(s\wedge\tau^{K'}))ds\quad\mbox{in }D([0,T])^d.
\end{align*}
To see this, $G(\bw_{\gamma'}(\cdot\wedge\tau_\gamma^K))\rightsquigarrow G(\bw(\cdot\wedge\tau^{K'}))$ in $D([0,\infty))^d$ by \eqref{eq:weakwt} and the continuous mapping theorem with the continuity of $G(\cdot)$ in Assumption \ref{as:M}, this and Proposition 7.27 on page 118-119 of \cite{K08} yield
\begin{align}
	r_{\gamma'}(\cdot\wedge\tau_{\gamma'}^K) = \int_0^{\cdot} G(\bw_{\gamma'}(s\wedge\tau_\gamma^K))ds \rightsquigarrow \int_0^{\cdot} G(\bw_K'(s\wedge\tau^{K'}))ds \quad\mbox{in }D([0,T])^d,\label{eq:limr}
\end{align}
but this implies the convergence in $D([0,\infty))^d$ by Lemma \ref{lem:fint} as $T>0$ is arbitrary. Hence, \eqref{eq:bvprime} is proved by the fact that the limit on the right hand side of \eqref{eq:limr} is continuous and $\cdot\wedge\tau_{\gamma'}^K\weakto\cdot\wedge\tau^{K'}$ in $D[0,\infty)$. This proves \eqref{eq:weak0v} of Theorem \ref{th:at}(a). The proof for \eqref{eq:weak0w} in Theorem \ref{th:at}(a) is complete by noting \eqref{eq:weakwt}.

\bigskip
To prove (b), we note that if the solution to \eqref{eq:weak0v} is unique, then we have 
\begin{align}
	\rho_d^\infty\big(\bv_{\gamma}(\cdot \wedge \tau_{\gamma}^K),\bv(\cdot)\big)\stackrel{p}{\longrightarrow} 0\label{eq:in_p_conv}
\end{align}
by the discussion above Eq. (3.7) on p.27 of \cite{B99}. By the discussion below \eqref{eq:metricT}, the condition \eqref{eq:in_p_conv} is equivalent to
\begin{align}
	\rho_{d,\circ}^T\big(\bv_{\gamma}(\cdot \wedge \tau_{\gamma}^K),\bv(\cdot)\big)\stackrel{p}{\longrightarrow} 0, \quad\mbox{for any $T>0$}\label{eq:in_p_conv1}.
\end{align}
Since $\bv(\cdot)$ is continuous, the distance between $\bv_{\gamma}(\cdot \wedge \tau_{\gamma}^K)$ and $\bv(\cdot)$ in product uniform metric also converges to 0, by the argument at 5th line above Eq. (18.3) on p.150 in the first edition of \cite{B99}. This proves the first assertion. 

Similar arguments show $\bw_\gamma(\cdot\wedge\tau_{\gamma}^K)\stackrel{p}{\to}\bw(\cdot)$ in product uniform metric on $[0,T]$. Hence, (b) is proved.

\bigskip
To prove (c), since $\{\bv_{\gamma}(\cdot \wedge \tau_{\gamma}^K)\}_\gamma$ is relatively compact with all limits satisfying \eqref{eq:weak0v}, it is enough to show that for a sequence $\gamma'\to 0$ and $K\to\infty$ that \eqref{eq:jointstop_at} (at the same time, \eqref{eq:jointstop_at1}) holds,
\begin{align}
\rho_d^\infty(\bv_{\gamma'}(\cdot),\bv_{\gamma'}(\cdot \wedge \tau_{\gamma'}^K)) &\leq d e^{-\tau_{\gamma'}^K}\stackrel{p}{\longrightarrow} 0, \label{eq:at_c1}\\
\mbox{and}\quad\quad\quad\quad	\rho_d^\infty(\bv(\cdot\wedge\tau^K),\bv) &\leq d e^{-\tau^K}\stackrel{p}{\longrightarrow} 0, \label{eq:at_c2}
\end{align}
then it follows by Corollary 3.3.3 on p.110 of \cite{EK86} that $\bv_{\gamma'} \weakto \bv$ in $D([0,\infty))^d$, where $\bv$ is one of the solutions of \eqref{eq:weak0v} satisfying \eqref{eq:jointstop_at1}. 

To see \eqref{eq:at_c2}, since all the solutions of \eqref{eq:weak0v} are bounded in finite time,  $\tau^{K}\to\infty$ as $K\to\infty$ along a countable sequence. The proof of \eqref{eq:at_c1} is complete once we show $\tau_{\gamma'}^K\stackrel{p}{\to}\infty$, but this can be done by similar argument as in \ref{ws:stopping} in the proof of Theorem \ref{th:dd} in Section \ref{sec:pfdisdy}, so the details are omitted for brevity.

\bigskip
In the following, we show \ref{p1}-\ref{p3}. Before that, we note a fact: 
 \begin{align}
 \exists K_\Psi>0 \mbox{ depending on $\Psi$  such that $\sup_\gamma\sup_{t\in[0,T]}\|\nabla\Psi_\gamma^*(t,\bv)\|_2\leq K_\Psi$ on $\|\bv\|_2\leq K$.}\label{eq:bdwv}
 \end{align}
 This fact follows by Lemma \ref{lem:cptpres}, and the facts \ref{as:Plim0}, \ref{as:Plim1}, \ref{as:Plim2} and Lemma \ref{lem:escreg}(i) and (ii).

\begin{itemize}[itemindent=45pt,leftmargin=0pt]	

\labitem{Step 1}{p1}	{\bf $L_\gamma(\cdot\wedge \tau_\gamma^K)\weakto 0$ for every $K$ in $D([0,\infty))^d$.}

	We will proceed in three parts: first, we show that $L_\gamma(t)\in D([0,\infty))^d$ is relatively compact. Next, by the relative compactness of $L_\gamma(t\wedge \tau_\gamma^K)$ shown above, there exists a subsequence $\gamma'$ such that $L_{\gamma'}(t\wedge \tau_\gamma^K)\weakto L(t\wedge \tau_\gamma^K)$ for any $K$, we will show that $L(t)$ is a \emph{continuous martingale} with \emph{finite variation}, then $L(t\wedge \tau_\gamma^K)=0$ for all $t$ and every $K$ a.s. by, e.g. Proposition 15.2 on page 276 of \cite{K97}.  
	
	 
	 \bigskip
	 \begin{itemize}[itemindent=54pt,leftmargin=0pt]
	 \labitem{Part 1.1}{p1-1}\textbf{$\{L_\gamma(\cdot\wedge \tau_\gamma^K)\}_\gamma \subset D([0,\infty))^d$ is relatively compact and any limit of $L_\gamma(\cdot\wedge \tau_\gamma^K)$ is a.s. continuous.} 
	 
	 	Without loss of generality, we will assume $\|\bv_\gamma(t)\|_2\leq K$ for all $\gamma$ and $t$, and drop $\tau_\gamma^K$ in $L_\gamma(\cdot\wedge \tau_\gamma^K)$.
	 
	 Recall the definition of $L_\gamma(t)$ in \eqref{eq:at1}. Using \eqref{eq:bdwv}, for any $s>0$,
	 \begin{align*}
	 	\|L_\gamma(t+s)-L_\gamma(t)\|_2 \leq 2 \sum_{k=\lfloor t/\gamma\rfloor +1}^{\lfloor (t+s)/\gamma\rfloor} \sup_{\|\bw\|_2\leq K_\Psi}\|\nabla f(\bw,Q(\bw,\psi_{k+1}))\|_2.
	 \end{align*}
	  Since $\E\big[\sup_{\|\bw\|_2\leq K_\Psi}\|\nabla f(\bw,Q(\bw,\psi_{k+1}))\|_2\big]<\infty$ by \eqref{eq:GL1} in Assumption \ref{as:M}, the last display and the strong law of large number imply that for any fixed $t$,
	\begin{align*}
		\sup_{s\leq \delta}\|L_\gamma(t+s)-L_\gamma(t)\|_2
		&\leq 2\gamma \sum_{k=\lfloor t/\gamma\rfloor +1}^{\lfloor (t+\delta)/\gamma\rfloor} \sup_{\|\bw\|_2\leq K_\Psi}\|\nabla f(\bw,Q(\bw,\psi_{k+1}))\|_2\notag\\
		&\to 2\E\big[\sup_{\|\bw\|_2\leq K_\Psi}\|\nabla f(\bw,Q(\bw,\psi_{k+1}))\|_2\big] \delta \leq K_3 \delta \quad\mbox{a.s.}
	\end{align*}
	where $K_3$ depends on $K_\Psi$. By similar arguments on page 976 of \cite{BKS93}, 
	there exists a constant $K_4>0$ depends on $K_3$ that
	\begin{align}
	\limsup_{\gamma\to 0}\sup_{0\leq t\leq T}\sup_{s\leq \delta}\|L_\gamma(t+s)-L_\gamma(t)\|_2\leq K_4 \delta, \mbox{ a.s.}	\label{eq:conti}
	\end{align}
	This verifies Eq. (28) on page 976 of \cite{BKS93}. 
	
	On the other hand, again by $\E\big[\sup_{\|\bw\|_2\leq K_\Psi}\|\nabla f(\bw,Q(\bw,\psi_{k+1}))\|_2\big]<\infty$ in \eqref{eq:GL1} of Assumption \ref{as:M}, the SLLN implies that for each $T>0$,
	\begin{align*}
		\limsup_{\gamma\to 0}\sup_{0\leq t\leq T}\|L_\gamma(t)\|_2&\leq 2\sup_{0\leq t\leq T}\limsup_{\gamma\to 0}\gamma\sum_{k=0}^{\lfloor t/\gamma\rfloor-1} \sup_{\|\bw\|_2\leq K_\Psi}\|\nabla f(\bw,Q(\bw,\psi_{k+1}))\|_2\\
		&\to 2 T \E\Big[\sup_{\|\bw\|_2\leq K_\Psi}\|\nabla f(\bw,Q(\bw,\psi_{k+1}))\|_2\Big]<\infty, \quad\mbox{a.s.,}
	\end{align*}
	where the bound for $G(\bw)$ follows from the continuity of $G(\bw)$ on $\R^d$ by Assumption \ref{as:M}. Hence, by Lemma 3 on page 976 of \cite{BKS93}, $L_\gamma(t)$ is relatively compact.
	
	To show that any possible limit of $L_\gamma(t\wedge\tau_\gamma^K)$ is continuous almost surely, observe that
	\begin{align}
		\sup_{0<t\leq T}\big\|L_\gamma(t\wedge\tau_\gamma^K)-L_\gamma(t_-\wedge\tau_\gamma^K)\big\|_2\leq \gamma\sup_{\|\bw\|_2\leq K_\Psi}\|\nabla f(\bw,Q(\bw,\psi_{k+1}))\|_2\to 0, \quad\mbox{a.s.}\notag
	\end{align}
	since $\sup_{\|\bw\|_2\leq K_\Psi}\|\nabla f(\bw,Q(\bw,\psi_{k+1}))\|_2<\infty$ a.s. by \eqref{eq:GL1} in Assumption \ref{as:M}. The desired continuity follows by Theorem 13.4 of \cite{B99}.
	
	\bigskip
\labitem{Part 1.2}{p1-2}\textbf{$\{L_\gamma(t\wedge \tau_\gamma^K)\}_\gamma$ is uniformly integrable for all $t$ and any possible limit $L(\cdot\wedge \tau^K)$ of $L_{\gamma'}(t\wedge \tau_\gamma^K)$ is a martingale. }

First we show that $\{L_{\gamma}(t\wedge \tau_\gamma^K)\}_{\gamma}$ is uniformly integrable for all $t$, note that
	\begin{align}
		&\sup_{\gamma}\E\big[\|L_{\gamma}(t\wedge \tau_\gamma^K)\|_2 \IF\{\|L_{\gamma}(t\wedge \tau_\gamma^K)\|_2\geq M\}\big]\notag\\
		&\leq 4 \sup_{\gamma} \gamma\sum_{k=0}^{\lfloor (t\wedge \tau_\gamma^K)/\gamma\rfloor-1} \E\Big[\sup_{\|\bw\|_2\leq K_\Psi} \big\{\|\nabla f(\bw,Q(\bw,\psi_k))\|_2\IF\{\|L_{\gamma}(t\wedge \tau_\gamma^K)\|_2\geq M\}\Big] \label{eq:ui0}
	\end{align}
	For every $\epsilon>0$, there exists $M_\epsilon>0$ such that the right hand side of the display above is less than $\epsilon$ by Lemma \ref{lem:ui_iid} and \eqref{eq:GL1} in Assumption \ref{as:M}. Therefore, $\{\|L_{\gamma}(t\wedge \tau_\gamma^K)\|_2\}_{\gamma}$ is uniformly integrable, and this implies  $\{L_{\gamma,j}(t\wedge \tau_\gamma^K)\}_{\gamma}$ is uniformly integrable for each $j=1,...,d$.

	$L_\gamma(t\wedge \tau_\gamma^K)$ is a stopped process of a martingale by the defintion \eqref{eq:at1}, and is itself a martingale by the optional stopping theorem (see page 105 of \cite{K97}). 
	By \ref{p1-1}, $\{L_\gamma(\cdot\wedge \tau_\gamma^K)\}_\gamma$ is relatively compact. Take any subsequence $\gamma'$ such that 
	\begin{align}
		(L_{\gamma'}(\cdot\wedge \tau_{\gamma'}^K),\bv_{\gamma'}(\cdot\wedge\tau_{\gamma'}^K),\tau_{\gamma'}^K)\weakto(L(\cdot\wedge \tau^{K'}),\bv_{K}'(\cdot\wedge\tau^{K'}),\tau^{K'}) \quad\mbox{in }D([0,\infty))^d\times D([0,\infty))^d\times \R.\label{eq:jointconvL}
	\end{align}
	
	Therefore, for any event $A_s\in\Fc_s$ for $s\leq t$,
	\begin{align*}
		\E[L(t\wedge \tau^{K'})\IF\{A_s\}] = \lim_{\gamma'\to 0} \E[L_{\gamma'}(t\wedge  \tau_{\gamma'}^K)\IF\{A_s\}] = \lim_{\gamma'\to 0} \E[L_{\gamma'}(s\wedge  \tau_{\gamma'}^K)\IF\{A_s\}] = \E[L(s\wedge \tau^{K'})\IF\{A_s\}].
	\end{align*}
	where in the first and the last equality we apply Theorem 25.12 of \cite{B95}, and the fact that $\{L_{\gamma'}(t\wedge \tau_{\gamma'}^K)\IF(A_s)\}_{\gamma'}$ is uniformly integrable and $L_{\gamma'}(t\wedge \tau_{\gamma'}^K)\IF\{A_s\}\weakto L(t\wedge \tau^{K'})\IF\{A_s\}$ for any $t$ and $s\leq t$ by \eqref{eq:jointconvL}; the second equality follows from that $L_\gamma(t\wedge \tau_{\gamma'}^K)$ is a martingale. The last display shows that 
	\begin{align}
		\E\big[\E[L(t\wedge \tau^{K'})|\Fc_s]\IF\{A_s\}\big] =\E[L(t\wedge \tau^{K'})\IF\{A_s\}] = \E[L(s\wedge \tau^{K'})\IF\{A_s\}].
	\end{align}
	Hence, $\E[L(t\wedge \tau^{K'})|\Fc_s]=L(s\wedge \tau^{K'})$ and $L(\cdot\wedge \tau^{K'})$ is a martingale.

	\bigskip
	\labitem{Part 1.3}{p1-3}\textbf{Any possible limit $L(\cdot\wedge\tau^{K'})$ of $L_{\gamma}(\cdot\wedge\tau_{\gamma}^K)$ is of locally finite variation almost surely.} 
	
	Assume \eqref{eq:jointconvL} holds along some subsequence $\gamma'$. The total variation process is defined by $V_t(L_{\gamma'}(\cdot\wedge\tau_{\gamma'}^K)):=\sup\sum\|L_{\gamma'}(t_{k+1}\wedge\tau_{\gamma'}^K)-L_{\gamma'}(t_k\wedge\tau_{\gamma'}^K)\|_2$ with the $\sup$ ranging over all partitions $t_k$ of the interval $[0,t]$. $V_t(L_{\gamma'}(\cdot\wedge\tau_{\gamma'}^K))$ satisfies
	\begin{align}
		\E[V_t(L_{\gamma'}(\cdot\wedge\tau_{\gamma'}^K))] \leq 2 t \E \Big[\sup_{\|\bw\|_2\leq K_\Psi}\|\nabla f(\bw,Q(\bw,\psi_{k+1}))\|_2\Big] \leq Ct, \label{eq:bvLg}
	\end{align}
	for some $C>0$ by \eqref{eq:GL1} in Assumption \ref{as:M}. 
	Using the fact that $V_t(\cdot)$ is lower semicontinuous and is bounded below by 0, by the Portmanteau theorem and the weak convergence \eqref{eq:jointconvL}, we obtain from \eqref{eq:bvLg} that for any $t>0$,
	\begin{align*}
		\E[V_t(L(\cdot\wedge\tau^{K'}))] \leq \liminf_{\gamma'\to 0}\E[V_t(L_{\gamma'}(\cdot\wedge\tau_{\gamma'}^K))] \leq Ct.
	\end{align*}
	This implies that $V_t(L(\cdot\wedge\tau^{K'}))$ is finite for any $t>0$ almost surely and the proof is complete.
	
	\end{itemize}
	
	\bigskip
	\labitem{Step 2}{p2}{\bf $r_\gamma(\cdot \wedge \tau_\gamma^K)$ is relatively compact on $D([0,\infty))^d$, and all its possible limits are continuous a.s.}
	
Without loss of generality, we will assume $\|\bv_\gamma(t)\|_2\leq K$ for all $\gamma$ and $t$, and drop $\tau_\gamma^K$ in $r_\gamma(\cdot\wedge \tau_\gamma^K)$.

	Observe that
	\begin{align*}
		\sup_{s\leq \delta}\|r_\gamma(t+s)-r_\gamma(t)\|_2
		&\leq \gamma \lfloor \delta/\gamma\rfloor  \sup_{\|\bw\|_2\leq K_\Psi}\|\nabla G(\bw)\|_2
		\leq \delta K_1', \quad\mbox{a.s.}
	\end{align*}
	where the constant $K_1'>0$ depends on $K_\Psi$ and $\sup_{\|\bw\|_2\leq K_\Psi}\|\nabla G(\bw)\|_2<\infty$ [which follows by the continuity of $G$ in Assumption \ref{as:M}]. By similar arguments on page 976 of \cite{BKS93}, 
	there exists a constant $K_2'>0$ depends on $K_1'$ that
	\begin{align}
	\limsup_{\gamma\to 0}\sup_{0\leq t\leq T}\sup_{s\leq \delta}\|r_\gamma(t+s)-r_\gamma(t)\|_2\leq K_2' \delta, \mbox{ a.s.}	\label{eq:conti_r}
	\end{align}
	This verifies Eq. (28) on page 976 of \cite{BKS93}. 
	
	Similarly, for any $T>0$,
	\begin{align*}
		\limsup_{\gamma\to 0}\sup_{0\leq t\leq T}\|L_\gamma(t)\|_2&\leq \limsup_{\gamma\to 0}\gamma \lfloor T/\gamma\rfloor\sup_{\|\bw\|_2\leq K_\Psi}\|\nabla G(\bw)\|_2<\infty, \quad\mbox{a.s.,}
	\end{align*}
	Hence, by Lemma 3 on page 976 of \cite{BKS93}, $L_\gamma(t)$ is relatively compact.
	
	Lastly, the continuity of any possible limit of $r_\gamma(\cdot\wedge\tau_\gamma^K)$ follows by Theorem 13.4 of \cite{B99} based on the fact
	\begin{align}
		\sup_{0<t\leq T}\big\|r_\gamma(t\wedge\tau_\gamma^K)-r_\gamma(t_-\wedge\tau_\gamma^K)\big\|_2\leq \gamma\sup_{\|\bw\|_2\leq K_\Psi}\|G(\bw)\|_2\to 0, \quad\mbox{a.s.}\notag
	\end{align}
	
	\labitem{Step 3}{p3} \textbf{$\bw_\gamma\weakto\bw:=\nabla\Psi(\cdot,\bv(\cdot))$ in $D([0,\infty))^d$, for any sequence $\gamma$ such that $\bv_\gamma\weakto\bv$ in $D([0,\infty))^d$ with a deterministic function $\bv\in C([0,\infty))^d$.}
		
	By Lemma \ref{lem:fint}, it is enough to show the restricted $\br_T \bw_\gamma(t)\weakto \br_T \bw(t)$ in $(D[0,T))^d$ for any $T>0$, where $\br_t:D([0,\infty))^d\to D([0,T])^d$ is given by $\br_t \bw \mapsto (r_t w_1,...,r_t w_d)$, and $r_t w_j$ is the restriction of $w_j\in D[0,\infty)$ on $[0,t]$.

	The algorithm \eqref{eq:md} and \eqref{eq:weak0w} suggests that for each $t>0$,
	\begin{align}
		\bw_\gamma(t)&=\arg\min_{\bu\in\R^d}\{\Psi_\gamma(t,\bu)-\bu^\top\bv_\gamma(t)\}=\bw_{\lfloor t/\gamma\rfloor},\label{eq:step4tmp1_at}\\
		\bw(t)&=\arg\min_{\bu\in\R^d}\{\Psi(t,\bu)-\bu^\top\bv(t)\}.
	\end{align}
	We will apply Lemma \ref{lem:kato1} to show $\br_T \bw_\gamma(t)\weakto \br_T \bw(t)$ in $(D[0,T))^d$ for any $T>0$. 
	
	To verify condition (a) in Lemma \ref{lem:kato1}, we note that $t\mapsto\Psi_\gamma(t,\bu)-\bu^\top\bv_\gamma(t)$ is in $D[0,T]$ for each $\bu$ since $\bv_\gamma(t)\in D([0,T])^d$ and $\Psi_\gamma(t,\bu)\in D[0,T]$ for each $\bu$. Furthermore, $\bu\mapsto\Psi_\gamma(t,\bu)-\bu^\top\bv_\gamma(t)$ is $(\beta,q)$-e.u.c. in $\bu$ for each $t$ by the $(\beta,q)$-e.u.c. of $\Psi_{\lfloor t/\gamma\rfloor,\gamma}$ and Lemma \ref{lem:eucchar}. On the other hand, $t\mapsto\Psi(t,\bu)-\bu^\top\bv(t)$ is continuous in $t$ for each $\bu$ since $\bv$ is an a.s. continuous process,  
	and $\Psi(\cdot,\bu)$ is continuous on $[0,\infty)$ by fact \ref{as:Plim0}, and $\bu\mapsto\Psi(t,\bu)-\bu^\top\bv(t)$ is $\beta$-e.s.c. with respect to $t$ by Lemma \ref{lem:escreg}(ii) and Lemma \ref{lem:eucchar}. Thus, condition (a) in Lemma \ref{lem:kato1} holds. 

	For the condition (b) in Lemma \ref{lem:kato1}, we note that $\bu\mapsto\Psi(t,\bu)-\bu^\top\bv(t)$ is $\beta$-e.s.c. with respect to $t$ by Lemma \ref{lem:escreg}(ii) and Lemma \ref{lem:eucchar}, so Lemma \ref{lem:uc}(b) implies that the minimizer is unique for each $t\in[0,T]$. 

	For the condition (c) in Lemma \ref{lem:kato1}, from \eqref{eq:step4tmp1_at} it is clear that $\bw_\gamma(t)\in D([0,\infty))^d$; Lemma \ref{lem:conticonj} [conditions verified by facts \ref{as:Plim0}, \ref{as:Plim2} and Lemma \ref{lem:escreg}(ii)] suggests that $\nabla\Psi^*$ is continuous on $[0,T]\times\R^d$ for any $T>0$, so $\bw(t)=\nabla\Psi^*(t,\cdot)\circ \bv(t)$ is almost surely continuous on $[0,T]$ because $\bv(t)$ is almost surely continuous.

	To verify the finite dimensional convergence \eqref{eq:kato1}, take arbitrary points $\{\bu_l\}_{l\leq L}\subset\R^d$ where $L\in\N$, and any bounded, nonnegative Lipschitz functions $f_1,...,f_L:D[0,T]\to\R$. By $\bv_\gamma\weakto\bv$ in $D([0,T])^d$ and the uniform convergence $\sup_{t\in[0,T]}|\Psi_\gamma(t,\bu)-\Psi(t,\bu)|$ for every $\bu$. We have 
	\begin{align*}
		\int \prod_{l=1}^L f_l\big(\Psi_\gamma(t,\bu_l)-\bu_l^\top\bv_\gamma\big) dP_{\bv_\gamma} \to \prod_{l=1}^L f_l\big(\Psi(\cdot,\bu_l)-\bu_l^\top\bv\big) 
	\end{align*} 
	by the boundedness of $f_1,...,f_L$ and the definition of the weak convergence of each $\Psi_\gamma(\cdot,\bu_l)-\bu_l^\top\bv_\gamma$. Hence, the desired joint convergence \eqref{eq:kato1} follows by Corollary 1.4.5 on page 31 of \cite{VW96}. 

	The proof is complete by invoking Lemma \ref{lem:kato1}.	
		
\end{itemize}
\hfill$\qed$

\subsection{Proofs for Section \ref{sec:disdy}}\label{sec:pfdisdy}

\begin{proof}[Proof of Lemma \ref{lem:exprW}]
	It is enough to show that for each $t$,
	\begin{align*}
		0\in \partial\big(\tP_\gamma(t,\bu)-\langle \bu,\VV_\gamma(t)\rangle\big)\big|_{\bu=\WW_\gamma(t)}. 
	\end{align*}
	First we observe that
	\begin{align*}
		\tP_\gamma(t,\bu)-\langle \bu,\VV_\gamma(t)\rangle = \gamma^{-1}\big(\Psi_\gamma(t,\bw(t)+\sqrt{\gamma}\bu)-\Psi_\gamma(t,\bw(t)) -\langle \sqrt{\gamma}\bu,\bv_{\lfloor t/\gamma\rfloor}\rangle\big).
	\end{align*}
	
	Note that from Lemma \ref{lem:sumsdif} and \ref{lem:chsdif}, and the fact that $\langle \sqrt{\gamma}\bu,\bv_{\lfloor t/\gamma\rfloor}\rangle$ is differentiable w.r.t. $\bu$ and $\sqrt{\gamma}\bu$ is an invertible linear transformation, for any $\bu\in$dom$\Psi_\gamma(t,\bw(t)+\sqrt{\gamma}\cdot)=\R^d$,
	\begin{align*}
		\partial \big(\gamma^{-1}\big(\Psi_\gamma(t,\bw(t)&+\sqrt{\gamma}\bu)-\Psi_\gamma(t,\bw(t)) -\langle \sqrt{\gamma}\bu,\bv_{\lfloor t/\gamma\rfloor}\rangle\big)\big) \\
		&= \partial (\gamma^{-1}\Psi_\gamma(t,\bw(t)+\sqrt{\gamma}\bu)) - \gamma^{-1/2} \bv_{\lfloor t/\gamma\rfloor} \\
		&= \gamma^{-1/2}\partial \Psi_\gamma(t,\bw(t)+\sqrt{\gamma}\bu)- \gamma^{-1/2} \bv_{\lfloor t/\gamma\rfloor}.
	\end{align*}
	Take $\bu=\WW_\gamma(t)$, we need to show $0\in \gamma^{-1/2}\partial \Psi_\gamma(t,\bw_{\lfloor t/\gamma\rfloor})- \gamma^{-1/2} \bv_{\lfloor t/\gamma\rfloor}$. By \eqref{eq:md}, $\bw_{\lfloor t/\gamma\rfloor}=\nabla\Psi_{\lfloor t/\gamma\rfloor,\gamma}^*(\bv_{\lfloor t/\gamma\rfloor})$, so $\bv_{\lfloor t/\gamma\rfloor}\in\partial \Psi_\gamma(t,\bw_{\lfloor t/\gamma\rfloor})$ by Lemma \ref{lem:uc}(c). Hence, $0\in \gamma^{-1/2}\partial \Psi_\gamma(t,\bw_{\lfloor t/\gamma\rfloor})- \gamma^{-1/2} \bv_{\lfloor t/\gamma\rfloor}$ and the proof is complete.
	\end{proof}

\subsubsection{Proof of Theorem \ref{th:dd}}\label{sec:pfweak}


Observe that from \eqref{eq:md} and Assumption \ref{as:M},
	\begin{align}
		\bv_{n+1} &= \bv_n - \gamma \nabla f_{n+1}(\bw_n)
				= \bv_n - \gamma\DD_{n+1}^\gamma-\gamma G(\bw_n).\label{eq:wtmp1}
	\end{align}
where $\DD_{n+1}^\gamma := \nabla f_{n+1}(\bw_n)-G(\bw_n)$ is a sequence of martingale difference. Since $\bv_0$ is a constant vector, \eqref{eq:wtmp1} implies $\VV_\gamma(t)$ defined in \eqref{eq:defvw} satisfies
	\begin{align}
		\VV_\gamma(t) = -\MM_\gamma(t) &- \gamma^{-1/2}\Big(\gamma\sum_{k=0}^{\lfloor t/\gamma\rfloor-1} G(\bw_k)-\underbrace{\int_0^t G(\bw(s))ds}_{\mbox{\scriptsize from }  \bv(t)}\Big)\label{eq:recurv}
	\end{align}
	where $\MM_\gamma(t):=\sqrt{\gamma}\sum_{k=0}^{\lfloor t/\gamma\rfloor-1} \DD_{k+1}^\gamma$ is a martingale. 
	
	Define $\tilde\tau_\gamma^K:=\inf\{t:\|\VV_\gamma(t)\|_2>K\}$. It will be shown in \ref{ws:M} that $\MM_\gamma(\cdot\wedge\tilde\tau_\gamma^K)$ is relatively compact in $D([0,\infty))^d$. 
	
	In \ref{ws:R}, we show that {almost surely} for any $t>0$,
	\begin{align}
		\gamma^{-1/2}\Big(\gamma\sum_{k=0}^{\lfloor t\wedge \tilde\tau_\gamma^K/\gamma\rfloor-1} G(\bw_k)-\int_0^{t\wedge \tilde\tau_\gamma^K} G(\bw(s))ds\Big) = R_\gamma(t\wedge \tilde\tau_\gamma^K) + o(1).\label{eq:appRg}
	\end{align}
	where the $o(1)$ holds almost surely as $\gamma\to 0$ and
	\begin{align}
		R_\gamma(t):=\int_0^t \nabla G\big(\bw(s)\big) \cdot \WW_\gamma(s)ds, \label{eq:defR}
	\end{align}
	and the first term on the right hand side of \eqref{eq:appRg} is relatively compact.

	 By \ref{ws:M} and \ref{ws:R}, for every compact $K\subset\R^d$, all terms on the right of \eqref{eq:recurv} converge to an element in $C([0,\infty))^d$ respectively, which is a Polish space. This implies that both $\MM_\gamma(\cdot\wedge\tilde\tau_\gamma^K)$ and $R_\gamma(\cdot\wedge \tilde\tau_\gamma^K)$ are asymptotically tight by the discussion on page 23 of \cite{VW96}. It follows that $\VV_\gamma(\cdot\wedge\tilde\tau_\gamma^K)$ is relative compact by Lemma 1.4.3 of \cite{VW96} and Problem 3.22(c) on p.153 of \cite{EK86}. 

\bigskip
Now we identify the limit. For any $K$ and any convergent subsequence $\VV_{\gamma'}(\cdot\wedge\tilde\tau_{\gamma'}^K)$ with limit $\VV_{K}'$ in $D([0,\infty))^d$, by the continuous mapping theorem and the fact that $\inf\{t:\cdot>K\}$ is a continuous,
\begin{align}
	(\VV_{\gamma'}(\cdot\wedge\tilde\tau_{\gamma'}^K),\tilde\tau_{\gamma'}^K)\weakto(\VV_{K}'(\cdot\wedge\tilde\tau^{K'}),\tilde\tau^{K'}) \quad\mbox{in }D([0,\infty))^d\times \R,\label{eq:jointstop}
\end{align}
for all but countably many $K$ at which $\tilde\tau^{K'}:=\inf\{t:\|\VV_{K}'(t)\|_2>K\}$ is discontinuous, where $\VV_{K}'(\cdot\wedge\tilde\tau^{K'})$ is continuous a.s. on $[0,\infty)$ by the almost sure continuity on $[0,\infty)$ of the limit of $\MM_\gamma(\cdot\wedge\tilde\tau_{\gamma'}^K)$ (from \ref{ws:M}) and $R_\gamma(\cdot\wedge\tilde\tau_{\gamma'}^K)$ (from \ref{ws:R}). 

	 A consequence followed by \eqref{eq:jointstop} and \ref{ws:cmt} is that 
\begin{align}
	\WW_{\gamma'}(\cdot\wedge\tilde\tau_{\gamma'}^K)\weakto\WW_{K}'(\cdot\wedge\tilde\tau^{K'}):=\nabla\tPs(\cdot,\VV_{K}'(\cdot\wedge\tilde\tau^{K'})) \quad\mbox{in $D([0,\infty))^d$}.\label{eq:Wconv}
\end{align}	  
where $\WW_{K}'(\cdot\wedge\tilde\tau^{K'})\in C([0,\infty))^d$ almost surely for all but countably many $K$, since $\VV_{K}'(\cdot\wedge\tilde\tau^{K'})\in C([0,\infty))^d$ a.s. and $\nabla\tPs$ is continuous by Lemma \ref{lem:conticonj}, where the conditions of Lemma \ref{lem:conticonj} are verified by \ref{as:tPlim1}, \ref{as:tPlim2} and Lemma \ref{lem:escreg}(iii), 

We will show that $\VV_{K}'$ satisfies
	\begin{align}
		\VV_{K}'(t\wedge \tilde\tau^{K'})&=-\int_0^{t\wedge \tilde\tau^{K'}} \nabla G(\bw(s)) \cdot \nabla\tPs(s,\VV_{K}'(s))ds+\MM'(t\wedge\tilde\tau^{K'}), \quad\mbox{for all $t\geq 0$},\label{eq:weakvprime}
	\end{align}
To see \eqref{eq:weakvprime}, first note that $\MM_{\gamma'}(\cdot\wedge\tilde\tau_{\gamma'}^K)\weakto\MM'(\cdot\wedge\tilde\tau^{K'})$ in $D([0,\infty))^d$ as $\MM'$ is continuous martingale a.s. by \ref{ws:M}. Second, if
\begin{align}
	\int_0^{\cdot\wedge \tilde\tau_{\gamma'}^K} \nabla G\big(\bw(s)\big) \cdot \WW_{\gamma'}(s)ds\weakto \int_0^{\cdot\wedge \tilde\tau^{K'}} \nabla G\big(\bw(s)\big) \cdot \WW_K'(s)ds \quad\mbox{in }  D([0,\infty))^d, \label{eq:wcinte}
\end{align}
then \eqref{eq:weakvprime} follows by the continuous mapping theorem and the continuity of the additive operator by Theorem 4.1 on page 79 of \cite{W80} (note that all terms on the right hand side of \eqref{eq:weakvprime} are continuous in $t$ almost surely). 

To show \eqref{eq:wcinte}, 
consider two processes $\cdot\wedge\tilde\tau_{\gamma'}^K$ and $\cdot\wedge\tilde\tau^{K'}$ in $D[0,\infty)$. 
Using the fact that $(s,t)\mapsto s\wedge t=(1/2)(s+t-|s-t|)$ is continuous in joint topology, finite dimensional convergence and tightness can be verified and we have $\cdot\wedge\tilde\tau_{\gamma'}^K\weakto\cdot\wedge\tilde\tau^{K'}$ in $D[0,\infty)$. $\nabla G\big(\bw(\cdot)\big)$ is continuous by \ref{as:L} and Theorem \ref{th:at}, so continuous mapping theorem gives $\nabla G\big(\bw(\cdot\wedge\tilde\tau_{\gamma'}^K)\big)\weakto \nabla G\big(\bw(\cdot\wedge\tilde\tau^{K'})\big)$ in $(D[0,\infty))^{d\times d}$. Together with \eqref{eq:Wconv}, continuous mapping theorem and Theorem 4.2 of \cite{W80} [$G\big(\bw(\cdot\wedge\tilde\tau^{K'})\big)$ is continuous by Assumption \ref{as:L} and Theorem \ref{th:at}], we obtain 
\begin{align}
\nabla G\big(\bw(\cdot\wedge\tilde\tau_{\gamma'}^K)\big)\WW_{\gamma'}(\cdot\wedge\tilde\tau_{\gamma'}^K) \weakto \nabla G\big(\bw(\cdot\wedge\tilde\tau^{K'})\big)\WW_{K}'(\cdot\wedge\tilde\tau^{K'}) \mbox{ in $D([0,\infty))^d$}. \label{eq:integrandweakto}	
\end{align} 
Using \eqref{eq:integrandweakto} and Proposition 7.27 on page 118-119 of \cite{K08} [taking $G=G_\gamma=$Lebesgue measure in that proposition], we have for every $T>0$,
\begin{align}
	\int_0^{\cdot} &\nabla G\big(\bw(s\wedge \tilde\tau_{\gamma'}^K)\big) \cdot \WW_{\gamma'}(s\wedge \tilde\tau_{\gamma'}^K)ds \weakto \int_0^{\cdot} \nabla G\big(\bw(s\wedge \tilde\tau^{K'})\big) \cdot \WW_{K}'(s\wedge \tilde\tau^{K'})ds,\notag\\
	 &\quad\quad\quad\mbox{ in }D([0,T])^d, \label{eq:wcinte1}
\end{align}
but this implies the convergence in $D([0,\infty))^d$ by Lemma \ref{lem:fint}. Hence, \eqref{eq:wcinte} is proved by the fact that the limit on the right hand side of \eqref{eq:wcinte1} is continuous and $\cdot\wedge\tilde\tau_{\gamma'}^K\weakto\cdot\wedge\tilde\tau^{K'}$ in $D[0,\infty)$.

%
If the stochastic integral equation \eqref{eq:weakv} allows for the unique solution $\VV$ (we will show its uniqueness below), we claim that for any $K>0$,
\begin{align}
		(\VV_K'(\cdot\wedge\tilde\tau_{\gamma}^{K'}),\tilde\tau_{\gamma}^{K'})\stackrel{d}{=}(\VV(\cdot\wedge\tilde\tau^K),\tilde\tau^K),\label{eq:limitchar}
\end{align}
where $\VV_{K}'$ satisfies \eqref{eq:weakvprime} and $\tilde\tau^K:=\inf\{t:\|\VV(t)\|_2>K\}$. This claim will be shown later before we show \ref{ws:M} - \ref{ws:stopping}. Hence, combining \eqref{eq:limitchar} and \eqref{eq:jointstop}, we conclude that for all but countably many $K>0$,
\begin{align}
	(\VV_{\gamma}(\cdot\wedge\tilde\tau_{\gamma}^K),\tilde\tau_{\gamma}^K)\weakto(\VV(\cdot\wedge\tilde\tau^K),\tilde\tau^K) \quad\mbox{in }D([0,\infty))^d\times \R.\label{eq:wctrunc}
\end{align}
To prove the desired $\VV_\gamma\weakto\VV$ in $D[0,\infty))^d$ in \eqref{eq:weakv}, note that we have the estimations
\begin{align}
	\rho_d^\infty(\VV_\gamma(\cdot),\VV_\gamma(\cdot \wedge \tilde\tau_\gamma^K)) &\leq d e^{-\tilde\tau_\gamma^K},\label{eq:ridofK1}\\
	\rho_d^\infty(\VV(\cdot\wedge\tilde\tau^K),\VV) &\leq d e^{-\tilde\tau^K}. \label{eq:ridofK2}
\end{align}
By the above claim, the almost sure continuity of $\VV$ on $[0,\infty)$ by \ref{ws:sol} (on a rich enough probability space) ensures that $\tilde\tau^K\uparrow\infty$ almost surely as $K\uparrow\infty$ ($K$ increases along a countable sequence of compact sets), so $\rho_d^\infty(\VV_K',\VV)\stackrel{p}{\to} 0$ as $K\to\infty$. We will show in \ref{ws:stopping} that $\tilde\tau_\gamma^K\to \infty$ in probability as $K\to\infty$ and $\gamma\to 0$, and this implies that $\rho_d^\infty(\VV_\gamma(t),\VV_\gamma(t \wedge \tilde\tau_\gamma^K))\stackrel{p}{\to} 0$ as $K\to\infty$ and $\gamma\to 0$. The proof of \eqref{eq:weakv} in Theorem \ref{th:dd} is complete by Corollary 3.3.3 on p.110 of \cite{EK86}. 

Finally, we complet the proof of Theorem \ref{th:dd} by noting that \eqref{eq:weakw} follows by $\VV_\gamma\weakto\VV$ in $D[0,\infty))^d$ and \ref{ws:cmt}.

\bigskip
Now it is left to prove \eqref{eq:limitchar}. We will first establish the uniqueness of the stopped martingale problem $(\delta_0,\Ac,\{\bv:\|\bv\| < K\})$ (see Section 4.6 of \cite{EK86} for the details) with solution $\VV(\cdot\wedge\tilde\tau^K)$, where the martingale problem $(\delta_0,\Ac)$ is defined as follows: let $C_c^\infty(\R^d)$ be the set of infinite times differentiable functions on $\R^d$ with compact support,
\begin{align}\label{eq:mp}
	\begin{split}
		&\mbox{initial distribution: $\delta_0$, Dirac measure at 0}\\
		&\{(g,\Ac g): g\in C_c^\infty(\R^d)\},\mbox{ and }\\
		&\Ac = \frac{1}{2}\sum_{j,k=1}^d \Sigma_{jk}(\bw(t))\partial_j\partial_k + \sum_{j=1}^d \nabla G_{j\cdot}(\bw(t))^\top\nabla\tPs(t,\bu)\partial_j.
	\end{split}
\end{align}
We will show this via the uniqueness of the solution of the SDE \eqref{eq:weakv}. The pathwise uniqueness of the solution $\VV$ of the stochastic integral equation \eqref{eq:weakv} is proved in \ref{ws:sol}, and this implies the distribution uniqueness by Theorem 5.3.6 on page 296 of \cite{EK86} [locally boundedness conditions are implied by Assumption \ref{as:L}]. 
By Corollary 5.3.4 on page 295 of \cite{EK86}, the distribution uniqueness of the SDE \eqref{eq:weakv} implies that the solution to the martingale problem $(\delta_0,\Ac)$ in \eqref{eq:mp} is unique with solution $\VV$ [the association between the SDE and the martingale problem is described in Eq. (3.3)-(3.5) on page 291 of \cite{EK86}]. This also implies that the stopped martingale problem $(\delta_0,\Ac,\{\bv:\|\bv\| < K\})$ is unique with solution $\VV(\cdot\wedge\tilde\tau^K)$ by Theorem 4.6.1 (and its proof) on page 216 of \cite{EK86}.

To show \eqref{eq:limitchar}, we will show that $\VV_{K}'$ also solves the stopped martingale problem $(\delta_0,\Ac,\{\bv:\|\bv\| < K\})$ and the desired results will follow by its uniqueness. That is, we need to show that for any $g\in C_c^\infty(\R^d)$,
\begin{align}
	g(\VV_K'(t\wedge \tilde\tau^{K'})) - \int_0^{t\wedge \tilde\tau^{K'}} \Ac g(\VV_K'(s))ds
\end{align}
is a martingale. To this goal, by the Ito's formula (Eq. (3.6) on page 292 of \cite{EK86}), it is sufficient to show that $\MM'(t\wedge\tilde\tau^{K'})$ in \eqref{eq:weakvprime} and 
\begin{align}
	\big\langle\MM_i'(t\wedge\tilde\tau^{K'}), \MM_j'(t\wedge\tilde\tau^{K'})\big\rangle - \int_0^{t\wedge \tilde\tau^{K'}} \Sigma_{ij}(\bw(s)) ds \label{eq:anglebracket1}
\end{align} 
are martingales for all $1\leq i,j\leq d$, where the $\langle\cdot,\cdot\rangle$ is defined on page 79 of \cite{EK86}. By \ref{ws:M}, $\MM'(t\wedge\tilde\tau^{K'})$ is shown to be a martingale. By Proposition 2.6.2 on p.79 of \cite{EK86}, $\langle\MM_i'(t\wedge\tilde\tau^{K'}), \MM_j'(t\wedge\tilde\tau^{K'})\rangle-\MM_i'(t\wedge\tilde\tau^{K'})\MM_j'(t\wedge\tilde\tau^{K'})$ is a martingale because $\MM_i'(t\wedge\tilde\tau^{K'})$ for each $1\leq i\leq d$ is a martingale. Therefore, we will check if the following process is a martingale:
\begin{align}
	\MM_i'(t\wedge\tilde\tau^{K'})\MM_j'(t\wedge\tilde\tau^{K'}) - \int_0^{t\wedge \tilde\tau^{K'}} \Sigma_{ij}(\bw(s)) ds. \label{eq:anglebracket2}
\end{align} 
If we show that the process $\{\MM_{\gamma',i}(t\wedge\tilde\tau^{K'})\MM_{\gamma',j}(t\wedge\tilde\tau^{K'})\}_{\gamma'}$ for every $1\leq i,j\leq d$ is uniformly integrable, then by \eqref{eq:Mmtgle}, the process in \eqref{eq:anglebracket2} is a martingale. This can be done by similar arguments in the proof of Theorem 7.1.4 case (b) on page 344-345 of \cite{EK86} (see a nice exposition on 299-301 of \cite{W07}), and we omit the details for brevity.

\bigskip
Before we formally show \ref{ws:M} - \ref{ws:stopping}, we note a fact that will be used repeatedly: For $K>0$, $\|\VV_\gamma(t)\|_2\leq K$ for all $t\leq \tilde\tau_\gamma^K$ and $\gamma$, applying Lemma \ref{lem:exprW} to get
\begin{align}
	&\sup_\gamma\sup_{k\leq \lfloor (t\wedge \tilde\tau_\gamma^K)/\gamma\rfloor -1}\big\|\WW_\gamma(k\gamma)\big\|_2 
	= \sup_\gamma\sup_{k\leq \lfloor (t\wedge \tilde\tau_\gamma^K)/\gamma\rfloor -1}\big\|\nabla\tPs_\gamma(k\gamma,\VV_\gamma(k\gamma))\big\|_2\leq
	\bar K_t,\label{eq:weak0}
\end{align}
where the last inequality follows by Lemma \ref{lem:cptpres}. The conditions there are verified by \ref{as:tPlim} and Lemma \ref{lem:escreg}(i) and (iii).

\begin{itemize}[itemindent=45pt,leftmargin=0pt]
\labitem{Step 1}{ws:M} \textbf{$\{\MM_\gamma(\cdot\wedge \tilde\tau_\gamma^K)\}_\gamma$ is relatively compact in $D([0,\infty))^d$, and any possible limit is a continuous martingale a.s.}

This step follows by the proof of Theorem 7.1.4 case (b) of \cite{EK86} (in particular, on page 344). We will verify the conditions of Theorem 7.1.4 case (b) of \cite{EK86}, and the desired properties follow. By the optional stopping theorem (see page 105 of \cite{K97}) and the fact that $\MM_\gamma(t)$ is a martingale, $\MM_\gamma(t\wedge \tilde\tau_\gamma^K)$ is a martingale. Without loss of generality, we will assume $\|\VV_\gamma(t)\|_2\leq K$ for all $\gamma$ and $t$, and drop $\tilde\tau_\gamma^K$ in $\MM_\gamma(t\wedge \tilde\tau_\gamma^K)$ for brevity. 

Define 
\begin{align}
	\ZZ_\gamma(t):=\MM_\gamma(t)\MM_\gamma(t)^\top-\gamma\sum_{k=0}^{\lfloor t/\gamma\rfloor-1} \Sigma(\bw_k).\label{eq:Mmtgle}
\end{align}
By the fact that $\MM_\gamma(t)$ is a sum of martingale differences, the law of iterative expectation and the orthogonality of martingale differences (e.g. page 250 of \cite{D05}) show that $\ZZ_\gamma(t)$ is a martingale with respect to the filtration $\Fc_{\lfloor t/\gamma\rfloor}$. Thus, Eq. (7.1.18) on page 340 of \cite{EK86} is proved. 

The hypothesis of this Theorem says that the solution $\bv(t)$ exists for $t\geq 0$, and Assumption \ref{as:L} implies that the solution $\bv(t)$ is unique. Hence, Theorem \ref{th:at}(b) and (c) implies that 
\begin{align}
\sup_{0\leq t\leq T}\|\bv_{\lfloor t/\gamma\rfloor}-\bv(t)\|_2\pto 0\mbox{ and }\sup_{0\leq t\leq T}\|\bw_{\lfloor t/\gamma\rfloor}-\bw(t)\|_2\pto 0 \mbox{ for any $T>0$.}\label{eq:ptosol}	
\end{align}
Recall that (see e.g. Theorem 1.6.2 on page 46 of \cite{D05}) to prove 
\begin{align}
	\Big|\gamma\sum_{k=0}^{\lfloor t/\gamma\rfloor-1} \Sigma_{ij}(\bw_k) - \int_0^t \Sigma_{ij}(\bw(s)) ds\Big| \pto 0,\mbox{ for each $t\geq 0$, $1\leq i\leq j \leq d$,}
\end{align}
it is enough to show that for any subsequence $\gamma'$ of $\gamma$, there exists a further subsequence $\gamma''$ such that 
\begin{align}
	\Big|\gamma''\sum_{k=0}^{\lfloor t/\gamma''\rfloor-1} \Sigma_{ij}(\bw_k) - \int_0^t \Sigma_{ij}(\bw(s)) ds\Big| \asto 0,\mbox{ for each $t\geq 0$, $1\leq i\leq j \leq d$,.}  \label{eq:step1temp1}
\end{align}
Indeed, by \eqref{eq:ptosol}, for any subsequence $\gamma'$, there is a subsequence $\gamma''$ such that $\sup_{0\leq s\leq t}\|\bw_{\lfloor s/\gamma''\rfloor}-\bw(s)\|_2\asto 0$. Along $\gamma''$, \eqref{eq:step1temp1} is proved by Lemma 1(a) of \cite{BKS93} using the continuity of $\Sigma(\cdot)$ in Assumption \ref{as:L}. Hence, Eq. (7.1.19) on page 340 of \cite{EK86} is shown. 

By the unique existence of the solution $\bv(t)$ of \eqref{eq:weak0v} and $\bw(t)$ for $t\geq 0$, $\bw(t)$ is continuous for $t\geq 0$ by \eqref{eq:weak0w}. We conclude that 
\begin{align}
	\forall T>0, \exists K_T>0 \mbox{ such that }\sup_{0\leq t\leq T}\|\bw(t)\|_2\leq K_T. 
	\label{eq:step1tmp2}
\end{align}

For any $T>0$, $i,j=1,...,d$ and $\delta$, recall that $\|\VV_\gamma(t)\|_2\leq K$ for all $t\leq T$ and $\gamma$ from the beginning of this step, so $\WW_\gamma(k\gamma)=\frac{\bw_k-\bw(k\gamma)}{\sqrt{\gamma}}\leq \bar K_T$ for all $k\leq\lfloor T/\gamma\rfloor$ and $\gamma>0$ by \eqref{eq:weak0} and the bound for $\bw(t)$ in \eqref{eq:step1tmp2},
\begin{align}
	\lim_{\gamma\to 0}\gamma\E\Big[\sup_{0\leq t\leq T}\Sigma_{ij}(\bw_{\lfloor t/\gamma\rfloor-1})\Big]&\leq \lim_{\gamma\to 0}\gamma\E\Big[\sup_{\|\bw\|_2\leq K_T+\sqrt{\gamma}\bar K_T}\Sigma_{ij}(\bw)\Big]\notag\\
	&\leq 4 \lim_{\gamma\to 0}\gamma\E\Big[\sup_{\|\bw\|_2\leq K_T+\sqrt{\gamma}\bar K_T}\big\|\nabla f\big(\bw,Z_n\big)\big\|_2^2\Big]=0,\label{eq:step1eq3}
\end{align}
where in the last equality of apply \eqref{eq:mom2bdd} in Assumption \ref{as:L}. Therefore, Eq. (7.1.16) on page 340 of \cite{EK86} is proved. 

The other condition in Eq. (7.1.17) on page 340 of \cite{EK86} is that the second moment of the maximum jump converges to 0:
\begin{align}
	\lim_{\gamma\to 0}\E\Big[\sup_{0\leq t\leq T}\big\|\MM_\gamma(t)-\MM_\gamma(t_-)\big\|_2^2\Big]=\lim_{\gamma\to 0}\gamma\E\Big[\sup_{0\leq t\leq T}\big\|\nabla f_{\lfloor t/\gamma\rfloor}(\bw_{\lfloor t/\gamma\rfloor-1})-G(\bw_{\lfloor t/\gamma\rfloor-1})\big\|_2^2\Big]=0.\label{eq:step1eq4}
\end{align}
This can be verified by similar arguments for proving \eqref{eq:step1eq3}, and the details are omitted. This also implies that any possible limit of $\MM_\gamma$ is a.s. continuous on $[0,T]$ for every $T>0$ by Theorem 13.4 on p.142 of \cite{B99}.

\labitem{Step 2}{ws:R} \textbf{$R_\gamma(\cdot\wedge\tilde\tau_K^\gamma)$ defined in \eqref{eq:defR} is relatively compact in $D([0,\infty))^d$, \eqref{eq:appRg} holds and all possible limits of $R_\gamma(\cdot\wedge \tilde\tau_\gamma^K)$ are continuous a.s.}

From the mean value theorem Theorem 4.2 on page 341 of \cite{L93}, 
	\begin{align}
		&\gamma^{-1/2}\Big(\gamma\sum_{k=0}^{\lfloor t\wedge \tilde\tau_\gamma^K/\gamma\rfloor-1} G(\bw_k)-\int_0^{t\wedge \tilde\tau_\gamma^K} G(\bw(s))ds \Big)\notag\\
		&= \gamma^{-1/2}\int_0^{t\wedge \tilde\tau_\gamma^K} \big\{G(\bw_{\lfloor s/\gamma\rfloor})-G(\bw(s))\big\}ds\notag\\
		&= \gamma^{-1/2}\int_0^{t\wedge \tilde\tau_\gamma^K} \big\{G(\bw(s)+\sqrt{\gamma}\WW_\gamma(s))-G(\bw(s))\big\}ds\notag\\
		&= \int_0^{t\wedge \tilde\tau_\gamma^K} \int_0^1 \nabla G\big(\bw(s)+u_s \sqrt{\gamma} \WW_\gamma(s)\big)du_s \cdot \WW_\gamma(s)ds\notag\\
		&= R_\gamma(t\wedge \tilde\tau_\gamma^K) + \underbrace{\int_0^{t\wedge \tilde\tau_\gamma^K} \Big\{\int_0^1 \nabla G\big(\bw(s)+u_s \sqrt{\gamma} \WW_\gamma(s)\big)du_s-\nabla G(\bw(s))\Big\} \cdot \WW_\gamma(s)ds}_{:=\hat R_\gamma(t\wedge\tilde\tau_\gamma^K)}.\label{eq:Rgexpr}
	\end{align}
	Now we show $\sup_{t\in[0,\infty)}\|\hat R_\gamma(t\wedge\tilde\tau_\gamma^K)\|_2=o(1)$ a.s. Note that $\|\WW_\gamma(s)\|_2\leq \bar K_t$ for all $s\leq t\wedge\tilde\tau_\gamma^K$ by \eqref{eq:weak0}. Moreover, $\nabla G(\bw(\cdot))$ is continuous on $[0,t]$, so $\sup_{\bw\in\Wc_t}\|\nabla G(\bw)\|_2\leq K_t$ for some constant $K_t>0$, where the compact set $\Wc_t \supset \{\bw(s):0\leq s\leq t\}+\sqrt{\bar\gamma} B_{\bar K_t}(0)$, where the positive value $\bar\gamma>\gamma$ for all $\gamma$, and $B_r(0)$ is a ball centered at the origin with radius $r>0$. Therefore, by the dominated convergence, for any $t>0$, almost surely,
	\begin{align*}
		\|\hat R_\gamma(t\wedge\tilde\tau_\gamma^K)\|_2 \leq K_t \int_0^{t\wedge \tilde\tau_\gamma^K} \int_0^1 \big\|\nabla G\big(\bw(s)+u_s \sqrt{\gamma} \WW_\gamma(s)\big)-\nabla G(\bw(s))\big\|_2du_s ds \to 0, \mbox{ }\gamma\to 0.
	\end{align*}
	
It is left to show that $R_\gamma(\cdot\wedge\tilde\tau_K^\gamma)$ is relatively compact. By exactly the same argument as in the last paragraph, we verify that
\begin{align*}
	\limsup_{\gamma\to 0}\sup_{0\leq t\leq T}\|R_\gamma(t\wedge\tilde\tau_K^\gamma)\|_2<\infty, \mbox{ for any $T>0$ a.s.},
\end{align*}
which is Eq. (27) on page 976 of \cite{BKS93}, and we can also verify that for any $T$,
\begin{align}
	\sup_{s\leq\delta}\|R_\gamma((t+s)\wedge \tau_K^\gamma)-R_\gamma(t\wedge \tau_K^\gamma)\|_2 \leq K_T \delta, \mbox{ a.s. for any $t<T$ and $\delta>0$.}\label{eq:step2temp1}
\end{align}
Hence, by similar argument as that after Eq. (30) on page 976 of \cite{BKS93}, Eq. (28) on page 976 of \cite{BKS93} is verified. Thus, the relative compactness follows from Lemma 3 of \cite{BKS93}.

By similar argument for verifying \eqref{eq:step2temp1}, we can show that the second moment of the maximum jump
$$
\lim_{\gamma\to 0}\E\Big[\sup_{0\leq t\leq T}\big\|R_\gamma(t\wedge\tilde\tau_K^\gamma)-R_\gamma(t_-\wedge\tilde\tau_K^\gamma)\big\|_2^2\Big]=0.
$$
Therefore, all possible limits of $R_\gamma(\cdot\wedge \tilde\tau_\gamma^K)$ are continuous a.s. on $[0,T]$ for every $T>0$ by Theorem 13.4 of \cite{B99}.

\labitem{Step 3}{ws:cmt} \textbf{$\WW_\gamma\weakto\WW:=\nabla\tP(\cdot,\VV(\cdot))$ in $D([0,\infty))^d$, for any sequence $\gamma$ such that $\VV_\gamma\weakto\VV$ in $D([0,\infty))^d$ where $\VV\in C([0,\infty))^d$.}

By Lemma \ref{lem:fint}, it is enough to show the restricted $\br_T \WW_\gamma(t)\weakto \br_T \WW(t)$ in $(D[0,T))^d$ for any $T>0$, where $\br_t:D([0,\infty))^d\to D([0,T])^d$ is given by $\br_t \bX \mapsto (r_t X_1,...,r_t X_d)$, and $r_t X_j$ is the restriction of $X_j\in D[0,\infty)$ on $[0,t]$.

Lemma \ref{lem:exprW} and \eqref{eq:weakw} suggests that for each $t>0$,
\begin{align}
	\WW_\gamma(t)&=\arg\min_{\bu\in\R^d}\{\tP_\gamma(t,\bu)-\bu^\top\VV_\gamma(t)\}=\frac{\bw_{\lfloor t/\gamma\rfloor}-\bw(t)}{\sqrt{\gamma}},\label{eq:step4tmp1}\\
	\WW(t)&=\arg\min_{\bu\in\R^d}\{\tP(t,\bu)-\bu^\top\VV(t)\}.
\end{align}
We will apply Lemma \ref{lem:kato1} to show $\br_T \WW_\gamma(t)\weakto \br_T \WW(t)$ in $(D[0,T))^d$ for any $T>0$. To verify condition (a) in Lemma \ref{lem:kato1}, we note that $t\mapsto\tP_\gamma(t,\bu)-\bu^\top\VV_\gamma(t)$ is in $D[0,T]$ for each $\bu$ since $\VV_\gamma(t)\in D([0,T])^d$ and $\tP_\gamma(t,\bu)\in D[0,T]$ by the hypothesis in this Theorem. Furthermore, $\bu\mapsto\tP_\gamma(t,\bu)-\bu^\top\VV_\gamma(t)$ is $\beta$-e.s.c. for each $t$ by Lemma \ref{lem:escreg}(i) and Lemma \ref{lem:eucchar}. On the other hand, $t\mapsto\tP(t,\bu)-\bu^\top\VV(t)$ is continuous for each $\bu$ since $\VV$ is an a.s. continuous process, and $\tP(\cdot,\bu)$ is continuous on $[0,\infty)$ for each $\bu$ by Condition \ref{as:tPlim1}, and $\bu\mapsto\tP(t,\bu)-\bu^\top\VV(t)$ is $\beta$-e.s.c. for each $t$ by Lemma \ref{lem:escreg}(iii) and Lemma \ref{lem:eucchar}. Thus, condition (a) in Lemma \ref{lem:kato1} holds. 

For the condition (b) in Lemma \ref{lem:kato1}, we note that $\bu\mapsto\tP(t,\bu)-\bu^\top\VV(t)$ is $\beta$-e.s.c. with respect to $t$ by Lemma \ref{lem:escreg}(iii), so Lemma \ref{lem:uc}(b) implies the minimizer is unique for each $t\in[0,T]$. 

For the condition (c) in Lemma \ref{lem:kato1}, from \eqref{eq:step4tmp1} it is clear that $\WW_\gamma(t)\in D([0,\infty))^d$; Lemma \ref{lem:conticonj} suggests that $\nabla\tPs$ is continuous on $[0,T]\times\R^d$ for any $T>0$, where the conditions of Lemma \ref{lem:conticonj} are verified by \ref{as:tPlim1}, \ref{as:tPlim2} and Lemma \ref{lem:escreg}(iii). Hence, $\WW(t)=\nabla\tPs(t,\cdot)\circ \VV(t)$ is almost surely continuous on $[0,T]$ because $\VV(t)$ is almost surely continuous.

To verify the finite dimensional convergence \eqref{eq:kato1}, take arbitrary points $\{\bu_l\}_{l\leq L}\subset\R^d$ where $L\in\N$, and any bounded, nonnegative Lipschitz functions $f_1,...,f_L:D[0,T]\to\R$. By $\VV_\gamma\weakto\VV$ in $D([0,T])^d$ and the uniform convergence $\sup_{t\in[0,T]}|\tilde\Psi_\gamma(t,\bu)-\tilde\Psi(t,\bu)|$ for every $\bu$, we have 
\begin{align*}
	\int \prod_{l=1}^L f_l\big(\tP_\gamma(\cdot,\bu_l)-\bu_l^\top\VV_\gamma\big) dP_{\VV_\gamma} \to \int \prod_{l=1}^L f_l\big(\tP(\cdot,\bu_l)-\bu_l^\top\VV\big) dP_{\VV}
\end{align*} 
by the boundedness of $f_1,...,f_L$ and the definition of weak convergence of each $\tP_\gamma(\cdot,\bu_l)-\bu_l^\top\VV_\gamma$. Hence, the desired joint convergence \eqref{eq:kato1} follows by Corollary 1.4.5 on page 31 of \cite{VW96}. 

The proof is complete by invoking Lemma \ref{lem:kato1}.

\labitem{Step 4}{ws:sol} \textbf{There exists a weak solution $\VV(t)$ of \eqref{eq:weakv} which is pathwise unique.}

We will show below that there exist a weak solution of \eqref{eq:weakv} which is pathwise unique (see the definitions on, e.g. page 300-301 of \cite{KS98}). This yield the existence of a strong solution in $C([0,\infty))^d$ on a rich enough probability space by Corollary 5.3.23 on page 310 of \cite{KS98}. Since the solution is continuous on $[0,\infty)$, this suggests that $\VV(t)$ does not explode in finite time almost surely and $\tilde\tau^K\to\infty$ a.s. as $K\to\R^d$.

To show the existence of the weak solution, it is enough to verify the Lipschitz condition Eq. (3.35) and linear growth condition Eq. (3.34) on page 300 of \cite{EK86} for each $T>0$. Since $\nabla G(\cdot)$ is continuous from Assumption \ref{as:L}, for any $T>0$, by \eqref{eq:step1tmp2}
	\begin{align}
		\sup_{0\leq t\leq T}\|\nabla G(\bw(t))\|_2 \leq \sup_{\|\bw\|\leq K_T}\|\nabla G(\bw)\|_2<C_T,\label{eq:bddG}
	\end{align}
for some $C_T>0$. By Lemma \ref{lem:escreg}(iii) and Lemma \ref{lem:uc}(a), $\nabla G(\bw(t)) \cdot \nabla\tPs(t,\cdot)$ is uniformly Lipschitz continuous in $t$. Hence, the Lipschitz condition Eq. (3.35) on page 300 of \cite{EK86} holds. 

Next, recall that $\nabla\tPs(t,\cdot)$ is uniformly Lipschitz continuous in $t$ by Lemma \ref{lem:escreg}(iii) and Lemma \ref{lem:uc}(a), for any $0\leq t\leq T$,
\begin{align}
	\|\nabla\tPs(t,\bv)\|_2 \leq \beta^{-1}\|\bv\|_2+\|\nabla\tPs(t,0)\|_2\leq \|\bv\|_2+C_T,\label{eq:wbdd0}
\end{align}
where $C_T>0$ satisfies $\max_{t\in[0,T]}\|\nabla\tPs(t,0)\|_2\leq C_T$ since $\nabla\tPs(t,0)=\arg\,\min_{\bw\in\R^d}\tP(t,\bw)$ is continuous on $[0,T]$ for any $T>0$ by Lemma \ref{lem:conticonj}. 

On the other hand, by H\"older's inequality and \eqref{eq:step1tmp2},
\begin{align*}
\sup_{0\leq t\leq T}\|\Sigma(\bw(t))\|_2 &\leq \sup_{\|\bw\|_2\leq K_T} \|\Sigma(\bw)\|_2 \\
&\leq  \E\Big[\sup_{\bw:\|\bw\|_2\leq K_T}\big\|\nabla f\big(\bw,Z_n\big)\big\|_2^2\Big] + 3\sup_{\|\bw\|_2\leq K_T} \|G(\bw)\|_2^2\\
&<C_T',	
\end{align*}
for some $C_T'>0$ by \eqref{eq:mom2bdd} in Assumption \ref{as:L} and the continuity of $G$ on $\R^d$ in Assumption \ref{as:M}. Hence, the linear growth condition Eq. (3.34) on page 300 of \cite{EK86} holds. By Theorem 5.3.11 on page 300 of \cite{EK86}, the weak solution of \eqref{eq:weakv} exists.

The pathwise uniqueness follows straightforwardly from Theorem 5.3.7 on page 297 of \cite{EK86} and the Lipschitz property of $\nabla G(\bw(t)) \cdot \nabla\tPs(t,\cdot)$ shown above. Hence, the proof for this step is complete.

\labitem{Step 5}{ws:stopping} \textbf{$\tilde\tau_\gamma^K\to\infty$ in probability as $\gamma\to 0$, $K\to \infty$.}

Recall from \eqref{eq:wctrunc} that $\VV_\gamma(\cdot\wedge\tilde\tau_\gamma^K)\rightsquigarrow \VV(\cdot\wedge\tilde\tau^K)$ in $D([0,\infty))^d$ for all but countably many $K$ (pointwisely for each of them), if $T<\tilde\tau_\gamma^K$, we have $\VV_\gamma \rightsquigarrow \VV$ in $D([0,T])^d$ (with the sigma field defined with metric $\rho_{d,\circ}^T$ in Section \ref{sec:argpr}), and by the continuous mapping theorem we also have
\begin{align}
\sup_{0\leq t\leq T}\|\VV_\gamma(t)\|_2 \rightsquigarrow \sup_{0\leq t\leq T}\|\VV(t)\|_2 \quad\mbox{as }\gamma\to 0,\label{eq1:step5}
\end{align}
where $\sup_{0\leq t\leq T}\|\VV(t)\|_2$ is a continuous random variable (r.v.) because $\VV(t)$ is a Gaussian process with a.s. continuous sample path. 

Therefore, for any $T>0$, for all but countably many $K$ that $\VV_\gamma(\cdot\wedge\tilde\tau_\gamma^K)\rightsquigarrow \VV(\cdot\wedge\tilde\tau^K)$ in $D([0,\infty))^d$,
$$
P(\tilde\tau_\gamma^K>T)=P\big(\sup_{0 \leq t\leq T}\|\VV_\gamma(t)\|_2<K\big) = P\big(\sup_{0 \leq t\leq T}\|\VV(t)\|_2<K\big) + \Rc(K)
$$
where $\Rc(K) = P(\sup_{0\leq t\leq T}\|\VV_\gamma(t)\|_2<K)-P(\sup_{0\leq t\leq T}\|\VV_\gamma(t)\|_2<K)$ satisfies $\sup_{K\in\mathbb R_+}|\Rc(K)|=o(1)$ as $\gamma\to 0$ because \eqref{eq1:step5} implies {uniform convergence of CDFs} of $\sup_{0\leq t\leq T}\|\VV_\gamma(t)\|_2$ and $\sup_{0\leq t\leq T}\|\VV_\gamma(t)\|_2$ since the cdf of $\sup_{0\leq t\leq T}\|\VV_\gamma(t)\|_2$ is continuous. 

As $\gamma\to 0$ and $K\to\infty$ along an increasing countable sequence, since $\sup_{K\in\mathbb R_+}\Rc(K)=o(1)$, we obtain that $P(\tilde\tau_\gamma^K>T)\to 1$ (since $P(\sup_{0 \leq t\leq T}\|\VV(t)\|_2<K)\to 1$ as $K\to\infty$ along an increasing countable sequence) for any $T>0$. This verifies that $\tilde\tau_\gamma^K\stackrel{p}{\to}\infty$ as $r$ and $n$ $\to\infty$.
\end{itemize}

\subsection{Proof for \eqref{eq:avg} in Corollary \ref{th:avg}}\label{sec:pfavg}

Using the expression in \eqref{eq:barwT},
\begin{align*}
		\bar\bw_{\lfloor T/\gamma\rfloor}-T^{-1}\int_0^T\bw(s)ds
		 &= \frac{1}{\gamma\lfloor T/\gamma\rfloor}\int_0^T \bw_{\lfloor s/\gamma\rfloor} ds - \frac{1}{T}\int_0^T\bw_{\lfloor s/\gamma\rfloor} ds\\ 
		 &\quad+ \frac{1}{T}\int_0^T\bw_{\lfloor s/\gamma\rfloor} ds - \frac{1}{T}\int_0^T \bw(s)ds.
\end{align*}
By continuous mapping theorem (Proposition 7.27 of \cite{K08}),
\begin{align*}
	\gamma^{-1/2}\Big(\frac{1}{T}\int_0^T\bw_{\lfloor s/\gamma\rfloor} ds - \frac{1}{T}\int_0^T \bw(s)ds\Big) = \frac{1}{T}\int_0^T \frac{\bw_{\lfloor s/\gamma\rfloor}-\bw(s)}{\sqrt{\gamma}}ds\weakto \frac{1}{T}\int_0^T \WW(s) ds.
\end{align*}
It is left to show that
\begin{align}
	\frac{1}{\gamma\lfloor T/\gamma\rfloor}\int_0^T \bw_{\lfloor s/\gamma\rfloor} ds - \frac{1}{T}\int_0^T\bw_{\lfloor s/\gamma\rfloor} ds = \Big(\frac{1}{\gamma\lfloor T/\gamma\rfloor}-\frac{1}{T}\Big) \Big(\int_0^T \bw_{\lfloor s/\gamma\rfloor} ds\Big) =o_p(\sqrt{\gamma}).\label{eq:pfavg1}
\end{align}
To show \eqref{eq:pfavg1}, we will show that $\gamma^{-1/2}\big(\frac{1}{\gamma\lfloor T/\gamma\rfloor}-\frac{1}{T}\big)=o(1)$ and $\int_0^T \bw_{\lfloor s/\gamma\rfloor} ds=O_p(1)$. To see the former, note that 
\begin{align*}
	\gamma^{-1/2}\Big|\frac{1}{\gamma\lfloor T/\gamma\rfloor}-\frac{1}{T}\Big| = \frac{\gamma^{-1/2}}{T\lfloor T/\gamma\rfloor}\big|T/\gamma-\lfloor T/\gamma\rfloor\big| \leq \frac{\gamma^{-1/2}}{T(T/\gamma-1)} = \frac{\gamma^{1/2}}{T(T-\gamma)} = o(1).
\end{align*}

To see the latter, since $\tau^K=\inf\{t:\|\bv(t)\|_2\geq K\}\to\infty$ in probability as $K\to\infty$ since $\bv(t)$ exists for all $t>0$, we can drop $K$ in $\bw_\gamma^K$ in Theorem \ref{th:at}(b) and obtain that for any $T>0$,
\begin{align}
	\sup_{0\leq t < T}\|\bw_\gamma(t)-\bw(t)\|_2\to 0 \quad\mbox{in probability}. \label{eq:pfavg2}
\end{align}
Observe that by \eqref{eq:pfavg2},
\begin{align}
	\Big|\int_0^T \bw_{\lfloor s/\gamma\rfloor} ds\Big| \leq \int_0^T \sup_{0\leq t < T}\|\bw_\gamma(t)-\bw(t)\|_2 ds + \Big|\int_0^T \bw(s) ds\Big| = o_p(1) + \Big|\int_0^T \bw(s) ds\Big|.
\end{align}
By Theorem \ref{lem:conticonj}, $\bw(\cdot)=\nabla\Psi^*(\cdot,\bv(\cdot))$ continuous since $\bv(\cdot)$ is continuous. Hence, $\big|\int_0^T \bw(s)ds\big|=O(1)$, and the proof is complete.

\bigskip
\begin{proof}[Proof of Remark \ref{rem:excess}]
	Since $\frac{1}{T}\int_0^T \WW(s) ds=O_p(1/\sqrt{T})$, Corollary \ref{th:avg} suggests that 
		\begin{align}
			\bar\bw_{\lfloor T/\gamma\rfloor}-T^{-1}\int_0^T\bw(s)ds = O_p\bigg(\sqrt{\frac{\gamma}{T}}\bigg) \label{eq:avgerr}
		\end{align}
		Since $f_0$ is globally Lipschitz, by \eqref{eq:avgerr},
		\begin{align}
		f_0(\bar\bw_{\lfloor T/\gamma\rfloor})-f_0\bigg(\frac{1}{T}\int_0^T\bw(s)ds\bigg)=O_p\bigg(\sqrt{\frac{\gamma}{T}}\bigg).\label{eq:excess1}	
		\end{align}
		On the other hand, proofs for Eq. (7) from \cite{RB12} give 
		\begin{align}
		f_0\bigg(\frac{1}{T}\int_0^T\bw(s)ds\bigg)-\inf_{\bw} f_0(\bw)\leq \frac{\mbox{Breg}_\Psi(\bw(0),\bw^*)}{2T}, \label{eq:excess2}
		\end{align}
		where $\mbox{Breg}_\Psi(\bx,\by)=\Psi(\bx)-\Psi(\by)-\langle\nabla\Psi(\by),\bx-\by\rangle$ is the Bregman divergence induced by $\Psi$. Combining \eqref{eq:excess1} and \eqref{eq:excess2}, we obtain a bound for the excess risk
		\begin{align*}
			f_0(\bar\bw_{\lfloor T/\gamma\rfloor}) - \inf_{\bw} f_0(\bw) = O_p\bigg(\frac{\mbox{Breg}_\Psi(\bw_0,\bw^*) \vee \sqrt{\gamma}}{\sqrt{T}}\bigg).
		\end{align*}
\end{proof}

\subsection{Proofs for Section \ref{sec:rdal1}}\label{sec:pfl1gen}

\begin{proof}[Proof of Corollary \ref{cor:l1at}]
	$F(\bw)=\frac{1}{2}\|\bw\|_2^2$ is strongly convex with factor 1, and $\Pc(\bw)=\|\bw\|_1$ is convex. The $\ell_2$ and $\ell_1$ norms are continuous. Hence, \ref{as:R} holds. The desired result follows by Theorem \ref{th:at}(c). 
	\end{proof}

\bigskip
\begin{proof}[Proof of Theorem \ref{th:biasrda}]
	By hypothesis, the data $Z_n=(X_n,Y_n)$ follows a linear regression model in \eqref{eq:modrdal1}. The least square loss is $f(\bw;Z_n)=(Y_n-X_n \bw)^2/2$, with gradient $\nabla f(\bw;Z_n)=-X_n(Y_n-X_n \bw)$ and $G(\bw)=H(\bw-\bw^*)$.
	
	Note that \ref{as:S} and \ref{as:M} are satisfied under this model. Setting tuning function $g(n,\gamma)=c_0 n \gamma$ translates \eqref{eq:rda} to the framework of \eqref{eq:grda_gen}. \ref{as:Plim} is satisfied by the observation that
	\begin{align*}
		\lim_{\gamma\to 0}\sup_{t\in[0,T]}\big|c_0 \lfloor t/\gamma\rfloor \gamma -c_0 t\big| = 0, \quad \mbox{for any $T>0$}.
	\end{align*}
	So $g^\dagger(t)=c_0 t$. Corollary \ref{cor:l1at} implies that the mean trajectory
	\begin{align}
		\frac{d\bv}{dt} = -H \big(\nabla\Lc^*(t,\bv)-\bw^*\big), \quad \bv(0)=\bw_0, \label{eq:dv}
	\end{align}
	where $\nabla\Lc^*(t,\bv)=\sgn(\bv)\cdot\big(|\bv|-c_0 t\big)_+$, and $\bw(t)=\nabla\Lc^*(t,\bv(t))$. Because $H=\mbox{diag}(\sigma_1^2,\sigma_2^2,...,\sigma_d^2)$, \eqref{eq:dv} can de-couple, and each $v_j(t)$ in $\bv(t)=(v_1(t),v_2(t),...,v_d(t))$ satisfies
	\begin{align}
		\frac{d v_j}{dt} = -\sigma_j^2 \big(\nabla\Lc^*(t,v_j(t)) -w_j^*\big), \quad v_j(0)=w_{0,j}, \label{eq:dvj}
	\end{align}
	and $w_j(t)=\nabla\Lc^*(t,v_j(t))$.
	
	\begin{itemize}[itemindent=45pt,leftmargin=0pt]
	\labitem{Step 1}{brda1} \textbf{Solution of \eqref{eq:dv} exists on $[0,\infty)$}.
	Letting 
		\begin{align}
			\theta(t,\bv) := -\big(\nabla\Lc^*(t,\bv(t))-\bw^*\big).
		\end{align}
		Observe that for any $\bv,\bv'\in\R$, Lemma \ref{lem:uc} implies
		\begin{align*}
			\sup_{t\in[0,\infty)}\big|\nabla\Lc^*(t,\bv)-\nabla\Lc^*(t,\bv')\big| \leq |\bv-\bv'|,
		\end{align*}
		so $\theta(t,\bv)$ is globally Lipschitz uniformly in $t$. By Corollary 2.6 of \cite{T12}, there exists a global solution of \eqref{eq:dv}.

	\labitem{Step 2}{brda2} \textbf{Proof of the main statement}.
	
	We focus on $j\in\{k:w_k^\infty\neq 0\}$. For sufficiently large $T_0$ such that $\sgn(w_j(t))=\sgn(w_j^\infty)$ for all $t>T_0$, in order to maintain the equality $w_j(t)=\nabla\Lc^*(t,v_j(t))$ for $t>T_0$, the solution $v_j(t)$ (which exists by \ref{brda1}) must satisfy
	\begin{align}
		v_j(t) = w_j(t) + \sgn(w_j^\infty)c_0 t\quad\mbox{for $t>T_0$}.\label{eq:vsol}
	\end{align}
	By a change of variable $w_j(t)=v_j(t)-\sgn(w_j^\infty)c_0 t$, and \eqref{eq:dvj},
		\begin{align}
			\frac{dw_j}{dt} = -\sigma_j^2\big(w_j-w_j^*+\sgn(w_j^\infty) c_0/\sigma_j^2\big), \quad \bw(0)=\bw_0.\label{eq:dwj}
		\end{align}
	The solution of \eqref{eq:dwj} exists by \ref{brda1}, and is unique for $t>T_0$ by Problem 2.5 and Theorem 2.2 (Picard-Lindl\"of) of \cite{T12} and \citep[p.19]{H69}, because the function of $w_j$ on the right hand side of \eqref{eq:dwj} is continuously differentiable with respect to $w_j$. Furthermore, the unique solution is in the form
		\begin{align}
		w_j(t)=e^{-\sigma_j^2 t}\bw_0 + (1-e^{- \sigma_j^2 t})(w_j^*-\sgn(w_j^\infty) c_0/\sigma_j^2), \quad t > T_0. \label{eq:wjsol}
		\end{align}
		Hence,
		\begin{align}
			|w_j^\infty-w_j^*| = \lim_{t\to\infty} \big|e^{-\sigma_j^2 t}(\bw_0 - w_j^*) - (1-e^{- \sigma_j^2 t}) \sgn(w_j^\infty) c_0/\sigma_j^2 \big| = c_0/\sigma_j^2.
		\end{align}
	\end{itemize} 
	
%

\end{proof} 

\begin{rem}\label{rem:wsta}
	 The difficulty for showing the convergence of $\bw(t)$ results from the non-smoothness of the $\ell_1$ penalty. Recall that $\bw(t)=\nabla\Lc^*(t,\bv(t))$ in \eqref{eq:weak0wl1} is not a differentiable transformation of $\bv(t)$ under the $\ell_1$ penalty; see \eqref{eq:sop}. Thus, the ODE for $\bw(t)$ does not exist. Moreover, the stability of $\bv(t)$ is unclear, because a well-behaved Lyapunov function required in the standard stability analysis of ODE, e.g. Chapter 4.5 of \cite{K02}, is hard to find, due to the non-smoothness of the $\ell_1$ penalty. 
\end{rem}

\bigskip
\begin{proof}[Proof of Lemma \ref{lem:lbrel1time}]
	Using \eqref{eq:lamcond}, Lemma \ref{lem:l1dd} suggests that \eqref{eq:tPlim1l1} holds 
	if and only if the remainder term 
	\begin{align}
		\sup_{t\in\Tc}\gamma^{-1}\bigg|g(\itg,\gamma) \sum_{j:w_j(t)\neq 0} R_j^t(\sqrt{\gamma})\bigg| \to 0, \mbox{ as }\gamma \to 0,\label{eq:rdal1rem}
	\end{align}
	where
	\begin{align}
	\big|R_j^t(\sqrt{\gamma})\big|= 2\cdot\IF\{0< |w_j(t)|\leq \sqrt{\gamma}|u_j|\} \big|\sqrt{\gamma}|u_j|-|w_j(t)|\big|, \quad j\in\{k=1,...,d:w_k(t)\neq 0\}.\label{eq:estiR}
	\end{align}
	Since $\sup_{t\in\Tc}|g(\itg,\gamma)|=O(\gamma^{1/2})$ by \eqref{eq:lamcond}, under the estimation \eqref{eq:estiR}, then \eqref{eq:rdal1rem} holds if and only if pointwise for any $u_j\neq 0$,
	$$
	\IF\{0< |w_j(t)|\leq \sqrt{\gamma}|u_j|\}\to 0 \mbox{ as }\gamma\to 0, \quad \mbox{ for all }t\in\Tc.
	$$
	This is equivalent to that for any $u_j\neq 0$ with $j\in\{k:\sup_{t\in \Tc}|w_k(t)|\neq 0\}$, 
	\begin{align}
				\min_{j\in\{k:\ \sup_{t\in\Tc}|w_k(t)|\neq 0\}}\frac{\min_{t\in\Tc}|w_j(t)|}{|u_j|}>0. \label{eq:pathrestrict} 
	\end{align}
	However, \eqref{eq:pathrestrict} is equivalent to the sign stability defined in Definition \eqref{def:sist}.
\end{proof}

\bigskip
\begin{proof}[Proof of Theorem \ref{th:l1dd}]
	For (a), it is clear that \ref{as:R} holds by $F(\bw)=\frac{1}{2}\|\bw\|_2^2$ and $\Pc(\bw)=\|\bw\|_1$. By the hypothesis that $g(\lfloor\cdot/\gamma\rfloor,\gamma)\in D([0,\infty))$ satisfies \eqref{eq:lamcond} with $\Tc=[0,T]$ for any $T>0$, $\sup_{t\in[0,T]}\big|g(\itg,\gamma)\big|=o(\gamma^{1/2})$ for any $T>0$, so condition \ref{as:Plim} holds with $g^\dagger(t)=0$ for all $t$. This implies that $\Lc(t,\bw)=\frac{1}{2}\|\bw\|_2^2$ and $\nabla\Lc^*(t,\bv)=\bv$ for any $t>0$ and $\bv\in\R^d$. Using Corollary \ref{cor:l1at}, $\bv_\gamma(t)\pto\bv(t)$ in uniform metric on $[0,\infty)$ by Theorem \ref{th:at}(b) and \eqref{eq:weak0vl1} in Corollary \ref{cor:l1at}, where $\bv(t)$ is the unique solution of \eqref{eq:bvdd}  and it follows that $\bw(t)=\nabla\Lc^*(t,\bv(t))=\bv(t)$ for any $t>0$ by \eqref{eq:weak0wl1}.

\bigskip	
	For (b), from the sign stability, 
Lemma \ref{lem:lbrel1time} implies $\tilde\Lc$ is of the form as \eqref{eq:rdal1tP}. For $t>T_0$, similar to \eqref{eq:recurv},
		\begin{align}
			\VV_\gamma(t) = \VV_\gamma(T_0)-\MM_\gamma(t) &- \gamma^{-1/2}\Big(\gamma\sum_{k=0}^{\lfloor t/\gamma\rfloor-1} G(\bw_k)-\underbrace{\int_{T_0}^t G(\bw(s))ds}_{\mbox{\scriptsize from   $\bv(t)$}}\Big).\label{eq:recurvloc}
		\end{align}
	Following similar argument as the proof of Theorem \ref{th:dd} and \ref{ws:M}-\ref{ws:stopping}, the process convergence on $D(\Tc)^d$ can be shown for $\MM_\gamma(t)$ and $\gamma^{-1/2}\big(\gamma\sum_{k=0}^{\lfloor t/\gamma\rfloor-1} G(\bw_k)-\int_{T_0}^t G(\bw(s))ds\big)$ in \eqref{eq:recurvloc}. The details are omitted. The result under the global sign stability follows immediately from Theorem \ref{th:dd}.
\end{proof}

\begin{lemma}\label{lem:l1dd}
	For any $\bw=(w_1,...,w_d)$ and $\bu=(u_1,...,u_d)$ in $\R^d$,
\begin{align*}
	\hspace{-1cm}\tilde\Lc_\gamma(t,\bu)&= \frac{1}{2}\|\bu\|_2^2 + (g(\itg,\gamma)/\sqrt{\gamma}) \Big(\sum_{j=1}^d u_j\sgn(w_j(t))\IF\{w_j(t)\neq 0\} + \sum_{j=1}^d|u_j|\IF\{w_j(t)=0\}\Big)\\
	&\quad + \gamma^{-1}g(\itg,\gamma) \sum_{j:w_j(t)\neq 0} R_j^t(\sqrt{\gamma}).
\end{align*}	
where
$$
\big|R_j^t(\sqrt{\gamma})\big|= 2\cdot\IF\{0< |w_j(t)|\leq \sqrt{\gamma}|u_j|\}\big|\sqrt{\gamma} |u_j|-|w_j(t)|\big|, \quad j\in\{k=1,...,d:w_k(t)\neq 0\}.
$$
\end{lemma}
\begin{proof}[Proof of Lemma \ref{lem:l1dd}] 
		Fix arbitrary $\bw=(w_1,...,w_d)$ and $\bu=(u_1,...,u_d)$. Using Knight's identity (in the proof of Theorem 1 of \cite{K98}),
		\begin{align*}
			\|\bw+\sqrt{\gamma}\bu\|_1-\|\bw\|_1 
			= \sqrt{\gamma}\Big(\sum_{j=1}^d u_j\sgn(w_j)\IF\{w_j\neq 0\} + \sum_{j=1}^d|u_j|\IF\{w_j=0\}\Big)+ \sum_{j:w_j\neq 0} R_j^t,
		\end{align*}
		where 
		\begin{align*}
			R_j^t = 2 \int_0^{-\sqrt{\gamma}u_j} (\IF\{w_j\leq s\}-\IF\{w_j\leq 0\})ds.
		\end{align*}
		Therefore,
		\begin{align*}
		&\gamma^{-1}\big(\Psi_\gamma(t,\bw+\sqrt{\gamma}\bu)-\Psi_\gamma(t,\bw) -\langle \sqrt{\gamma}\bu,\bw\rangle\big)\\
		&=\gamma^{-1}\bigg(\frac{1}{2}\|\bw+\sqrt{\gamma}\bu\|_2^2-\frac{1}{2}\|\bw\|_2^2-\sqrt{\gamma}\bu^\top\bw+g(\itg,\gamma) \|\bw+\sqrt{\gamma}\bu\|_1-g(\itg,\gamma) \|\bw\|_1\bigg)\\
		&= \frac{1}{2}\|\bu\|_2^2 + (g(\itg,\gamma)/\sqrt{\gamma}) \Big(\sum_{j=1}^d u_j\sgn(w_j)\IF\{w_j\neq 0\} + \sum_{j=1}^d|u_j|\IF\{w_j=0\}\Big) \\
		&\quad\quad\quad+ \gamma^{-1}g(\itg,\gamma) \sum_{j:w(t)\neq 0} R_j^t(\sqrt{\gamma}).
		\end{align*}
		
		Notice that
		\begin{align*}
			&\int_0^{-\sqrt{\gamma}u_j} (\IF\{w_j\leq s\}-\IF\{w_j\leq 0\})ds \\
			&= \begin{cases}
				\IF\{w_j\leq -\sqrt{\gamma}u_j\} \big(-\sqrt{\gamma}u_j-w_j\big) &w_j>0\\
				\IF\{w_j<-\sqrt{\gamma}u_j\}(-\sqrt{\gamma}u_j)+\IF\{w_j>-\sqrt{\gamma}u_j\}w_j + \sqrt{\gamma}u_j&w_j<0.\\
				\quad= \IF\{w_j>-\sqrt{\gamma}u_j\}(w_j+\sqrt{\gamma}u_j)&
			\end{cases}
		\end{align*}	
		Combine the above to get
		\begin{align*}
			&\bigg|\int_0^{-\sqrt{\gamma}u_j} (\IF\{w_j\leq s\}-\IF\{w_j\leq 0\})ds\bigg|\\
			&= \big(\IF\{0<w_j\leq -\sqrt{\gamma}u_j\}+\IF\{0>w_j>-\sqrt{\gamma}u_j\}\big)|\sqrt{\gamma}u_j+w_j|\\
			&= \IF\{0< |w_j|\leq \sqrt{\gamma}|u_j|\}\big||\sqrt{\gamma}u_j|-|w_j|\big|.
		\end{align*}
		 This proof is finished by replacing $w_j$ by $w_j(t)$.
\end{proof}

\section{Proofs for Section \ref{sec:lrl1}}\label{sec:pfslr}

\begin{proof}[Proof of Corollary \ref{cor:lrrdaat}]
	The asymptotic trajectory \eqref{eq:weakatrda1} follows by Theorem \ref{th:l1dd}. Moreover, \eqref{eq:solatrda1} stems from elementary results of linear ODE (e.g. Section 7.5 of \cite{BD05}), and that is the unique solution of \eqref{eq:weakatrda1} for all $t\in[0,\infty)$. The conclusion in (b) directly results from (a)-(c) of Theorem \ref{th:l1dd}.
\end{proof}

\bigskip
\begin{proof}[Proof of Theorem \ref{th:lrsu}]
	
Since $\lfloor t/\gamma\rfloor \gamma \to t$ as $\gamma\to 0$, the tuning function \eqref{eq:lrlamlong} implies $\lim_{\gamma\to 0}\sup_{0<t<T}|g(\itg,\gamma)-g^\dagger(t)|=0$ where $g^\dagger(t)=0$ for all $t$, and $\lim_{\gamma\to 0}\sup_{0<t<T}|g(\itg,\gamma)/\sqrt{\gamma}-g^\ddagger(t)|=0$ for any $T>0$, with
	\begin{align}\label{eq:lrlamlong1}
		g^\ddagger(t)=
		\left\{\begin{array}{ll}
			0,&\ t \leq t_0;\\
			c (t-t_0)^\mu, &\ t > t_0,
		\end{array}\right.
	\end{align}
so $g^\ddagger(t)$ is continuous on $[0,\infty)$, and $g(\lfloor\cdot/\gamma\rfloor,\gamma)\in D([0,\infty))$ satisfies \eqref{eq:lamcond} with $\Tc=[0,T]$ for any $T>0$. Assumptions \ref{as:S}, \ref{as:M} and \ref{as:L} are clear from the model \eqref{eq:modrdal1}. Corollary \ref{cor:lrrdaat}(a) suggests that as $\gamma\to 0$, the mean dynamics $\bw(t)$ is \eqref{eq:lintra}. 

\bigskip
For the distributional dynamics, we will use Theorem \ref{th:dd}. The only condition left to be verified is \ref{as:tPlim} with $g^\ddagger$ defined in \eqref{eq:lrlamlong1}. Recall $\tilde\Lc_\gamma(t,\bu)$ in \eqref{eq:locbregl1} and $\tilde\Lc(t,\bu)$ in \eqref{eq:rdal1tP}. Lemma \ref{lem:lbrel1time} verifies Condition \ref{as:tPlim0} for $[t_0,T]$ for any $T>t_0$, because $\bw(t)$ is sign stable on $[t_0,T]$. $\tilde\Lc(t,\bu)$ satisfies Condition \ref{as:tPlim1} because it is continuous in both $t$ (by the continuity of $g^\ddagger(t)$) and $\bu$. Lastly, for any $T>0$ and any fixed $\bu_0\in\R^d$, $\sup_{0\leq t\leq T}\|\tilde\Lc(t,\bu_0)\|_2 \leq \|\bu_0\|_2 + \sup_{0\leq t\leq T}|g^\ddagger(t)|$. Because $g^\ddagger(t)$ is continuous, \ref{as:tPlim2} holds. 

Applying Theorem \ref{th:dd} gives $\frac{\bv_{\lfloor t/\gamma\rfloor}-\bv(t)}{\sqrt{\gamma}}=\VV_\gamma(t)\weakto \VV(t)$ on $D([0,\infty))^d$ where $\VV(t)$ satisfies the SDE
	\begin{align}
		\VV(t) &= -\int_0^t H \cdot \nabla \tilde\Lc^*(s,\VV(s)) ds + \MM(t),\quad t\geq 0, \label{eq:conV}
	\end{align}
	where $\tilde\Lc^*_j(t,\bV)=\bV$ is the identity map for $t\leq t_0$, while for $t>t_0$,
	\begin{align}\label{eq:rdasoftthreshold_again}
		\nabla \tilde\Lc^*_j(t,\bV) = \begin{cases}
			V_j-g^\ddagger(t), &\mbox{ for }j:w_j^*>0;\\			
			V_j+g^\ddagger(t), &\mbox{ for }j:w_j^*<0;\\			
			\sgn(V_j)\big[|V_j|-g^\ddagger(t)\big]_+. &\mbox{ for }j:w_j^*=0,\\
		\end{cases} 
	\end{align}
	for $j=1,...,d$, and $g^\ddagger$ is of the form in \eqref{eq:lrlamlong1}. The reason $\nabla \tilde\Lc^*_j(t,\bV)$ depends on $w_j^*$ rather than $w_j(t)$ like \eqref{eq:rdasoftthreshold} is that $\sgn(w_j(t))=\sgn(w_j^*)$ for $t>t_0$; see \eqref{eq:lintra}. $H$ is a diagonal matrix from Condition \ref{as:D1}; the covariance kernel of $\MM(t)$ is of the form
		\begin{align*}
			\Sigma(\bw)=\E\big[(XX^\top-H)(\bw-\bw^*)(\bw-\bw^*)^\top(XX^\top-H)\big] + \sigma_\varepsilon^2 H.
	\end{align*}
	from \eqref{eq:covrdadd}. 
	
	\bigskip
	Since $\MM(t)$ has independent Gaussian increments, by the Ito isometry, $\VV(t)$ in \eqref{eq:conV} can equivalently be written as
	\begin{align}
		\VV(t) &= -\int_0^t H \cdot \nabla \tilde\Lc^*(s,\VV(s)) ds + \int_0^t \Sigma^{1/2}(\bw(s))d\BB(s), \quad t\geq 0, \label{eq:tempV} 
	\end{align}
	where $\BB(t)=(B_1(t),...,B_d(t))^\top$ is the standard $d$-Brownian motion with $\Cov(B_i(t),B_j(t))=0$.
Observe that $H$ is diagonal and the operator $\nabla \tilde\Lc^*$ defined in \eqref{eq:rdasoftthreshold_again} applies coordinatewise, and the diffusion term does not depend on $\VV(t)$, so the process $\VV(t)$ can be represented coordinatewise by 
\begin{align}
	V_j(t) = -\int_0^t \sigma_j^2 \nabla \tilde\Lc_j^*(s,V_j(s)) ds + \int_0^t \Sigma_{j\cdot}^{1/2}(\bw(s))^\top d\BB(s),\ t\geq 0, \ j=1,2,...,d,\label{eq:tempV1_0}
\end{align}
where $\Sigma(\bw(s))_{j\cdot}^{1/2}$ denotes the $j$th row of the matrix $\Sigma(\bw(s))^{1/2}$. 

For $t<t_0$, $g^\ddagger=0$ by \eqref{eq:lrlamlong1} and $\nabla \tilde\Lc_j^*(s,V_j(s))=V_j(s)$, so the process \eqref{eq:tempV1_0} is essentially a Ornstein-Uhlenbeck process
\begin{align}
	V_j(t) = - \sigma_j^2\int_0^t V_j(s) ds + \int_0^t \Sigma_{j\cdot}^{1/2}(\bw(s))^\top d\BB(s), \quad t\leq t_0, \ j=1,2,...,d,\label{eq:tempV1_sgd}
\end{align} 
which has a closed-form solution by Lemma \ref{lem:solV}. Therefore, the process in \eqref{eq:tempV1_0} for $t>t_0$ can be rewritten 
\begin{align}
	V_j(t) = V_j(t_0)-\int_{t_0}^t \sigma_j^2 \nabla \tilde\Lc_j^*(s,V_j(s)) ds + \int_{t_0}^t \Sigma_{j\cdot}^{1/2}(\bw(s))^\top d\BB(s), \quad t\geq t_0, \ j=1,2,...,d,\label{eq:tempV1}
\end{align}
where 
\begin{align}
	V_j(t_0)=e^{-\sigma_j^2 t_0}\int_0^{t_0} e^{\sigma_j^2 s}\ \Sigma_{j\cdot}^{1/2}(\bw(s))^\top d\BB(s) \label{eq:V0}
\end{align} 
is the solution at $t=t_0$ of \eqref{eq:tempV1_sgd} by Lemma \ref{lem:solV}. Note that 
\begin{align}
	\var\big\{V_j(t_0)\big\} \leq (1-e^{-2\sigma_j^2 t_0}) (C\|\bw^*\|_2^2 + \sigma_\varepsilon^2\|H\|_2)<\infty,\label{eq:finsecm}
\end{align}
by Ito isometry and Condition \ref{as:D1}.

\bigskip
In the following, statements in \ref{lrsu1} and \ref{lrsu2} will be proved in several steps.
 
\begin{itemize}[itemindent=45pt,leftmargin=0pt]
\labitem{Step 1}{lrsta1} \textbf{Proof of \ref{lrsu1}}.

Using the explicit form of $\nabla \tilde\Lc_j^*$ defined in \eqref{eq:rdasoftthreshold_again} for $j\in\{k:w_k^*=0\}$, as $t>t_0$, $W_j(t)=\sgn(V_j(t))\big[|V_j(t)|-c(t-t_0)^\mu\big]_+$ is a soft-thresholding operator on $V_j$ with threshold $c(t-t_0)^\mu$. By \eqref{eq:tempV1}, $\forall j\in\{k:w_k^*=0\}$, $V_j(t)$ is characterized by
	\begin{align}
		dV_j(t) = -\sigma_j^2 \sgn(V_j(t))\big[|V_j(t)|-c(t-t_0)^\mu\big]_+ dt + \Sigma_{j\cdot}^{1/2}(\bw(s))^\top d\BB(t),\ t>t_0,\label{eq:vprime}
	\end{align}
with initialization $V_j(t_0)$ defined in \eqref{eq:V0}. 

\bigskip
We claim that for $t>t_0$,
\begin{align}
	W_j(t)=\sgn(V_j(t))\big[|V_j(t)|-c(t-t_0)^\mu\big]_+\pto 0,\ \mbox{ as $t\to\infty$}. \label{eq:Wlim}
\end{align}
Note that the mean dynamics $w_j(t)=0$ for all $t$ by \eqref{eq:lintra} when both $w_{0,j}=0$ and $w_j^*=0$, and this with \eqref{eq:Wlim} yield that $|W_j(t)|=|w_{\lfloor t/\gamma \rfloor,j}/\sqrt{\gamma}|=o_p(1)$ as $\gamma\to 0$ and $t\to\infty$. Hence, $|w_{\lfloor t/\gamma \rfloor,j}|=o_p(\sqrt{\gamma})$ when $\gamma\to 0$ and $t\to\infty$ as desired in this step.

	\bigskip
	It is now left to show \eqref{eq:Wlim}. To simplify expressions, for $t>t_0$, define the event
	\begin{align}
		E_1(t):= \bigg\{\sup_{t_0 \leq \tau\leq t}\bigg|V_j(t_0)+\int_{t_0}^\tau \Sigma_{j\cdot}^{1/2}(\bw(s))^\top d\BB(s)\bigg|\leq c(t-t_0)^\mu\bigg\}.\label{eq:evE1}
	\end{align}
	Take $u>0$,
	\begin{align}
		P\big(|W_j(t)|>u\big)&=P\big(|V_j(t)-c (t-t_0)^\mu|>u\big)\notag\\
		&\leq P\big(\big\{|V_j(t)-c (t-t_0)^\mu|>u\big\}\cap E_1(t)\big)+P\big(E_1(t)^c\big),\label{eq:step1dec}
	\end{align}
	where $c>0$ and $\mu>1/2$ are defined as in the hypothesis of this Theorem, and $E_1(t)^c$ is the complement of $E_1(t)$ defined in \eqref{eq:evE1}. It will be shown that the first probability of \eqref{eq:step1dec} is zero, while the second probability converges to 0 as $t\to\infty$.
	
	\bigskip
	We now prove that the first probability in \eqref{eq:step1dec} is 0. We first show that that under $E_1(t)$, 
	\begin{align}
		V_j(\tau)=V_j(t_0)+\int_{t_0}^\tau \Sigma_{j\cdot}^{1/2}(\bw(s))^\top d\BB(s), \quad t_0\leq\tau\leq t,\label{eq:solV_E1}
	\end{align}
	  is the unique weak solution to \eqref{eq:vprime} before time $t$. To see this, by the restriction on $E_1(t)$, $\big[|V_j(t)|-c(t-t_0)^\mu\big]_+=0$, so the drift term in \eqref{eq:vprime} is zero, and then $dV_j(t)=\Sigma_{j\cdot}^{1/2}(\bw(s))^\top d\BB(t)$ for $t\geq t_0$ with $V_j(t_0)$ in \eqref{eq:V0}. Hence, \eqref{eq:solV_E1} holds by uniqueness proven in Theorem \ref{th:dd}. Plugging the solution \ref{eq:solV_E1} in the first probability of \eqref{eq:step1dec},
	 \begin{align}
	 	&P\big(\big\{|V_j(t)-c (t-t_0)^\mu|>u\big\}\cap E_1(t)\big)\notag\\
		&=P\bigg(\bigg\{\bigg|V_j(t_0)+\int_{t_0}^t \Sigma_{j\cdot}^{1/2}(\bw(s))^\top d\BB(s)-c (t-t_0)^\mu\bigg|>u\bigg\}\bigcap E_1(t)\bigg)\notag\\
		&=0,\quad \forall t\geq t_0 \mbox{ and }u>0. \label{eq:pbd1}
	 \end{align} 
	
	\bigskip
	Next, we bound the probability $P\big(E_1(t)^c\big)$ in \eqref{eq:step1dec}. Probability bound on maximal deviation of Gaussian process in Theorem D.4 on page 20 of \cite{P96} will be applied. Some preliminary estimation will be made below. Observe that for any $t_1,t_2$ such that $t_0\leq t_1 < t_2 \leq t$, Ito isometry implies
	\begin{align}
		&\E\bigg[\bigg(\int_{t_0}^{t_1}\Sigma_{j\cdot}^{1/2}(\bw(s))^\top d\BB(s)-\int_{t_0}^{t_2}\Sigma_{j\cdot}^{1/2}(\bw(s))^\top d\BB(s)\bigg)^2\bigg]\notag\\
		&=\E\bigg[\bigg(\int_{t_1}^{t_2}\Sigma_{j\cdot}^{1/2}(\bw(s))^\top d\BB(s)\bigg)^2\bigg] \notag\\
		&= \int_{t_1}^{t_2} \big\|\Sigma_{j\cdot}^{1/2}(\bw(s))\big\|_2^2 ds \notag\\
		&= \Big(\sup_{t_1\leq s\leq t_2} \big\|\Sigma_{j\cdot}^{1/2}(\bw(s))\big\|_2^2\Big) |t_2-t_1|\notag\\
		&\leq \big(C \|\bw^*\|_2^2+\sigma_\varepsilon^2 \|H\|_2\big)\ |t_2-t_1|,\label{eq:covfunbd}
	\end{align}
where 
the last inequality follows from $\|\bw(s)\|_2 \leq \|\bw^*\|_2$ for all $s\geq 0$ using the explicit mean trajectory \eqref{eq:lintra} and $\bw_0=0$, and the constant $C$ depends on the almost sure bound of the covariates $X$, which exists by Condition \ref{as:D1}. 

In addition, by similar argument as in \eqref{eq:covfunbd}, Ito isometry yields
\begin{align}
	\sup_{t_0\leq \tau\leq t}\E\bigg[\bigg(\int_{t_0}^\tau\Sigma_{j\cdot}^{1/2}(\bw(s))^\top d\BB(s)\bigg)^2\bigg]&=\sup_{t_0\leq \tau\leq t}\E\bigg[\int_{t_0}^\tau \big\|\Sigma_{j\cdot}^{1/2}(\bw(s))\big\|_2^2 ds\bigg]\notag\\
	&\leq \big(C \|\bw^*\|_2^2+\sigma_\varepsilon^2 \|H\|_2\big)(t-t_0). \label{eq:sigfunbd}
\end{align}

Using the bounds in \eqref{eq:covfunbd} and \eqref{eq:sigfunbd}, Theorem D.4 on page 20 of \cite{P96} and elementary probability bound give
\begin{align}
	&P\big(E_1(t)^c\big) \notag\\
	&\leq P\bigg(\big|V_j(t_0)\big|+\sup_{t_0 \leq \tau\leq t}\bigg|\int_{t_0}^\tau \Sigma_{j\cdot}^{1/2}(\bw(s))^\top d\BB(s)\bigg|>c(t-t_0)^\mu\bigg)\notag\\
	&\leq 2 \bigg(1-\Phi\Big(\frac{c(t-t_0)^\mu}{2\sqrt{\var\{V_j(t_0)\}}}\Big)\bigg)+ 2 c\ C_0\ t (t-t_0)^\mu\ \frac{1}{2}\exp\bigg\{-\frac{c^2 (t-t_0)^{2\mu}}{8\big(C \|\bw^*\|_2^2+\sigma_\varepsilon^2 \|H\|_2\big)\ (t-t_0)}\bigg\}\notag\\ 
	&\to 0, \quad \mbox{if }\mu>1/2,\mbox{ and } c>0,
	\label{eq:pbd2}
\end{align}
where $C_0$ depends on $\bw^*,H,\sigma_\varepsilon$ and the almost sure bound on $X$, and note that the Gaussian tail probability satisfies $1-\Phi(u)\leq \frac{1}{2}e^{-u^2/2}$ for any $u>0$, where $\Phi(\cdot)$ is the cdf of $\Nc(0,1)$. 

The proof \eqref{eq:Wlim} is complete by plugging \eqref{eq:pbd1} and \eqref{eq:pbd2} in \eqref{eq:step1dec}.

\labitem{Step 2}{lrsta2} \textbf{Proof of \eqref{eq:limw} in \ref{lrsu2}}.

By \eqref{eq:tempV1}, for $t>t_0$,
	\begin{align}
		dV_j(t) = -\sigma_j^2 \big(V_j(t)-c \sgn(w_j^*)(t-t_0)^\mu\big) dt + \Sigma_{j\cdot}^{1/2}(\bw(s))^\top d\BB(t),\quad \forall j\in\{k:w_k^*\neq 0\},\label{eq:vprime_act}
	\end{align}
	with initialization $V_j(t_0)$ defined in \eqref{eq:V0}, and
	\begin{align}
		W_j(t) = V_j(t)-c \sgn(w_j^*)(t-t_0)^\mu.
	\end{align}
An application of Ito's lemma shows that $W_j(t)$ satisfies the SDE 
	\begin{align}
		d W_j(t) = \big(-\sigma_j^2 W_j(t)-c\sgn(w_j^*) \mu (t-t_0)^{\mu-1}\big) dt + \Sigma_{j\cdot}^{1/2}(\bw(s))^\top d\BB(t), \quad j\not\in\{k:w_k^*=0\}. \label{eq:wprime0}
	\end{align}
	with initialization $W_j(t_0)=V_j(t_0)$. Lemma \ref{lem:solV} gives the solution 
	\begin{align}
		W_j(t) =  e^{-\sigma_j^2(t-t_0)}W_j(t_0)+h_j(t) + U_j(t)
	\end{align}
	where
	\begin{align}
		h_j(t) &= - c\sgn(w_j^*) \mu\ e^{-\sigma_j^2 t}\int_{t_0}^t (s-t_0)^{\mu-1} e^{\sigma_j^2 s} ds,\label{eq:ht_pr}\\ 
		U_j(t) &= e^{-\sigma_j^2 t}\int_{t_0}^t e^{\sigma_j^2 s}\  \Sigma_{j\cdot}^{1/2}(\bw(s))^\top d\BB(s).
	\end{align}
	This finishes the proof of this step.

	\labitem{Step 3}{lrsta3} \textbf{The estimation of $h_j(t)$ in \eqref{eq:hest} of \ref{lrsu3}}.
Recall $h_j(t)$ from \eqref{eq:ht_pr}, 
	\begin{align*}
		h_j(t) &= - \mu c\ \sgn(w_j^*)\ e^{-\sigma_j^2 t}\ \int_{t_0}^t (s-t_0)^{\mu-1} e^{\sigma_j^2 s} ds.
	\end{align*}
	Integration by substituting $s=\sigma_j^{-2}u+t_0$ gives
	\begin{align}
		h_j(t) &= - \mu c\ \sgn(w_j^*)\ \sigma_j^{-2\mu}\ e^{-\sigma_j^2 (t-t_0)} \int_{0}^{\sigma_j^2 (t-t_0)} u^{\mu-1} e^u\ du. \label{eq:rech}
	\end{align}
	Repeated integration by parts shows that %
	\begin{align}
		e^{-\sigma_j^2 (t-t_0)} \int_{0}^{\sigma_j^2 (t-t_0)} u^{\mu-1} e^u\ du
		&= \sum_{k=0}^\infty \frac{(\sigma_j^2 (t-t_0))^{\mu+k}}{\mu^{\overline{k+1}}}(-1)^k,\label{eq:exprh0}
	\end{align}
where $\mu^{\overline{k+1}}=\mu(\mu+1)(\mu+2)...(\mu+k)$ is the rising factorial. In addition, for any $t>t_0$,
\begin{align}
	\sum_{k=0}^\infty \frac{(\sigma_j^2 (t-t_0))^{\mu+k}}{\mu^{\overline{k+1}}}(-1)^k =  (-1)(\sigma_j^2 (t-t_0))^{\mu-1}\sum_{k=0}^\infty \frac{(-\sigma_j^2 (t-t_0))^{k+1}}{(k+1)!} \frac{(k+1)!}{\mu^{\overline{k+1}}}. \label{eq:ingam}
\end{align}
Hence, combining \eqref{eq:exprh0} and \eqref{eq:ingam} with \eqref{eq:rech}, and let $C_j := \mu \sigma_j^{-2} c > 0$,
\begin{align}
	h_j(t) &= C_j \sgn(w_j^*) (t-t_0)^{\mu-1} \sum_{k=0}^\infty \frac{(-\sigma_j^2 (t-t_0))^{k+1}}{(k+1)!} \frac{(k+1)!}{\mu^{\overline{k+1}}}
	.\label{eq:hest1}
\end{align}

Now we separate discuss different cases of $\mu$. If $\mu \geq 1$, for any $k\geq 0$,
\begin{align}
	\mu^{-(k+1)}\leq \frac{(k+1)!}{\mu^{\overline{k+1}}} = \underbrace{\frac{1}{\mu}\cdot \frac{2}{\mu+1}\cdot \frac{3}{\mu+2}\cdot  \cdots\cdot \frac{k+1}{\mu+k}}_{\mbox{\scriptsize $k+1$ terms are less than or equal to 1}}\leq 1.\label{eq:bdfactorial}
\end{align}
Observing that by the Taylor series for exponential function, \eqref{eq:bdfactorial} implies
\begin{align}
	e^{-\sigma_j^2 (t-t_0)}-1 \leq \sum_{k=0}^\infty \frac{(-\sigma_j^2 (t-t_0))^{k+1}}{(k+1)!}\frac{(k+1)!}{\mu^{\overline{k+1}}} \leq e^{-\sigma_j^2 (t-t_0)/\mu}-1, \quad \forall t\geq t_0. \label{eq:expexpr}
\end{align}
Note that if $\mu=1$, the inequalities \eqref{eq:bdfactorial} and \eqref{eq:expexpr} hold with equality. Applying \eqref{eq:expexpr} in \eqref{eq:hest1} obtains
\begin{align}
	C_j (t-t_0)^{\mu-1}\ (e^{-\sigma_j^2 (t-t_0)}-1) \leq \sgn(w_j^*) h_j(t) \leq C_j (t-t_0)^{\mu-1}\ (e^{-\sigma_j^2 (t-t_0)/\mu}-1), \label{eq:hest2}
\end{align}

\bigskip
If $0<\mu<1$, the inequality \eqref{eq:bdfactorial} reverses and therefore inequality \eqref{eq:expexpr} also reverses. Hence, \eqref{eq:hest1} gives
\begin{align}
	C_j (t-t_0)^{\mu-1}\ (e^{-\sigma_j^2 (t-t_0)/\mu}-1) \leq \sgn(w_j^*) h_j(t) \leq C_j (t-t_0)^{\mu-1}\ (e^{-\sigma_j^2 (t-t_0)}-1). \label{eq:hest3}
\end{align}

\bigskip
Finally, combining \eqref{eq:hest2} and \eqref{eq:hest3}, we obtain that for any $j\in\{k:w_k^*\neq 0\}$ and $\mu>0$,
\begin{align}
	&C_j (t-t_0)^{\mu-1}\ (e^{-\sigma_j^2 (t-t_0)\max\{1,\mu^{-1}\}}-1)\\
	&\leq\sgn(w_j^*) h_j(t) \leq C_j (t-t_0)^{\mu-1}\ (e^{-\sigma_j^2 (t-t_0)\min\{1,\mu^{-1}\}}-1).\notag
\end{align}
From the last display, one can estimate $\sgn(w_j^*) h_j(t)$ by
\begin{align}
	\big|\sgn(w_j^*) h_j(t) - (-1) C_j (t-t_0)^{\mu-1}\big| \leq C_j (t-t_0)^{\mu-1} e^{-\sigma_j^2 (t-t_0)\min\{1,\mu^{-1}\}}=:\tilde h_j(t). \label{eq:deftilh}
\end{align}
%
By the hypothesis of this theorem that $t_0<\infty$ is fixed, so $\tilde h_j(t)=o(1)$ as $t\to\infty$. This completes the proof.

\labitem{Step 4}{lrsta4} \textbf{$U_j(t)\weakto\Nc(0,\sigma_\varepsilon^2/2)$ as $t\to\infty$ for any $j=1,...,d$ in \ref{lrsu3}}.

To prove this step, it is enough to show that $\Var\big\{U_j(t)\big\} \to \sigma_\varepsilon^2/2$ as $t\to\infty$. To simplify notations, let 
\begin{align}
	H_j(t):=\int_{t_0}^t e^{\sigma_j^2 s}\  \Sigma_{j\cdot}^{1/2}(\bw(s))^\top d\BB(s),
\end{align}
so $U_j(t)=e^{-\sigma_j^2 t} H_j(t)$. 

Observe that 
\begin{align*}
	\hat H_j(t):= e^{-2\sigma_j^2 t}\int_{t_0}^t e^{2\sigma_j^2 s} \sigma_\varepsilon^2 \sigma_j^2 ds &= \frac{\sigma_\varepsilon^2}{2}\big(1-e^{-2\sigma_j^2(t-t_0)}\big) \to \frac{\sigma_\varepsilon^2}{2}, \quad \mbox{ as }t\to\infty.
\end{align*}
Therefore, it is left to show that the difference $\big|\Var\big\{U_j(t)\big\}-\hat H_j(t)\big|$ converges to 0 as $t\to\infty$.

For $t \geq u>t_0$, the covariance function of $U_j(t)$ in \eqref{eq:Uj} can be computed
\begin{align}
	\Cov\big\{U_j(t),U_j(u)\big\} 
	&= e^{-\sigma_j^2(t+u)}\E\big[\E\big[\big(H_j(t)-H_j(u)+H_j(u)\big)H_j(u)\big]\big|\Fc_u\big]\notag\\
	&= e^{-\sigma_j^2(t+u)}\E\big[H_j(u)^2\big] \notag\\
	&= e^{-\sigma_j^2(t+u)}\int_{t_0}^u e^{2\sigma_j^2 s}\big\|\Sigma_{j\cdot}^{1/2}(\bw(s))\big\|_2^2ds.\label{eq:covU}
\end{align}
By virtue of the triangular inequality, 
\begin{align}
	\big|\|A\|^2-\|B\|^2\big| &= \big|(\|A\|+\|B\|)(\|A\|-\|B\|)\big| \leq (2\|B\|+\|A-B\|)\|A-B\|\notag\\ 
	&\leq (2\|B\|+1)\|A-B\|,\label{eq:trineq}
\end{align}
for any matrices $A,B$ and matrix norm $\|\cdot\|$ with $\|A-B\|\leq 1$. From \eqref{eq:covU},
\begin{align}
	&\bigg|\Cov\big\{U_j(t),U_j(u)\big\}-e^{-\sigma_j^2(t+u)}\int_{t_0}^u e^{2\sigma_j^2 s} \sigma_\varepsilon^2 \sigma_j^2 ds\bigg| \notag\\
	&\leq (2\sigma_\varepsilon \sigma_j+1) e^{-\sigma_j^2(t+u)}\int_{t_0}^u e^{2\sigma_j^2 s}\big\|\Sigma_{j\cdot}^{1/2}(\bw(s))-\sigma_\varepsilon \sigma_j\big\|_2ds\notag\\
	&\leq (2\sigma_\varepsilon \sigma_j+1) e^{-\sigma_j^2(t+u)}\int_{t_0}^u e^{2\sigma_j^2 s}\big\|\Sigma^{1/2}(\bw(s))-\sigma_\varepsilon H^{1/2}\big\|_2 ds \label{eq:tmplr4}
\end{align}
given that $\big\|\Sigma^{1/2}(\bw(s))-\sigma_\varepsilon H^{1/2}\big\|_2\leq 1$ for the spectral norm $\|\cdot\|_2$, where the first inequality follows by \eqref{eq:trineq} and the last inequality follows by H\"older's inequality.

We claim that for some $T_1>0$ (explicitly stated in \eqref{eq:T1}), the spectral norm of the two matrix square roots difference is bounded as
\begin{align}
	\big\|\Sigma^{1/2}(\bw(s))-\sigma_\varepsilon H^{1/2}\big\|_2 = O\big(e^{-(\min_j \sigma_j^2) s}\|\bw^*\|_\infty\big), \ \forall s>T_1,\label{eq:matroot}
\end{align}
where the right hand size $\leq 1$ for some $T_1>0$. This claim will be proved at the end of this step. 

Since $\sup_{0<s<T_1}\big\|\Sigma^{1/2}(\bw(s))\big\|_2$ is uniformly bounded by Condition \ref{as:D1} for any fixed $T_1$, we can suppose that $t_0>T_1$ without loss of generality. Otherwise, the integral $\int_{t_0}^u$ in \eqref{eq:tmplr4} can be separated as $\int_{t_0}^u = \int_{t_0}^{T_1}+\int_{T_1}^u$ (we are interested in large $t,u$), and the part $e^{-\sigma_j^2(t+u)} \int_{t_0}^{T_1}$ converges to 0 as $t,u\to\infty$. 

If \eqref{eq:matroot} is true, then \eqref{eq:tmplr4} can be bounded by 
\begin{align}
	&(2\sigma_\varepsilon \sigma_j+1) e^{-\sigma_j^2(t+u)}\int_{t_0}^u e^{(2\sigma_j^2-\min_j \sigma_j^2) s}ds \notag\\
	&=(2\sigma_\varepsilon \sigma_j+1)(2\sigma_j^2-\min_j \sigma_j^2)^{-1} \|\bw^*\|_\infty e^{-\sigma_j^2(t+u)} \big(e^{(2\sigma_j^2-\min_j \sigma_j^2) u}-e^{(2\sigma_j^2-\min_j \sigma_j^2)  t_0}\big)\label{eq:templrsta4}
\end{align}
Letting $u=t$, the last display implies
\begin{align*}
	&\big|\Var\big\{U_j(t)\big\}-\hat H_j(t)\big|\\
	&\leq (2\sigma_\varepsilon \sigma_j+1)(2\sigma_j^2-\min_j \sigma_j^2)^{-1} \|\bw^*\|_\infty \big(e^{-(\min_j \sigma_j^2) t}-e^{-2\sigma_j^2(t-t_0)-(\min_j \sigma_j^2) t_0}\big) =o(1),
\end{align*}
as $t\to 0$, which concludes the proof.

\bigskip
It is left to show the claim \eqref{eq:matroot}. Note first that from the form of $\bw(t)$ in \eqref{eq:lintra}, $\|\bw(t)-\bw^*\|_2=O(e^{-(\min_j \sigma_j^2) t}\|\bw^*\|_\infty)$ for any $t>0$ and the bounded $X$ in Condition \ref{as:D1}, so
	\begin{align}
		\big\|\Sigma(\bw(t))-\sigma_{\varepsilon}^2 H\big\|_2 = O\big(e^{-(\min_j \sigma_j^2) t}\|\bw^*\|_\infty\big), \ \forall t>0. \label{eq:sig2bound}
	\end{align}
The minimal eigenvalue of $\sigma_{\varepsilon}^2 H$ is $\sigma_{\varepsilon}^2 \min_j \sigma_j^2$, and the lower bound for the eigenvalue of $\Sigma(\bw(t))$ is $\sigma_{\varepsilon}^2 \min_j \sigma_j^2-e^{-(\min_j \sigma_j^2) t}\|\bw^*\|_\infty$, which is strictly positive for all $t>T_1$ where 
\begin{align}
	T_1 = -C\big(\min_j\sigma_j^2\big)^{-1}\log\bigg\{\frac{(\sigma_\varepsilon^2 \min_j\sigma_j^2) \wedge 1}{\|\bw^*\|_\infty}\wedge 1\bigg\},\label{eq:T1}
\end{align}
where $C$ depends on the almost sure bound for $X$ in Condition \ref{as:D1}. Applying Lemma 2.2 of \cite{S92} yields that for $t>T_1$,
\begin{align*}
	\big\|\Sigma^{1/2}(\bw(t))-\sigma_{\varepsilon} H^{1/2}\big\|_2 \leq \sigma_{\varepsilon}^{-1} (\min_j \sigma_j)^{-1} \big\|\Sigma(\bw(t))-\sigma_{\varepsilon}^2 H\big\|_2 = O\big(e^{-(\min_j \sigma_j^2) t}\|\bw^*\|_\infty\big).
\end{align*}
Note that the choice of $T_1$ in \eqref{eq:T1} also ensures that the right hand side of \eqref{eq:matroot} is less than 1 for all $s>T_1$. This verifies the claim \eqref{eq:matroot}.

\labitem{Step 5}{lrsta5} \textbf{$\lim_{t\to \infty}\Cov\big(U_i(t),U_j(t)\big)=0$ for all $i\neq j$ where $i,j\in\{k:w_k^*\neq 0\}$ in \ref{lrsu3}}.

Observe that the matrix valued function $\Sigma^{1/2}(\bw(s))$ is symmetric and positively definite, the inner product of $i$th row and $j$th column $\Sigma_{i\cdot}^{1/2}(\bw(s))^\top \Sigma_{j\cdot}^{1/2}(\bw(s))=\Sigma_{ij}(\bw(s))$. Hence,
\begin{align*}
	\Cov\big\{U_i(t),U_j(t)\big\} &= e^{-2\sigma_j^2t}\int_{t_0}^t e^{2\sigma_j^2 s}\big\{\Sigma_{i\cdot}^{1/2}(\bw(s))^\top \Sigma_{j\cdot}^{1/2}(\bw(s))\big\}ds=
	e^{-2\sigma_j^2t}\int_{t_0}^t e^{2\sigma_j^2 s} \Sigma_{ij}(\bw(s)) ds.
\end{align*}
Therefore, the bound \eqref{eq:sig2bound} and similar argument as \eqref{eq:templrsta4} yield that $\Cov\big\{U_i(t),U_j(t)\big\}=o(1)$ as $t\to\infty$.

\end{itemize}
\end{proof}

\begin{rem}[Comparison with SGD]\label{rem:comp_sgd}
	\cite{BMP90} (p.335) and \cite{BKS93} prove that for $w_{n,j}^{SGD}$ generated from \eqref{eq:sgd}, 
	\begin{align}
		\frac{w_{n,j}^{SGD}-w_j^*}{\sqrt{\gamma}} \weakto \Nc(0,\sigma_\epsilon^2/2), \quad j=1,...,d, \label{eq:sgdstat}
	\end{align}
	 as $\gamma\to 0$ and $n\to\infty$. For inactive coefficients $j\in\{k:w_k^*=0\}$, \eqref{eq:sgdstat} implies $|w_{n,j}^{SGD}|=O_p(\sqrt{\gamma})$. By contrast, Theorem \ref{th:lrsu}(a) shows that $|w_{n,j}|=o_p(\sqrt{\gamma})$ for inactive coefficients, where $w_{n,j}$ is generated by \eqref{eq:grda_gen}. This is due to the $\ell_1$ penalization. For the active coefficients $j\in\{k:w_k^*\neq 0\}$, if $h_j(t)\to 0$, Theorem \ref{th:lrsu}(b) and (c) show that $W_{\gamma,j}(t)$ has the same stationary distribution as \eqref{eq:sgd} in \eqref{eq:sgdstat}. This can be understood as the ``oracle property.''
\end{rem}

\section{Supplementary materials for Section \ref{sec:pca}}\label{sec:pfspca}

\subsection{A literature review of SPCA}\label{sec:respca}
SPCA is powerful data analytic tool that is used extensively in scientific applications; see the introduction of \cite{WBS16} for an overview. Existing algorithms of sparse PCA includes \citep{JTU03,ZHT06,DEJL07,DBG08,SH08,JL09,WTH09,JNRS10}. Some recent interest is on the minimax lower bound and adaptive algorithms under high dimensional settings ($d\gg n$) \citep{BJNP13,CMW13,M13,VL13,WBS16,SSM13}, or semidefinite relaxation \citep{VCLR13} for the incurred NP-hardness of the optimization problems. More recently, \cite{GWS18} proposes an interesting random projection method. \cite{ABEKM18} proposes a projection based method.

\bigskip
For online sparse PCA, \cite{MBPS10} propose an online successive orthogonal projection algorithm to solve the SPCA of \cite{ZHT06}. \cite{YX15} propose row truncation with iterative deflation, which improves the algorithm of \cite{MBPS10} in terms of computational efficiency. \cite{WL16} estimate the first principal component with $\ell_1$ proximal operator, but they obtain interesting asymptotic results for the joint distribution between the estimator and the true component as $d\to\infty$.

\subsection{Proof for Corollary \ref{cor:ospca}}\label{sec:pfospca}

We will verify the conditions of Corollary \ref{cor:l1at} and then Part (a) follows. Condition \ref{as:S} is verified by the hypothesis that $X_n$ are i.i.d. Condition \ref{as:M} follows by the fact that $X_n$ is bounded almost surely. Lastly, $g^\dagger(t)$ in Corollary \ref{cor:l1at} is $0$ for all $t$, by the hypothesis that $g(\lfloor\cdot/\gamma\rfloor,\gamma)\in D([0,\infty))$ satisfies \eqref{eq:lamcond} with $\Tc=[0,T]$ for any $T>0$..

Next, the existence of the uniform attraction region of \eqref{eq:atospca} follows directly by Lemma 4 on page 76 of \cite{OK85}, in which the condition A5 is verified by our hypothesis that the $k$th largest eigenvalues of $\Cc$ are positive with unit multiplicity. To show the global existence of the unique solution of \eqref{eq:atospca}, note that the right hand side of the ODE \eqref{eq:atospca} is differentiable with
\begin{align}
	\begin{split}\label{eq:hesospca}
		\nabla_j G_j(\bu_1,...,\bu_k) = \frac{\partial G_j}{\partial \bu_j}(\bu_1,...,\bu_j) &= \Cc - \Big(\Id_d \bu_j^\top \Cc \bu_j + 2 \sum_{i=1}^j \bu_i \bu_i^\top \Cc\Big),\quad j=1,...,k,\\
		\nabla_l G_j(\bu_1,...,\bu_k) = \frac{\partial G_j}{\partial \bu_l} (\bu_1,...,\bu_j) &= -2 \big(\Id_d \bu_l^\top \Cc \bu_j + \bu_l \bu_j^\top \Cc\big),\quad l=1,...,j-1,\\
	\end{split}
\end{align}
and the derivatives are continuous on $\R^{dk}$. Therefore, by Problem 2.5 and Theorem 2.2 (Picard-Lindl\"of) of \cite{T12}, there exists unique local solution of \eqref{eq:atospca}. The local uniqueness implies the global uniqueness \citep[p.19]{H69}. The global existence follows by the fact that $\UU(t)^\top\UU(t)$ is finite and nonsingular for all $t>0$ by Theorem 2.1 of \cite{P95} and the discussion that follows, as the initial $\UU(0)=\UU_0$ is finite.

\bigskip
For Part (b), we need to verify the conditions of Theorem \ref{th:l1dd}. \ref{as:L} follows from the fact that $X_n$s are bounded almost surely. The solution of the ODE \eqref{eq:atospca} uniquely exists for all $t\geq 0$ by the discussion above. Hence, conditions of Theorem \ref{th:l1dd} are verified. Note that $\nabla G(\UU)$ in \eqref{eq:hesspca} follows from \eqref{eq:hesospca}, and $\Sigma(\UU)$ in \eqref{eq:covpca} follows from \eqref{eq:sig}.

\subsection{ODE trajectories of OSPCA}\label{sec:ode_ospca}
\begin{figure}[!h]
	\centering
    \includegraphics[scale=0.15]{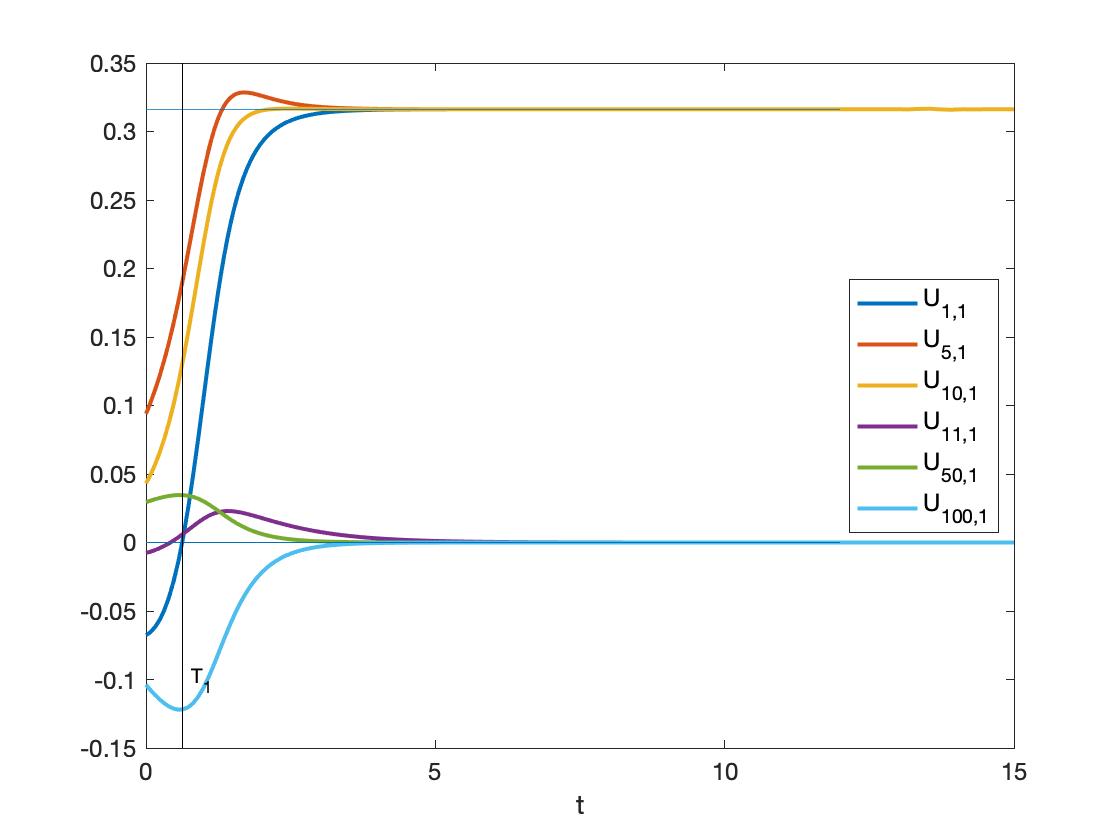}
	 \includegraphics[scale=0.15]{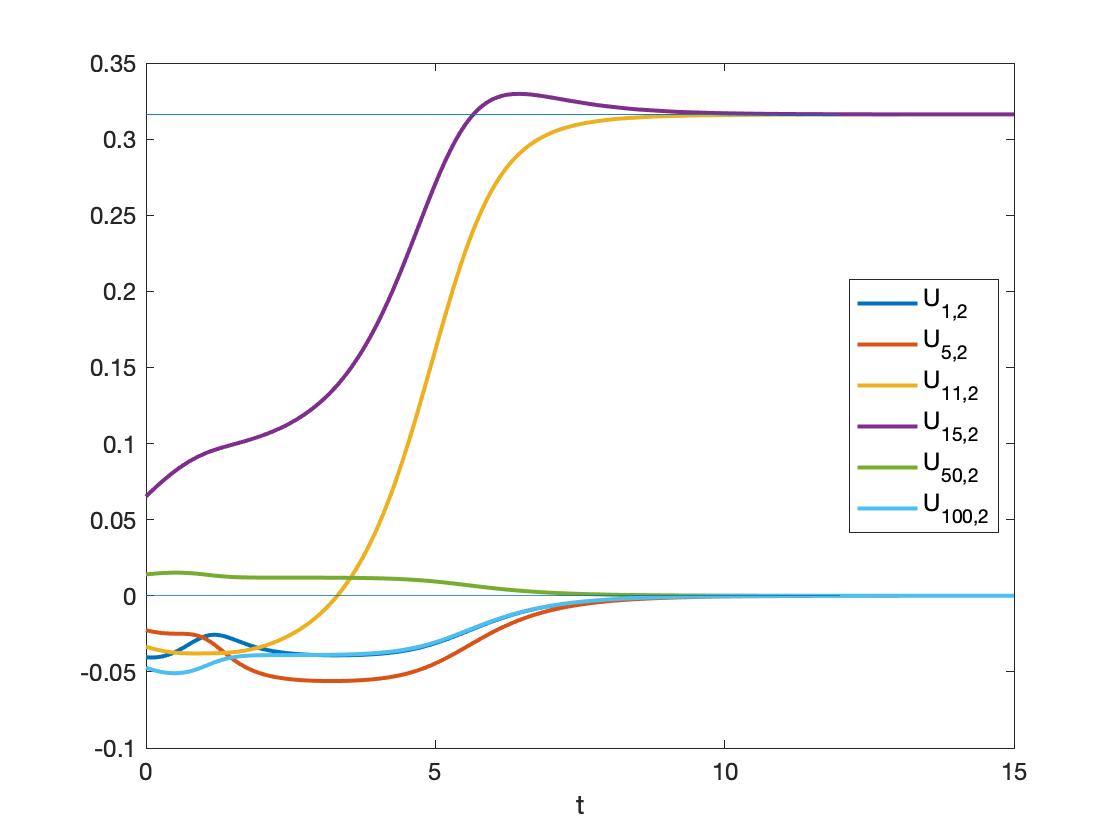}
   \caption{$\UU_{\cdot 1}(t)$ (left panel) and $\UU_{\cdot 2}(t)$ (right panel) in \eqref{eq:atospca}. $\Cc = 2*\UU_{\cdot 1}^*\UU_{\cdot 1}^{*\top}+\UU_{\cdot 2}^*\UU_{\cdot 2}^{*\top}+\Id_d$, where $d=100$ and the only two principal components are $\UU_{\cdot 1}^*=10^{-1/2}(1_{10},0_{90})$ and $\UU_{\cdot 2}^*=10^{-1/2}(0_{10},1_{10},0_{80})$. $\UU_{\cdot 1}(t)$ is sign stable on $(0,T_1)$ and $(T_1,\infty)$. 
	}\label{fig:pca_ode}
\end{figure}

\section{Auxiliary Results}\label{sec:aux}
\subsection{Uniform Boundedness of Sequences of Time-Varying Equi-Strongly Convex Functions}

In the next results, we consider a sequence of functions with a limit satisfying the conditions in Lemma \ref{lem:conticonj}. We will show that such function sequence is uniformly bounded on compact sets.

Recall that $D[0,\infty)$ be the space of functions that are right continuous with left limit.

\begin{lemma}[Uniformly boundedness]\label{lem:cptpres}
	Let $\Phi:[0,\infty)\times\R^d\to\R$ be a function satisfying the conditions of Lemma \ref{lem:conticonj}. Let $\Phi_{\gamma}:[0,\infty)\times\R^d\to\R$ be such that $\Phi_{\gamma}(t,\cdot)$ is $\beta$-e.s.c. with respect to $t$ and $\gamma$ and l.s.c., and $\Phi_{\gamma}(\cdot,\bu)\in D[0,\infty)$ for any fixed $\bu$. In addition, if for any $T$,
	\begin{align}
		\sup_{0\leq t\leq T}\big|\Phi_\gamma(t,\bu)-\Phi(t,\bu)\big|\to 0 \mbox{ pointwise in $\bu$}, \label{eq:unifPhiconv}
	\end{align}
	as $\gamma\to 0$ along a countable sequence. Then, for any constant $C$,
	\begin{align*}
		&\sup_{\gamma}\sup_{0\leq t\leq T}\sup_{\|\bu\|_2\leq C}\|\nabla\Phi_\gamma^*(t,\bu)\|_2 \leq C_T'
	\end{align*}
	where $\nabla\Phi_\gamma^*(t,\bv):=\arg\,\min_{\bw\in\R^d}\big\{\Phi_\gamma(t,\bw)-\bw^\top\bv\big\}$, and $C_T'>0$ is a constant depending on $T$ and $C$ (and $\beta,q$).
\end{lemma}
\begin{proof}[Proof of Lemma \ref{lem:cptpres}]
	Observe that
	\begin{align}
		&\sup_\gamma\sup_{0\leq t\leq T}\sup_{\|\bu\|_2\leq C}\big\|\nabla\Phi^*_\gamma(t,\bu)\big\|_2\notag\\
		&\leq \sup_\gamma\sup_{0\leq t\leq T}\sup_{\|\bu\|_2\leq C}\big\{\big\|\nabla\Phi^*_\gamma(t,\bu)-\nabla\Phi^*_\gamma(t,0)\big\|_2+ \|\nabla\Phi^*_\gamma(t,0)-\nabla\Phi^*(t,0)\|_2+\|\nabla\Phi^*(t,0)\|_2\big\} \label{eq:decom}\\
		&\leq C\beta^{-1} + \sup_\gamma\sup_{t\in[0,T]}\|\nabla\Phi^*_\gamma(t,0)-\nabla\Phi^*(t,0)\|_2+C_t'\notag
	\end{align}
	where the first term in \eqref{eq:decom} is bounded by using Lemma \ref{lem:uc}(a); the third term in \eqref{eq:decom} can be bounded by some $C_t'>0$, which follows from the continuity of $\nabla\Phi^*$ from Lemma \ref{lem:conticonj} since $\Phi$ satisfies the assumptions there.
	
	\bigskip
	Next, we will bound the second term in \eqref{eq:decom}. We will show that $\sup_{t\in[0,T]}\|\nabla\Phi^*_\gamma(t,0)-\nabla\Phi^*(t,0)\|_2$ converges to 0 as $\gamma\to 0$, and therefore this is a bounded sequence with the bound depending on $t$. Note that 
	\begin{align*}
		\tilde\bu_\gamma(t)&=\nabla\Phi^*_\gamma(t,0)=\arg\min_{\bu\in\R^d}\Phi_\gamma(t,\bu)\\
		\tilde\bu(t)&=\nabla\Phi^*(t,0)=\arg\min_{\bu\in\R^d}\Phi(t,\bu)
	\end{align*}
	by Proposition 11.3 on page 476 of \cite{RW09}. 

	Suppose to the contrary that 
	\begin{align}
			\exists\epsilon>0, \mbox{ such that $\sup_{t\in[0,T]}\|\tilde\bu_\gamma(t)-\tilde\bu(t)\|_2>\epsilon$ for infinitely many $\gamma$.}\label{eq:convmin0}
	\end{align}
	From now on, we will fix such an $\epsilon>0$. Define
	\begin{align}
		\tilde\eta_\epsilon := \inf_{t\in[0,T]}\inf_{\bu\in\Sc^{d-1}}\big\{\Phi(t,\tilde\bu(t)+\epsilon\bu)-\Phi(t,\tilde\bu(t))\big\},\label{eq:etaeps}
	\end{align}
	where $\Sc^{d-1}\subset\R^d$ is the unit sphere. We claim that $\tilde\eta_\epsilon>0$. By Lemma \ref{lem:conticonj}, $\tilde\bu(t)$ is continuous. Moreover, $(t,\bu)\mapsto\Phi(t,\bu)$ is continuous on $[0,T]\times\{\bu:\|\bu\|_2\leq C\}$ by Lemma \ref{lem:rock107}. Thus, $(t,\bu)\mapsto \Phi(t,\tilde\bu(t)+\epsilon\bu)-\Phi(t,\tilde\bu(t))$ is continuous. It follows that $\tilde\eta_\epsilon>0$, because $[0,T]\times\Sc^{d-1}$ is compact and $\tilde\bu(t)$ is the unique minimizer of $\Phi(t,\cdot)$ (since $\Phi(t,\cdot)$ is strongly convex for any $t$ and Lemma \ref{lem:uc}(b)).
	
	Notice that for {any} $t\in [0,T]$, $\bx\in\Sc^{d-1}$, from the definition of $\tilde\eta_\epsilon$ in \eqref{eq:etaeps}, we have
	\begin{align}
		0<\tilde\eta_\epsilon
		&\leq \Phi(t,\tilde\bu(t)+\epsilon\bx)-\Phi(t,\tilde\bu(t)) \notag\\
		&= \Phi(t,\tilde\bu(t)+\epsilon\bx)-\Phi_\gamma(t,\tilde\bu(t)+\epsilon\bx)+\Phi_\gamma(t,\tilde\bu(t))-\Phi(t,\tilde\bu(t))+\Phi_\gamma(t,\tilde\bu(t)+\epsilon\bx)-\Phi_\gamma(t,\tilde\bu(t))\notag\\
		&\leq 2\Delta_{\gamma,\epsilon}+\Phi_\gamma(t,\tilde\bu(t)+\epsilon\bx)-\Phi_\gamma(t,\tilde\bu(t)), \mbox{ for all $\gamma>0$,}\label{eq:convmin1}
	\end{align}
	where 
	\begin{align*}
		\Delta_{\gamma,\epsilon} := \sup_{t\in[0,T]}\sup_{\bx:\|\bx-\tilde\bu(t)\|_2\leq\epsilon}\big|\Phi_\gamma(t,\bx)-\Phi(t,\bx)\big|
	\end{align*}
	From \eqref{eq:convmin1}, we deduce that
	\begin{align}
		0 <\tilde\eta_\epsilon \leq 2\limsup_\gamma \Delta_{\gamma,\epsilon}+\liminf_\gamma\inf_{t\in [0,T]}\inf_{\bx\in\Sc^{d-1}} \big\{\Phi_\gamma(t,\tilde\bu(t)+\epsilon\bx)-\Phi_\gamma(t,\tilde\bu(t))\big\}. \label{eq:convmin2}
	\end{align}
	
	We will show at the end of the proof that 
	\begin{align}
		\Delta_{\gamma,\epsilon}\to 0 \mbox{ as $\gamma\to 0$.}\label{eq:deltato0}
	\end{align}
	 
	On the other hand, by Lemma \ref{lem:conticonj}, both $\tilde\bu(t)$ and $\tilde\bu_\gamma(t)$ are continuous. Thus, there exists a sequence $t_\gamma$ such that $\sup_{t\in[0,T]}\|\tilde\bu_\gamma(t)-\tilde\bu(t)\|_2=\|\tilde\bu_\gamma(t_\gamma)-\tilde\bu(t_\gamma)\|_2>\epsilon$ for infinitely many $\gamma$ by \eqref{eq:convmin0}. Because the minimizer of $\Phi_\gamma(t_\gamma,\cdot)$ is unique as it is $\beta$-e.s.c. for any $t_\gamma$, there exists a sequence $\{\bx_\gamma\}_\gamma\subset\Sc^{d-1}$, such that
	\begin{align}
		\Phi_\gamma(t,\tilde\bu(t_\gamma)+\epsilon\bx_\gamma)-\Phi_\gamma(t_\gamma,\tilde\bu(t_\gamma))\leq 0 \mbox{ for infinitely many $\gamma$.}\label{eq:convmin2_5}
	\end{align}
However, this together with \eqref{eq:deltato0} and \eqref{eq:convmin2} imply that 
	\begin{align}
0 <\tilde\eta_\epsilon \leq \liminf_\gamma\inf_{t\in [0,T]}\inf_{\bx\in\Sc^{d-1}} \big\{\Phi_\gamma(t,\tilde\bu(t)+\epsilon\bx)-\Phi_\gamma(t,\tilde\bu(t))\big\} \leq 0
	\end{align}
	which forms a contradiction. Thus, \eqref{eq:convmin0} is false and the proof is complete.
	
	\bigskip
	It is now left to show \eqref{eq:deltato0}. Since $\tilde\bu(t)$ is continuous by Lemma \ref{lem:conticonj}, there exists $M_T>0$ such that $\sup_{t\in[0,T]}\|\tilde\bu(t)\|_2\leq M_T$. Define the compact set $\Kc_{T,\epsilon}:=\{\bx:\|\bx\|_2\leq M_T+\epsilon\}$. Since $\{\bx:\|\bx-\tilde\bu(t)\|_2\leq\epsilon\}\subset\Kc_{T,\epsilon}$ for all $t\in[0,T]$, we have 
	\begin{align}
		\Delta_{\gamma,\epsilon} \leq \sup_{t\in[0,T]}\sup_{\bx\in\Kc_{T,\epsilon}}\big|\Phi_\gamma(t,\bx)-\Phi(t,\bx)\big|.\label{eq:convmin3}
	\end{align}
	 We will apply Lemma 3 on page 1827 of \cite{K09} to prove this claim. From the conditions of this Lemma, both $\{\Phi_\gamma(t,\cdot)\}_\gamma$ and $\Phi(t,\cdot)$ are convex for each $t\in[0,T]$, and both $\{\Phi_\gamma(\cdot,\bu)\}_\gamma$ and $\Phi(\cdot,\bu)$ are bounded since elements in $D[0,T]$ are bounded by Eq. (12.5) on page 122 of \cite{B99}. By \eqref{eq:unifPhiconv}, $\sup_{t\in[0,T]}\big|\Phi_\gamma(t,\bx)-\Phi(t,\bx)\big|\to 0$ for every $\bx\in\R^d$. Hence, since $\Kc_{T,\epsilon}$ is compact, by Lemma 3 on page 1827 of \cite{K09} and \eqref{eq:convmin3},
	\begin{align*}
			0\leq \Delta_{\gamma,\epsilon} \leq \sup_{t\in[0,T]}\sup_{\bx\in\Kc_{T,\epsilon}}\big|\Phi_\gamma(t,\bx)-\Phi(t,\bx)\big| \to 0,
	\end{align*}
	and by squeezing the proof of the claim is finished.
\end{proof}

\subsection{Continuous Mapping Theorem for Argmin Processes}\label{sec:argpr}

In this section, we will state and prove a continuous mapping theorem which is an adaptation of Theorem 1 of \cite{K09} but with slightly weaker conditions. For completeness, we state all the notations in \cite{K09}.

Let $(\Omega,\Fc,P)$ be a probability space. $P^*$ and $\E^*$ be the outer probability and expectation (see \cite{VW96} for more details). Suppose $f_n:\R^d\times [0,T]\times \Omega\to\R$ $(n\in \N)$ and $f_\infty:\R^d\times [0,T]\times \Omega\to\R$, and $f_n(\bx,t,\cdot)$ and $f_\infty(\bx,t,\cdot)$ are measurable with respect to $\Fc$ for each $(\bx,t)\in \R^d\times [0,T]$. For {each} $(t,\omega)$, define
\begin{align}
	\bx_n(t,\omega)\in\arg\,\min_{\bx\in\R^d} f_n(\bx,t,\omega),\quad \bx_\infty(t,\omega)\in\arg\,\min_{\bx\in\R^d} f_\infty(\bx,t,\omega).\label{eq:min_xn}
\end{align}
We assume that each argmin set is nonempty. In the main result Lemma \ref{lem:kato1} of this section, we will assume that $\bx_\infty$ is unique, but $\bx_n$ may not be unique. We omit the argument $\omega$ if there is no confusion. 

Take an arbitrary probability space $(\tilde\Omega,\tilde\Fc,\tilde P)$ different from $(\Omega,\Fc,P)$. A measurable map $\phi:\tilde\Omega\to\Omega$ is called {perfect} if 
\begin{align}
	\E^*[H] = \tilde\E^*[H\circ \phi]
\end{align}
for every bounded function $H$ on $\Omega$. 


We will assume that $f_n(\bx,\cdot)$ to be {c\'adl\'ag} (right continuous with left limit) for each $\bx$, and we view $\bx_n$ as random element in $D([0,T])^d$. We metrize $D([0,T])^d$ with metric $\rho_{d,\circ}^{T}$, defined by $\rho_{d,\circ}^{T}(\bx,\by):=\sum_{j=1}^d \rho_\circ^T(x_j,y_j)$ where 
\begin{align}
		\rho_\circ^T(x,y)=\inf_{\nu\in\Vc_T}\Big\{\sup_{0\leq t<s\leq T}\Big|\log\frac{\nu(s)-\nu(t)}{s-t}\Big|\vee \sup_{0\leq t\leq T}\big\|x(t)-y(\nu(t))\big\|_2\Big\}
\end{align}
where $\Vc_T$ is a class of nondecreasing functions on mapping $[0,T]$ onto itself. $\rho_\circ^T$ is one of the common metrics that topologizes $D([0,T])$ with the $J_1$ topology. This metric makes $D([0,T])$ a separable and complete space (see e.g. page 125-129 of \cite{B99}).  

Now, we state and prove the main theorem in this section. 
\begin{lemma}\label{lem:kato1}
	Suppose 
	\begin{enumerate}
		\item[(a)] $f_n(\bx,t)$ ($n\in\N$) and $f_\infty(\bx,t)$ are convex in $\bx$ for each $t$, and $f_n(\bx,t)$ is {c\'adl\'ag} in $t$ for each $\bx$, while $f_\infty(\bx,t)$ is continuous in $t$ for each $\bx$;
		\item[(b)] $\bx_\infty(t)$ is the unique minimum point of $f_\infty(\cdot,t)$ for each $t\in[0,T]$;
		\item[(c)] $\bx_n(\cdot)$ are random elements in $D([0,T])^d$ under product Skorohod metric  ($n\in\N$) and $\bx_\infty(\cdot)\in (C[0,T])^d$. 
	\end{enumerate}
	Then $\bx_\infty(\cdot)$ is a random element of $(C[0,T])^d$, and if for each $k\in\N$,
	\begin{align}
		\big(f_n(\by_1,\cdot),f_n(\by_2,\cdot),...,f_n(\by_k,\cdot)\big)\weakto \big(f_\infty(\by_1,\cdot),f_\infty(\by_2,\cdot),...,f_\infty(\by_k,\cdot)\big) \quad \mbox{in }(D[0,T])^k,\label{eq:kato1}
	\end{align}
	where $\{\by_1,\by_2,...\}$ is a countable {dense} subset of $\R^d$, we have
	\begin{align}
		\bx_n(\cdot)\weakto\bx_\infty(\cdot) \quad \mbox{in }D([0,T])^d.\label{eq:wmini}
	\end{align}
\end{lemma}
\begin{proof}[Proof of Lemma \ref{lem:kato1}]
 By (a), $f_\infty(\cdot,t)$ is convex for each $t$ and $f_\infty(\bx,\cdot)$ is continous for each $\bx$, by Theorem 10.7 on page 89 of \cite{R70} the function $f_\infty$ is jointly continuous on $\R^d\times[0,T]$. Together with its measurability with respect to $\Fc$ at each $(\bx,t)\in\R^d\times[0,T]$, we therefore conclude that $f_\infty$ is an random element in $C(\R^d\times[0,T])$. Since $\bx_\infty(t)$ is the unique minimizer of $f_\infty(\cdot,t)$ for each $t$ and $\bx_\infty(t)$, Corollary 1 on page 1531 of \cite{N92} gives that $\omega\mapsto \bx_\infty(t,\omega)$ is measurable with respect to $\Fc$ for an arbitrary fixed $t\in[0,T]$. Together with assumption (c), it yields that $\bx_\infty(t)$ is a random element in $(C[0,T])^d$ by the discussion on page 84 of \cite{B99}.
 
Since $\big(f_\infty(\bx_1,\cdot),f_\infty(\bx_2,\cdot),...\big)\in (C[0,T])^\infty$ is separable under the product uniform metric [defined by setting $d_i$ as uniform metric in the either product metric on page 32 of \cite{VW96}, and see the discussion on the same page just above Theorem 1.4.8], by Theorem 1.4.8 on the same page of \cite{VW96}, \eqref{eq:kato1} is equivalent to 
\begin{align}
	\big(f_n(\bx_1,\cdot),f_n(\bx_2,\cdot),...\big)\weakto \big(f_\infty(\bx_1,\cdot),f_\infty(\bx_2,\cdot),...\big) \quad \mbox{in }(D[0,T])^\infty.\label{eq:kato1temp}
\end{align}
From the discussion on page 32 of \cite{VW96}, $(D[0,T])^\infty$ is a metric space with the metric on page 32 of \cite{VW96} and $\rho_\circ^T$ as its base (we do not specify this metric because this will not be used elsewhere), and the limit $\big(f_\infty(\bx_1,\cdot),f_\infty(\bx_2,\cdot),...\big)$ lies in $(C[0,T])^\infty$ by (a). Again from the separability of $(C[0,T])^\infty$ under the product uniform metric, applying Theorem 1.10.4 and Addendum 1.10.5 on page 59 of \cite{VW96} to \eqref{eq:kato1temp}, there exists a probability space $(\tilde\Omega,\tilde\Fc,\tilde P)$, perfect maps $\phi_n:\tilde\Omega\to\Omega$ and $\phi_\infty:\tilde\Omega\to\Omega$ such that
\begin{align}
	\sup_{t\in[0,T]}\big|f_n(\by_i,t,\phi_n(\tilde\omega))-f_\infty(\by_i,t,\phi_\infty(\tilde\omega))\big|\to 0 \quad\mbox{ a.s. in outer measure $\tilde P^*$}, \ \forall i. \label{eq:auconv}
\end{align}
(Recall the definition of such convergence in Definition 1.9.1 on page 52 of \cite{VW96}, and note that by Lemma 1.9.2(iii) on page 53 of \cite{VW96}, almost uniform convergence is equivalent to the almost sure convergence in outer measure $\tilde P^*$ since $f_n$ is a countable sequence). To simplify notations, in the following we denote 
$$
\tilde f_n(\by_i,t)=f_n(\by_i,t,\phi_n(\tilde\omega)) \mbox{ and }\tilde f_\infty(\by_i,t)=f_\infty(\by_i,t,\phi_\infty(\tilde\omega)).
$$

In the end of the proof, it will be shown that we can strengthen the convergence in \eqref{eq:auconv} to 
\begin{align}
	\sup_{t\in[0,T]}\sup_{\by\in K}\big|\tilde f_n(\by,t)-\tilde f_\infty(\by,t)\big|\to 0\quad  \mbox{a.s. in }\tilde P^* \mbox{ for any compact $K\subset\R^d$.}  \label{eq:kato2}
\end{align}

Define $\tilde\bx_n(\cdot,\tilde\omega)=\bx_n(\cdot,\phi_n(\tilde\omega))$ and $\tilde\bx_\infty(\cdot,\tilde\omega)=\bx_\infty(\cdot,\phi_\infty(\tilde\omega))$. Straightforwardly,
\begin{align}
	\tilde\bx_n(t)\in\arg\,\min_{\bx\in\R^d} \tilde f_n(\bx,t), \quad \tilde\bx_\infty(t)=\arg\,\min_{\bx\in\R^d} \tilde f_\infty(\bx,t) \label{eq:tilx}
\end{align}
for each $(t,\tilde\omega)$, this follows from the construction of $\bx_n$ and $\bx_\infty$ under every $t$ and $\omega$ as in \eqref{eq:min_xn}. Moreover, $\tilde\bx_\infty$ is a random element in $(C[0,T])^d$ measurable with respect to $\tilde\Fc$. This follows by the fact that $\phi_\infty:\tilde\Omega\to\Omega$ is measurable with respect to $\tilde\Fc$ and that $\bx_\infty$ is a random element in $(C[0,T])^d$ as shown in the first paragraph of this proof.

To show the weak convergence \eqref{eq:wmini}, we will show the stronger $\rho_{d,\circ}^T(\tilde\bx_n,\tilde\bx_\infty)\to 0$ in probability. 
In fact, in order to show $\rho_{d,\circ}^T(\tilde\bx_n,\tilde\bx_\infty)\to 0$ in probability, it is enough to show that $\sup_{t\in[0,T]}\|\tilde\bx_n(t)-\tilde\bx_\infty(t)\|_2\to 0$ in probability. To see this, note that $\rho_\circ^T(x,x')\leq \sup_{t\in[0,T]}|x(t)-x'(t)|$ for arbitrary $x,x'\in D[0,T]$ by the discussion on page 150 in the 1st edition of \cite{B99}, and for arbitrary $\bx,\bx'\in D([0,T])^d$, $\sup_{t\in[0,T]}|x_j(t)-x_j'(t)|\leq \sup_{t\in[0,T]}\|\bx(t)-\bx'(t)\|_2$ for every $j=1,...,d$. Hence, $\rho_{d,\circ}^T(\tilde\bx_n,\tilde\bx_\infty)\to 0$ in probability if $\sup_{t\in[0,T]}\|\tilde\bx_n(t)-\tilde\bx_\infty(t)\|_2\to 0$ in probability. 

Therefore, it is left to show that $\sup_{t\in[0,T]}\|\tilde\bx_n(t)-\tilde\bx_\infty(t)\|_2\to 0$ in probability; namely, for every $\delta>0$,
\begin{align}
	\lim_{n\to\infty} \tilde P^*\Big(\sup_{t\in[0,T]}\|\tilde\bx_n(t)-\tilde\bx_\infty(t)\|_2>\delta\Big) = 0.\label{eq:uniformmetric}
\end{align}

The proof of \eqref{eq:uniformmetric} follows exactly the same convexity arguments as the proof for Theorem 1 of \cite{K09}, by using \eqref{eq:kato2}, the fact that $\tilde\bx_\infty(\cdot)$ is a random element of $(C[0,T])^d$ [see the discussion below \eqref{eq:tilx}], and replacing all displays with $\tilde P(\cdot)$ by the outer probability $\tilde P^*(\cdot)$. 

\bigskip
It is left to show \eqref{eq:kato2}. The proof is similar to that of Lemma 3 on page 1827 of \cite{K09}. For each $\by_i$ in the dense set $\{\by_1,\by_2,...,\}\subset\R^d$, by \eqref{eq:auconv}, $\sup_{t\in[0,T]}|\tilde f_\infty(\by_i,t)|<\infty$ following from the continuity of $\tilde f_\infty(\by_i,\cdot)$, and $\sup_{t\in[0,T]}|\tilde f_\infty(\by_i,t)|<\infty$ for each $n$ from Eq. (12.5) on page 122 of \cite{B99}, we have $\sup_{n\in\N,t\in[0,T]}|\tilde f_n(\by_i,t)|<\infty$ for all $\by_i$ in the dense set $\{\by_1,\by_2,...\}\in\R^d$. This enables us to apply Lemma \ref{lem:lip} on $\{\tilde f_n(\cdot,t):t\in[0,T],\  n\geq 1\}$ to obtain that for some $L_1>0$ (depending on $K$),
\begin{align}
	|\tilde f_n(\bx,t)-\tilde f_n(\by,t)| \leq L_1\|\bx-\by\|_2, \quad \forall \bx,\by\in K,\ t\in[0,T],\  n\in\N. \label{eq:equiL1}
\end{align}
Applying Lemma \ref{lem:lip} on $\{\tilde f_\infty(\cdot,t):t\in[0,T]\}$ yields that for some $L_2>0$ (depending on $K$),
\begin{align}
	|\tilde f_\infty(\bx,t)-\tilde f_\infty(\by,t)| \leq L_2\|\bx-\by\|_2, \quad \forall \bx,\by\in K,\ t\in[0,T]. \label{eq:equiL2}
\end{align}
Define $L_0=\max\{L_1,L_2\}$. Take $\epsilon>0$. Because $\{\by_1,\by_2,...\}\in\R^d$ is dense in $\R^d$, there exists a finite set $K_\epsilon \subset \{\by_1,\by_2,...\}\cap K$ such that each point in $K$ is within $\epsilon/(3L_0)$ in $\ell_2$ distance to at least a point in $K_\epsilon$.

The cardinality $|K_\epsilon|<\infty$ depends on $K$ and $\epsilon$. By \eqref{eq:auconv}, there exists $n_\epsilon\in\N$ depending only on $\epsilon$ and $K$ such that
\begin{align}
	|\tilde f_n(\by,t)-\tilde f_\infty(\by,t)|<\epsilon/3,\quad \forall \by\in K_\epsilon,\ t\in[0,T] \mbox{ and }n\geq n_\epsilon. \label{eq:equiL3}
\end{align} 
For any $\bx\in K$, let $\by$ be the point in $K_\epsilon$ such that $\|\bx-\by\|_2\leq \epsilon/(3L_0)$. Then, for every $n\geq n_\epsilon$ and every $t\in[0,T]$, combining \eqref{eq:equiL1}, \eqref{eq:equiL2} and \eqref{eq:equiL3} to get
\begin{align*}
	|\tilde f_n(\bx,t)-\tilde f_\infty(\bx,t)|&\leq |\tilde f_n(\bx,t)-\tilde f_n(\by,t)|+|\tilde f_n(\by,t)-\tilde f_\infty(\by,t)|+|\tilde f_\infty(\by,t)-\tilde f_\infty(\bx,t)| \\
	&\leq \epsilon, \quad \forall t\in[0,T],\ \bx\in K \mbox{ and }n\geq n_\epsilon.
\end{align*}
Hence, \eqref{eq:kato2} is proved.

\end{proof}

\bigskip
Lemma \ref{lem:kato1} shows the weak convergence of process on a compact set $[0,T]$ of time. The next lemma suggests that to prove weak convergence in $D([0,\infty))^d$, it is enough to prove the weak convergence in $(D[0,T))^d$ for any $T>0$. This is a generalization of Theorem 16.7 on page 174 of \cite{B99} to multivariate processes. Denote $r_t X$ the restriction of $X\in D[0,\infty)$ on $[0,t]$.

\begin{lemma}\label{lem:fint}
	Let $P_n$ and $P$ be probability measures on $(D([0,\infty))^d,\Dc_\infty^d)$, where $\Dc_\infty^d$ is the product Borel $\sigma$-field in $D[0,\infty)$ under the topology induced by $\rho_d^\infty$ defined in \eqref{eq:metric}. Then $P_n\weakto P$ if and only if $P_n \br_t^{-1}\weakto P \br_t^{-1}$ for any $t\in T_{P,d}$, where $\br_t=(r_t,...,r_t)\in\R^d$ and $\br_t^{-1}=(r_t^{-1},...,r_t^{-1})$ and $T_{P,d}$ is a collection of $t$ such that $P\{\bX=(X_1,...,X_d): X_j \mbox{ is discontinuous at $t$ for each $j$}\}=0$.
\end{lemma}
\begin{proof}[Proof of Lemma \ref{lem:fint}]
	In the following, we will extensively use the results in Section 16 of \cite{B99}. The results there are based on the product metric with base $d_\infty^\circ(x,y)=\sum_{m=1}^\infty 2^{-m}(1\wedge \rho_\circ^m(x^m,y^m))$ defined on page 168 of \cite{B99}, where $x^m$ and $y^m$ are the restrictions of $x,y\in D[0,\infty)$ on $[0,m]$. However, the product metric with base $d_\infty^\circ$ also generates $\Dc_\infty^d$, because their base metric $d_\infty^\circ(x,y)\to 0$ if and only if $\rho^\infty(x,y)\to 0$ (this is a well-known fact in point-set topology, here, $\rho^\infty$ is defined in \eqref{eq:metric}), which follows by Theorem 16.1 on page 168 of \cite{B99} and Proposition 3.5.3 (a)$\Leftrightarrow$(b) on page 119 of \cite{EK86}.
	
	By the discussion on page 174 of \cite{B99}, $r_t$ is measurable, so elementary measure theoretic argument shows that $\br_t$ is measurable in $(D([0,\infty))^d,\Dc_\infty^d)$. Moreover, by the discussion in the proof of Theorem 16.7 on page 174 of \cite{B99}, $r_t$ is continuous on the set of $X_j$ that is continuous at $t$. Therefore, $\br_t$ is continuous on the set of $\bX$ that all $X_j$, $j=1,...,d$ is continuous at $t$. If $P_n\weakto P$ and $t\in T_{P,d}$, then by the definition of $T_{P,d}$, $\br_t$ is continuous on a set with $P$-measure 1, hence $P_n \br_t^{-1}\weakto P \br_t^{-1}$ by the continuous mapping theorem. 
	
	Suppose that $P_n \br_t^{-1}\weakto P \br_t^{-1}$ for any $t\in T_{P,d}$. By continuous mapping theorem and the continuity of the projection mapping $\pi_j:\R^d\to\R$, $P_n \br_t^{-1} \pi_j^{-1} \weakto P \br_t^{-1} \pi_j^{-1}$ for each $j=1,...,d$ and any $t\in T_{P,d}$. By Theorem 16.7 of \cite{B99}, this implies $P_n \pi_j \weakto P \pi_j$. By Theorem 1.4.3 of \cite{VW96}, this implies that $\{P_n\}$ is asymptotically tight, and is therefore relatively compact by the Prohorov's theorem. Now it is left to verify the finite dimensional convergence. Take arbitrary $t_1,...t_k\in [0,\infty)$. Since $T_{P,d}^c$ is at most countable by Lemma 3.5.1 of \cite{EK86} and $d<\infty$, there exist $\bar t\in T_{P,d}$ such that $\bar t\geq\max_{j\leq k} t_j$. By the hypothesis $P_n \br_{\bar t}^{-1}\weakto P \br_{\bar t}^{-1}$, we have $P_n \pi_{t_1...t_k}^{-1} = P_n \br_{\bar t}^{-1} \pi_{t_1...t_k}^{-1} \weakto P \br_{\bar t}^{-1} \pi_{t_1...t_k}^{-1}= P \pi_{t_1...t_k}^{-1}$.
\end{proof}

\subsection{Subdifferential and Uniform Convexity}
In this section, we collect concepts of convex analysis which are relevant to the main theoretical development. Take $\bar\R=\R\cup\{\infty,-\infty\}$ be the set of extended real number. We assume that $\Psi:\R^d\to\bar\R$ is proper, which means that there exists at least one $\bw\in\R^d$ such that $\Psi(\bw)<\infty$, and $\Psi(\bw)>-\infty$ for all $\bw\in\R^d$. Let $\mbox{dom}\Psi:=\{\bw\in\R^d:\Psi(\bw)<\infty\}$. On a metric space, $\Psi:\R^d\to\R$ is called {lower semicontinuous} if $\liminf_{\bw\to \bw_0}\Psi(\bw)\geq \Psi(\bw_0)$ for all $\bw_0\in \mbox{dom}\Psi$. Let $\bar\R$ be the set of extended real numbers.

\begin{lemma}[Sum rule, Moreau-Rockafellar theorem]\label{lem:sumsdif}
 Let $f_j:\R^d\to\bar\R$ for $j=1,2$ be convex functions. If there exists $\bx_0\in\mbox{dom}(f_1)\cap \mbox{dom}(f_2)$ such that either $f_1$ or $f_2$ is continuous at $\bx_0$, then 
 \begin{align*}
 	\partial(f_1+f_2)(\bx)=\partial f_1(\bx)+\partial f_2(\bx),\quad \forall \bx\in\mbox{dom}(f_1)\cap \mbox{dom}(f_2).
 \end{align*}
\end{lemma}
See Corollary 2.45 of \cite{MN14} for a proof of Lemma \ref{lem:sumsdif}.

\begin{lemma}[Chain rule]\label{lem:chsdif}
	 Let $f:\R^d\to\bar\R$ be a convex function, $A:\R^p\to\R^d$ be a linear mapping and $\bb\in\R^d$. Let $g:\R^d\to\bar\R$ be given by $g(\bx)=f(A\bx+\bb)$. If $\bx_0\in\R^p$ satisfies $A\bx_0+\bb\in\mbox{dom}(f)$, and that $f$ is continuous at $A\bx_0+\bb$ or $A$ is surjective, then
	 \begin{align*}
	 	\partial g(\bx_0) = A^\top (\partial f(A\bx_0+\bb))
	 \end{align*}
\end{lemma}
See Corollary 2.52 of \cite{MN14} for a proof of Lemma \ref{lem:chsdif}.

\begin{lemma}[Continuity of convex functions]\label{lem:conti}	Every proper convex function $f$ on a finite-dimensional separated topological linear space $X$ is continuous on the interior of dom$f$.
\end{lemma}
	Lemma \ref{lem:conti} is exactly Proposition 2.17 of \cite{BP12}. We note that the Euclidean space $\R^d$ is clearly a finite-dimensional separated topological linear space.

\bigskip
We recall a useful theorem from \cite{R70} that provides sufficient condition on joint continuity.
\begin{lemma}[Theorem 10.7 of \cite{R70}]\label{lem:rock107}
	Let $U\subset \R^d$ be an open set and $\Tc\subset\R$ be a an open or closed set. Let $f:\Tc\times U\to\R$ be such that $f(t,\cdot)$ is convex for each $t\in\Tc$ and $f(\cdot,\bu)$ is continuous on $\Tc$ for any $\bu\in U$. Then $f$ is jointly continuous on $\Tc\times U$.
\end{lemma}
	
	\begin{defin}\label{def:equ}
		(Page 87 of \cite{R70}) Let $\{h_n: n\in\N\}$, where $h_n:\R^d\to\bar\R=\R\cup\{-\infty,\infty\}$, be a collection of functions.
			$\{h_n: n\in\N\}$ is {equi-Lipschitzian} relative to $E\subset\R^d$ if there exists a real $L\geq 0$ such that
			\begin{align}
				|h_n(\bw')-h_n(\bw)|\leq L\|\bw'-\bw\|_2,\ \forall \bw',\bw\in E,\ n\in\N. \label{eq:elip}
			\end{align}
	\end{defin}

	\begin{lemma}[Theorem 10.6 on page 88 of \cite{R70}]\label{lem:lip}
		Let $E\subset\R^d$ be a relatively open set (see page 44 of \cite{R70}), and $\{h_\gamma:\R^d\to\bar\R|\gamma\in\Gamma\}$ be an arbitrary collection of convex functions finite on $E$, and in addition,
	\begin{enumerate}
		\item[(a)] there exists a subset $E'\subset E$ such that the convex hull of $\bar E'\supset E$ and $\sup_{\gamma\in\Gamma}h_\gamma(\bw)<\infty$ for every $\bw\in E'$, where $\bar E'$ is the closure of $E$;
		\item[(b)] there exists at least one $\bw\in E$ such that $\inf_{\gamma\in\Gamma} h_\gamma(\bw)>-\infty$.
	\end{enumerate}
	 Then $\{h_\gamma:\gamma\in\Gamma\}$ is uniformly bounded on and equi-Lipschitzian relative to any closed and bounded set $F\subset E$.
	\end{lemma}

\subsection{Stochastic Calculus}\label{sec:sc}

The goal of this section is to state a generalized version of Lemma 4.2 of \cite{DFK10} concerning the solution of Ito processes involving multivariate Brownian motions. For completeness, the proof is also given here. Denote $\BB(t)$ the $d$-dimensional standard Brownian motion.

\begin{lemma}\label{lem:solV}
	Let $\alpha\in\R$ be a constant. For $t_0\geq 0$, let $\zsig(t):[t_0,\infty)\to\R^d$ be a vector valued function and $L(t):[t_0,\infty)\to\R$ be a time dependent mean reversion level. Suppose $V(t_0)$ is a real-valued random variable with finite second moment. Then the solution of the univariate stochastic differential equation
	\begin{align}
		d V(t) = (L(t)-\alpha V(t)) dt + \zsig(t)^\top d\BB(t),\quad t> t_0,\label{eq:SDEv}
	\end{align}
	is
	\begin{align}
		V(t) = e^{-\alpha (t-t_0)}V(t_0) + h(t) + Z(t),
	\end{align}
	where  
	\begin{align}
		h(t) &= e^{-\alpha t}\int_{t_0}^t e^{\alpha s}L(s)ds,\\
		Z(t) &= e^{-\alpha t}\int_{t_0}^t e^{\alpha s}\zsig(s)^\top d\BB(s).
	\end{align}
\end{lemma}
\begin{proof}[Proof of Lemma \ref{lem:solV}]
	Let $H(t,v)=e^{\alpha (t-t_0)}v$. Applying multivariate Ito's lemma to $Y(t)=H(t,V(t))$ gives
	\begin{align*}
		dY(t) = \alpha e^{\alpha (t-t_0)} V(t) dt + e^{\alpha (t-t_0)} dV(t),
	\end{align*}
	and $Y(t_0)=V(t_0)$. Plugging \ref{eq:SDEv} in the last display yields
	\begin{align*}
		dY(t) = e^{\alpha (t-t_0)} \big(L(t) dt + \zsig(t)^\top d\BB(t)\big).
	\end{align*}
	Integrating from time $t_0$ to $t$ and multiplying by $e^{-\alpha (t-t_0)}$ finishes the proof of this Lemma.
\end{proof}

\section{Additional simulation results}\label{sec:addsim}

This section presents additional simulation analysis. The settings are the same as Section \ref{sec:sim}.

\subsection{Online sparse linear regression}\label{sec:addsimlr}

This section presents additional simulation analysis on the sign stable coefficients and the effect of step size. The simulation setup in this section is the same as the linear regression model in Section \ref{sec:simlr}. 

We check the coefficients that are sign stable on $(0,20]$ in Figure \ref{fig:LRpath}. For illustration, we pick inactive coefficient $j=9$ and active coefficient $j=10$ with $w_{10}^*=-0.4659$. Figure \ref{fig:LRpath} presents the empirical trajectories $w_{n,j}$, mean dynamics and confidence band of algorithms \eqref{eq:sgd} and \eqref{eq:grdal1}. For the left column panels of Figure \ref{fig:LRpath} corresponding to the inactive coefficient $j=9$, the trajectories are less dense than $j=10$ as $t$ increases for  \eqref{eq:grdal1}, whereas the trajectories of \eqref{eq:sgd} remain dense. For small $t$ (particularly between $t=0$ and $t=3$), the TACB of both \eqref{eq:sgd} and \eqref{eq:grdal1} demonstrate larger dispersion, due to large covariance kernel \eqref{eq:covrdadd} resulted from large $\|\bw(t)-\bw^*\|_2$ when $t$ is small. Asymmetry of \eqref{eq:TACB} of \eqref{eq:grdal1} around $\bw(t)$ is partly due to the bias induced by the drift in the SDE \eqref{eq:rdaV}, which has opposite sign to the mean trajectory $w_9(t)>0$, so the bias shifts the theoretical confidence band downwards. See the right panel of Figure \ref{fig:cover_lr} and its discussion for more about the bias.	

\begin{figure}[!h] 		
	\centering	
	\includegraphics[scale=0.22]{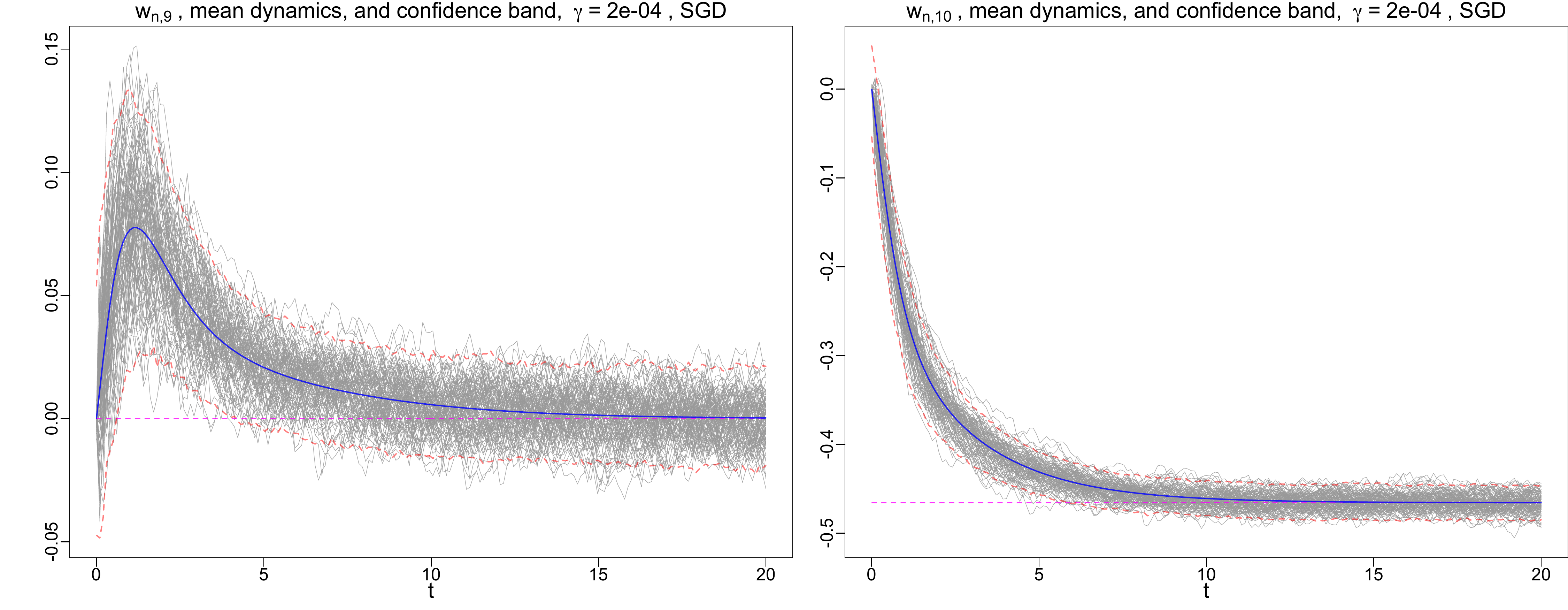}\\	 
\includegraphics[scale=0.22]{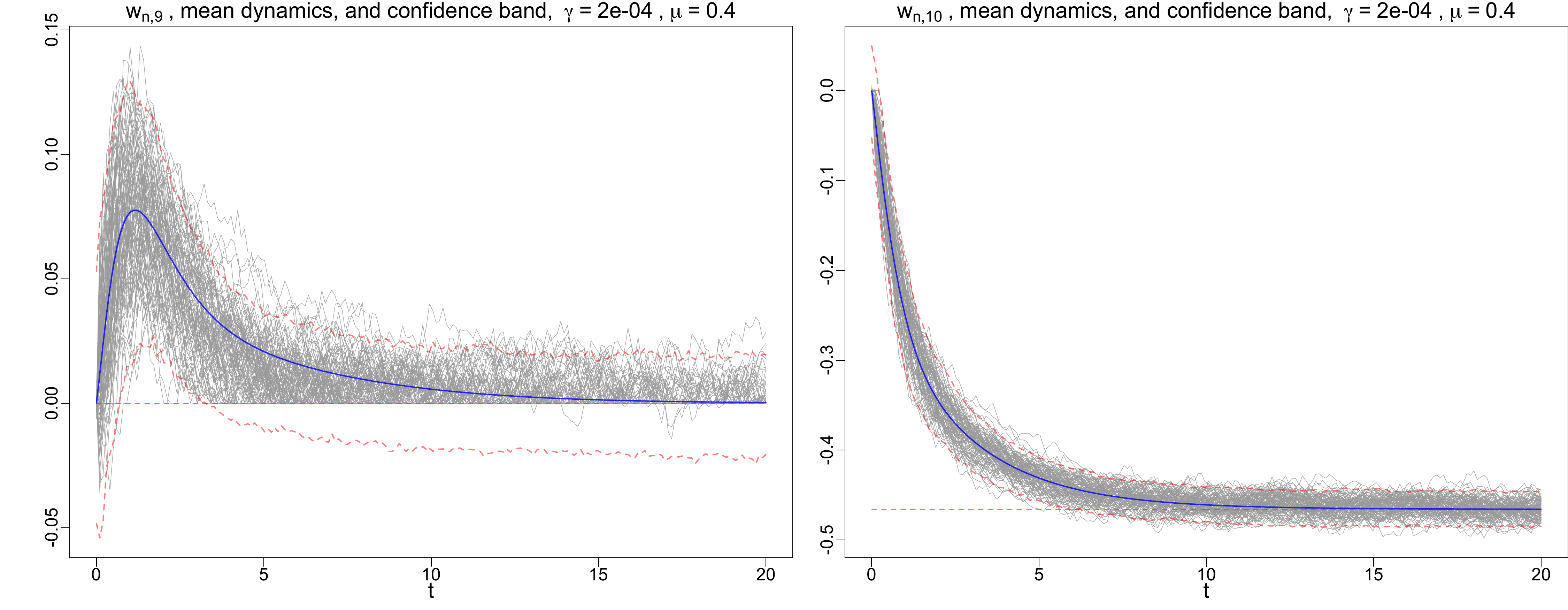}\\ 
\includegraphics[scale=0.22]{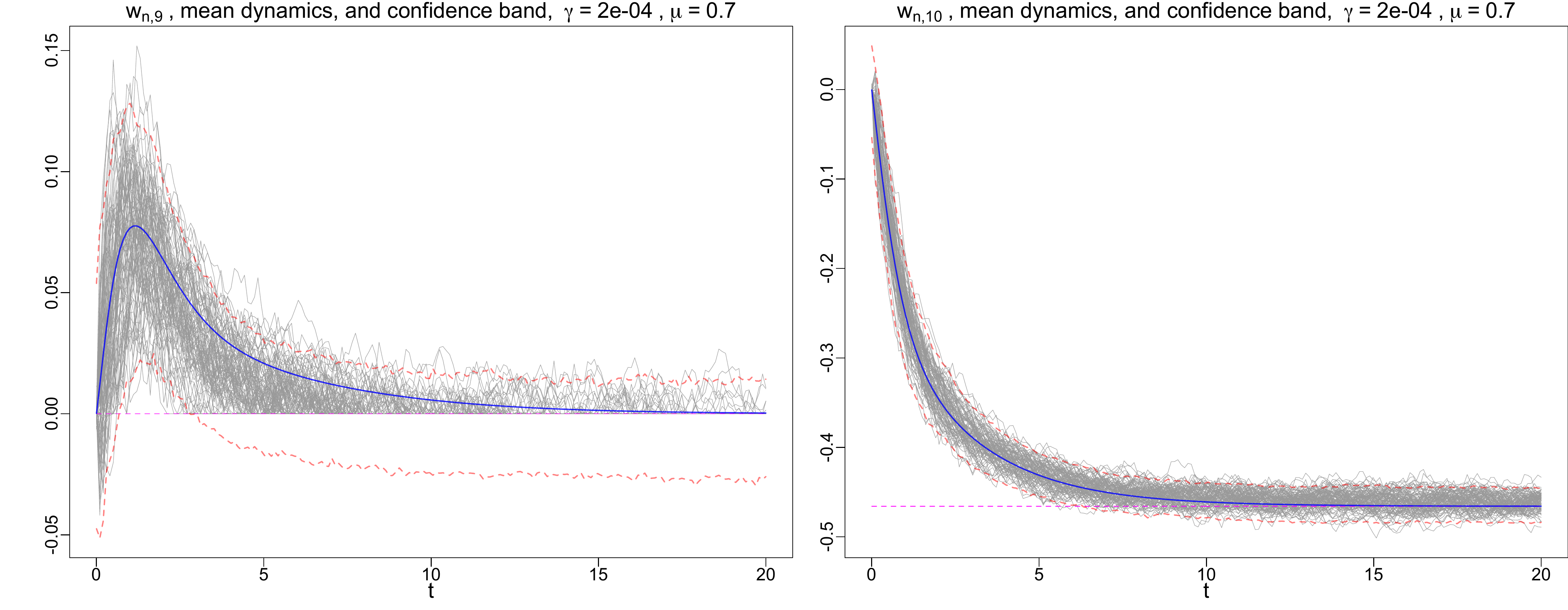}\\ 
\caption{Sign stable coefficients in $(0,20]$: 100 empirical trajectories $w_{n,j}$ (gray curves), $j=9,10$ mean dynamics (blue curve, \eqref{eq:solatrda1}) and confidence band (area bounded between two red dashed curves) for $w_9^*=0$ and $w_{10}^*=-0.4659$ (magenta dashed lines), under algorithms \eqref{eq:sgd} and \eqref{eq:grdal1} with $g(n,\gamma)=\gamma^{1/2+\mu}n^{\mu}$ with $\mu=0.4$ and 0.7, initiated at $\bw_0=0$. The number of steps is $N=20/\gamma$.}\label{fig:LRpath}
\vspace{-0.3cm}
\end{figure}

Figure \ref{fig:LRpath_mu15} presents the results for a relatively large $\mu=1.5$, with various step sizes. The three panels on the left column of Figure \ref{fig:LRpath_mu15} show that the trajectories tend to 0 as $t$ increases, and the speed is faster if the step size is greater. The severe bias of TACB and empirical trajectories indicates the inaccuracy of $\bw_n$ for $\bw^*$. When $\gamma$ is sufficiently small, the coverage of confidence band improves, as can be seen from the three panels on the right column of Figure \ref{fig:LRpath_mu15}. The results confirm the conclusion of $\mu>1$ case in Table \ref{tab:mu_bias} in Section \ref{sec:lrl1}.

	\begin{figure}[!h] 
		\centering
 \includegraphics[scale=0.22]{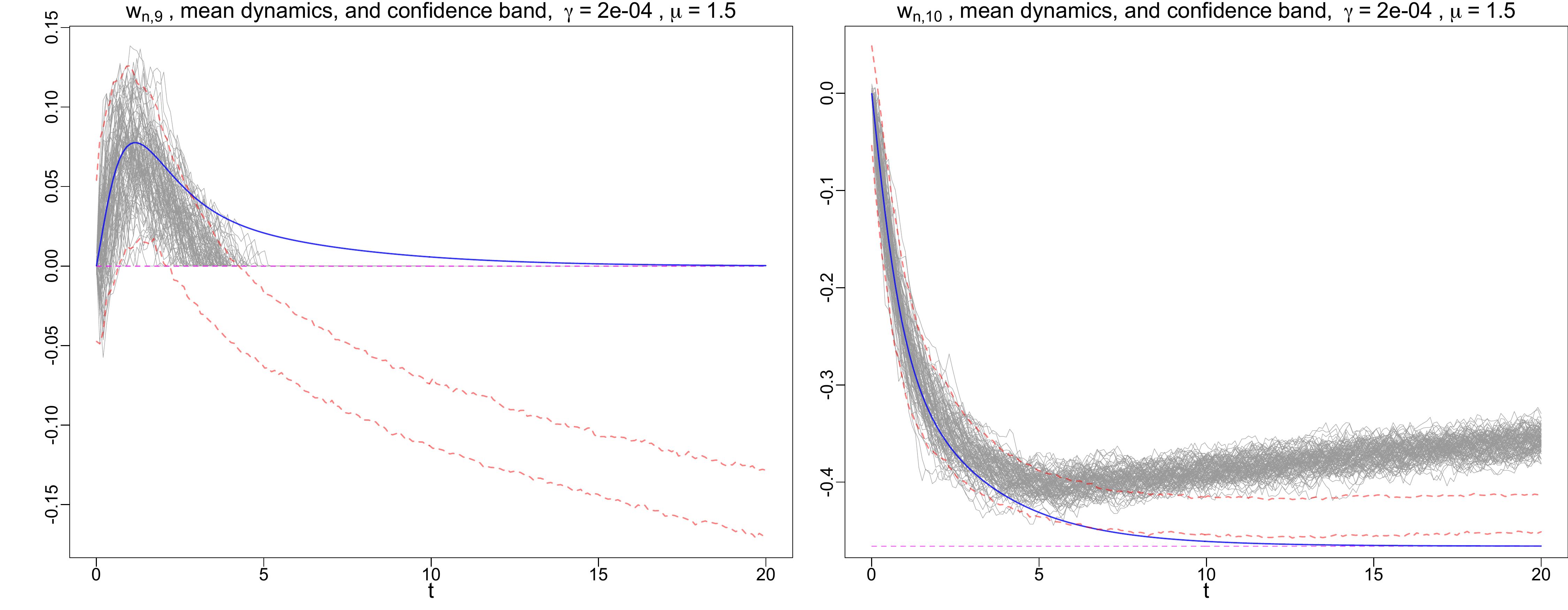}\\
 \includegraphics[scale=0.22]{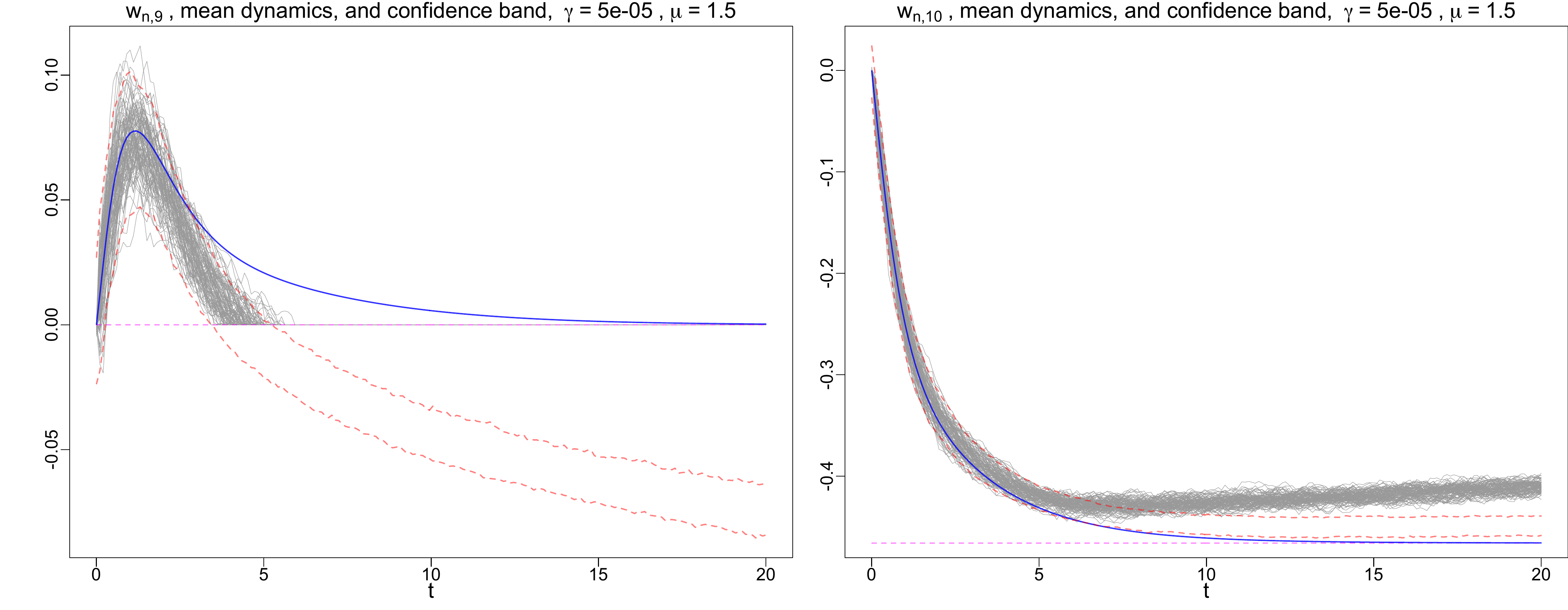}\\
 \includegraphics[scale=0.22]{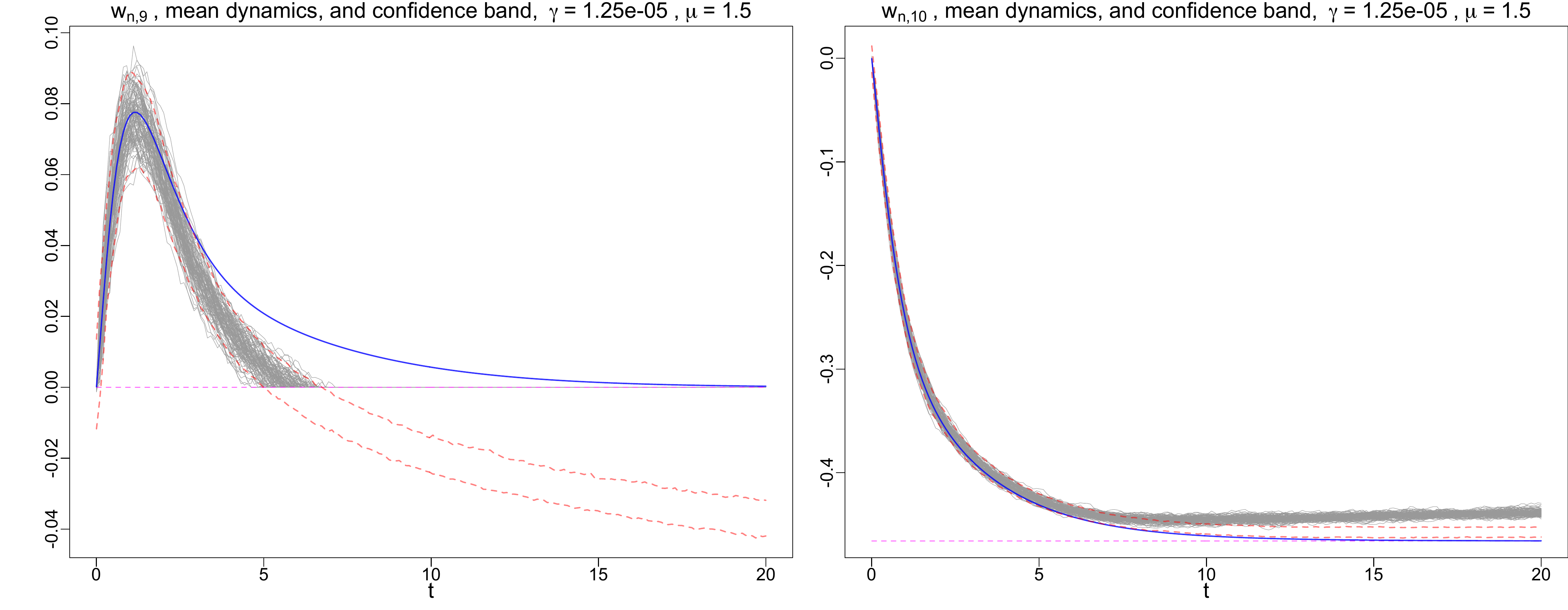}\\
 \caption{{Large $\mu$ regime}: 100 empirical trajectories $w_{n,j}$ (gray curves), $j=9,10$ mean dynamics (blue curve, \eqref{eq:solatrda1}) and confidence band (area bounded between two red dashed curves) for $w_9^*=0$ and $w_{10}^*=-0.4659$ (magenta dashed lines), under algorithms \eqref{eq:grdal1} with $g(n,\gamma)=\gamma^{1/2+\mu}n^{\mu}$ and $\mu=1.5$, initiated at $\bw_0=0$. The number of steps is $N=20/\gamma$. As step size decreases, the trajectories are getting closer to the mean dynamics.}\label{fig:LRpath_mu15}
\vspace{-0.3cm}
\end{figure}

Figure \ref{fig:LRcover15_varyg} shows that the averaged coverage probabilities over active coefficients improve when step size $\gamma$ decreases for $\mu=1.5$ in \eqref{eq:grdal1}, which supports the validity of our asymptotic theory.

	\begin{figure}[!h] 
		\centering
 \includegraphics[scale=0.5]{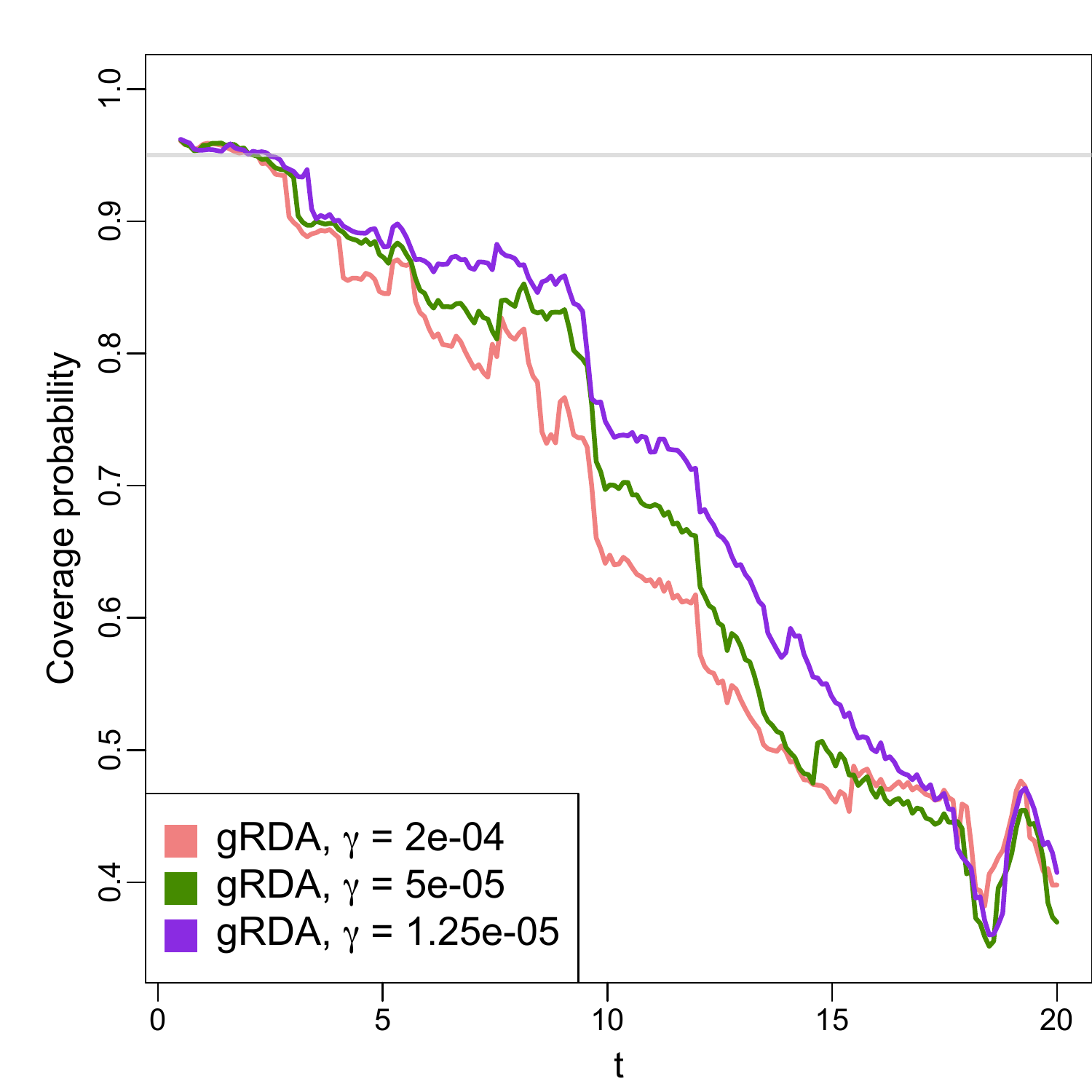}
 \caption{Averaged coverage probability of \eqref{eq:TACB} for empirical trajectories of \eqref{eq:grdal1} with $\mu=1.5$ over active coefficients, under the linear regression model as Section \ref{sec:lrl1}, based on 1000 simulation repetitions. The coverage probabilities improve as step size $\gamma$ decreases, which provides the asymptotic validity of \eqref{eq:TACB}.}\label{fig:LRcover15_varyg}
\vspace{-0.3cm}
\end{figure}

\subsection{Online sparse PCA}\label{sec:addsimpca}

This section presents the simulation analysis for the second principal component of the OSPCA model in Section \ref{sec:simpca}, and the simulation analysis on the effect of step size for the first principal component. The simulation setup is the same as the sparse PCA model in Section \ref{sec:simpca}. 

Figure \ref{fig:ospca2} presents results for the {second} principal component, by focusing on an inactive coordinate $j=1$ and an active coordinate $j=11$, with $\UU_{1,2}^*=0$ and $\UU_{11,2}^*=10^{-1/2}$. For the left column panels of Figure \ref{fig:ospca2}, the trajectories are increasingly sparse when large $\mu$ is adopted in \eqref{eq:ospca}. For the right column panels of Figure \ref{fig:ospca2}, a jump in the TACB of OSPCA also appears around the point where the mean trajectory hits the origin, but the TACB of OPCA is everywhere smooth. In addition, we observe an interesting pattern in the empirical trajectories that instead of always converging to $\UU_{11,2}^*=10^{-1/2}$, a significant number of empirical trajectories converges to the alternative stationary point $-10^{-1/2}$. Note that because the sign of principal component is indeterminable, $-10^{-1/2}$ is another valid stationary point for $\UU_{11,2}$. Some empirical trajectories are {heavily perturbed} by the large variance around $t=5$ from the mean dynamics (solid blue curve) and fall into the attraction region of $-10^{-1/2}$, and then converge along the reverse mean dynamics (dotted blue curve). Interestingly, such kind of perturbation has not occurred for the first principal component in all the 1000 simulation repetitions in Section \ref{sec:simpca}. The origin is a saddle point, and some trajectories are struggling before escaping. 

The large variance around $t=5$ in the right column panels of Figure \ref{fig:ospca2} can be explained by the eigenvalues of the Hessian-like matrix $-\nabla G(\UU(t))$ in \eqref{eq:hesspca}. If the eigenvalue of $-\nabla G(\UU(t))$ is positive\footnote{Note that the eigenvalues of $-\nabla G(\UU(t))$ are real, because the eigenvalues of $-\nabla G(\UU(t))$ are the union of eigenvalues of the symmetric and real sub-matrices on the main diagonal of $-\nabla G(\UU(t))$.}, then the SDE in \eqref{eq:rdaV} diverges with time; when the $-\nabla G(\UU(t))$ is negative, then the SDE converges to a stationary distribution. Figure \ref{fig:hess} shows that the top four eigenvalues of $-\nabla G(\UU(t))$ are non-negative for small $t$. This provides large variance for the empirical trajectories to escape the saddle point, but also causes the bifurcation in Figure \ref{fig:ospca2}. The eigenvalues converge to -1 and remain there after $t=5$, and this fact stabilizes the empirical trajectories.

Lastly, for the first principal component, Figure \ref{fig:OSPCAcover17_varyg} shows that the averaged coverage probabilities over active coefficients improve when step size $\gamma$ decreases for $\mu=1.7$ in \eqref{eq:grdal1}, which supports the validity of our asymptotic theory.

	\begin{figure}[!h] 
		\centering
	 \includegraphics[scale=0.22]{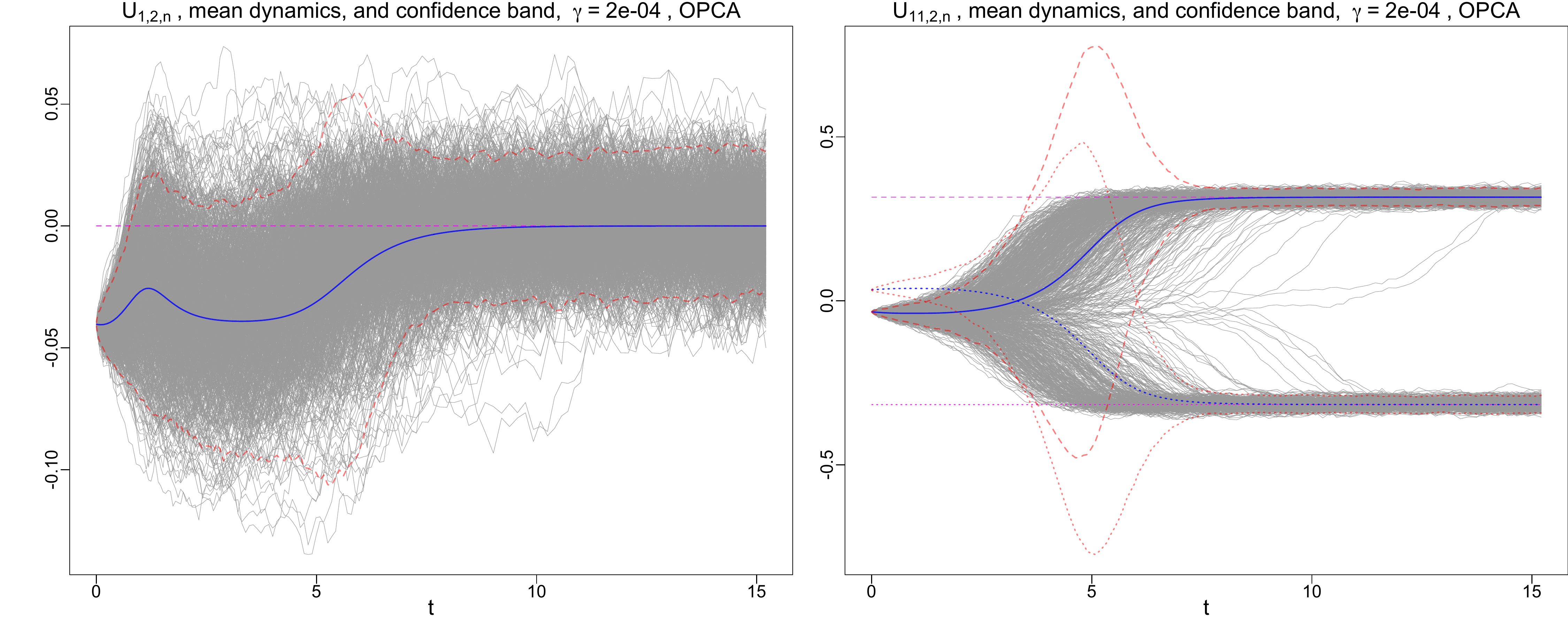}\\
	 \includegraphics[scale=0.22]{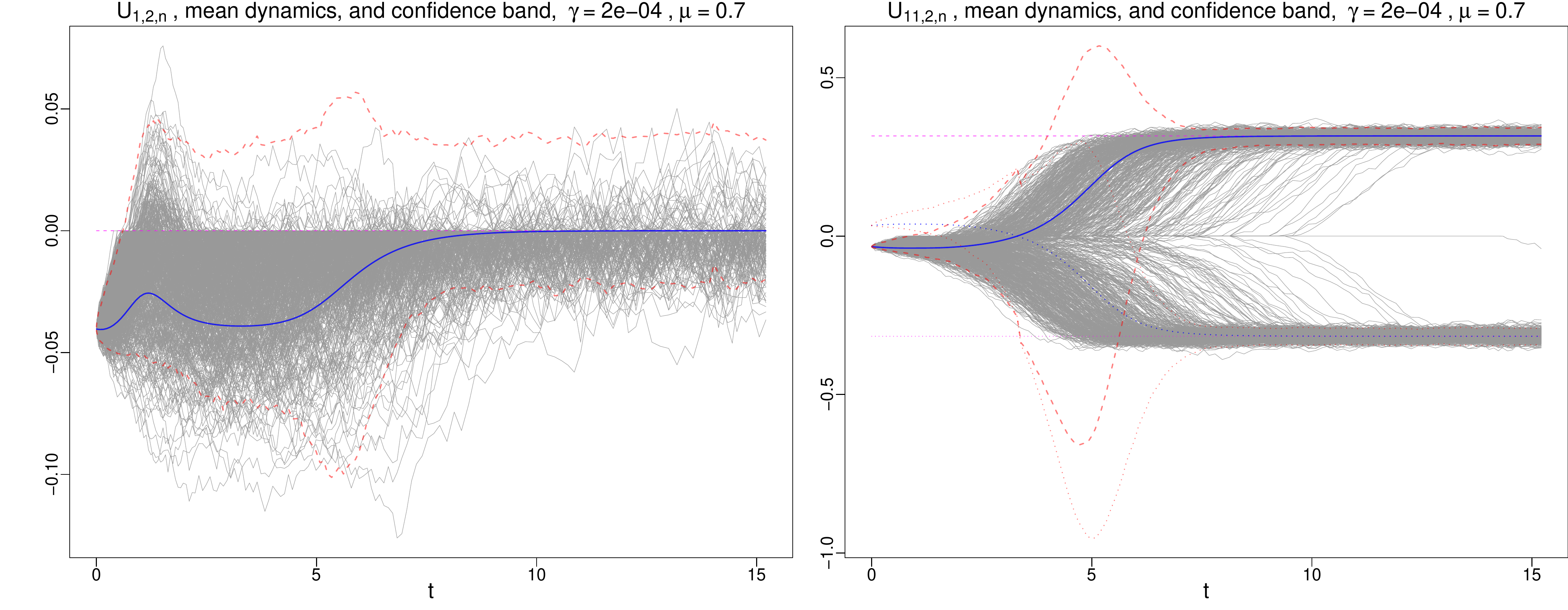}\\
	 \caption{{Second} PC: 1000 empirical trajectories (gray curves), mean dynamics (blue curve) and confidence band (area bounded between two red dashed curves) for $\UU_{1,2}^*=0$ and $\UU_{11,2}^*=10^{-1/2}$. Empirical trajectories are computed from algorithms \eqref{eq:opca} and \eqref{eq:ospca} with random initiation on the unit sphere. The number of steps is $N=15/\gamma$ with $\gamma=2\times 10^{-4}$. For the right column panels, the dotted blue curve is $-w_{11,2}(t)$, and the dotted red curves are the confidence band boundaries centered at the dotted blue curve.}\label{fig:ospca2}
 \vspace{-0.3cm}
 \end{figure}

\begin{figure}[!h]
	\centering
	 \includegraphics[scale=0.7]{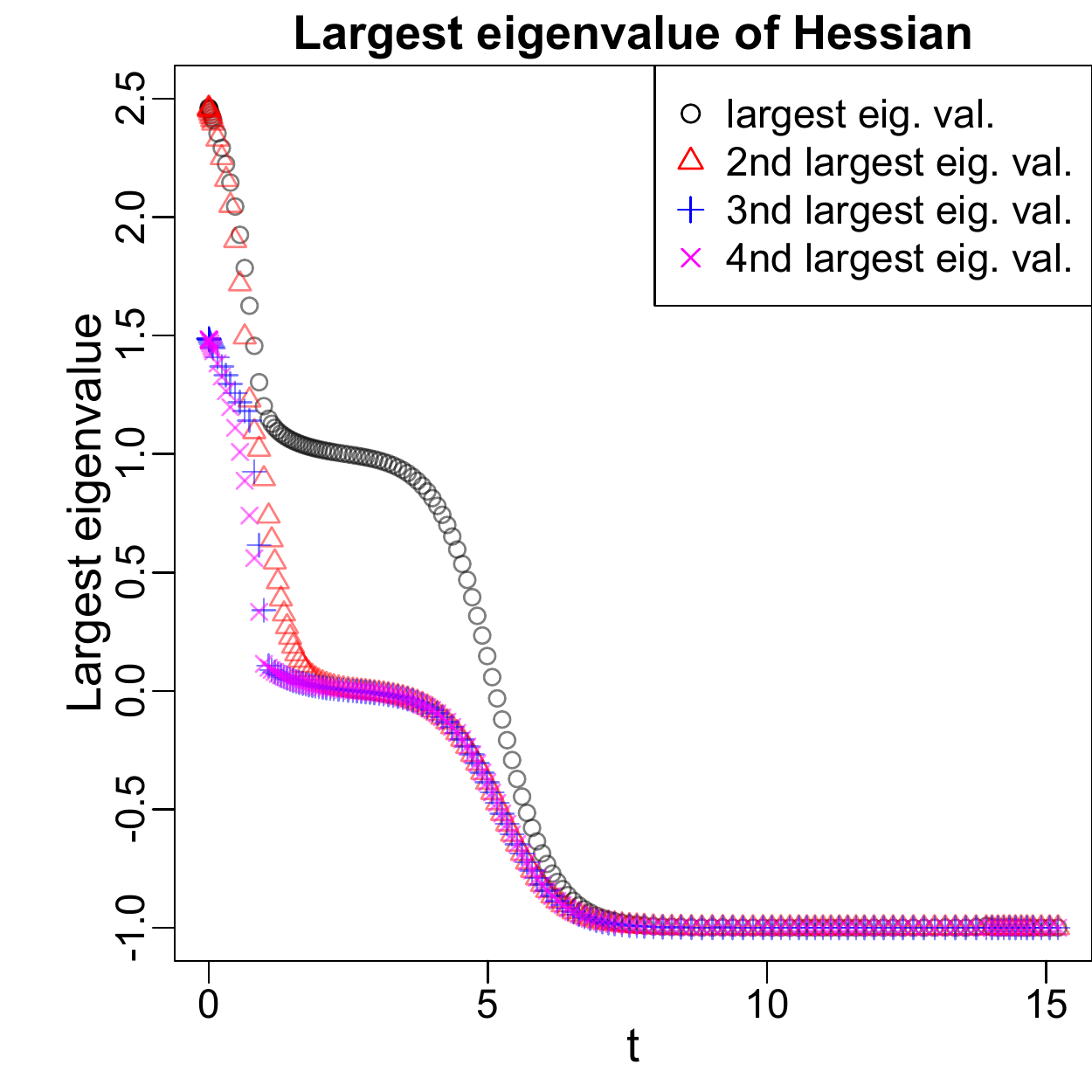}
	 \caption{Four largest eigenvalues of the Hessian matrix, which vary during training. For small $t$, the eigenvalues are positive and the width of the confidence bands in Figure \ref{fig:ospca1} and \ref{fig:ospca2} widen. For large $t$, the eigenvalues are negative, and the width of the confidence bands stabilize to a fixed level.}\label{fig:hess}
\end{figure}

	\begin{figure}[!h] 
		\centering
 \includegraphics[scale=0.5]{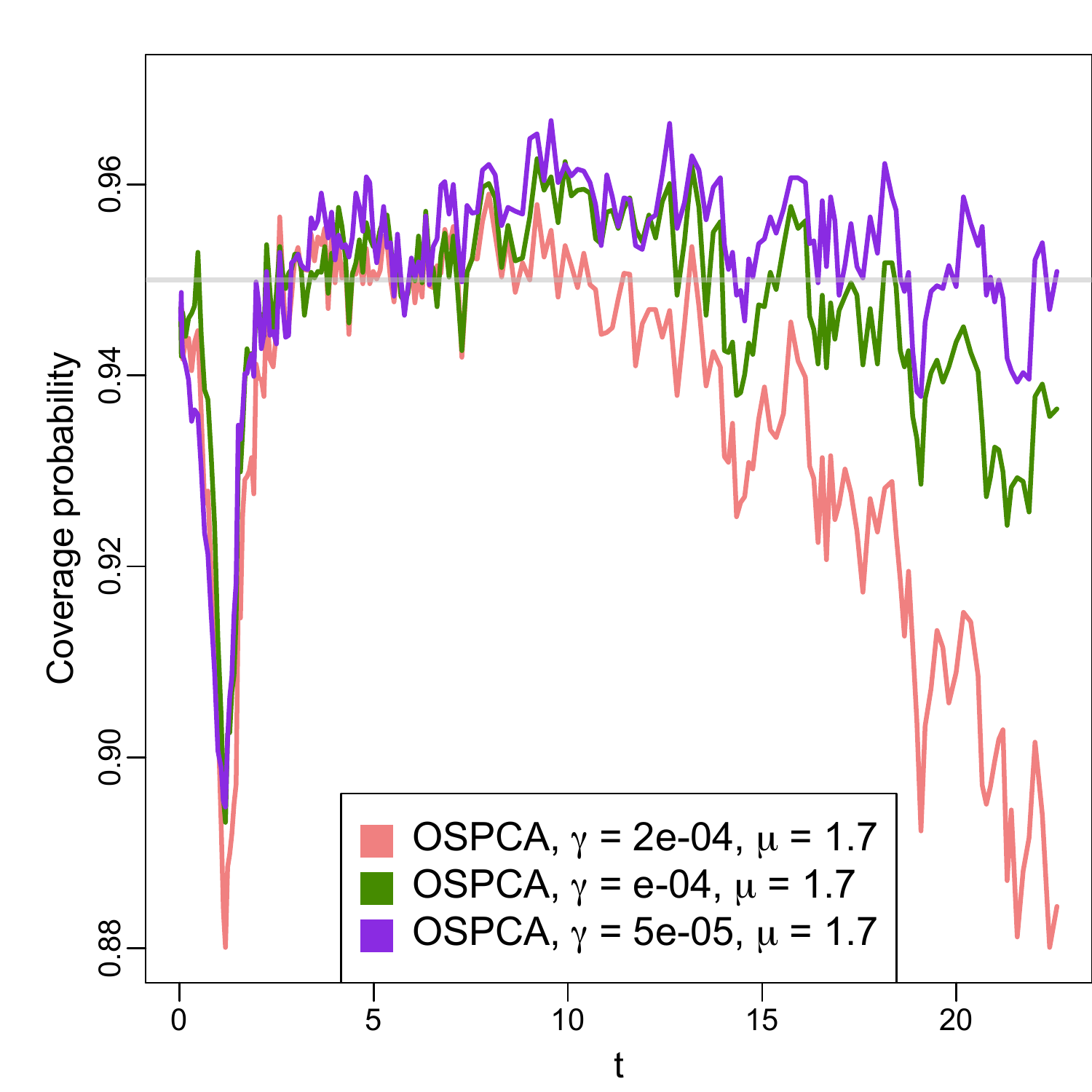}
 \caption{Averaged coverage probability of \eqref{eq:TACB} for empirical trajectories of \eqref{eq:grdal1} with $\mu=1.7$ over active coefficients of the first principal component, under the sparse PCA model as Section \ref{sec:simpca}, based on 1000 simulation repetitions. The coverage probabilities improve as step size $\gamma$ decreases, which provides the asymptotic validity of \eqref{eq:TACB}.}\label{fig:OSPCAcover17_varyg}
\vspace{-0.3cm}
\end{figure}

\bibliography{spglm.bib}
\end{document}